\documentclass[pdflatex,sn-mathphys-num]{sn-jnl}
\usepackage{amsfonts,amsmath,amssymb,amsthm}
\usepackage{bm,bbm}
\usepackage{booktabs}
\usepackage{float}
\usepackage{graphicx}
\usepackage{latexsym}
\usepackage{multirow}
\usepackage{newtxmath}
\usepackage{nicefrac}
\usepackage{xcolor}
\usepackage{url}
\usepackage{array}
\usepackage[ruled, vlined, noend]{algorithm2e}
\theoremstyle{thmstyleone}
\newtheorem{theorem}{Theorem}

\theoremstyle{thmstyletwo}

\theoremstyle{thmstylethree}

\SetKwRepeat{Do}{do}{while}
\newcolumntype{C}[1]{>{\hfil}m{#1mm}<{\hfil}}
\newcommand{\tcr}[1]{{\textcolor{red}{#1}}}
\definecolor{midori}{rgb}{0.0, 0.5, 0.0}

\newcommand{\sq}{{\rm sq}}
\renewcommand{\nll}{{\rm nll}}

\newcommand{\tra}{{\rm tra}}
\newcommand{\tes}{{\rm tes}}
\newcommand{\Binomial}{{\rm Binomial}}
\newcommand{\Logistic}{{\rm Logistic}}
\newcommand{\Cauchy}{{\rm Cauchy}}
\newcommand{\Normal}{{\rm Normal}}
\DeclareMathOperator{\argmin}{{arg\,min}}

\newcommand{\calY}{{\mathcal{Y}}}

\DeclareMathOperator{\bbI}{{\mathbbm{1}}}
\newcommand{\bbR}{{\mathbb{R}}}
\newcommand{\bbN}{{\mathbb{N}}}

\newcommand{\CF}[1]{{\setlength{\fboxrule}{0.4pt}\setlength{\fboxsep}{0.4pt}\fcolorbox{gray}{white}{#1}}}
\begin{document}
\title[Isotonic Bradley-Terry Model for Paired Comparison Data]{Isotonic Bradley-Terry Model for Paired Comparison Data}
\author{\fnm{Ryoya} \sur{Yamasaki}}\email{ryoya.yamasaki@r.hit-u.ac.jp}
\affil{\orgdiv{Hitotsubashi Institute for Advanced Study}, 
\orgname{Hitotsubashi University}, 
\orgaddress{\street{2-1 Naka}, \city{Kunitachi}, 
\postcode{186-8601}, \state{Tokyo}, \country{Japan}}}
\abstract{%
In this paper, we study prediction problems for paired comparison data, 
for example, predicting the win probability between two unmatched players 
and ranking all the players according to the order of their strengths
by using win probability data between two matched players.
Paired comparison data are typically analyzed 
using Bradley-Terry and Thurstone-Mosteller models.
These models predict the win probability 
by transforming the difference between learned rate parameters, 
which represent players'\;strengths, with a pre-specified inverse link function,
and employ the order of learned rate parameters for player ranking.
However, these models may suffer from model misspecification 
owing to the selection of a fixed inverse link function.
Therefore, in this study, we propose to learn the rate parameters by a (sub-)gradient method 
and the inverse link function by an isotonic regression technique alternately.
The proposed model guarantees monotonic improvement in training error, 
and is likely to yield an exact tie when the available data is insufficient to establish a strict ranking.
We also verified that the proposed model could improve the win probability prediction 
and ranking performance through numerical experiments with synthetic data and 
real-world data of football Premier League, baseball MLB, and tennis ATP tour.}
\maketitle
\section{Introduction}
\label{sec:Introduction}
In this paper, we study prediction problems for paired comparison data
\citep{bradley1976science,cattelan2012models}, for example, 
predicting the win probability between two unmatched players (or teams)
and ranking all the players according to the order of their strengths
by using win probability data between two matched players
for chess, tennis, and soccer matches.
Similar problems arise in surveys of people's perceptions and preferences, 
as in \citet{kissler2000effects,maydeu2008modeling,mazzuchi2008paired}
and Chatbot arena \citep{rafailov2023direct,Chatbot},
since it is burdensome for subjects to completely rank all items and 
a survey often only inquire comparisons between two items.

Paired comparison data are typically analyzed using 
generalized Bradley-Terry model, such as Bradley-Terry model 
\citep{Zermelo1929, ford1957solution, bradley1952rank, luce1959individual},
Thurstone-Mosteller model \citep{thurstone1927law, mosteller1951remarks1},
and variants by \citet{stern1990continuum}.
These models introduce real-valued rate parameters that represent players'\;strengths,
predict the win probability by transforming the difference between 
learned rate parameters with a pre-specified inverse link function,
and rank the players from highest to lowest according to 
the descending order of the learned rate parameters.

Generalized Bradley-Terry model has excellent interpretability and usability.
However, conventional usages of generalized Bradley-Terry model 
that select a fixed inverse link function before learning rate parameters
may suffer from model misspecification and perform poorly.
We are interested in how good a model can be designed while maintaining 
the excellent interpretability of generalized Bradley-Terry model
that represents players'\;strengths by real-valued parameters.
Therefore, in this paper, we propose to learn not only the rate parameters 
but also the inverse link function, and to mitigate model misspecification 
by iterating these learnings alternately:
there, we learn the rate parameters by a (sub-)gradient method
and the inverse link function by an isotonic (monotone) regression technique
\citep{ayer1955empirical,brunk1955maximum,van1956maximum,
barlow1972statistical,robertson1988order}.
We also study under what situations the proposed model can
improve the win probability prediction and ranking performance 
through synthetic data experiments,
and verify practical effectiveness of the proposed model 
through real-world data experiments.

The remaining part of this paper is organized as follows:
Section~\ref{sec:Preliminary} presents the problem formulation 
and review of generalized Bradley-Terry model.
We describe the proposed model in Section~\ref{sec:Proposal},
and discuss relationships between that model and related model,
nonparametric Bradley-Terry model \citep{chatterjee2015matrix,
shah2016stochastically,chatterjee2018matrix} in Section~\ref{sec:Discussion}.
In Sections~\ref{sec:Synthetic} and \ref{sec:RealWorld}, 
we present results of the numerical experiments
respectively for synthetic data and real-world data.
Our experimental codes are summarized at
\url{https://github.com/yamasakiryoya/IBT}.
Finally, we conclude this study and mention future prospects 
regarding the proposed model in Section~\ref{sec:Conclusion}.

\section{Preliminaries: Problem Formulation and Generalized Bradley-Terry Model}
\label{sec:Preliminary}
In this section, we describe the formulation of prediction problems 
for paired comparison data addressed in this paper,
and review a conventional statistical model for those problems, 
generalized Bradley-Terry model.
In order to facilitate concrete understanding, 
this paper adopts the description and terminologies supposing 
the win probability prediction and player ranking as application.

Suppose that there exists $n\in\bbN$ players (or teams) in total, 
and that we have data about the win probability for matches 
(e.g., chess, tennis, and soccer) between two different players:
let $D_\tra\coloneq\{(i,j)\in[n]^2\mid\text{players $i$ and $j$ have played a match at least once.}\}$
(note that $(j,i)\in D_\tra$ when $(i,j)\in D_\tra$),
where we suppose that a player can either win, lose, or draw in each match, 
and let win probability $y_{i,j}$ be 
\begin{align}
	\frac{\substack{\text{`the number of times that (\#) player $i$ has won player $j$'}\\+0.5\times\text{`\# players $i$ and $j$ have drawn'}}}
	{\text{`the number of matches between players $i$ and $j$'}} 
\end{align}
so that $y_{j,i}=1-y_{i,j}$ for every $(i,j)\in D_\tra$ 
according to the convention \citep{elo1978rating,hunter2004mm}.

We are interested in predicting `the probability that player $i$ will win player $j$'%
${}+0.5\times{}$`the probability that players $i$ and $j$ will draw', 
which is also denoted $y_{i,j}$, 
for an unmatched player pair $(i,j)\in D_\tes\coloneq\{(i,j)\in[n]^2\mid
\text{players $i$ and $j$ have not played a match.}\}
\subseteq\{(i,j)\in[n]^2\mid i\neq j\}\setminus D_\tra$.
Also, we are interested in ranking the players so that 
a higher-ranked player tends to win a lower-ranked player.

Generalized Bradley-Terry model serves as a basic
approach for addressing the above two interests.
This model introduces real-valued parameters $(r_i)_{i\in[n]}$
(which we call rate parameters) that represent players'\;strengths,
predicts the win probability $y_{i,j}$ by $\sigma(\hat{r}_i-\hat{r}_j)$ with 
learned rate parameters $(\hat{r}_i)_{i\in[n]}$ and a function $\sigma$,
which we call inverse link function following the terminology 
of generalized linear model \citep{nelder1972generalized}.
It can also rank the players from highest to lowest 
according to the descending order of $(\hat{r}_i)_{i\in[n]}$ 
(or some alternatives discussed in Section~\ref{sec:Proposal}).

\begin{figure}[t!]
\centering%
\includegraphics[width=8cm]{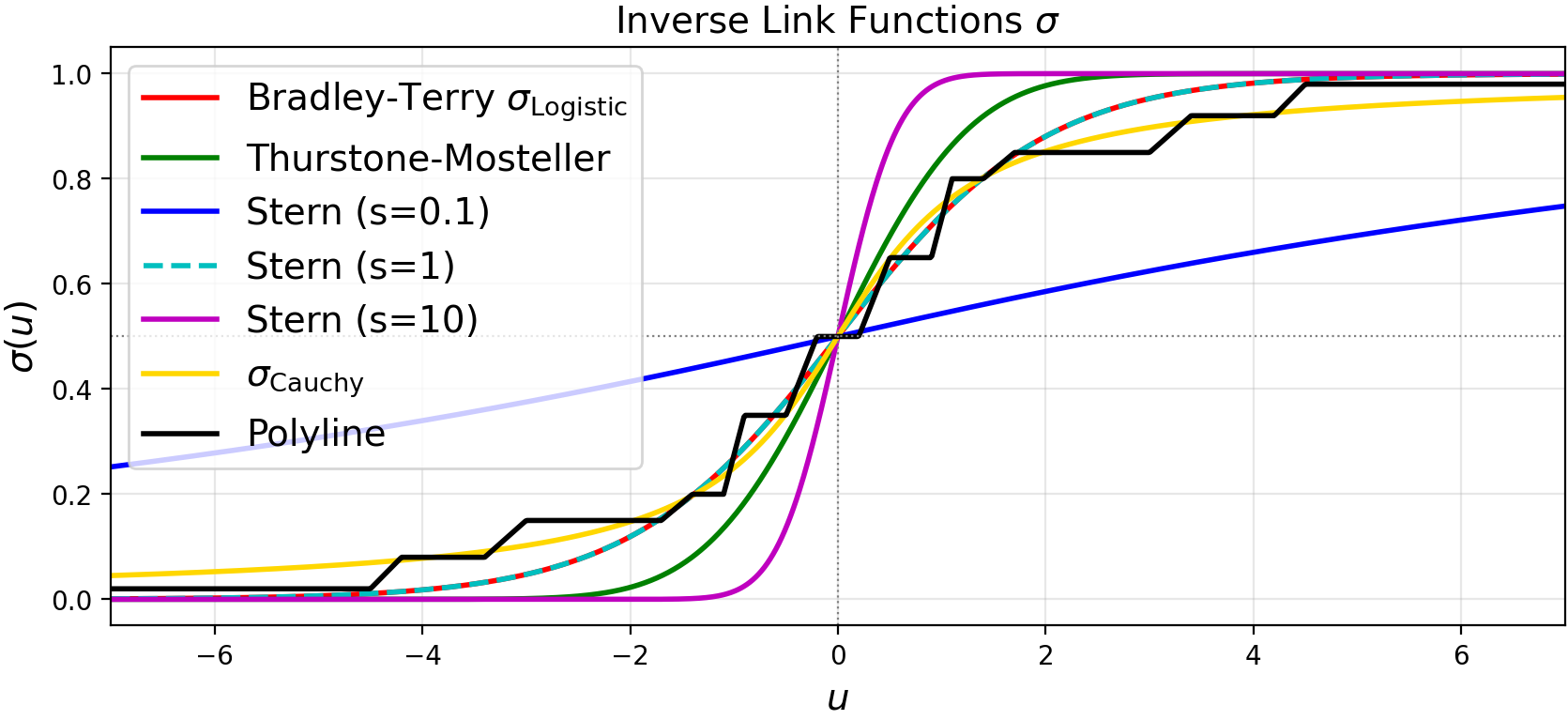}
\caption{%
Examples of the inverse link function $\sigma$.}
\label{fig:InvLinkFunc}
\end{figure}

For such a usage, the inverse link function $\sigma$
should be a non-decreasing function from $\bbR$ to $[0,1]$ and 
satisfy the symmetry $\sigma(-u)=1-\sigma(u)$ 
so that $\sigma(u)\le0.5$ for $u<0$, $\sigma(0)=0.5$, $\sigma(u)\ge0.5$ for $u>0$, 
and the model predicts that player $i$ is not weaker than player $j$ if $\hat{r}_i\ge\hat{r}_j$.
We denote the set of such functions as $\Sigma$.
As the inverse link function $\sigma\in\Sigma$, Bradley-Terry model 
\citep{Zermelo1929, ford1957solution, bradley1952rank, luce1959individual}
applies the logistic function $\sigma_\Logistic(u)\coloneq1/(1+e^{-u})$
so that $\sigma(r_i-r_j)=s_i/(s_i+s_j)$ with $s_i=e^{r_i}$, and
Thurstone-Mosteller model \citep{thurstone1927law, mosteller1951remarks1}
applies the cumulative distribution function of the Gauss distribution
$\sigma(u)=\int_{-\infty}^u (2\pi)^{-1/2}e^{-v^2/2}\,dv$
(see also Figure~\ref{fig:InvLinkFunc}).
Also, \citet{stern1990continuum} proposed to use $\sigma(u)=
\int_0^\infty\int_0^{e^u w} \{\Gamma(s)\}^{-2}(vw)^{s-1}e^{-(v+w)}\,dv\,dw$
with a hyper-parameter $s\in(0,\infty)$ and the gamma function $\Gamma$.
This proposal corresponds to predicting the win probability $y_{i,j}$ by
the probability $\Pr(X_i<X_j)$ 
with independent gamma random variables $X_i, X_j$ of 
common shape parameter $s$ and scale parameter $e^{r_i}, e^{r_j}$,
or equivalently, by the probability that player $i$ will be the first to score 
$s$ points under the assumption that players $i$ and $j$ score points 
according to a Poisson process with rates $e^{r_i}, e^{r_j}$ when $s\in\bbN$.
This inverse link function $\sigma(u)$ parameterizes a continuous transition 
between purely random and deterministic regimes by the hyper-parameter $s$:
Specifically, it degenerates to the constant function $\sigma(u) = 1/2$ as $s\to0$,
reduces to the logistic function when $s=1$, and 
approaches the deterministic step function $\sigma(u)\le0$ for $u<0$, 
$=0.5$ for $u=0$, and $\ge0.5$ for $u>0$ as $s\to\infty$.

A learning procedure for the rate parameters 
is typically designed according to the interpretation that 
$\sigma(r_i-r_j)$ is an estimation model of $y_{i,j}$ as
\begin{align}
\label{eq:Rate}
	(\hat{r}_i)_{i\in[n]}
	\in\underset{(r_i)_{i\in[n]}}{\argmin}\,
	\biggl[\frac{1}{|D_\tra|}\sum_{(i,j)\in D_\tra}\phi(\sigma(r_i-r_j),y_{i,j})\biggr]
\end{align}
with a non-negative-valued loss function $\phi$,
where $|D|$ is the cardinality of the set $D$.
Popular instances of the loss function $\phi$ are,
for example, the negative log-likelihood (NLL) loss $\phi_\nll(u,v)\coloneq-v\log u$ 
if $u\in(0,1]$, $0$ if $v=0$, or $\infty$ if $u=0$ and $v\neq0$
and robust alternatives \citep{bianco1996robust,yamasaki2023label}
including the squared loss $\phi_\sq(u,v)\coloneq(u-v)^2$.
For a loss function $\phi$ that is strictly convex with respect to the first argument,
under some regularity conditions\footnote{%
\label{fn:iden}
\cite{ford1957solution,hunter2004mm} assumed that
in every possible partition of $n$ players into two nonempty subsets,
some player $i$ in the second set wins some player $j$ in the first set at least once
(i.e., $y_{i,j}>0$ for $(i,j)\in D_\tra$).
This condition works in two ways:
First, it guarantees that there are no incomparable player subsets; 
if there are, the learning loses the identifiability of the rate parameters.
Second, it guarantees that there are no players with all loses or all wins; 
if there are, the rate parameter can take the value $\pm\infty$.}
the problem \eqref{eq:Rate} becomes convex and has a unique solution,
which can obtained by Majorize-Minimize (MM) algorithm \citep[Section~2]{hunter2004mm}.
For the win probability prediction, we can employ 
the same-form criterion as that for learning the rate parameters:
\begin{align}
\label{eq:WPE}
	\frac{1}{|D|}\sum_{(i,j)\in D}\phi(\sigma(r_i-r_j),y_{i,j})
\end{align}
with learned rate parameters 
$(r_i)_{i\in[n]}=(\hat{r}_i)_{i\in[n]}$ and $D=D_\tes$.
On the other hand, when we are further interested in the player ranking,
we can rely on Kendall's Tau (Kendall's rank correlation coefficient) \citep{kendall1938new},
\begin{align}
\label{eq:Kendall}
	\frac{n_+-n_-}{\sqrt{(|D|-n_+)(|D|-n_-)}}
\end{align}
with $(r_i)_{i\in[n]}=(\hat{r}_i)_{i\in[n]}$ and $D=D_\tes$,
$n_+=\sum_{(i,j)\in D}\bbI((\sigma(r_i-r_j)-0.5)\cdot(y_{i,j}-0.5)>0)$,
and $n_-=\sum_{(i,j)\in D}\bbI((\sigma(r_i-r_j)-0.5)\cdot(y_{i,j}-0.5)<0)$,
where $\bbI$ is the indicator function that values
$\bbI(c)=1$ if the condition $c$ is true or $0$ otherwise,
and Spearman's Rho (Spearman's rank correlation coefficient) \citep{Spearman04};
Kendall's Tau is considered to be more stable than 
Spearman's Rho in terms of ties (e.g., $\sigma(r_i-r_j)=0.5$),
mathematically easier to handle, and 
therefore preferred as described in \citet{noether1981kendall}.

\section{Proposal: Isotonic Bradley-Terry Model}
\label{sec:Proposal}
Generalized Bradley-Terry model has excellent interpretability,
which comes from real-valued representation of the players'\;strengths.
However, conventional usages of generalized Bradley-Terry model 
that select a fixed inverse link function $\sigma$ in advance
may suffer from model misspecification and perform poorly.
Even with model selection that selects an inverse link function 
from multiple pre-specified candidates based on learning results, 
it is not ensured to select a inverse link function 
that is optimal among the class $\Sigma$.
We are interested in how good a model can be designed while maintaining 
the excellent interpretability of generalized Bradley-Terry model 
that models players'\;strengths by real-valued parameters.
Therefore, in this paper, we propose to learn 
(optimize) the inverse link function $\sigma$ as well.

In this paper, for optimization of the inverse link function $\sigma$,
we relay on (generalized) isotonic regression 
\citep{brunk1955maximum,van1956maximum,barlow1972statistical,robertson1988order},
a nonparametric approach:
\begin{align}
\label{eq:Link}
	\begin{split}
	\hat{\sigma}
	&\coloneq\text{`a polyline function connecting the unique elements}
\\
	&\hphantom{\coloneq~~,}
	\text{of $(\hat{r}_i-\hat{r}_j,\hat{\sigma}_{i,j})_{(i,j)\in D_\tra}$ sorted in $(\hat{r}_i-\hat{r}_j)_{(i,j)\in D_\tra}$'}
	\end{split}
\end{align}
with $(\hat{\sigma}_{i,j})_{(i,j)\in D_\tra}$ that are defined, 
for learned rate parameters $(\hat{r}_i)_{i\in[n]}$, by
\begin{align}
\label{eq:Knots}
	\begin{split}
	(\hat{\sigma}_{i,j})_{(i,j)\in D_\tra}
	&\in\underset{(\sigma_{i,j})_{(i,j)\in D_\tra}}{\argmin}\,
	\biggl[\frac{1}{|D_\tra|}\sum_{(i,j)\in D_\tra}\phi(\sigma_{i,j},y_{i,j})\biggr]
\\
	&\text{~~~~~~~~~\,s.t.~~~~~~~\,}
	\sigma_{j,i}=1-\sigma_{i,j}\text{ and }
	\sigma_{i,j}\le\sigma_{i',j'}\text{ if }\hat{r}_i-\hat{r}_j\le\hat{r}_{i'}-\hat{r}_{j'}
\\
	&\hphantom{\text{~~~~~~~~~\,s.t.~~~~~~~\,}}
	\text{for all }(i,j),(i',j')\in D_\tra.
	\end{split}
\end{align}
See Figure~\ref{fig:InvLinkFunc} for an instance of the polyline function.
We propose to first select an inverse link function $\sigma$
(for example, $\sigma=\sigma_\Logistic$) and then iterate 
the updates \eqref{eq:Rate} and \eqref{eq:Link} alternately:
Denoting $\hat{\sigma}^{[0]}\coloneq\sigma$,
calculate the solution $(\hat{r}_i)_{i\in[n]}$ of \eqref{eq:Rate} 
with $\sigma=\hat{\sigma}^{[t-1]}$ as $(\hat{r}_i^{[t]})_{i\in[n]}$ 
and solution $\hat{\sigma}$ of \eqref{eq:Link} 
with $(\hat{r}_i)_{i\in[n]}=(\hat{r}_i^{[t]})_{i\in[n]}$ as $\hat{\sigma}^{[t]}$
as increasing the index $t\in\bbN$.
We call this model isotonic Bradley-Terry model.
While previous works employ $(\hat{\sigma}^{[0]},(\hat{r}_i^{[1]})_{i\in[n]})$,
the isotonic Bradley-Terry model employs $(\hat{\sigma}^{[t]},(\hat{r}_i^{[t]})_{i\in[n]})$
or $(\hat{\sigma}^{[t]},(\hat{r}_i^{[t+1]})_{i\in[n]})$ with $t\in\bbN$.

{\IncMargin{2em}
\begin{algorithm}[t]
\caption{PAV algorithm for the isotonic regression problem \eqref{eq:Knots}}
\label{alg:PAV-algorithm}
\SetAlgoNlRelativeSize{0}
\KwIn{Loss function $\phi$, inputs $(\hat{r}_i-\hat{r}_j)_{(i,j)\in D_\tra}$, and targets $(y_{i,j})_{(i,j)\in D_\tra}$.}

\tcc{Initialize estimates $L_k$, $R_k$, and $z_k$ for $k=1,\ldots,N$.}
\nl Let $\{x_k\}_{k=1}^N$ be unique sorted elements of $\{\hat{r}_i-\hat{r}_j\}_{(i,j)\in D_\tra}$:
	$x_1<\cdots<x_N$, $\hat{r}_i-\hat{r}_j\in\{x_k\}_{k=1}^N$ for all $(i,j)\in D_\tra$,
	and $x_k\in\{\hat{r}_i-\hat{r}_j\}_{(i,j)\in D_\tra}$ for all $k=1,\ldots,N$.\hypertarget{1}{}

\nl Left edge of $k$th input group $L_k\leftarrow x_k$, 
	right edge of $k$th input group $R_k\leftarrow x_k$, 
	and $k$th target group $\calY_k\leftarrow \{y_{i,j}\mid\hat{r}_i-\hat{r}_j=x_k\}_{(i,j)\in D_\tra}$ for $k=1,\ldots,N$.\hypertarget{3}{}

\nl Target estimate for $k$th group $z_k\in\argmin_{u}\sum_{v\in\calY_k}\{\phi(u,v)+\phi(1-u,1-v)\}$ for $k=1,\ldots,N$.

\tcc{Update estimates $L_k$, $R_k$, and $z_k$ for $k=1,\ldots,N$.}
\nl \While{$z_1<\cdots<z_N$ is false}{
	\tcc{Update groups: Merge adjacent groups if target estimates violate order.}
	\nl Index of updated group $l\leftarrow1$, $L'_1\leftarrow L_1$, $R'_1\leftarrow R_1$, and $\calY'_1\leftarrow\calY_1$.

	\nl \For{$k=1,\ldots,N-1$}{
		\nl \lIf{$z_k<z_{k+1}$}{
			$l\leftarrow l+1$, $L'_l\leftarrow L_{k+1}$, $R'_l\leftarrow R_{k+1}$, and $\calY'_l\leftarrow\calY_{k+1}$
		}\nl\lElse{
			$R'_l\leftarrow R_{k+1}$ and $\calY'_l\leftarrow\calY'_l\cup\calY_{k+1}$
		}
	}
	\nl $N\leftarrow l$, $L_k\leftarrow L'_k$, $R_k\leftarrow R'_k$, and $\calY_k\leftarrow\calY'_k$ for $k=1,\ldots,N$.

	\tcc{Update target estimates.}
	\nl $z_k\in\argmin_{u}\sum_{v\in\calY_k}\{\phi(u,v)+\phi(1-u,1-v)\}$ for $k=1,\ldots,N$.\hypertarget{10}{}
}
\KwOut{$(L_k,R_k,z_k)_{k\in[N]}$, 
or $\hat{\sigma}_{i,j}=z_k$ if $\hat{r}_i-\hat{r}_j\in[L_k,R_k]$ for all $(i,j)\in D_\tra$.}
\end{algorithm}
\DecMargin{2em}}

We can perform the update \eqref{eq:Link} 
(global minimization) of the inverse link function with 
pool-adjacent-violators (PAV) algorithm \citep{ayer1955empirical}%
\footnote{%
There are also other algorithms to solve isotonic regression problem; see, for instances, 
\citet{brunk1957minimizing,thompson1962problem,lee1983min,best1990active}.
However, PAV algorithm achieves the information-theoretic lower bound of computation complex.},
that is, Algorithm~\ref{alg:PAV-algorithm} in this paper:
\begin{theorem}
\label{thm:thm1}
If $\phi(u,v)$ is strictly convex in $u\in(0,1)$ at every $v\in[0,1]$
such as $\phi=\phi_\nll, \phi_\sq$,
Algorithm~\ref{alg:PAV-algorithm} provides a solution of \eqref{eq:Knots},
or equivalently, $\hat{\sigma}$ in \eqref{eq:Link} can be represented as 
\begin{align}
\label{eq:Link2}
	\hat{\sigma}(u)=
	\begin{cases}
	z_1&\text{if }u\le L_1\\
	z_k&\text{if }u\in [L_k,R_k]\text{ with some }k\in[N]\\
	z_k+(z_{k+1}-z_k)\frac{u-R_k}{L_{k+1}-R_k}&\text{if }u\in [R_k,L_{k+1}]\text{ with some }k\in[N-1]\\
	z_N&\text{if }u\ge R_N
	\end{cases}
\end{align}
with $(L_k,R_k,z_k)_{k\in[N]}$ given by Algorithm~\ref{alg:PAV-algorithm}.
\end{theorem}
\begin{proof}[Proof of Theorem~\ref{thm:thm1}]
This theorem can be proved based on the fact that the problem \eqref{eq:Knots} is 
a convex and separable problem under the assumption of Theorem~\ref{thm:thm1};
refer to \citet[Section~3]{best1990active} or \citet[Sections~2.3 and 3.1]{de2010isotone}
for general cases, and to \citet[Section~1.3]{barlow1972statistical} and 
\citet[Theorem~1.5.1]{robertson1988order} for the specific case with $\phi=\phi_\sq$.
\end{proof}

The inner optimization problem of Algorithm~\ref{alg:PAV-algorithm},
$\argmin_{u}\sum_{v\in\calY_k}\{\phi(u,v)+\phi(1-u,1-v)\}$, 
appeared in Lines~\hyperlink{3}{3} and \hyperlink{10}{10}
may be a significant computational burden.
However, for a special selection of the loss function $\phi$,
this inner problem has a closed-form solution
that can be calculated sequentially and quickly:
\begin{theorem}
\label{thm:thm2}
For $\phi=\phi_\nll, \phi_\sq$, 
$\argmin_{u}\sum_{v\in\calY_k}\{\phi(u,v)+\phi(1-u,1-v)\}$ in 
Lines~\hyperlink{3}{3} and \hyperlink{10}{10} of Algorithm~\ref{alg:PAV-algorithm}
becomes $\frac{1}{|\calY_k|}\sum_{v\in\calY_k}v$,
and $\hat{\sigma}$ in \eqref{eq:Link} and \eqref{eq:Link2} belongs to $\Sigma$.
\end{theorem}
\begin{proof}[Proof of Theorem~\ref{thm:thm2}]
The stationary condition becomes
$\partial_u[\sum_{v\in\calY_k}\{\phi(u,v)+\phi(1-u,1-v)\}]
=\partial_u[-\sum_{v\in\calY_k}\{v\log u+(1-v)\log(1-u)\}]
=\frac{1}{u(1-u)}\{u|\calY_k|-\sum_{v\in\calY_k}v\}$ for $\phi=\phi_\nll$ or 
$\partial_u[\sum_{v\in\calY_k}\{\phi(u,v)+\phi(1-u,1-v)\}]
=\partial_u[2\sum_{v\in\calY_k}(u-v)^2]
=4\{|\calY_k|u-\sum_{v\in\calY_k}v\}$ for $\phi=\phi_\sq$.
This concludes the proof of this theorem.
\end{proof}

However, since the mean $\frac{1}{|\calY_k|}\sum_{v\in\calY_k}v$ can take 0,
it should be noted that the test evaluation of the isotonic Bradley-Terry model,
$\frac{1}{|D_\tra|}\sum_{(i,j)\in D_\tes}\phi(\hat{\sigma}(\hat{r}_i-\hat{r}_j),y_{i,j})$
for $\phi=\phi_\nll$, can not be calculated (i.e., take NaN).
For the the player ranking,
the isotonic Bradley-Terry model does not suffer from 
this issue by using the evaluation criterion \eqref{eq:Kendall}, 
but for evaluation of the win probability prediction
this model may not be suitable when we have to 
use the criterion \eqref{eq:WPE} with $\phi=\phi_\nll$
(it does not suffer from this issue when using 
the criterion \eqref{eq:WPE} with $\phi=\phi_\sq$).

An important point to remark is that 
subsequent learning of rate parameter by minimizing \eqref{eq:Rate} with 
$\sigma=\hat{\sigma}^{[t]}$, $t\ge1$ can become non-convex generally.
However, we can employ a sub-gradient method with line search of the step size 
to update the rate parameters without increasing the training error,
though it has no guarantee of global optimization.
Thus, we propose to learn the rate parameters by a (sub-)gradient method 
(MM algorithm \citep{hunter2004mm} is also possible for the first time only) 
and the inverse link function by an isotonic regression technique alternately,
for the isotonic Bradley-Terry model, $(\hat{\sigma}^{[t]},(\hat{r}_i^{[t]})_{i\in[n]})$
or $(\hat{\sigma}^{[t]},(\hat{r}_i^{[t+1]})_{i\in[n]})$ with $t\in\bbN$.

Another noteworthy property of the isotonic Bradley-Terry model is
that it is more likely to result in a tie;
consider it can hold that $\hat{\sigma}(r_i-r_j)=0.5$ even if $r_i\neq r_j$.
The isotonic Bradley-Terry model can accurately output a tie as undecidable 
when one has insufficient training pair comparison data to accurately rank the pairs.
In relation to this property, it should be noted that in the isotonic Bradley-Terry model, 
ranking should be performed according to the descending order of Borda count
$(\sum_{j\in[n]}\sigma(\hat{r}_i-\hat{r}_j))_{i\in[n]}$ as \citet{Nihar18},
not that of the rate parameters $(\hat{r}_i)_{i\in[n]}$:
Borda count would better reflect the strength than the rate parameters for the isotonic Bradley-Terry model,
since (i) $\hat{r}_k>\hat{r}_l\iff\sum_{j\in[n]}\sigma(\hat{r}_k-\hat{r}_j)>\sum_{j\in[n]}\sigma(\hat{r}_l-\hat{r}_j)$
for strictly increasing $\sigma$, 
but (ii) $\hat{r}_k>\hat{r}_l\Leftarrow\sum_{j\in[n]}\sigma(\hat{r}_k-\hat{r}_j)>\sum_{j\in[n]}\sigma(\hat{r}_l-\hat{r}_j)$,
$\hat{r}_k>\hat{r}_l\Rightarrow\sigma(\hat{r}_k-\hat{r}_l)\ge0.5$, and
$\hat{r}_k>\hat{r}_l\Rightarrow\sigma(\hat{r}_k-\hat{r}_j)\ge\sigma(\hat{r}_l-\hat{r}_j)$ for all $j\in[n]$,
but $\hat{r}_k>\hat{r}_l\not\Rightarrow\sum_{j\in[n]}\sigma(\hat{r}_k-\hat{r}_j)>\sum_{j\in[n]}\sigma(\hat{r}_l-\hat{r}_j)$ 
for $\sigma$ that is only non-decreasing and not strictly increasing.

\section{Discussion: Relationships to Nonparametric Bradley-Terry Model}
\label{sec:Discussion}
A significant characteristic of the isotonic Bradley-Terry model
is the use of isotonic regression for 
the win probability prediction and player ranking.
A previous model, nonparametric Bradley-Terry model
\citep{chatterjee2015matrix,shah2016stochastically,chatterjee2018matrix},
for the win probability prediction also shares this characteristic.
We here discuss relationships between the isotonic Bradley-Terry model 
and nonparametric Bradley-Terry model.

Under the assumption that the player ranking is known,
nonparametric Bradley-Terry model predicts the win probability $y_{i,j}$ as $\hat{\sigma}_{i,j}$ 
defined by the following bivariate isotonic regression problem:
\begin{align}
	(\hat{\sigma}_{i,j})_{(i,j)\in D_\tra}
	&\in\underset{(\sigma_{i,j})_{(i,j)\in D_\tra}}{\argmin}\,
	\biggl[\frac{1}{|D_\tra|}\sum_{(i,j)\in D_\tra}\phi(\sigma_{i,j},y_{i,j})\biggr]
\nonumber\\
	&\text{~~~~~~~~~\,s.t.~~~~~~~\,}
	\sigma_{j,i}=1-\sigma_{i,j}\text{ and }
	\sigma_{i,j}\le\sigma_{i',j}
	\text{ if player $i'$ is higher-}
\nonumber\\
	&\hphantom{\text{~~~~~~~~~\,s.t.~~~~~~~\,}}
	\text{ranked than player $i$}
	\text{ for all }(i,j),(i',j)\in D_\tra
	\label{eq:NBT}
\end{align}
with $D_\tra=\{(i,j)\in[n]^2\mid i\neq j\}$.
Then, the solution $(\sigma_{i,j})_{(i,j)\in D_\tra}=(\hat{\sigma}_{i,j})_{(i,j)\in D_\tra}$ 
given by \eqref{eq:Knots} satisfies the conditions in \eqref{eq:NBT}, but 
the solution $(\sigma_{i,j})_{(i,j)\in D_\tra}=(\hat{\sigma}_{i,j})_{(i,j)\in D_\tra}$ 
given by \eqref{eq:NBT} does not necessarily satisfy 
the structural conditions as in generalized Bradley-Terry model.
Therefore, the training error $\frac{1}{|D_\tra|}\sum_{(i,j)\in D_\tra}\phi(\hat{\sigma}_{i,j},y_{i,j})$
by the isotonic Bradley-Terry model \eqref{eq:Knots} is larger than or equal to
the training error by the nonparametric Bradley-Terry model \eqref{eq:NBT}.
This point may be a demerit of the isotonic Bradley-Terry model.

On the other hand, the availability for predicting the win probability of an unmatched player pair
and the possibility to improve the player ranking are merits of the isotonic Bradley-Terry model:
Both models construct nonparametric estimates $(\hat{\sigma}_{i,j})_{(i,j)\in D_\tra}$, 
but the isotonic Bradley-Terry model can predict 
the win probability for an unmatched player pair by assuming 
the structural model $\hat{\sigma}_{i,j}=\hat{\sigma}(\hat{r}_i-\hat{r}_j)$.
Also, introducing the model $\hat{\sigma}_{i,j}=\hat{\sigma}(\hat{r}_i-\hat{r}_j)$
allows for the discussion of updating the rate parameters and player ranking;
the nonparametric Bradley-Terry model is designed under the assumption 
that the rankings are known in advance, making it difficult to update them.

\section{Numerical Experiments}
\label{sec:Experiments}
\subsection{Synthetic Data Experiments}
\label{sec:Synthetic}
We are interested in practical behaviors of the isotonic Bradley-Terry model
and situations where that model can improve the win probability prediction 
performance or ranking performance compared with the traditional Bradley-Terry model.
To investigate those aspects, we present synthetic data experiments in this subsection.

\paragraph{Procedure 1}
We generated synthetic data (Cauchy-$N$) consisted of $n=25,50,100,200,400$ players and
win probabilities $y_{i,j}=1-y_{j,i}\sim\Binomial(N,\sigma_\Cauchy(\tilde{r}_i-\tilde{r}_j))/N$
with the number $N=1,5,25$ of matches for each pair of players, underlying inverse link function 
$\sigma_\Cauchy(u)\coloneq\frac{1}{\pi}\arctan(u)+\frac{1}{2}$,
which is the cumulative distribution function of the standard Cauchy distribution,
and underlying players'\;strengths $\tilde{r}_i\sim\Normal(0,1)$,
where $\Binomial(u,v)$ is the binomial distribution with $u$ trials and probability $v$,
and where $\Normal(u,v)$ is the normal distribution with mean $u$ and variance $v$.
We divided all the player pairs into $D_\tra$ and $D_\tes$ in $1:9$, $3:7$, $\ldots$, $9:1$ ratios.
Because a random variable following $\Binomial(N,\sigma_\Cauchy(u))/N$ 
with $u\in\bbR$ has the expectation value $\sigma_\Cauchy(u)$, 
the Bradley-Terry model with $\sigma=\sigma_\Cauchy$ can perform
consistent estimation of the rate parameter $(\tilde{r}_i)_{i=1}^n$
(ignoring the degrees of freedom for translation) 
when $n\to\infty$, or $N\to\infty$, and so on 
\citep{bradley1952rank,han2020asymptotic,simons1999asymptotics}.
However, since we are interested in a situation with model misspecification,
we compared the Bradley-Terry model, $(\hat{\sigma}^{[0]},(\hat{r}_i^{[1]})_{i\in[n]})$,
and isotonic Bradley-Terry model, 
$(\hat{\sigma}^{[t]},(\hat{r}_i^{[t+1]})_{i\in[n]})$ for $t\in[9]$
and $(\hat{\sigma}^{[t]},(\hat{r}_i^{[t]})_{i\in[n]})$ for $t\in[10]$,
with $\hat{\sigma}^{[0]}=\sigma_\Logistic$,
not with $\hat{\sigma}^{[0]}=\sigma_\Cauchy$.
This setting is that the Bradley-Terry model can suffer from model misspecification, 
which leaves room for improvement by using the isotonic Bradley-Terry model;
See Appendix~\ref{sec:Specified} for experiments with a specified model.
We learned the rate parameters \eqref{eq:Rate}
by a coordinate sub-gradient method with line search of the step size 
and inverse link function \eqref{eq:Knots} by PAV algorithm
(exactly, we used the \texttt{IsotonicRegression} class of 
the Python package \texttt{sklearn.isotonic} \citep{scikit-learn}), 
with the objective function \eqref{eq:WPE} for $\phi=\phi_\sq$ (WPP error);
See Appendix~\ref{sec:NLL} for experiments with $\phi=\phi_\nll$.
We repeated this procedure 1000 trials with random data generation
to obtain 1000 test evaluations of the errors,
\eqref{eq:WPE} and \eqref{eq:Kendall}.
Also, we evaluated the tie rate defined by
\begin{align}
\label{eq:TIE}
	\frac{|\{(i,j)\in D\mid \sigma(r_i-r_j)=0.5\}|}{|D|}.
\end{align}
A larger value of the tie rate means more ties in $(\sigma(r_i-r_j))_{(i,j)\in D}$.
Note that, since nonparametric Bradley-Terry model reviewed in 
Section~\ref{sec:Discussion} is not applicable to the win probability 
prediction for unmatched players, we did not try that method.

\paragraph{Procedure 2}
It was observed from Figure~\ref{fig:Cauchy-SQ} that excessive updates in 
the isotonic Bradley-Terry model lead to performance degradation due to overfitting. 
Therefore, the appropriate selection of the number of updates is crucial.
In this additional experiment procedure, we compared 
the Bradley-Terry model with the isotonic Bradley-Terry model 
whose number of updates was determined via 10-fold cross-validation. 
We performed this procedure 1000 trials and evaluated \eqref{eq:WPE}, 
\eqref{eq:Kendall}, and \eqref{eq:TIE} similar to Procedure 1.

\paragraph{Consideration of Results}
Figure~\ref{fig:Cauchy-SQ} shows the dependence of behaviors of 
the isotonic Bradley-Terry model $(\hat{\sigma}^{[*]},(\hat{r}_i^{[*]})_{i\in[n]})$ 
learned with the loss $\phi=\phi_\sq$ on the number of updates.
It would be obvious that the training WPP error does not increase 
with each update due to the design of the isotonic Bradley-Terry model.
The test WPP error with $\phi=\phi_\sq$ was also improved 
by the isotonic Bradley-Terry model in many cases.
In particular, the improvement was more pronounced 
when $n$, $N$, and the ratio of $|D_\tra|$ were small
and when $n$ or $N$ was large;
while the isotonic Bradley-Terry model $(\hat{\sigma}^{[1]},(\hat{r}_i^{[1]})_{i\in[n]})$ 
was the best performer in most cases even when it improved, 
the isotonic Bradley-Terry model $(\hat{\sigma}^{[1]},(\hat{r}_i^{[2]})_{i\in[n]})$, 
which underwent more updates, could achieve the best performance 
under smaller $n$, $N$, and the ratio of $|D_\tra|$ and under larger $n$ and $N$.
The improvement under small $n$, $N$, and the ratio of $|D_\tra|$ is unrelated to model misspecification,
and is likely due to the Bradley-Terry model's severe failure from overfitting 
caused by the scarcity of training data, as Figure~\ref{fig:Res-Cauchy-SQ}
also illustrates that $|\hat{r}_i^{[1]}-\hat{r}_j^{[1]}|$ can take a large value
for the learned rate parameters $(\hat{r}_i^{[1]})_{i\in[n]})$.
On the other hand, the improvement under large $n$ and $N$ would be specific to 
cases of model misspecification, as verified in Appendix~\ref{sec:Specified}.
Regarding training and test ranking errors, 
the isotonic Bradley-Terry model improved Kendall's Tau in most cases.
The isotonic Bradley-Terry model can be interpreted as improving 
overall ranking performance by withholding judgment as a tie 
for player pairs whose strict ranking is difficult to be determined,
as suggested from Figure~\ref{fig:Cauchy-SQ}.
We can see similar results from Table~\ref{tab:Cauchy-SQ},
which summarizes the results of Procedure 2.

\subsection{Real-World Data Experiments}
\label{sec:RealWorld}
In this subsection, we address numerical experiments with real-world data
to see practical effectiveness of the isotonic Bradley-Terry model.

\paragraph{Procedure 1}
We here used win probability data of 
football Premier League (PL) for 2024/2025 season of $n=20$ teams and 380 matches
(\url{https://www.football-data.co.uk/mmz4281/2425/E0.csv}),
baseball MLB for 2025 season of $n=30$ teams and 2430 matches
(\url{https://www.retrosheet.org/gamelogs/gl2025.zip}), and
tennis ATP tour for 2025 season of $n=457$ players, 2944 matches, and 2522 matched player pairs
(\url{https://github.com/JeffSackmann/tennis_atp/blob/master/atp_matches_2025.csv}).
ATP win probability data is sparse in the sense that only 2522 patterns were played
out of $n(n-1)/2=104196$ possible match patterns between two different players.
Similar to the synthesis data experiments in Section~\ref{sec:Synthetic},
we compared the (isotonic) Bradley-Terry models, 
$(\hat{\sigma}^{[t]},(\hat{r}_i^{[t+1]})_{i\in[n]})$ and 
$(\hat{\sigma}^{[t]},(\hat{r}_i^{[t]})_{i\in[n]})$ for $t\in[10]$,
learned with $\phi=\phi_\sq$ and $\hat{\sigma}^{[0]}=\sigma_\Logistic$,
in terms of WPP error \eqref{eq:WPE} with $\phi=\phi_\sq$, 
Kendall's Tau \eqref{eq:Kendall}, and tie rate \eqref{eq:TIE}.
See also Appendix~\ref{sec:NLL} for experiments with $\phi=\phi_\nll$.

\paragraph{Procedure 2}
In order to study practical performance comparison,
we took an additional experiment procedure to compare
the Bradley-Terry model with the isotonic Bradley-Terry model 
whose number of updates was determined via 10-fold cross-validation.

\paragraph{Consideration of Results}
Figure~\ref{fig:Real-SQ} shows the dependence of behaviors of 
the isotonic Bradley-Terry model $(\hat{\sigma}^{[*]},(\hat{r}_i^{[*]})_{i\in[n]})$ 
on the number of updates,
and Table~\ref{tab:Real-SQ} summarizes the results of Procedure 2.
Results are similar to those for synthetic data experiments in Section~\ref{sec:Synthetic}:
The isotonic Bradley-Terry model improved the WPP error 
when the ratio of $|D_\tra|$ was small for the PL and MLB data with small $n$, 
whereas it improved the WPP error across a wider range of ratio of $|D_\tra|$ 
for the ATP data with large $n$;
Although verification is difficult due to the nature of real-world data, 
the latter improvement might be attributed to the mitigation of 
model misspecification by the isotonic Bradley-Terry model as 
Figure~\ref{fig:Real-Res-Cauchy-SQ}, rescaled version illustrates.
Also, isotonic Bradley-Terry model improved the ranking performance in most cases.
These results suggest practical effectiveness of the proposed isotonic Bradley-Terry model.

\begin{sidewaysfigure}
\centering%
\renewcommand{\arraystretch}{0.5}%
\renewcommand{\tabcolsep}{0.5pt}%
\begin{tabular}{cc|ccc|ccc|ccc}%
&&\multicolumn{3}{c|}{\tiny$N=1$, $|D_\tra|:|D_\tes|=$}&\multicolumn{3}{c|}{\tiny$N=5$, $|D_\tra|:|D_\tes|=$}&\multicolumn{3}{c}{\tiny$N=25$, $|D_\tra|:|D_\tes|=$}\\
&&{\tiny$1:9$}&{\tiny$5:5$}&{\tiny$9:1$}&{\tiny$1:9$}&{\tiny$5:5$}&{\tiny$9:1$}&{\tiny$1:9$}&{\tiny$5:5$}&{\tiny$9:1$}\\
\midrule
\multirow{3}{*}[-2.5mm]{\rotatebox{90}{\tiny\eqref{eq:WPE}, $n=$}}
&\rotatebox{90}{\tiny\,~~~\,$25$}&
\CF{\includegraphics[width=2.0cm]{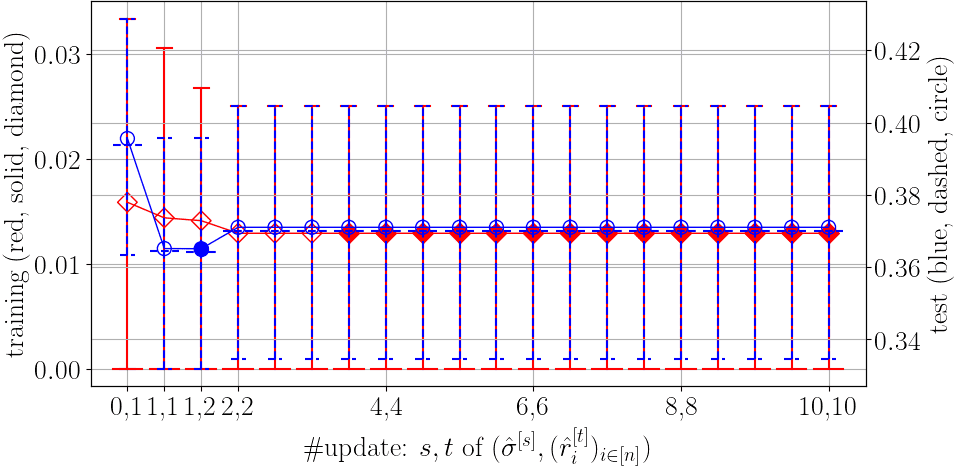}}&
\CF{\includegraphics[width=2.0cm]{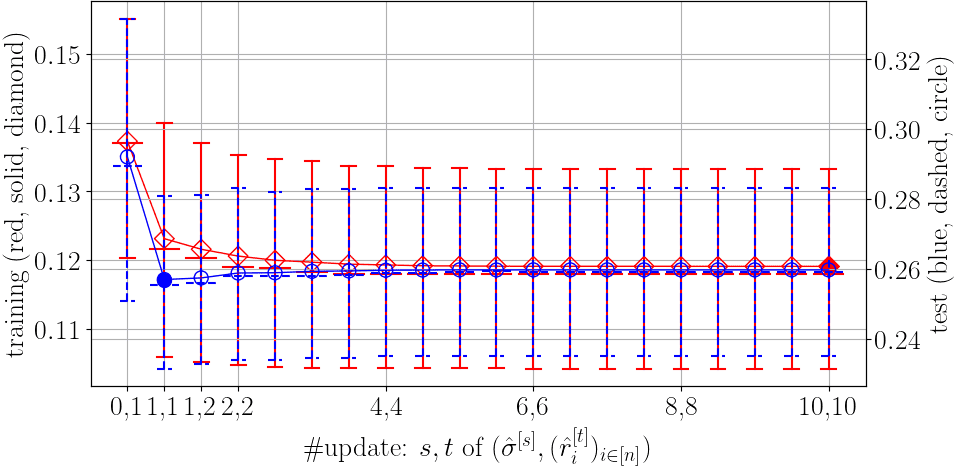}}&
{\includegraphics[width=2.0cm]{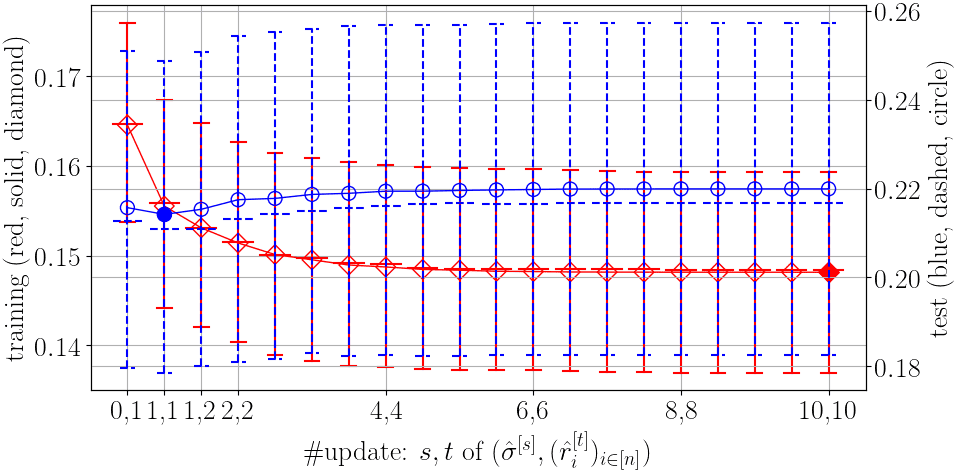}}&
\CF{\includegraphics[width=2.0cm]{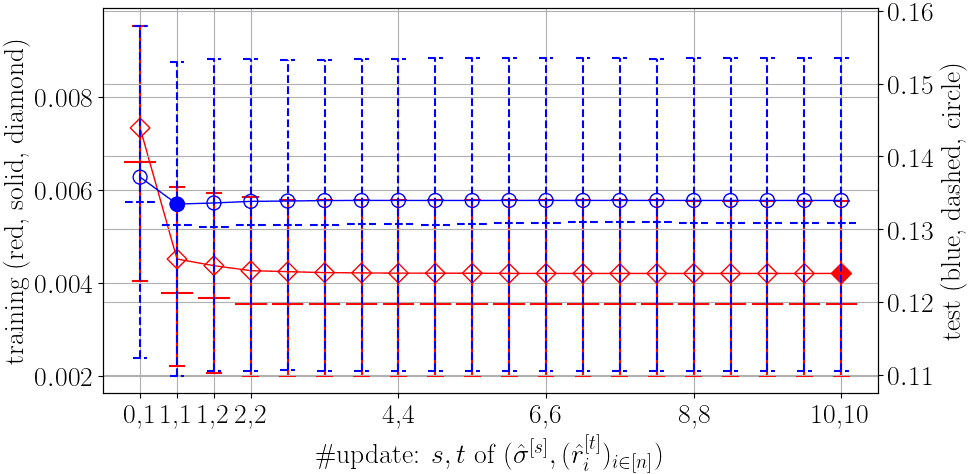}}&
{\includegraphics[width=2.0cm]{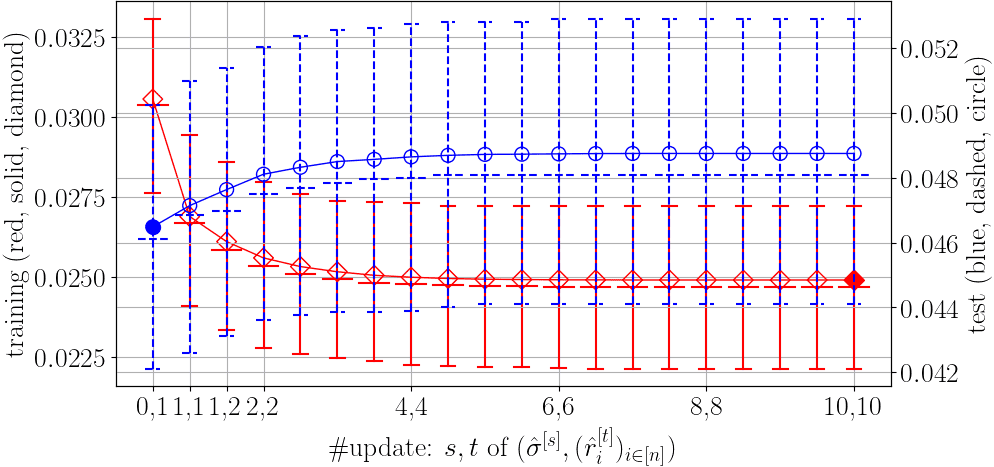}}&
{\includegraphics[width=2.0cm]{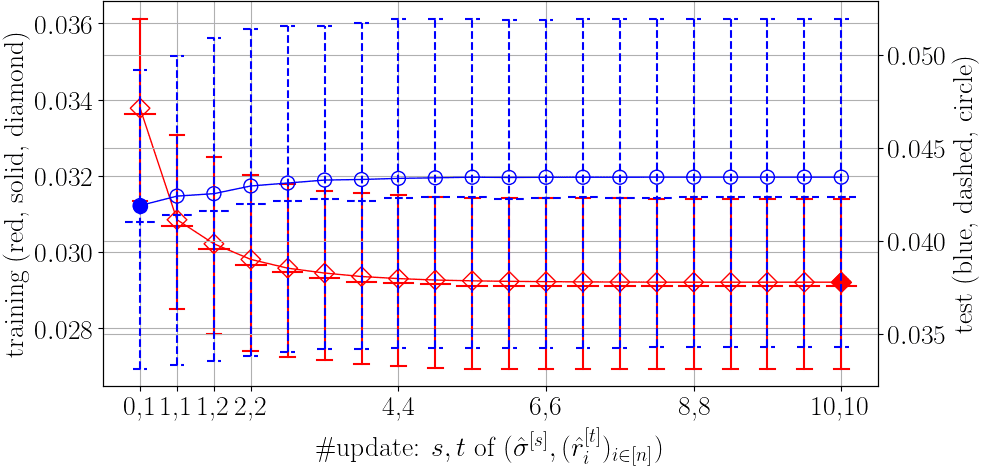}}&
{\includegraphics[width=2.0cm]{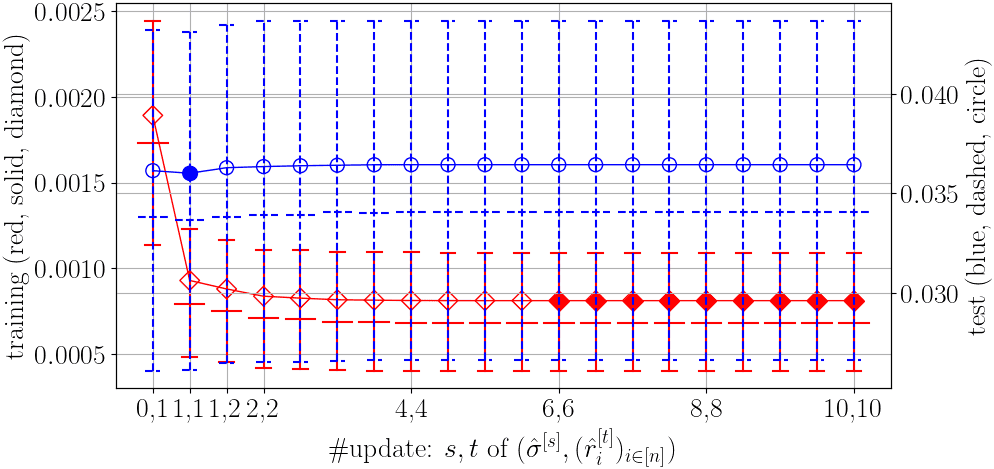}}&
\CF{\includegraphics[width=2.0cm]{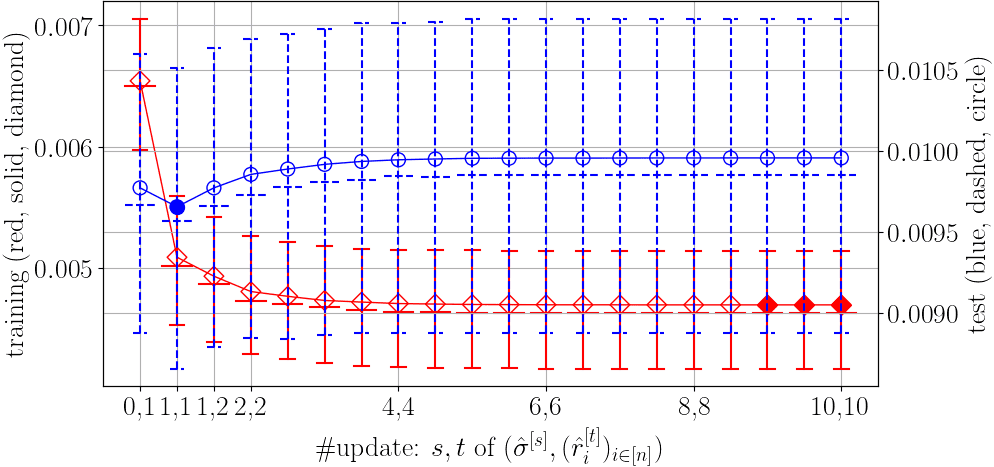}}&
\CF{\includegraphics[width=2.0cm]{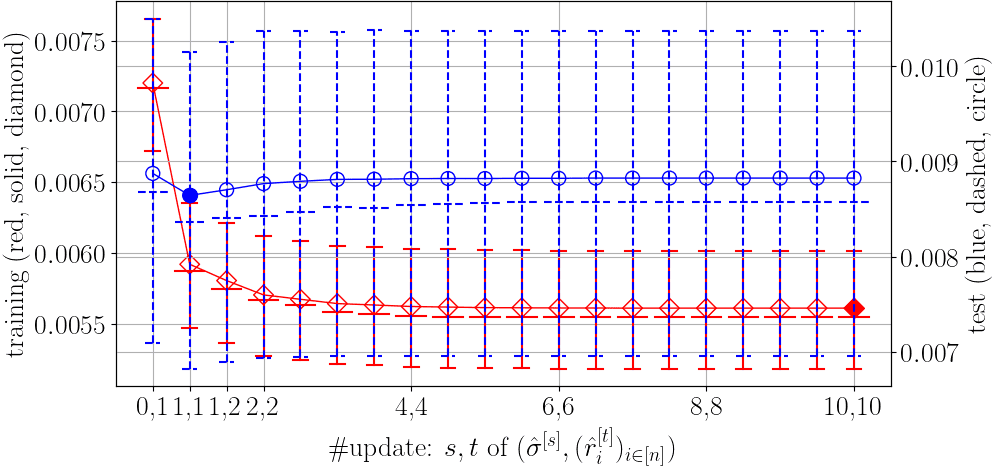}}\\
&\rotatebox{90}{\tiny\,~~\,$100$}&
\CF{\includegraphics[width=2.0cm]{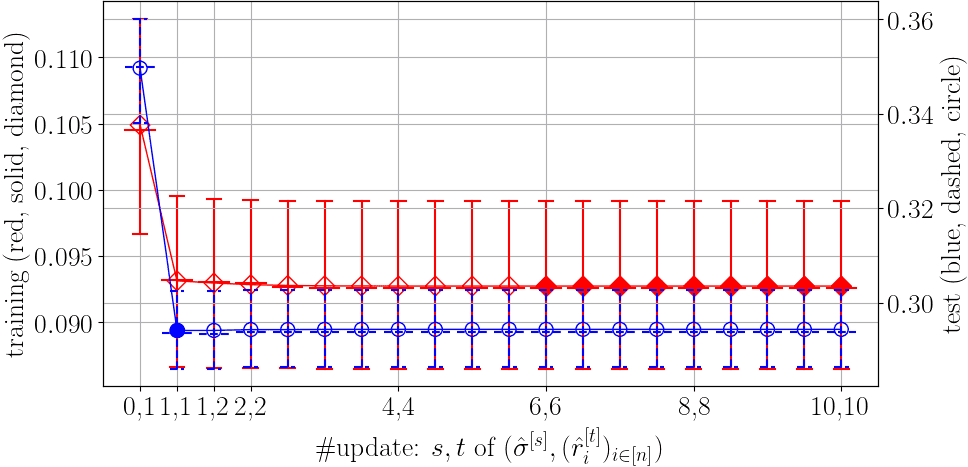}}&
{\includegraphics[width=2.0cm]{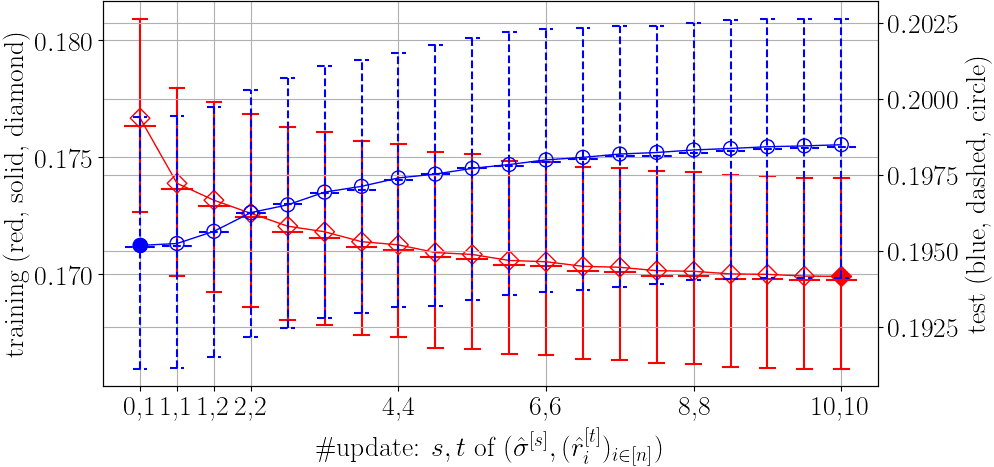}}&
{\includegraphics[width=2.0cm]{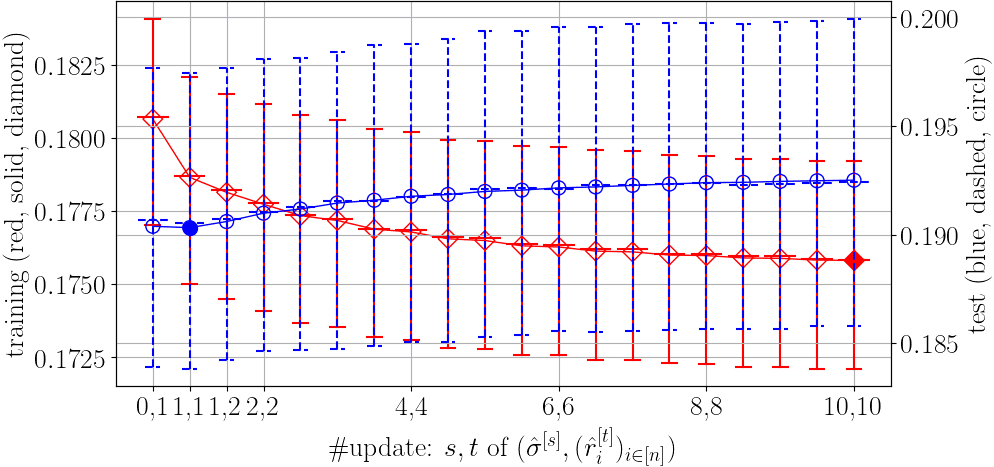}}&
{\includegraphics[width=2.0cm]{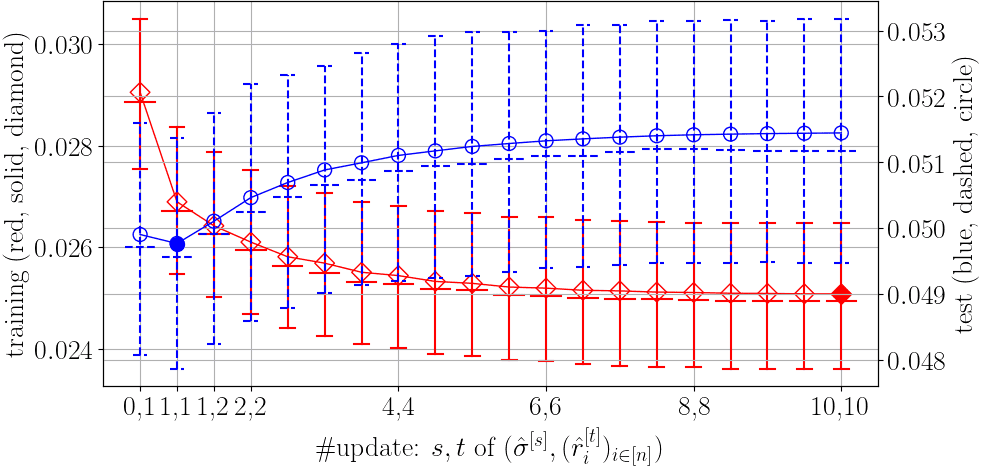}}&
\CF{\includegraphics[width=2.0cm]{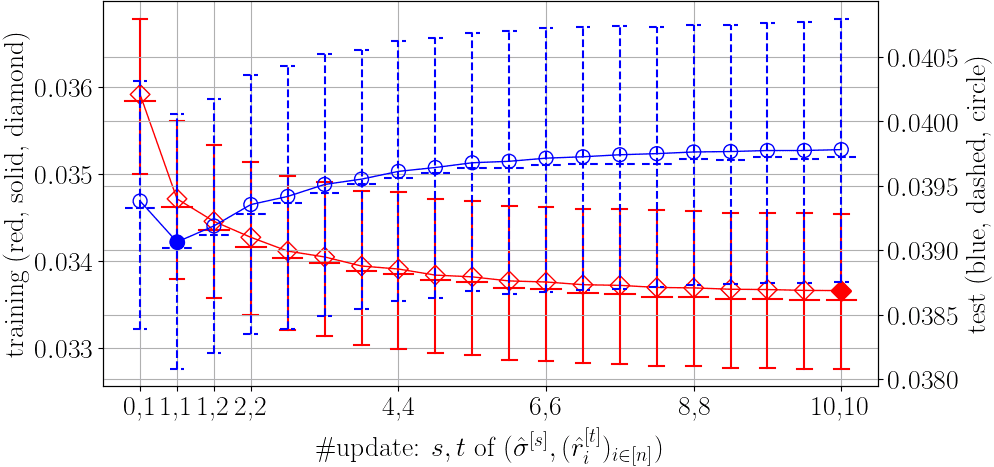}}&
\CF{\includegraphics[width=2.0cm]{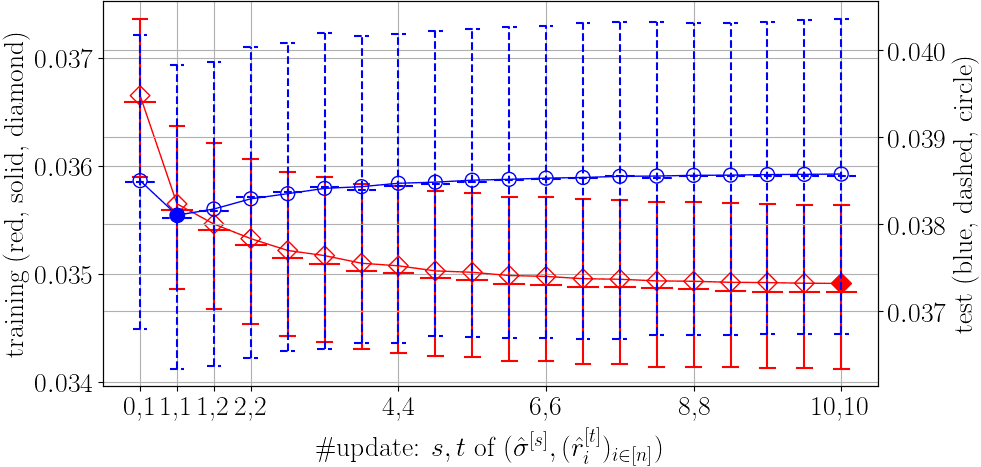}}&
\CF{\includegraphics[width=2.0cm]{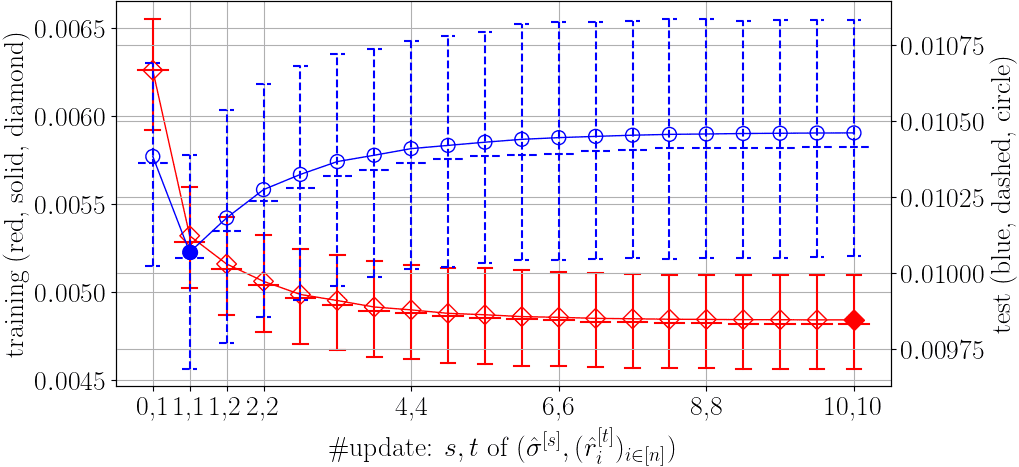}}&
\CF{\includegraphics[width=2.0cm]{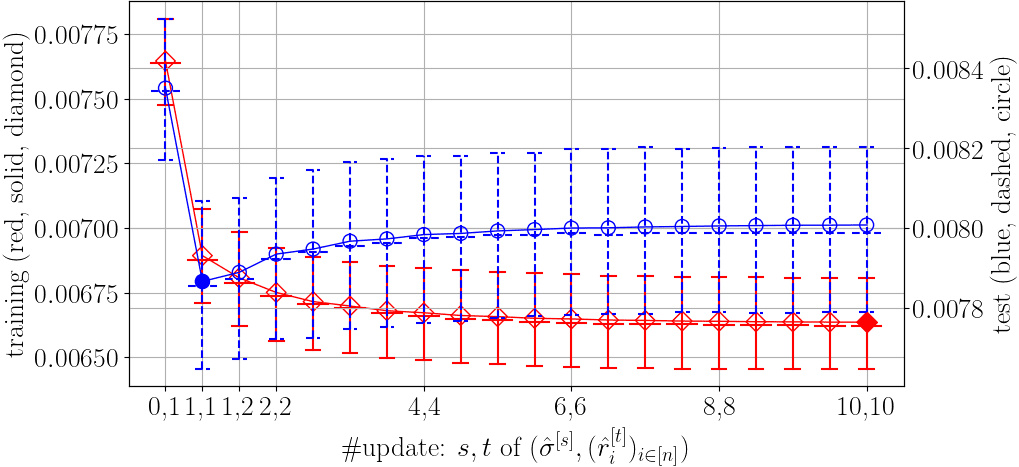}}&
\CF{\includegraphics[width=2.0cm]{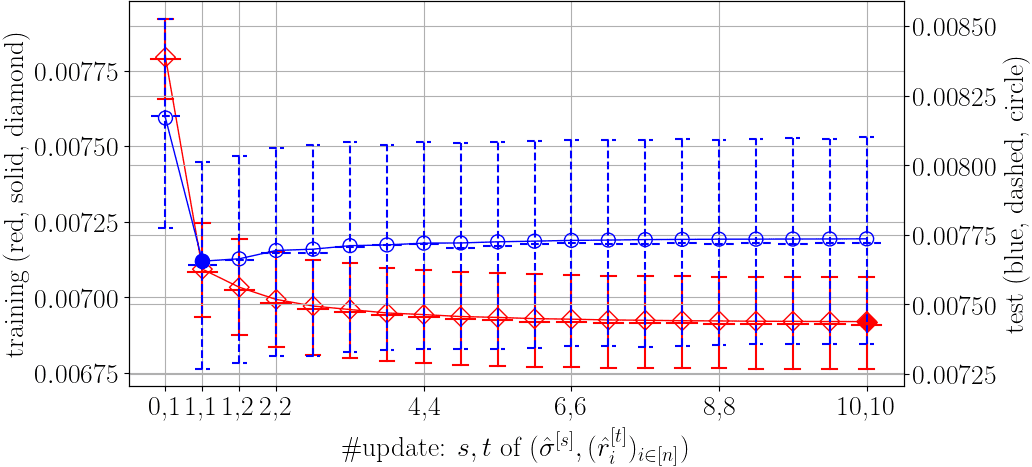}}\\
&\rotatebox{90}{\tiny\,~~\,$400$}&
\CF{\includegraphics[width=2.0cm]{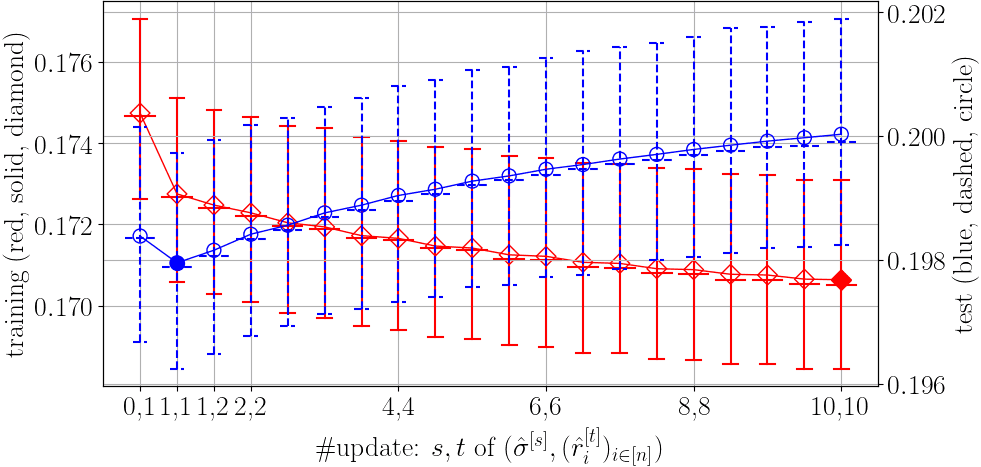}}&
\CF{\includegraphics[width=2.0cm]{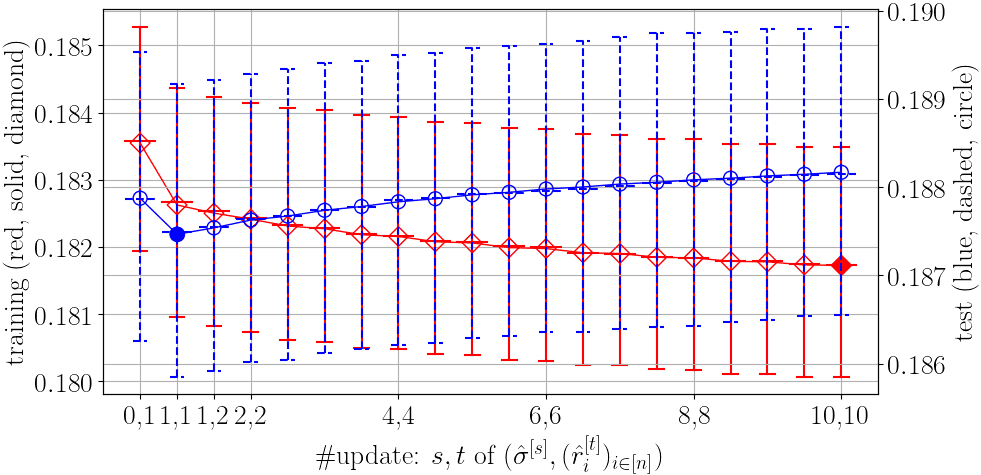}}&
\CF{\includegraphics[width=2.0cm]{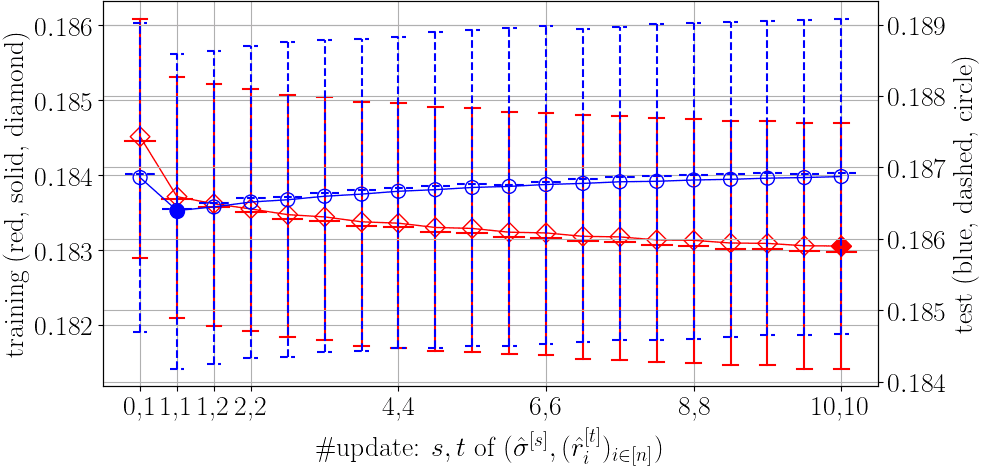}}&
\CF{\includegraphics[width=2.0cm]{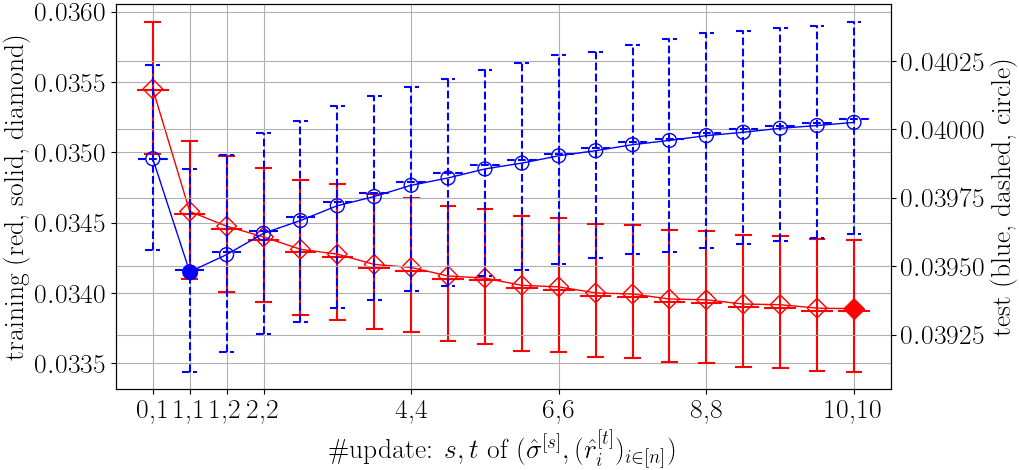}}&
\CF{\includegraphics[width=2.0cm]{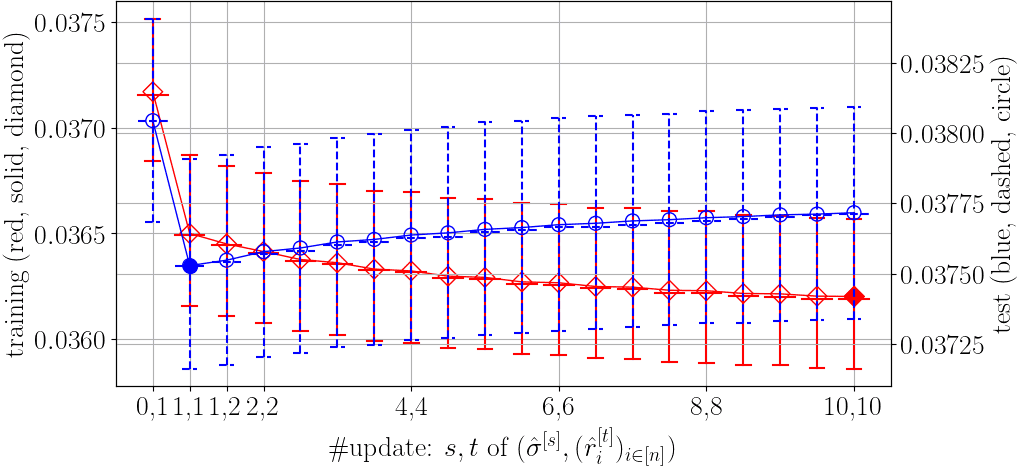}}&
\CF{\includegraphics[width=2.0cm]{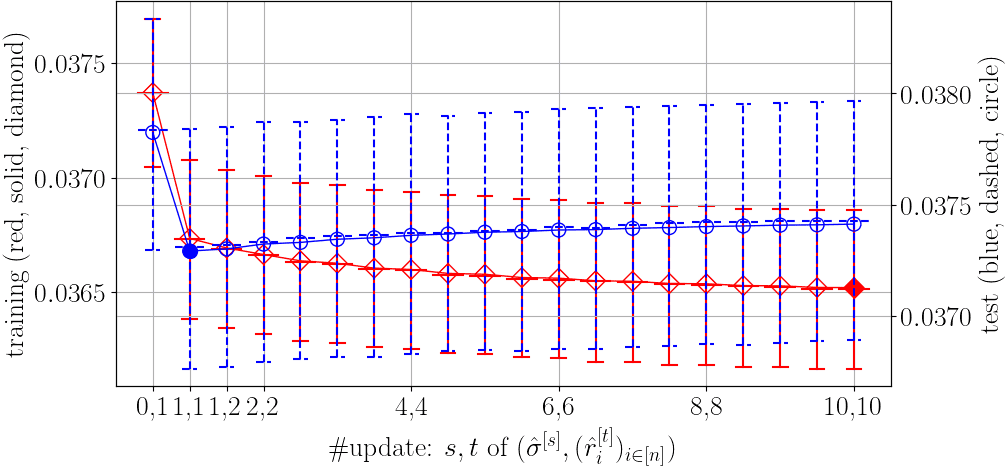}}&
\CF{\includegraphics[width=2.0cm]{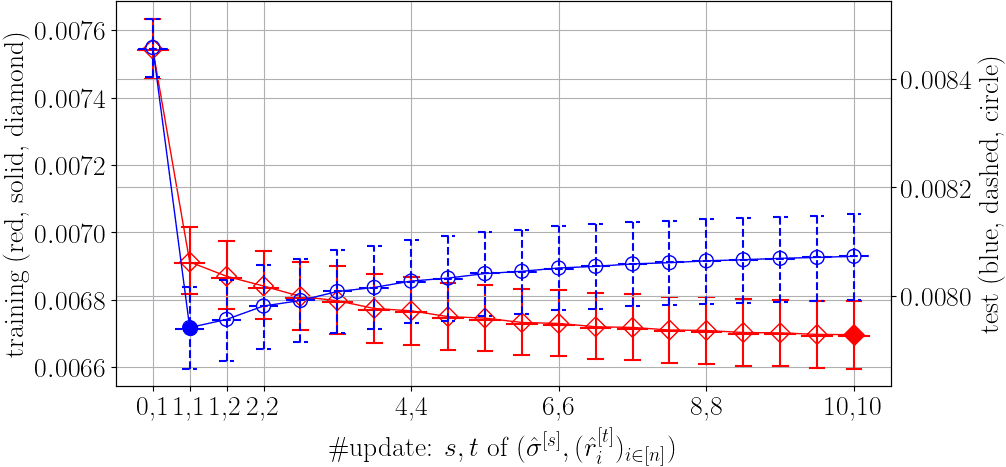}}&
\CF{\includegraphics[width=2.0cm]{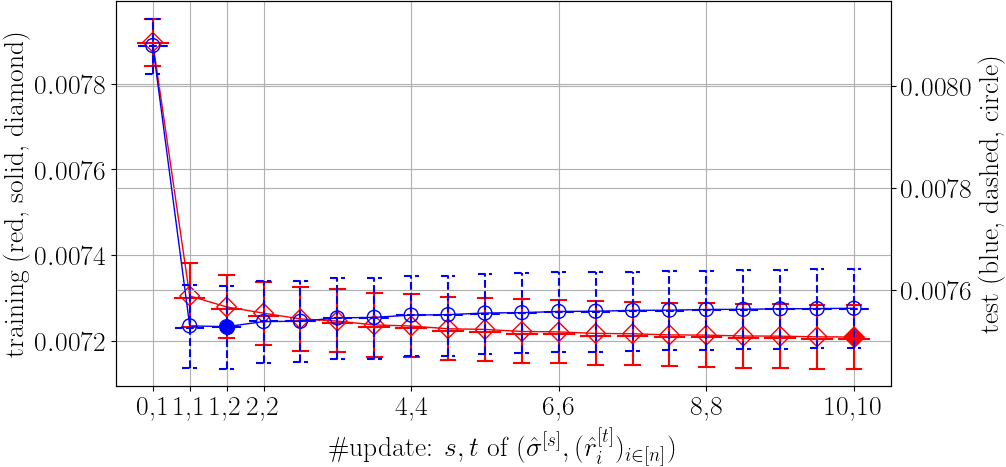}}&
\CF{\includegraphics[width=2.0cm]{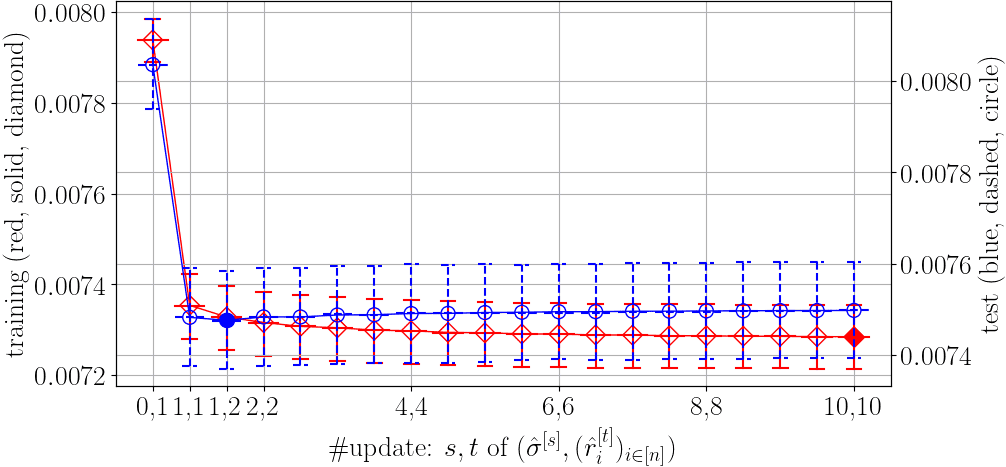}}
\\\midrule
\multirow{3}{*}[-2.5mm]{\rotatebox{90}{\tiny\eqref{eq:Kendall}, $n=$}}
&\rotatebox{90}{\tiny\,~~~\,$25$}&
{\includegraphics[width=2.0cm]{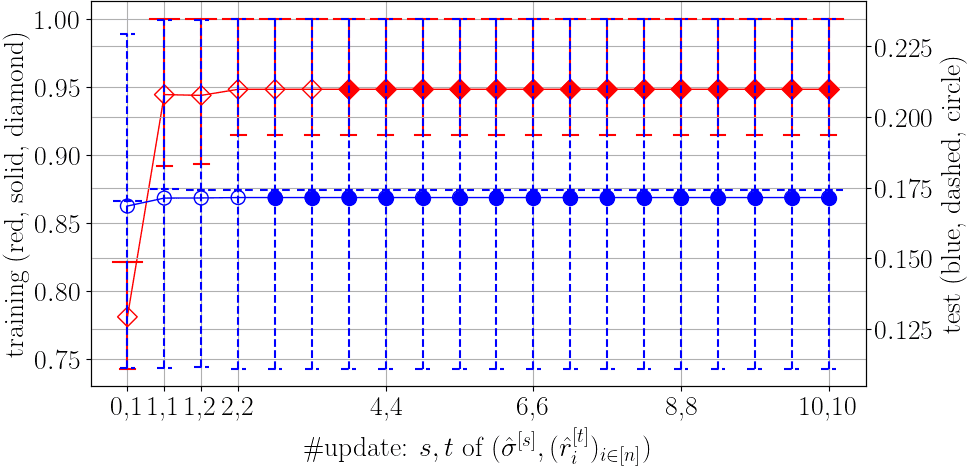}}&
\CF{\includegraphics[width=2.0cm]{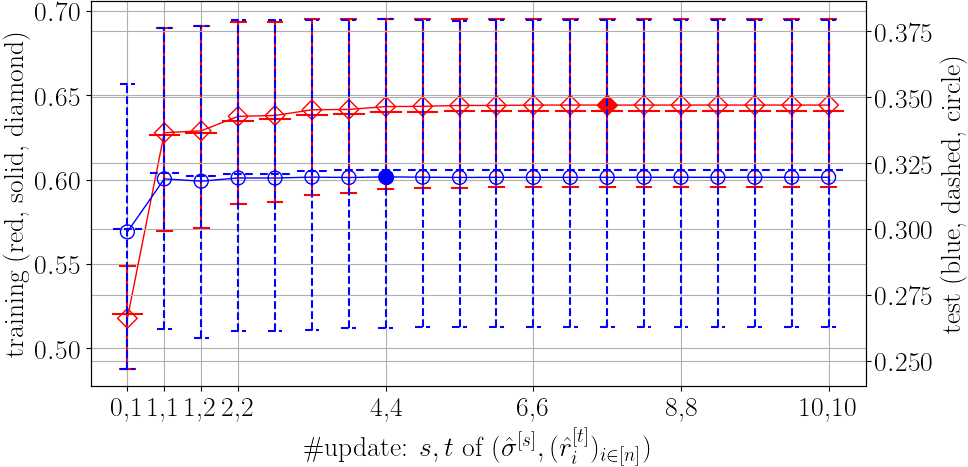}}&
\CF{\includegraphics[width=2.0cm]{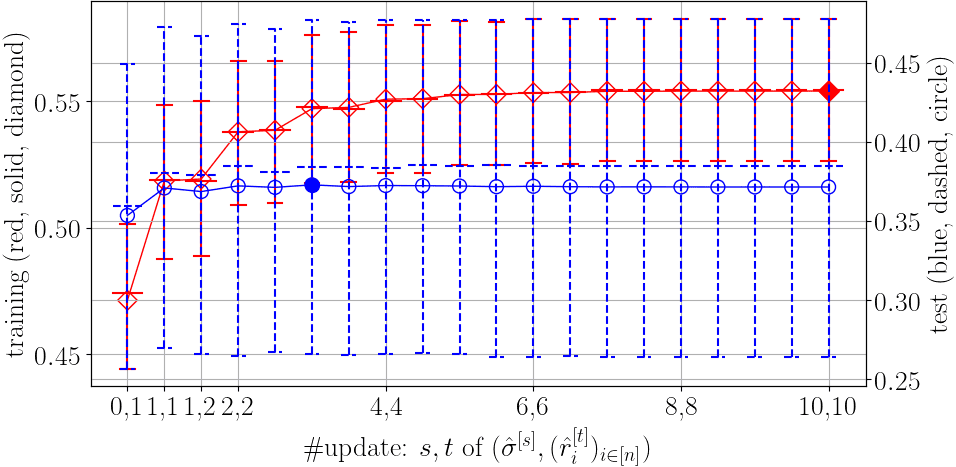}}&
\CF{\includegraphics[width=2.0cm]{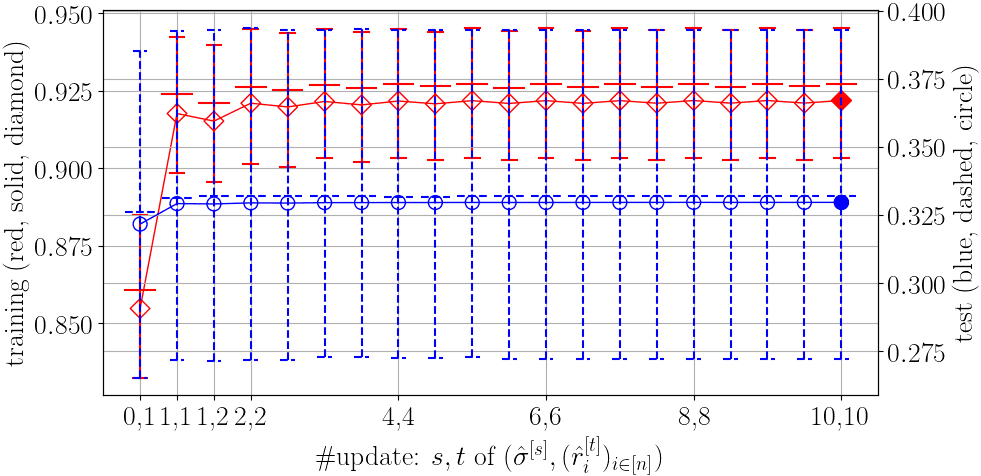}}&
\CF{\includegraphics[width=2.0cm]{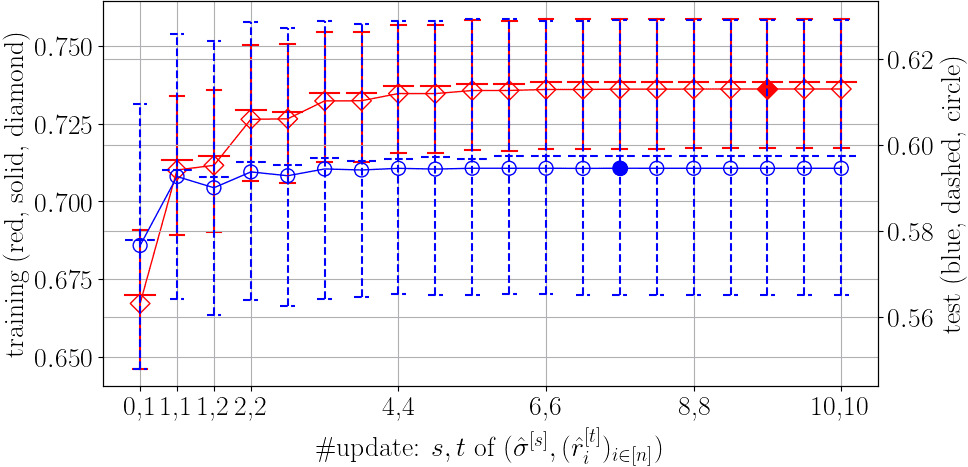}}&
\CF{\includegraphics[width=2.0cm]{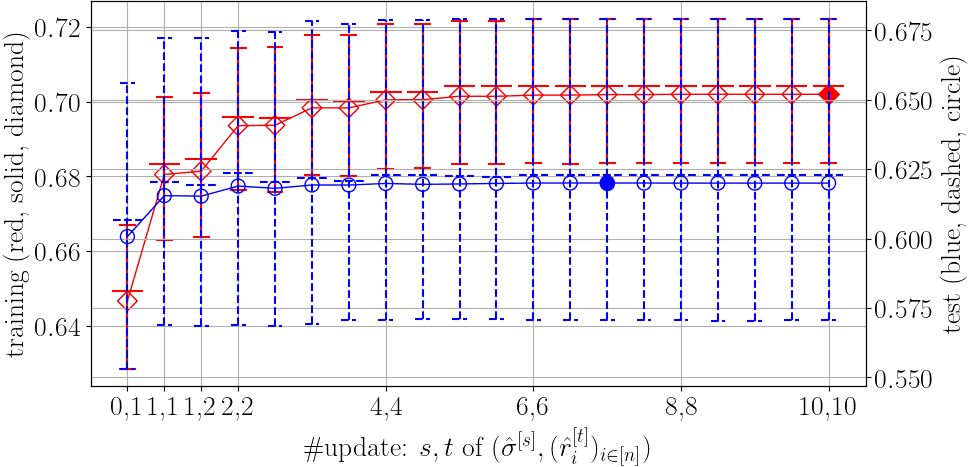}}&
\CF{\includegraphics[width=2.0cm]{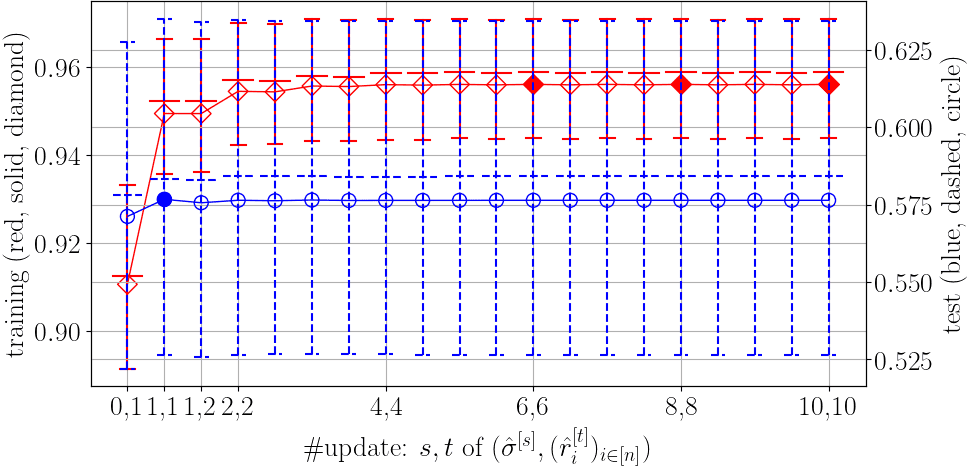}}&
\CF{\includegraphics[width=2.0cm]{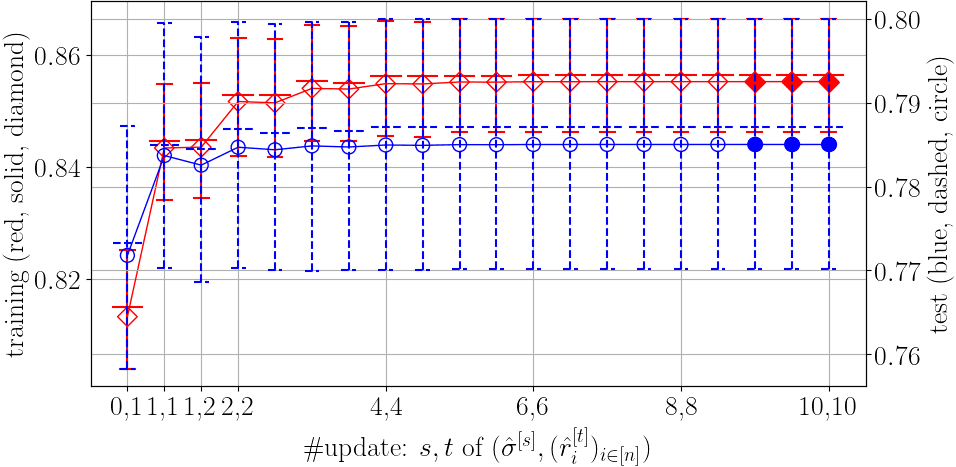}}&
\CF{\includegraphics[width=2.0cm]{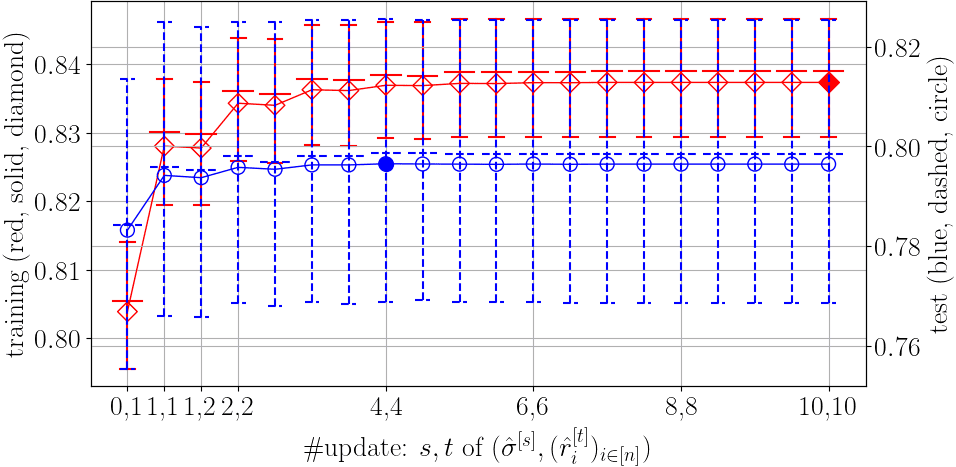}}\\
&\rotatebox{90}{\tiny\,~~\,$100$}&
\CF{\includegraphics[width=2.0cm]{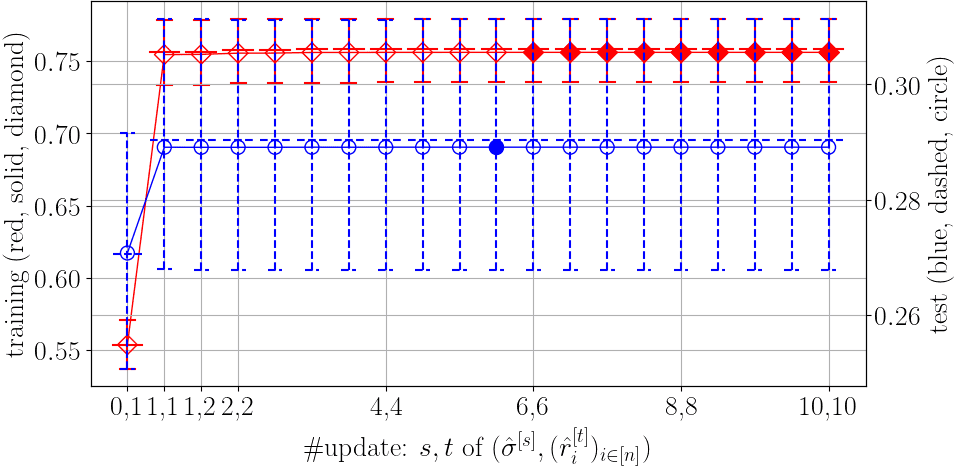}}&
\CF{\includegraphics[width=2.0cm]{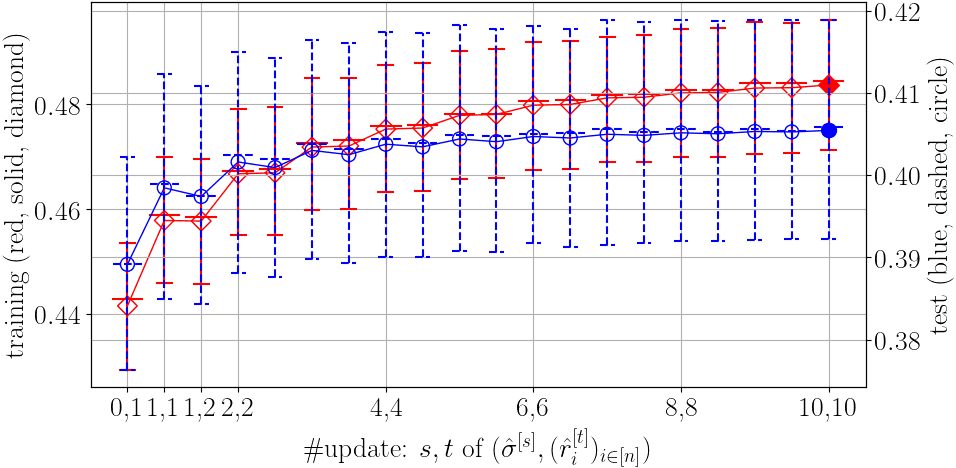}}&
\CF{\includegraphics[width=2.0cm]{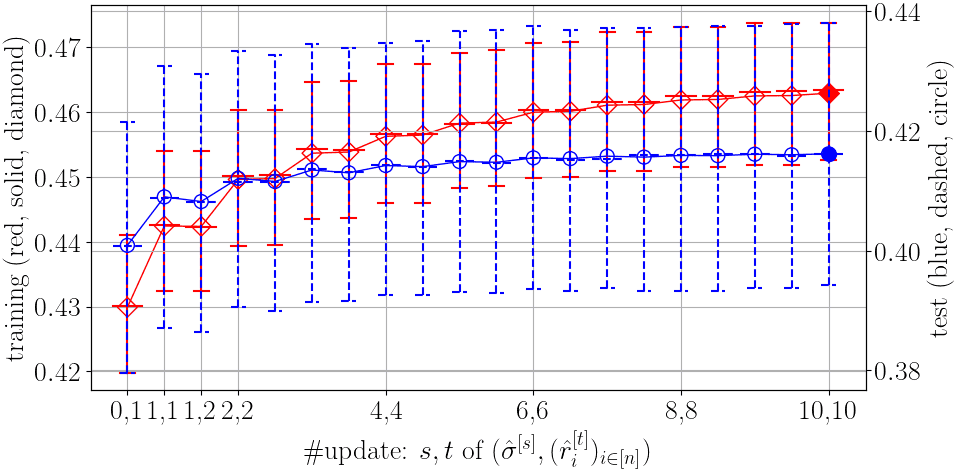}}&
\CF{\includegraphics[width=2.0cm]{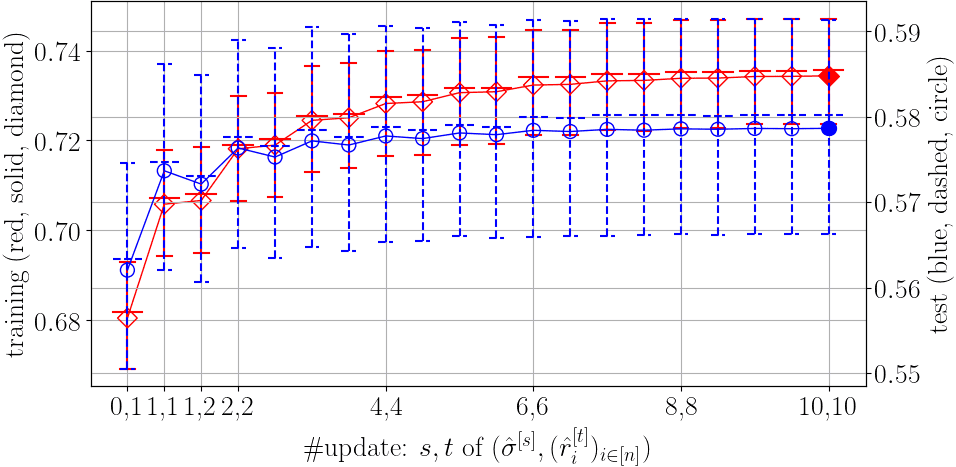}}&
\CF{\includegraphics[width=2.0cm]{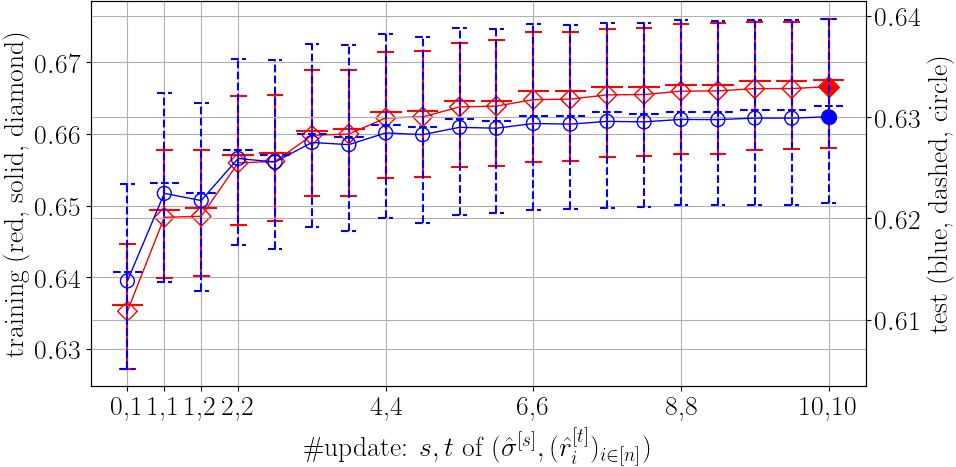}}&
\CF{\includegraphics[width=2.0cm]{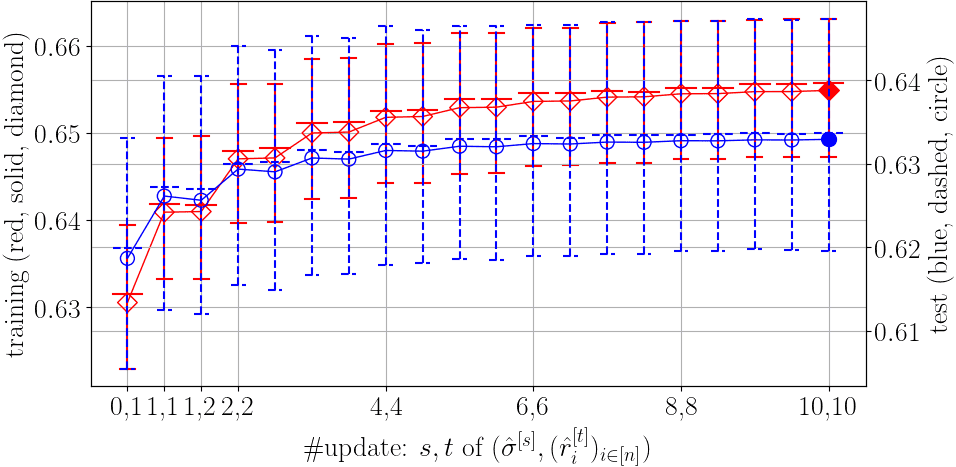}}&
\CF{\includegraphics[width=2.0cm]{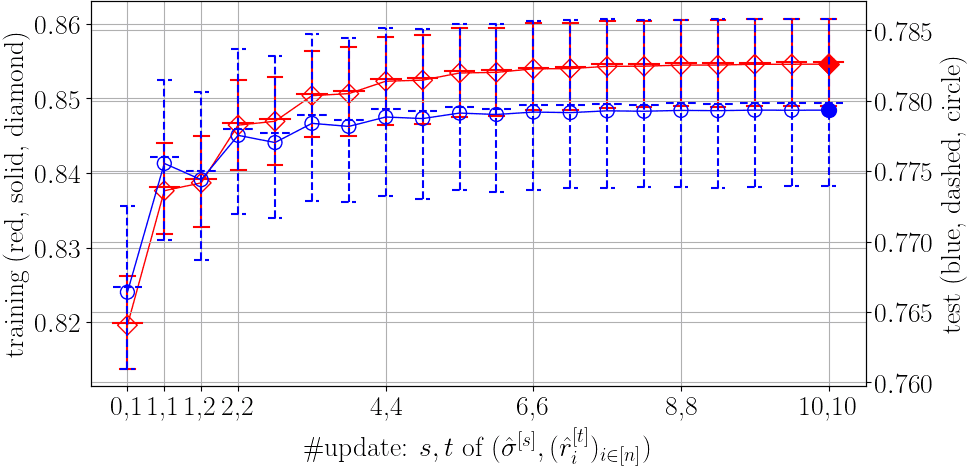}}&
\CF{\includegraphics[width=2.0cm]{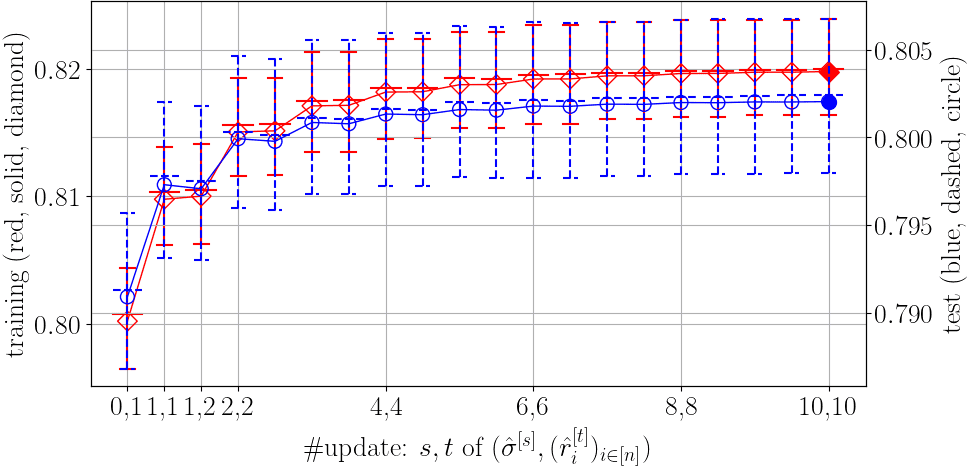}}&
\CF{\includegraphics[width=2.0cm]{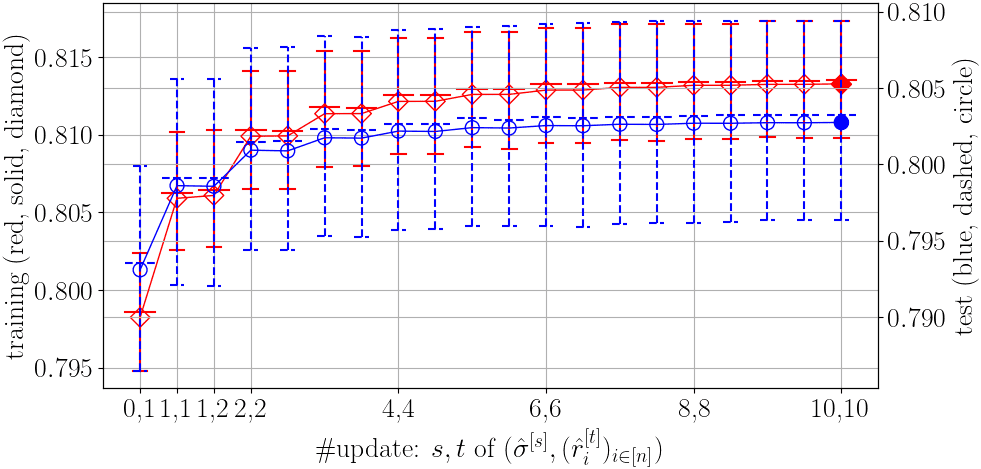}}\\
&\rotatebox{90}{\tiny\,~~\,$400$}&
\CF{\includegraphics[width=2.0cm]{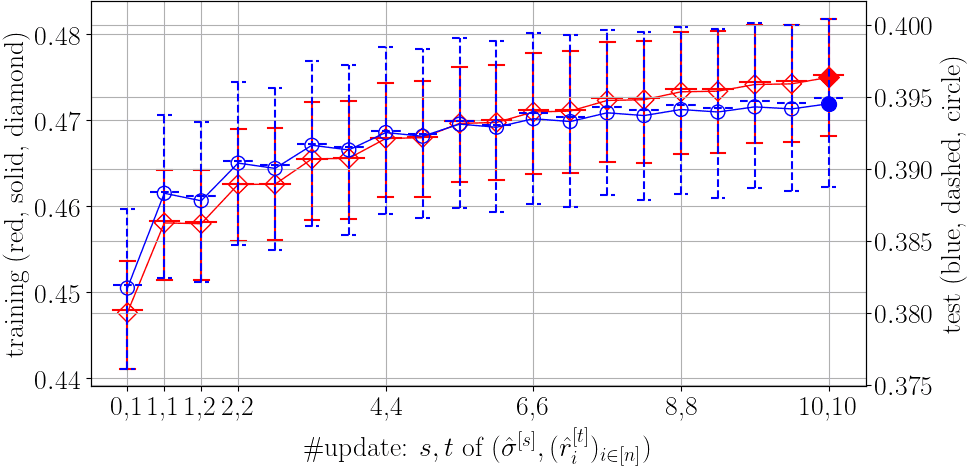}}&
\CF{\includegraphics[width=2.0cm]{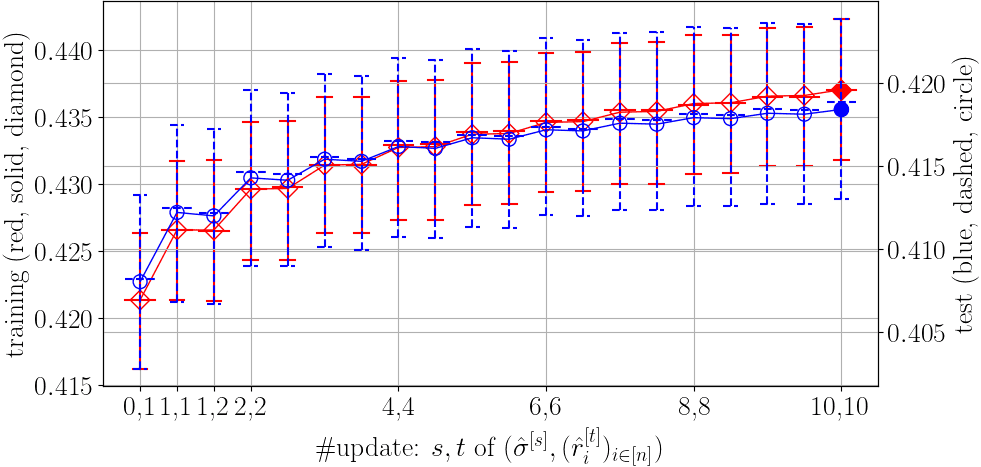}}&
\CF{\includegraphics[width=2.0cm]{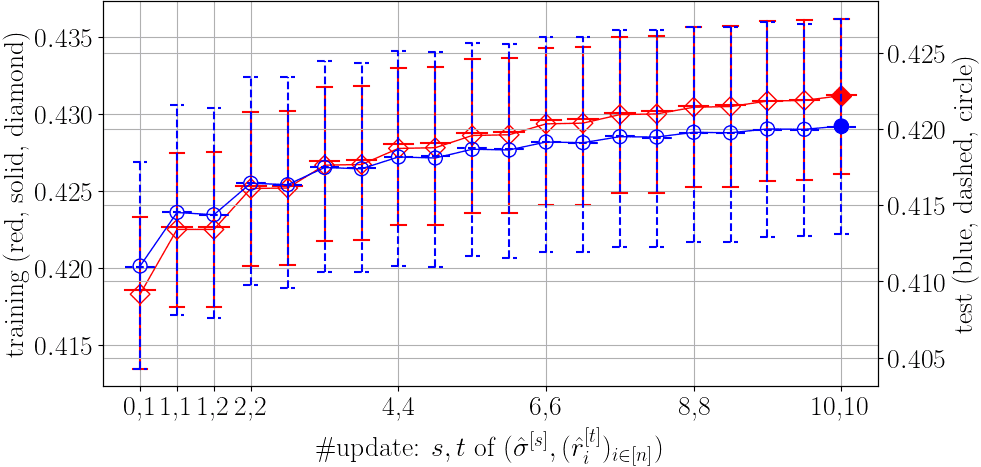}}&
\CF{\includegraphics[width=2.0cm]{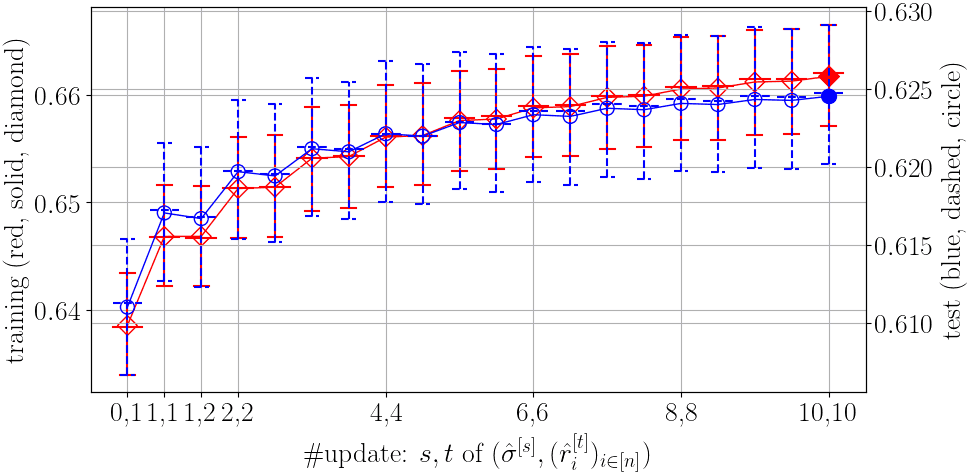}}&
\CF{\includegraphics[width=2.0cm]{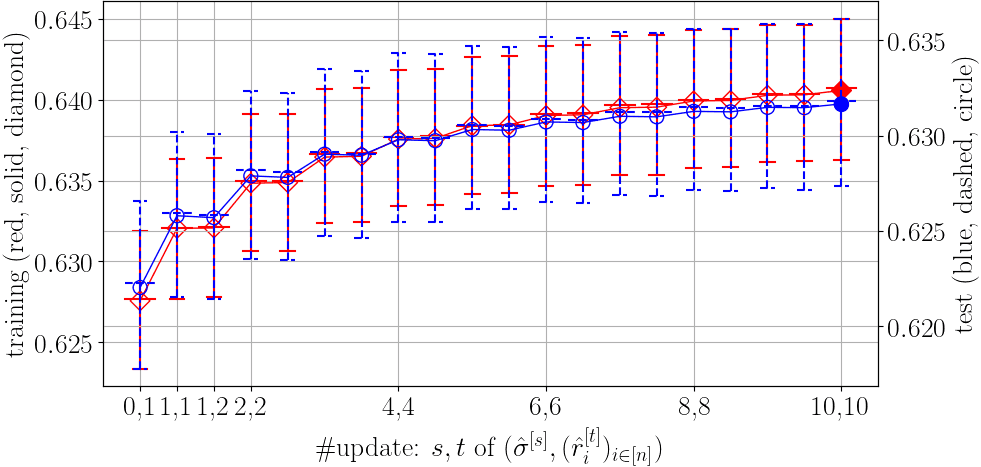}}&
\CF{\includegraphics[width=2.0cm]{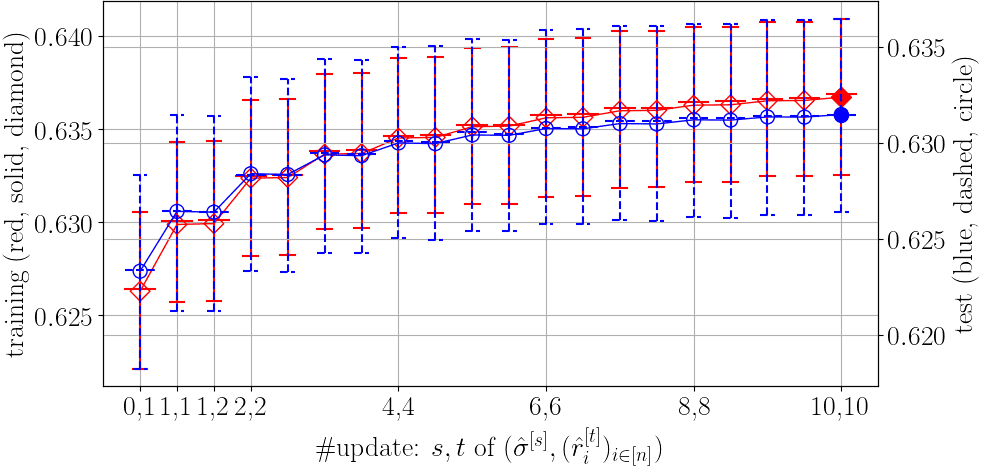}}&
\CF{\includegraphics[width=2.0cm]{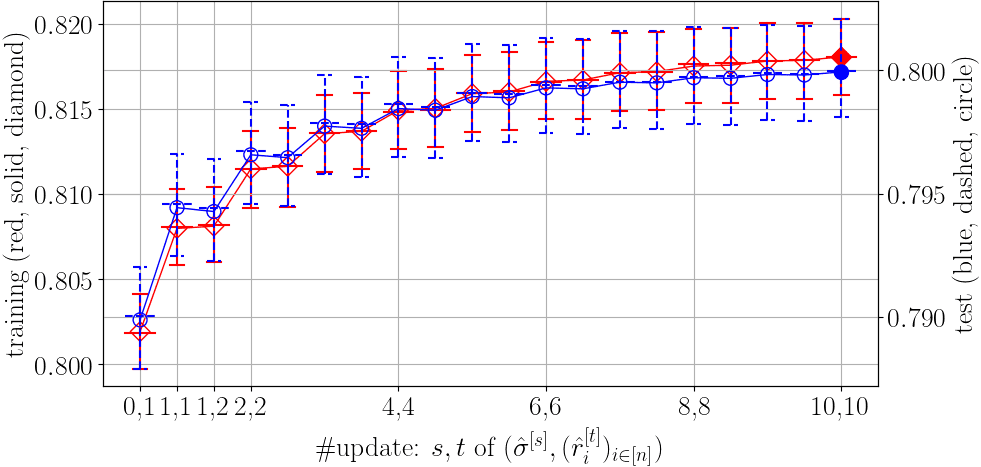}}&
\CF{\includegraphics[width=2.0cm]{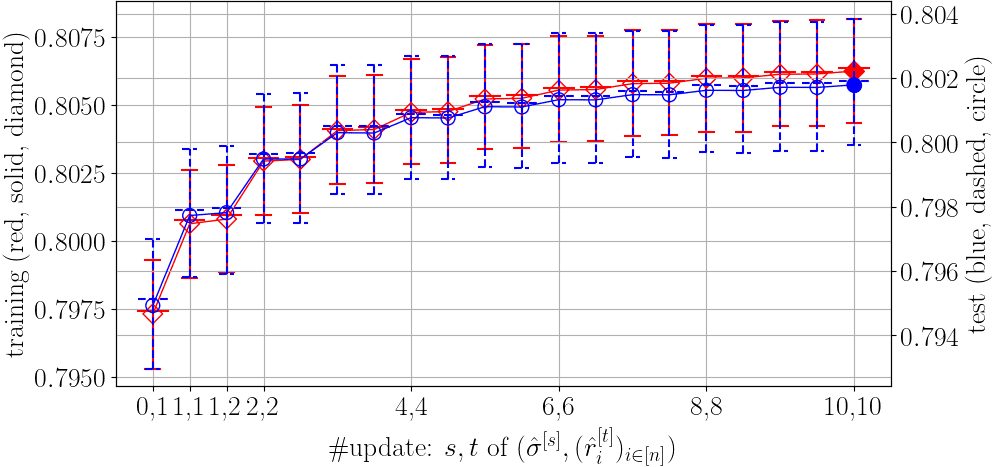}}&
\CF{\includegraphics[width=2.0cm]{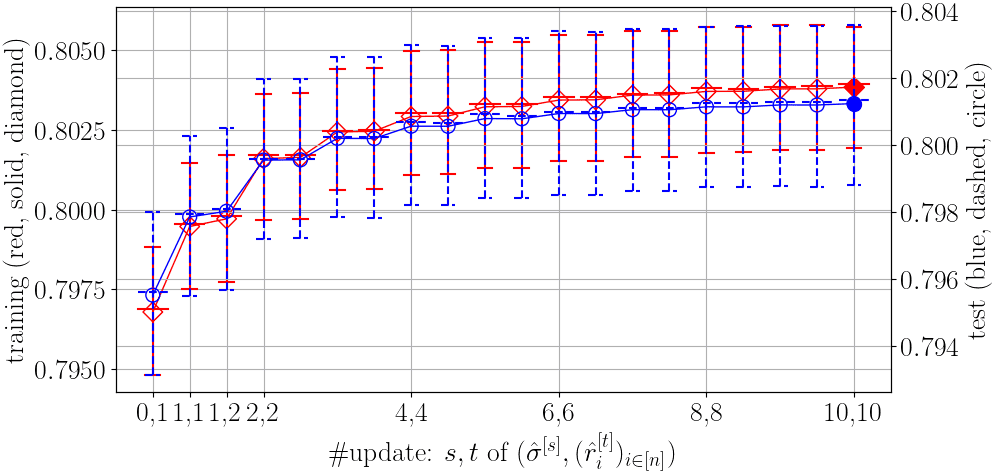}}
\\\midrule
\multirow{3}{*}[-2.5mm]{\rotatebox{90}{\tiny\eqref{eq:TIE}, $n=$}}
&\rotatebox{90}{\tiny\,~~~\,$25$}&
{\includegraphics[width=2.0cm]{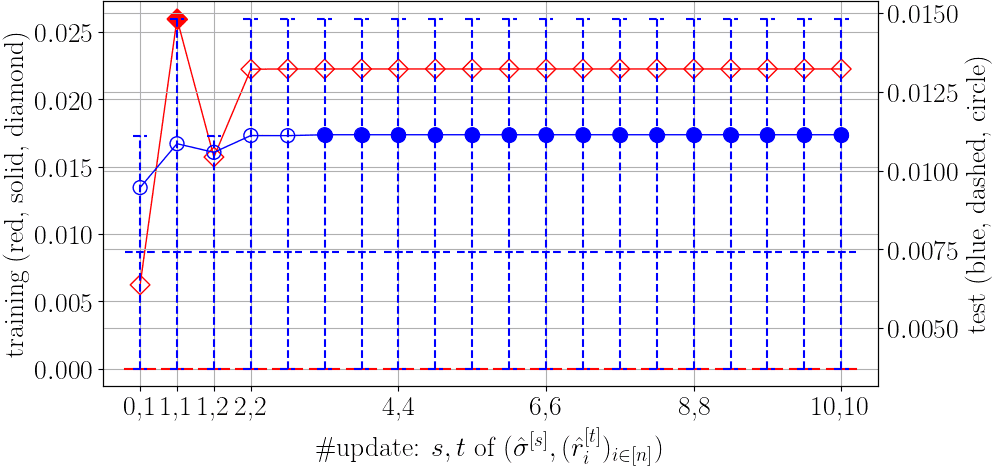}}&
{\includegraphics[width=2.0cm]{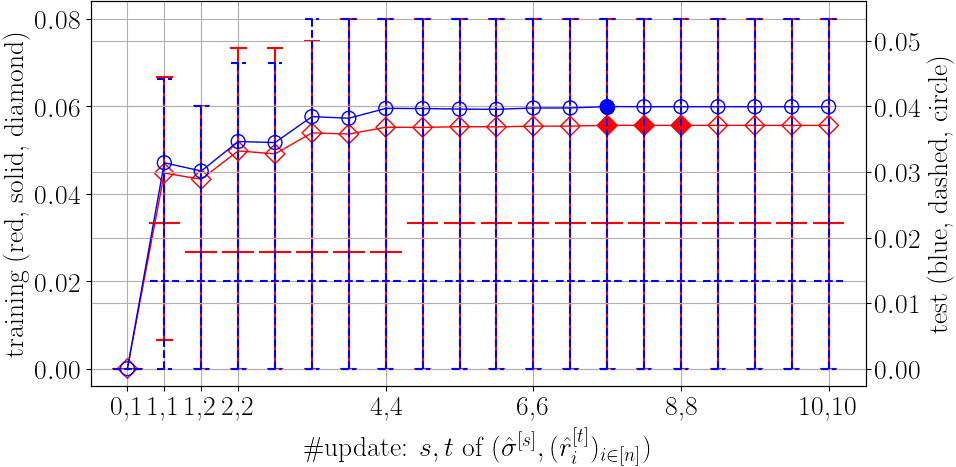}}&
{\includegraphics[width=2.0cm]{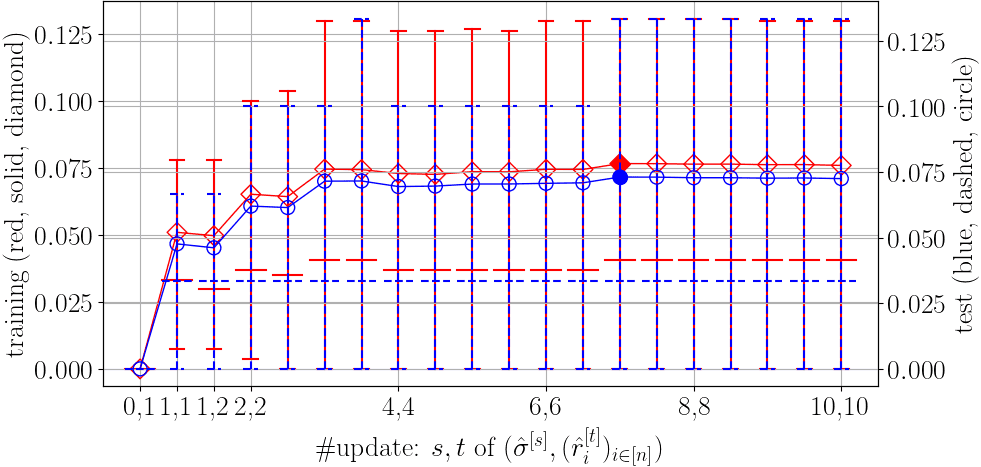}}&
{\includegraphics[width=2.0cm]{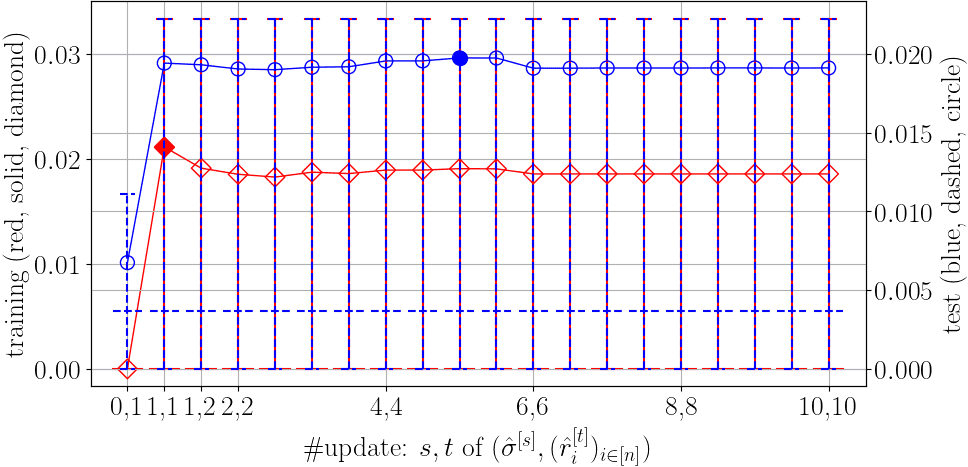}}&
{\includegraphics[width=2.0cm]{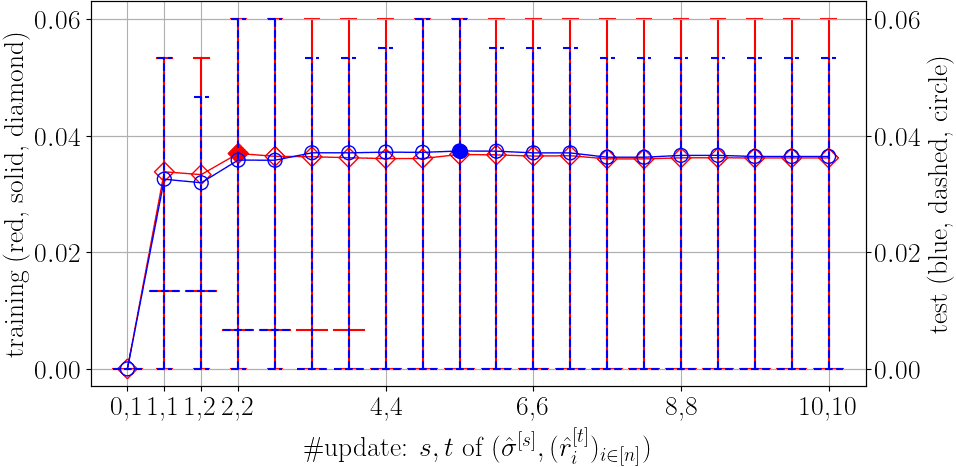}}&
{\includegraphics[width=2.0cm]{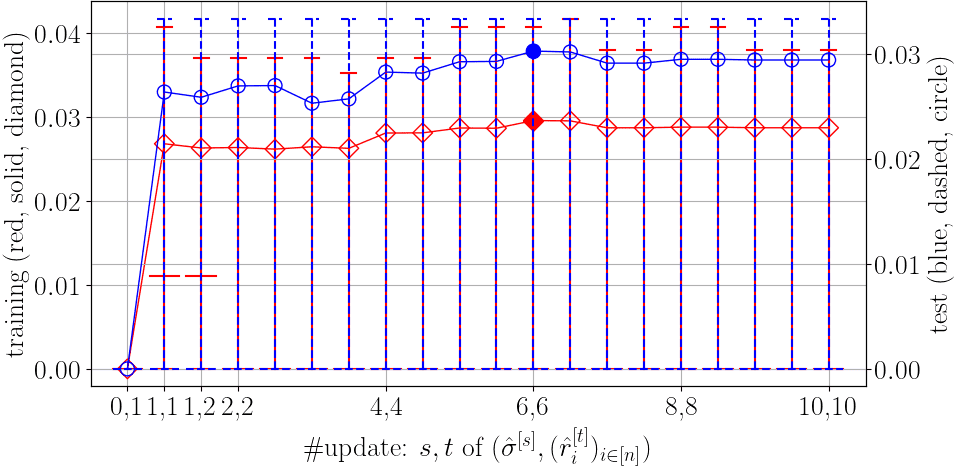}}&
{\includegraphics[width=2.0cm]{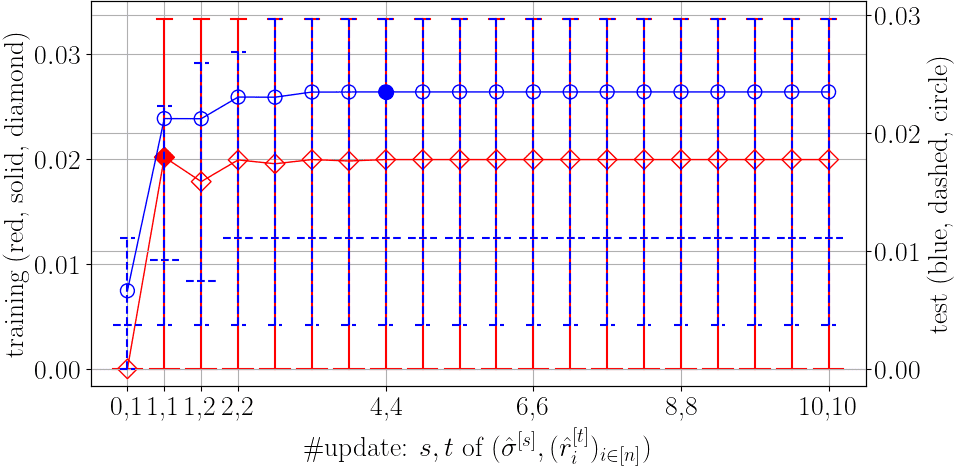}}&
{\includegraphics[width=2.0cm]{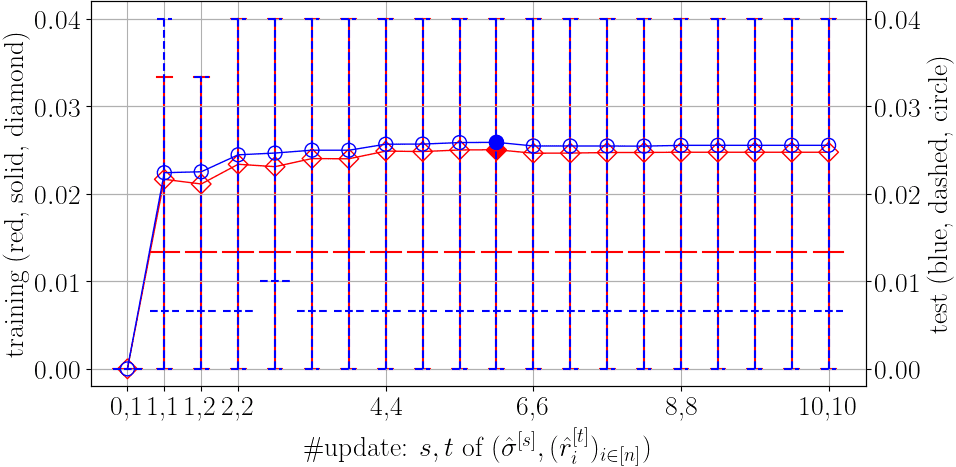}}&
{\includegraphics[width=2.0cm]{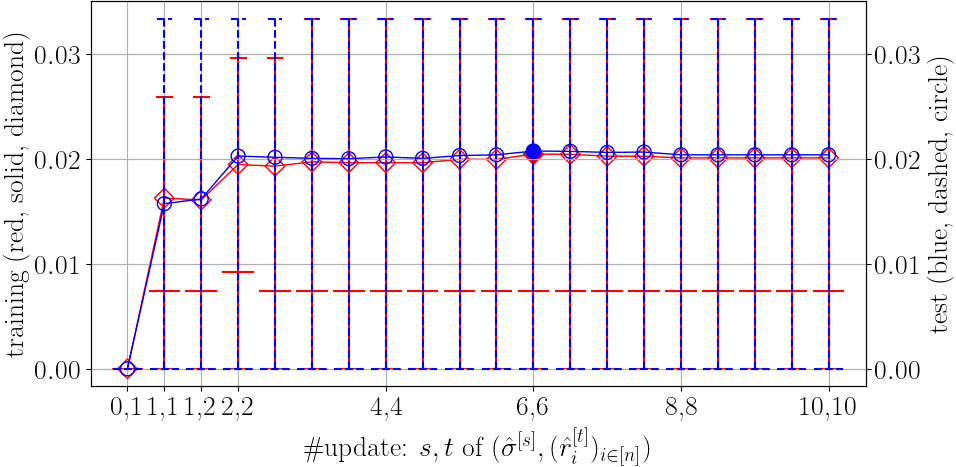}}\\
&\rotatebox{90}{\tiny\,~~\,$100$}&
{\includegraphics[width=2.0cm]{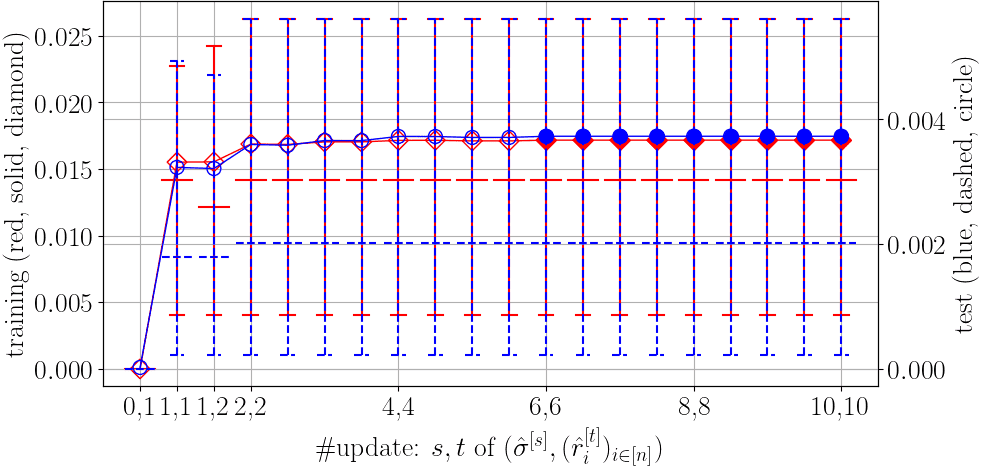}}&
{\includegraphics[width=2.0cm]{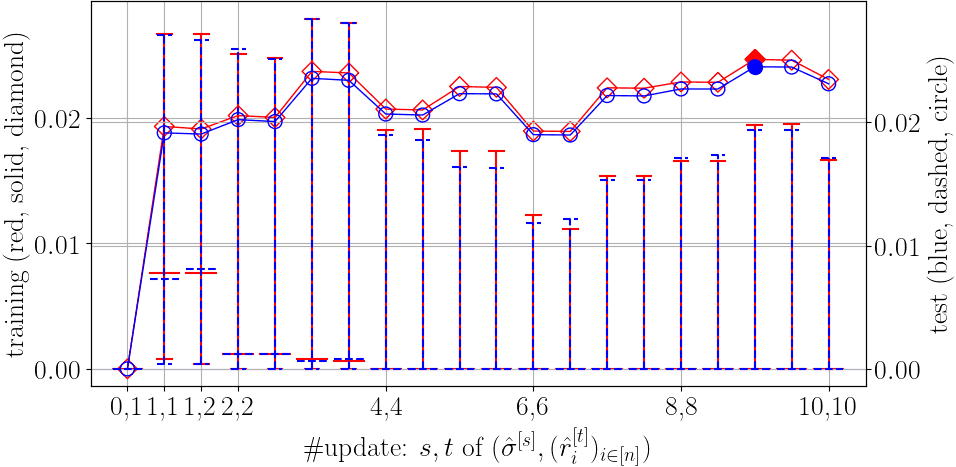}}&
{\includegraphics[width=2.0cm]{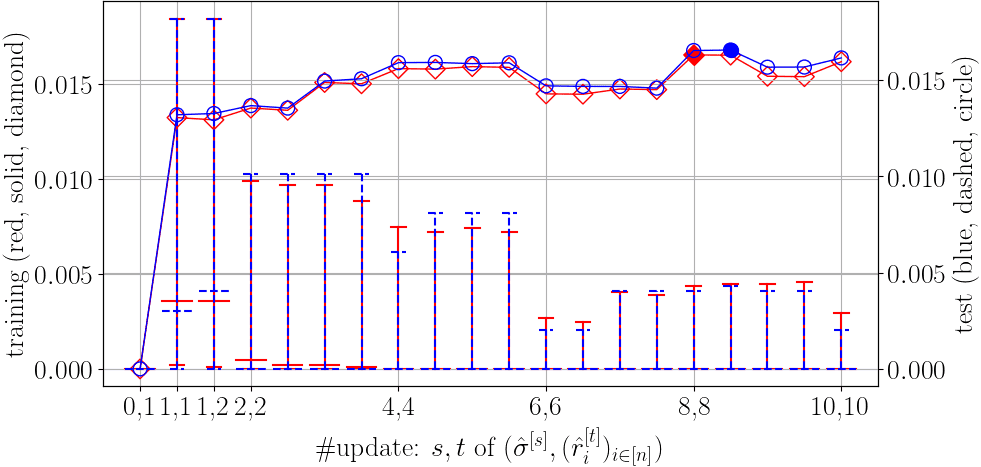}}&
{\includegraphics[width=2.0cm]{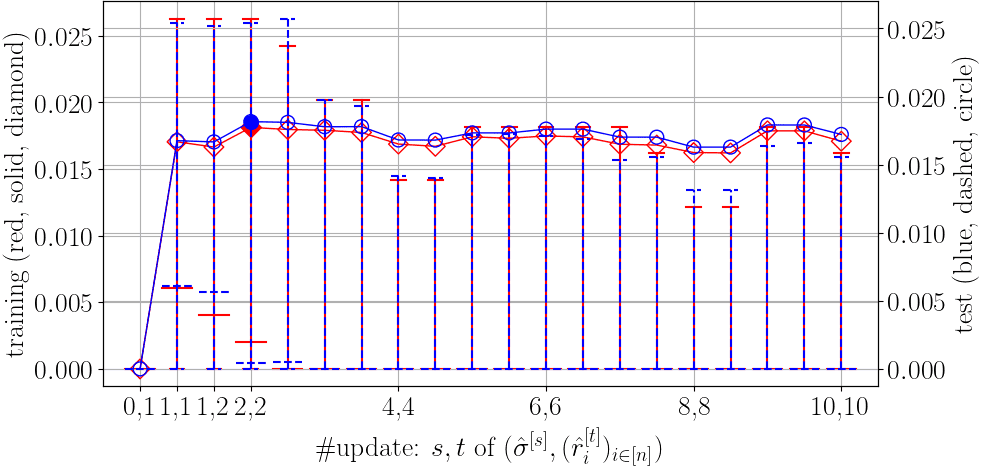}}&
{\includegraphics[width=2.0cm]{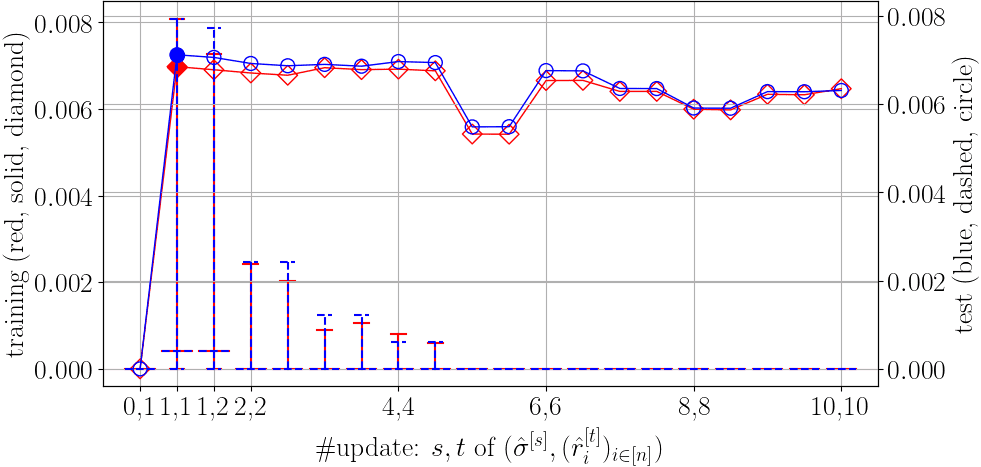}}&
{\includegraphics[width=2.0cm]{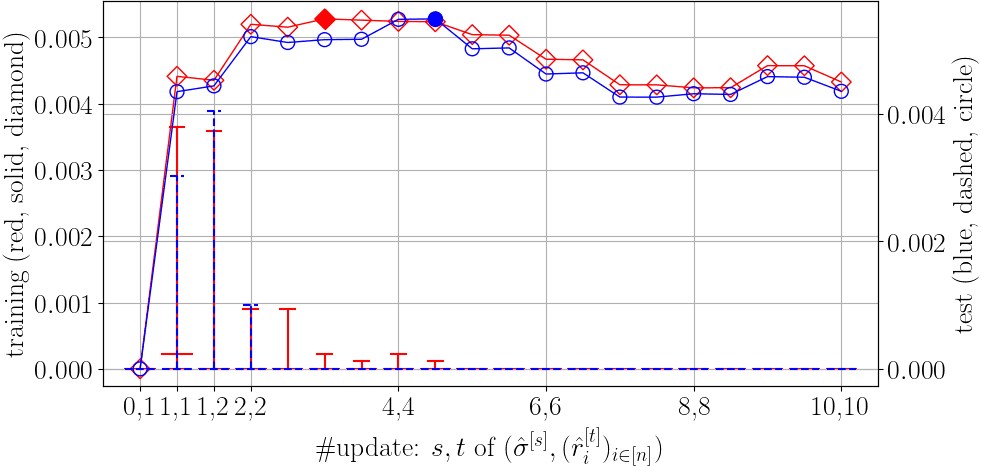}}&
{\includegraphics[width=2.0cm]{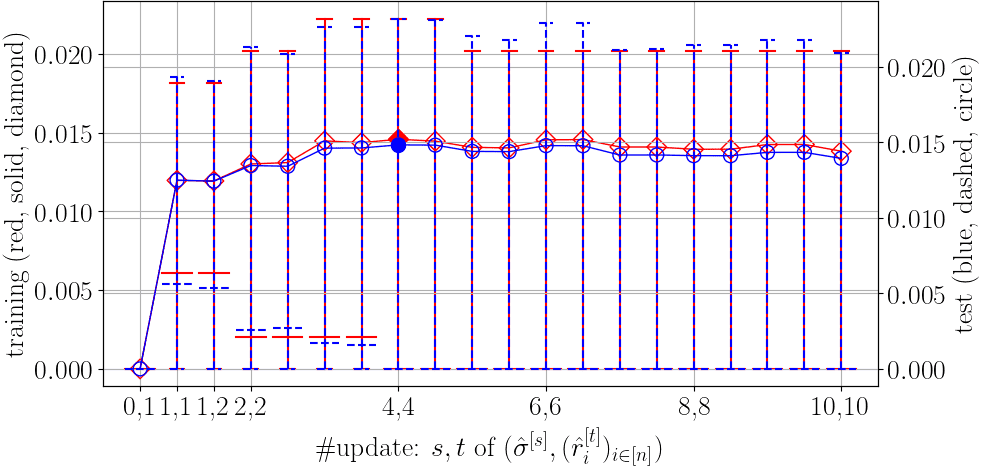}}&
{\includegraphics[width=2.0cm]{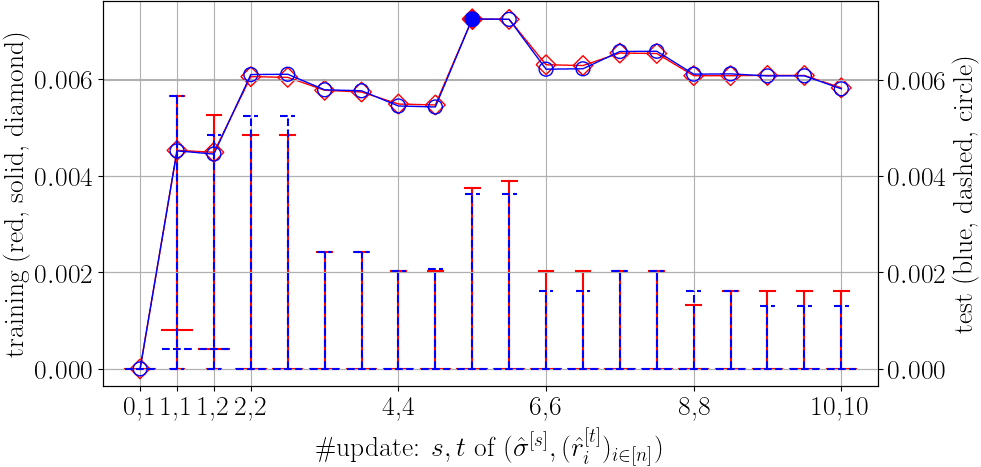}}&
{\includegraphics[width=2.0cm]{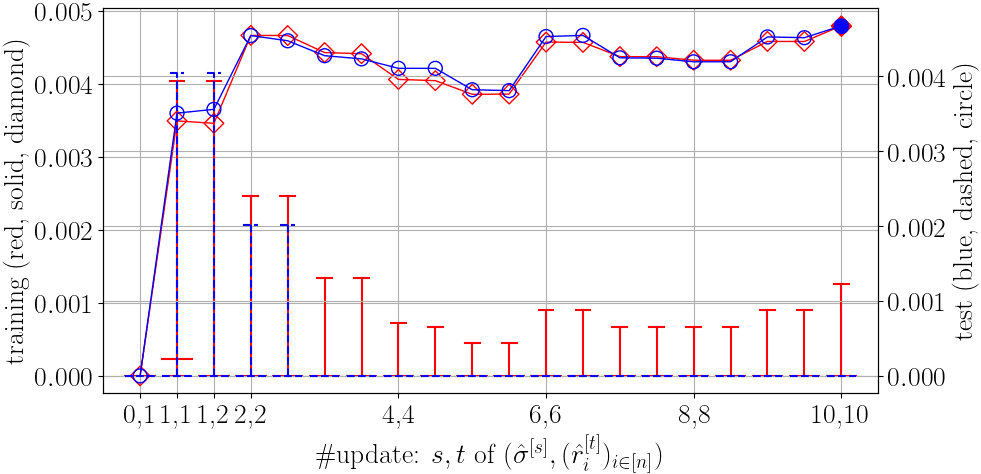}}\\
&\rotatebox{90}{\tiny\,~~\,$400$}&
{\includegraphics[width=2.0cm]{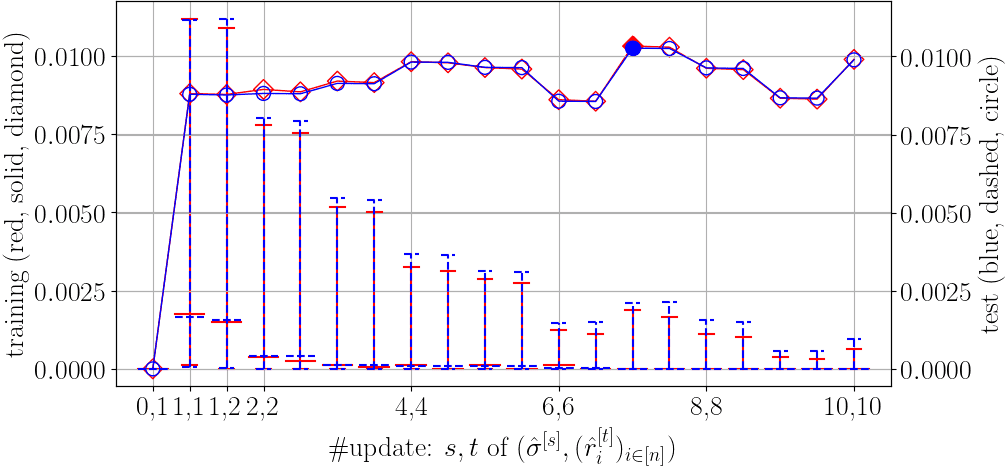}}&
{\includegraphics[width=2.0cm]{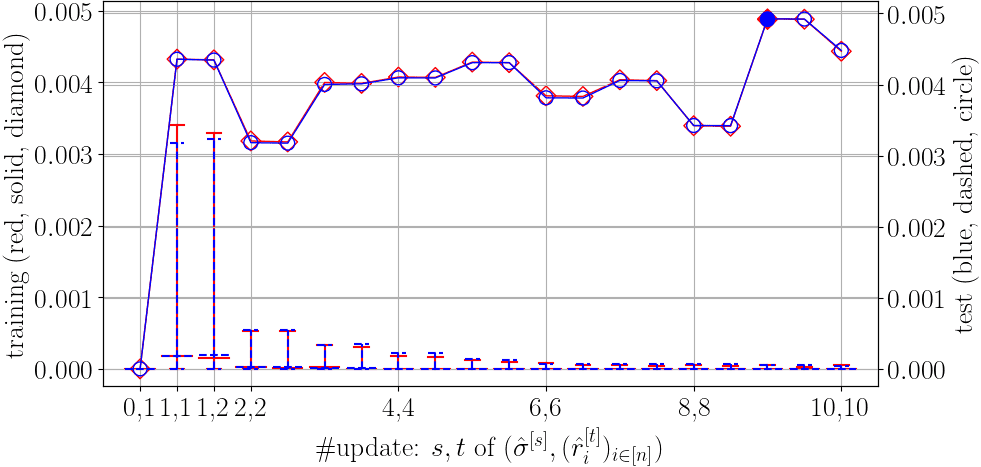}}&
{\includegraphics[width=2.0cm]{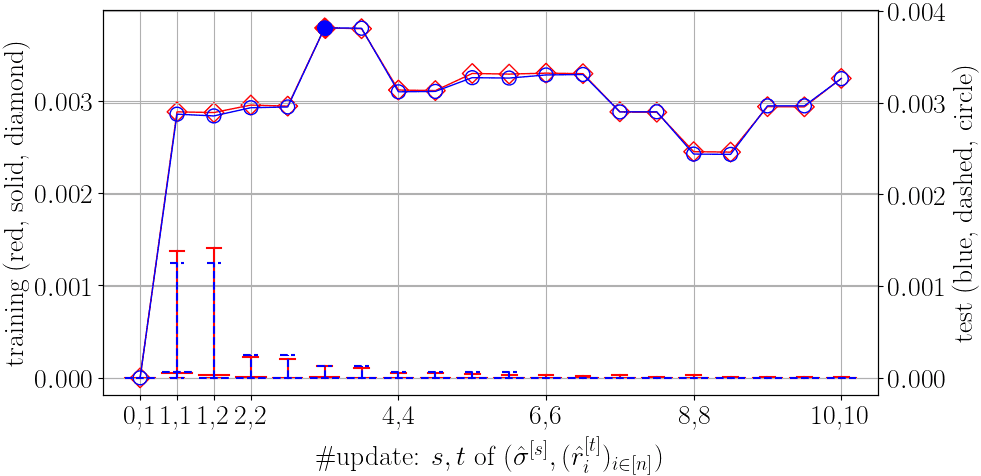}}&
{\includegraphics[width=2.0cm]{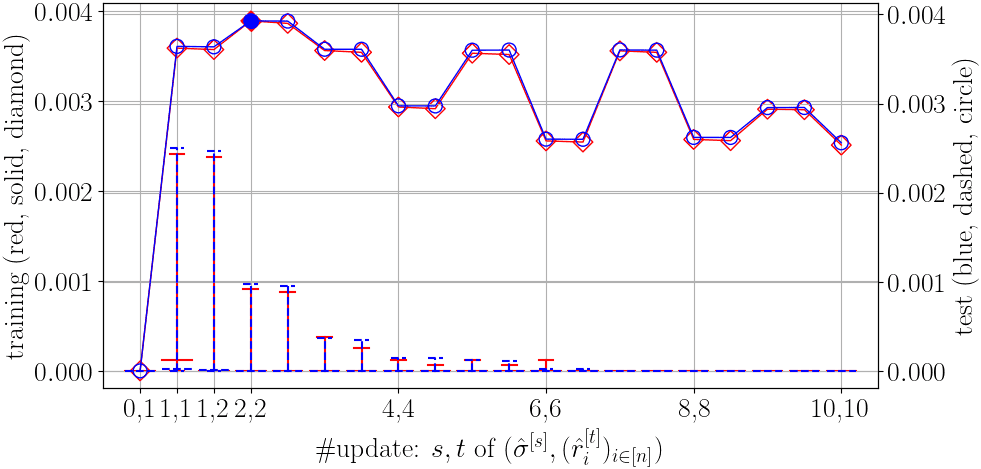}}&
{\includegraphics[width=2.0cm]{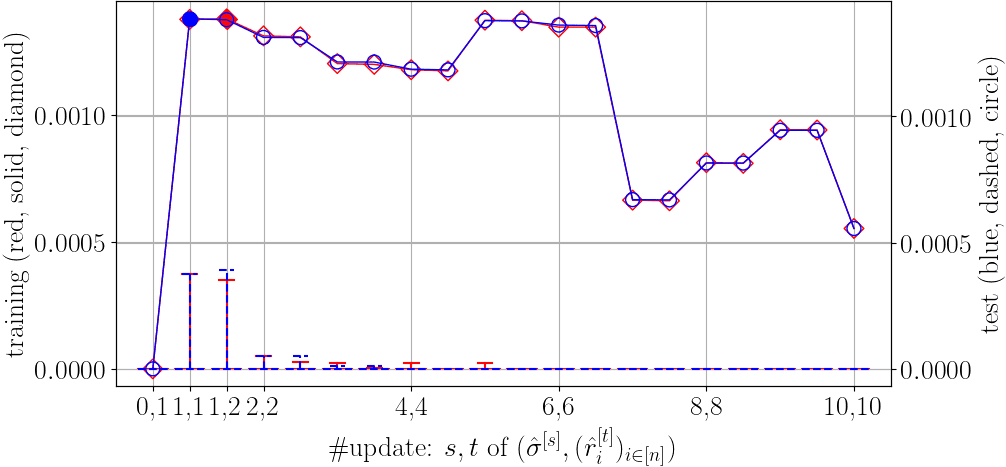}}&
{\includegraphics[width=2.0cm]{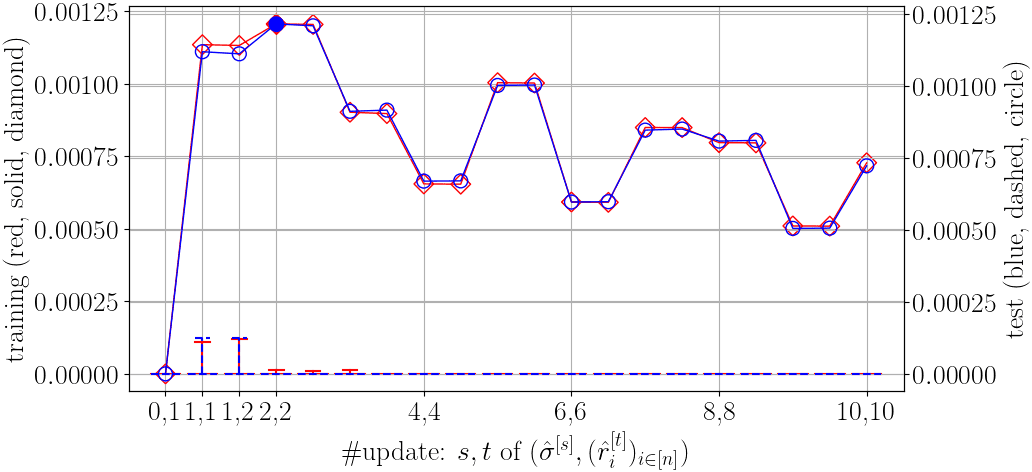}}&
{\includegraphics[width=2.0cm]{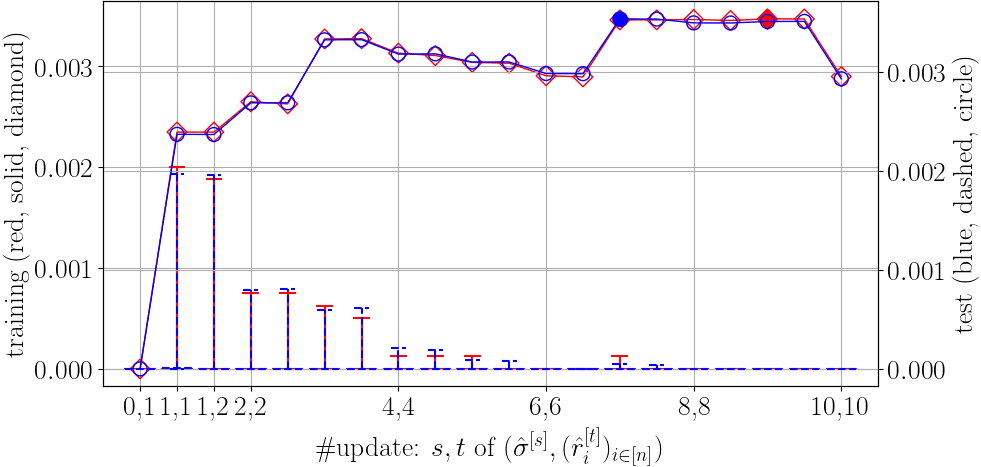}}&
{\includegraphics[width=2.0cm]{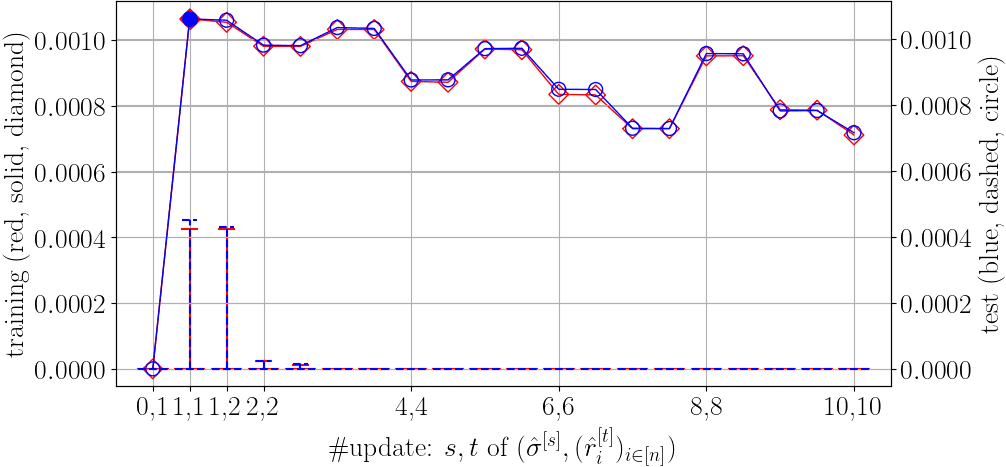}}&
{\includegraphics[width=2.0cm]{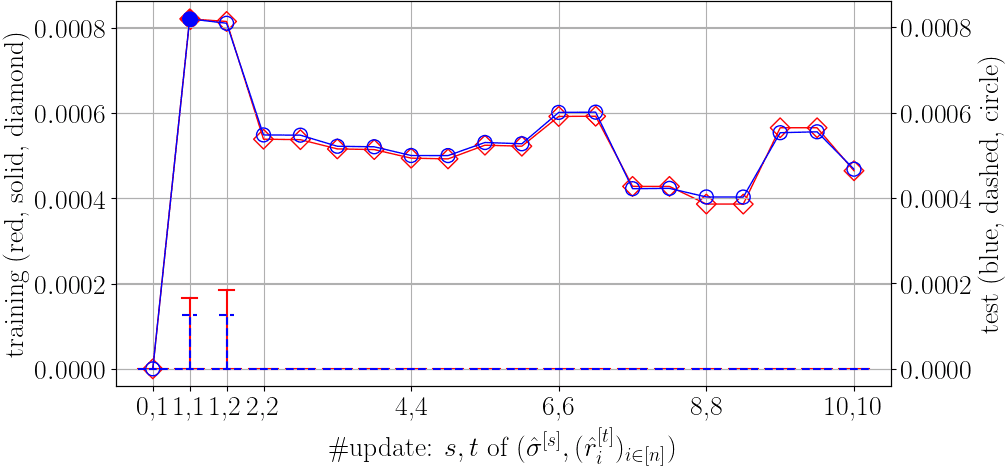}}
\end{tabular}
\caption{%
Part of results of synthetic data experiments, Procedure 1 in Section~\ref{sec:Synthetic}:
For Cauchy-$N$ synthetic data with $N=1,5,25$ (left to right),
mean (marker) and 0.25, 0.5, and 0.75 quantiles (lower, middle, and upper bars) of 
1000 trial training (red, solid, diamond) and test (blue, dashed, circle) evaluation of 
WPP error \eqref{eq:WPE} with the squared loss $\phi=\phi_\sq$, 
Kendall's Tau \eqref{eq:Kendall}, and tie rate \eqref{eq:TIE} (top to bottom)
for the isotonic Bradley-Terry model learned with the squared loss $\phi=\phi_\sq$.
Smaller \eqref{eq:WPE}, or larger \eqref{eq:Kendall} indicates a better model.
The marker for the best model and model with the most ties was filled in, 
and the outer frame of the figure was highlighted in gray
if the best model was significantly better than the Bradley-Terry model
with respect to Mann-Whitney U test of the significance level 0.05.}
\label{fig:Cauchy-SQ}
\end{sidewaysfigure}
\begin{sidewaysfigure}
\centering%
\renewcommand{\arraystretch}{0.5}%
\renewcommand{\tabcolsep}{0.5pt}%
\begin{tabular}{cc|ccc|ccc|ccc}%
&&\multicolumn{3}{c|}{\tiny$N=1$, $|D_\tra|:|D_\tes|=$}&\multicolumn{3}{c|}{\tiny$N=5$, $|D_\tra|:|D_\tes|=$}&\multicolumn{3}{c}{\tiny$N=25$, $|D_\tra|:|D_\tes|=$}\\
&&{\tiny$1:9$}&{\tiny$5:5$}&{\tiny$9:1$}&{\tiny$1:9$}&{\tiny$5:5$}&{\tiny$9:1$}&{\tiny$1:9$}&{\tiny$5:5$}&{\tiny$9:1$}\\
\midrule
\multirow{3}{*}[-3.4mm]{\rotatebox{90}{\tiny $n=$}}
&\rotatebox{90}{\tiny\,~~~~\,$25$}&
{\includegraphics[width=2.0cm]{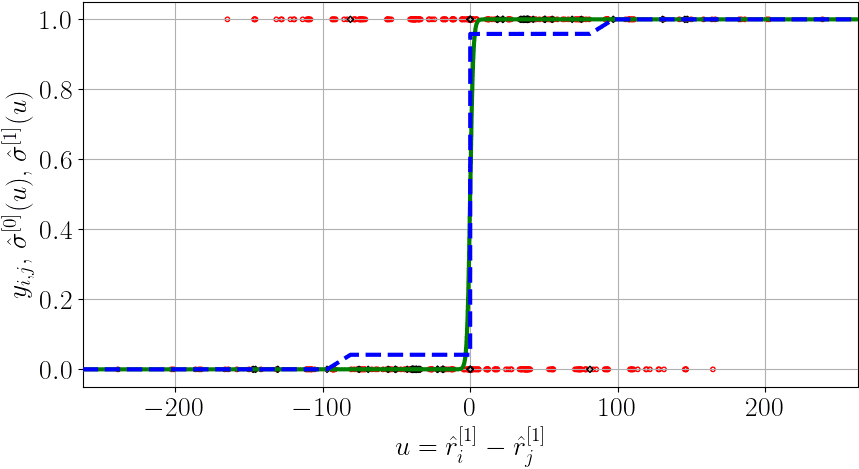}}&
{\includegraphics[width=2.0cm]{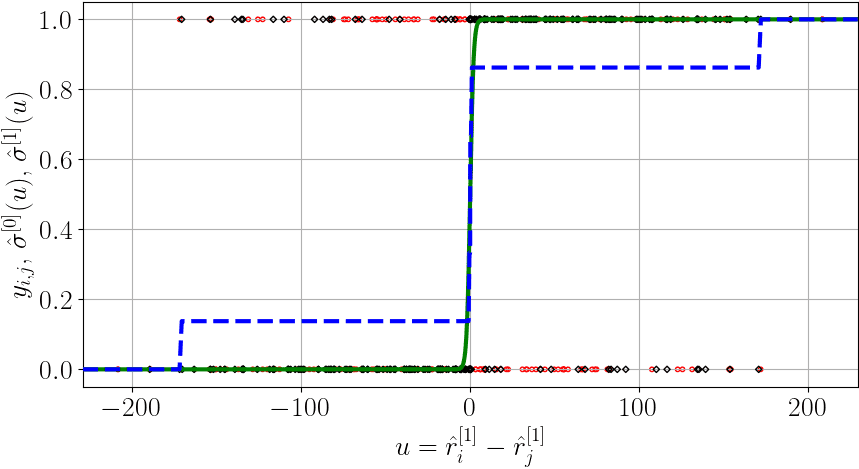}}&
{\includegraphics[width=2.0cm]{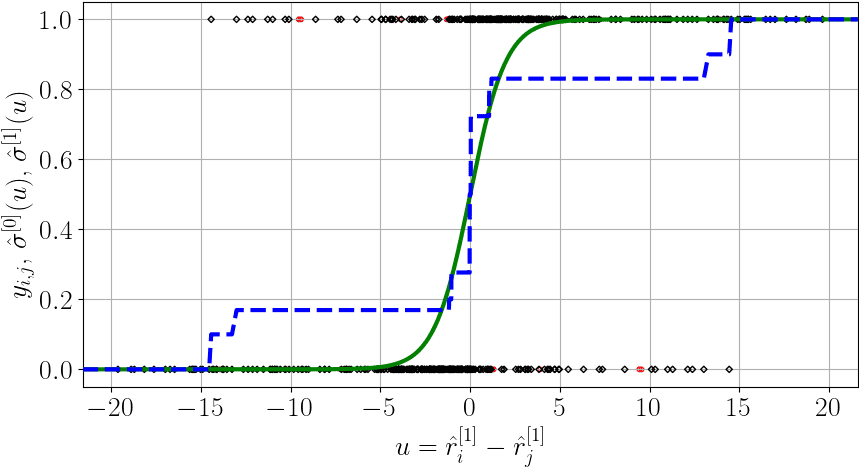}}&
{\includegraphics[width=2.0cm]{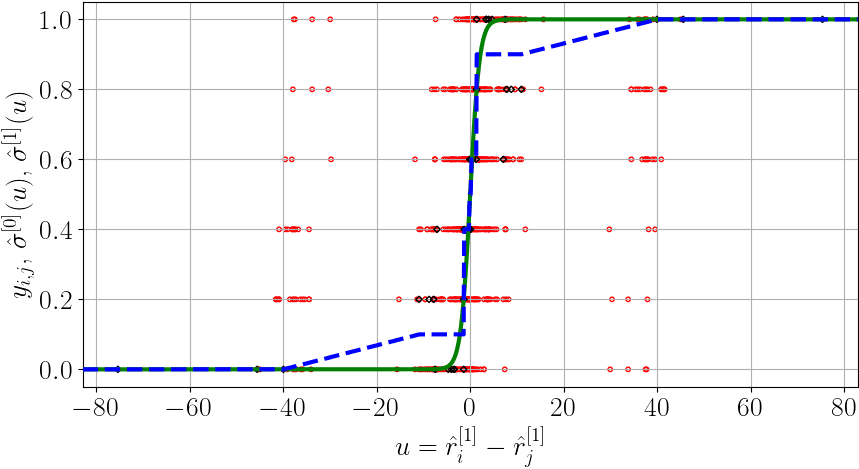}}&
{\includegraphics[width=2.0cm]{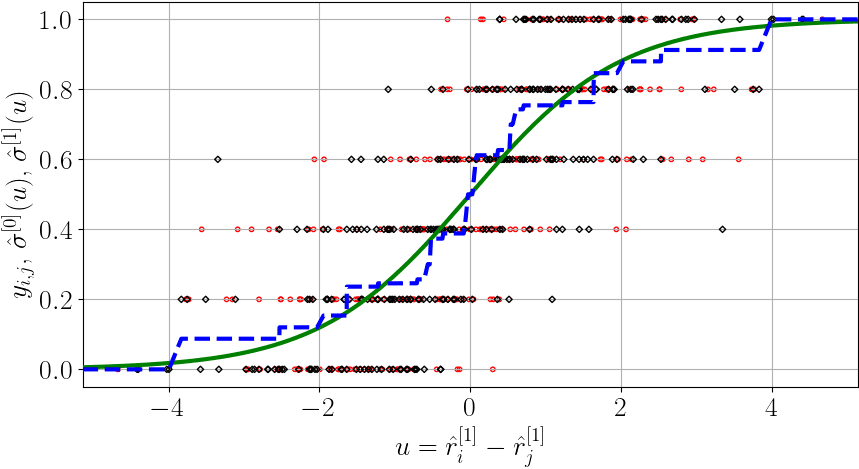}}&
{\includegraphics[width=2.0cm]{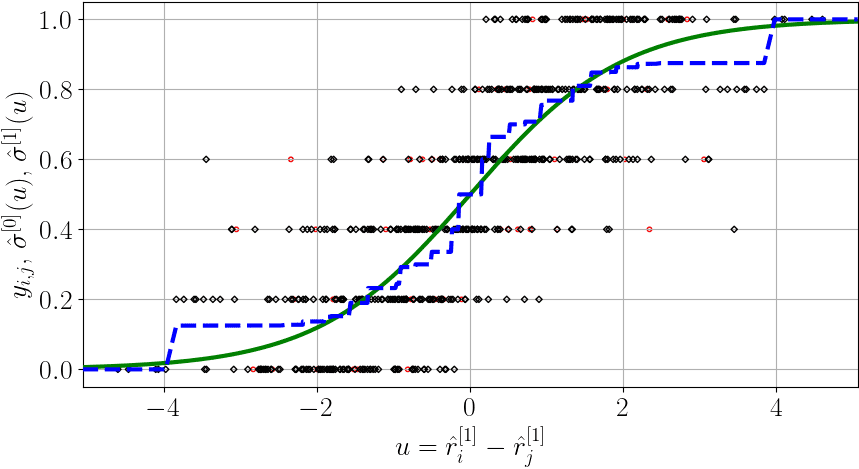}}&
{\includegraphics[width=2.0cm]{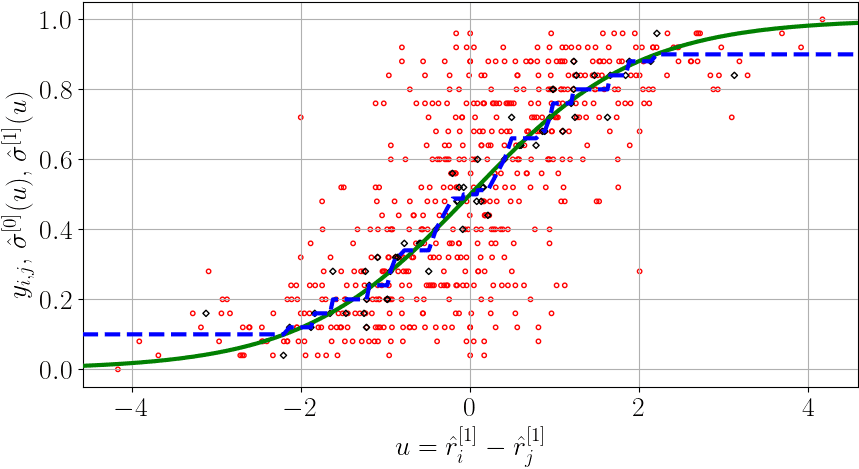}}&
{\includegraphics[width=2.0cm]{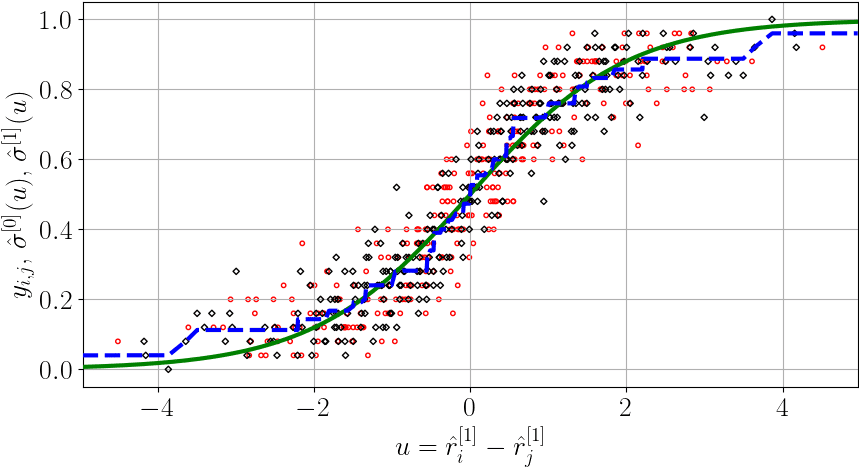}}&
{\includegraphics[width=2.0cm]{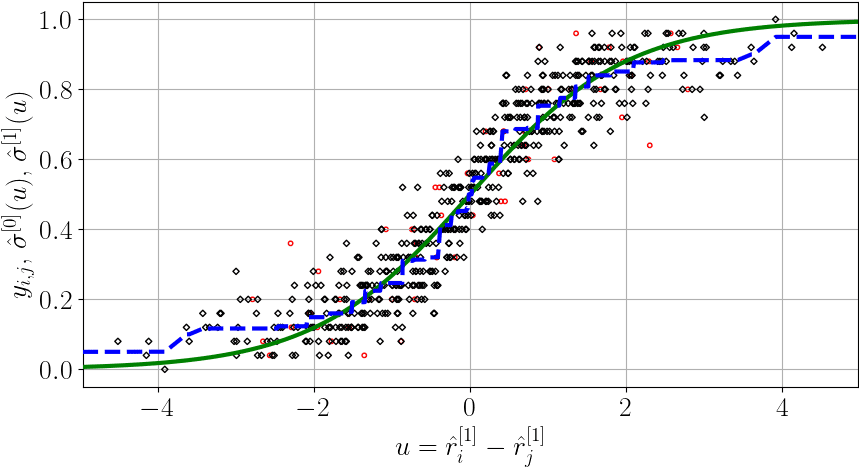}}\\
&\rotatebox{90}{\tiny\,~~~\,$100$}&
{\includegraphics[width=2.0cm]{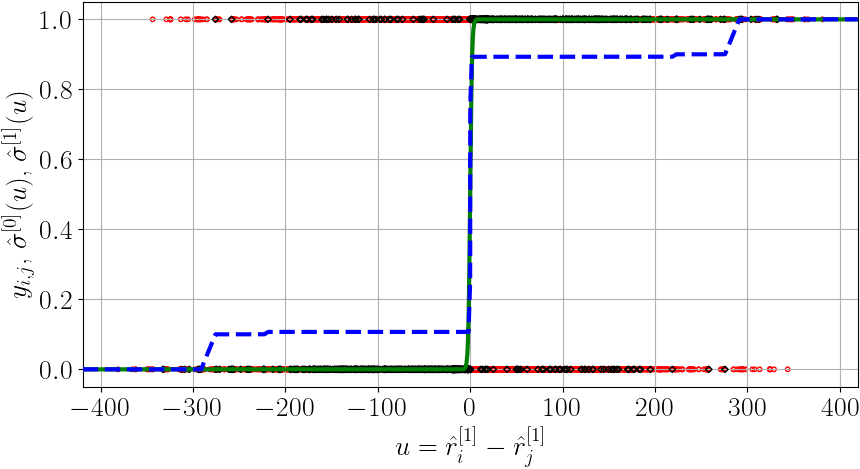}}&
{\includegraphics[width=2.0cm]{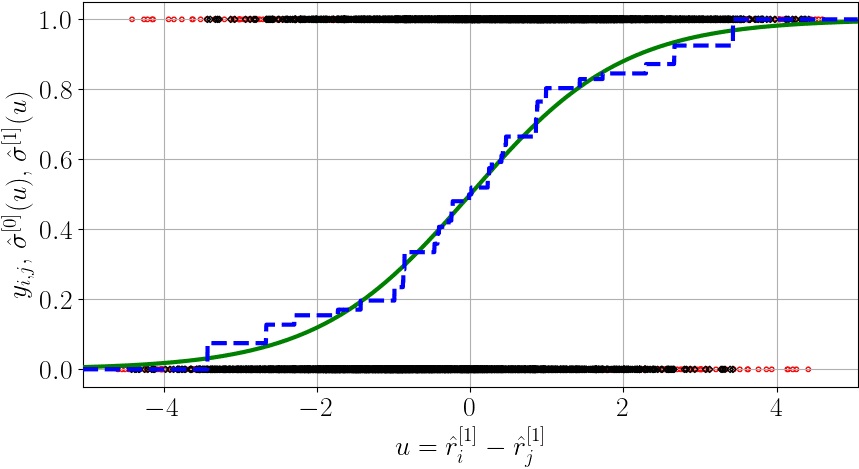}}&
{\includegraphics[width=2.0cm]{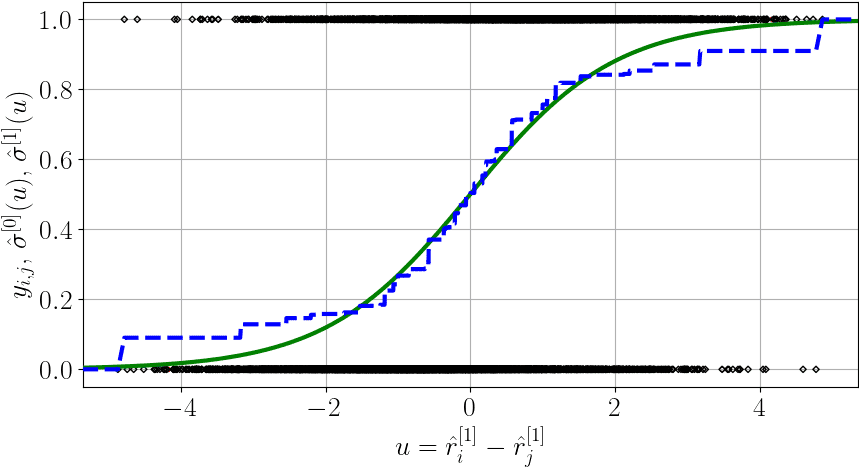}}&
{\includegraphics[width=2.0cm]{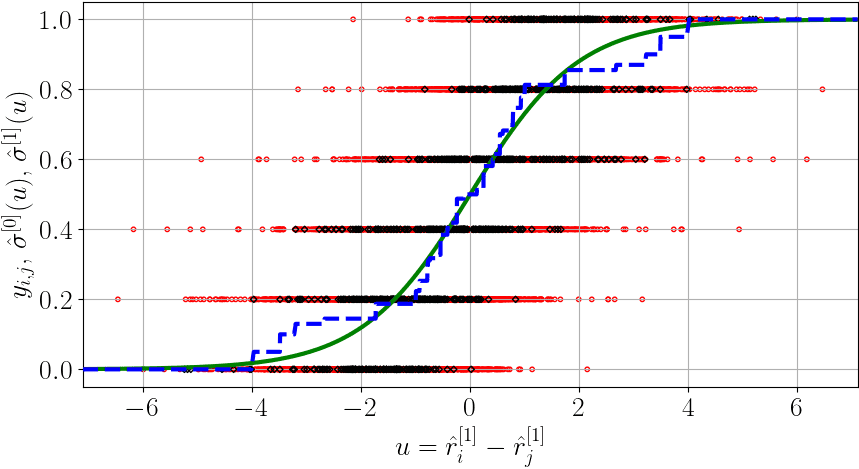}}&
{\includegraphics[width=2.0cm]{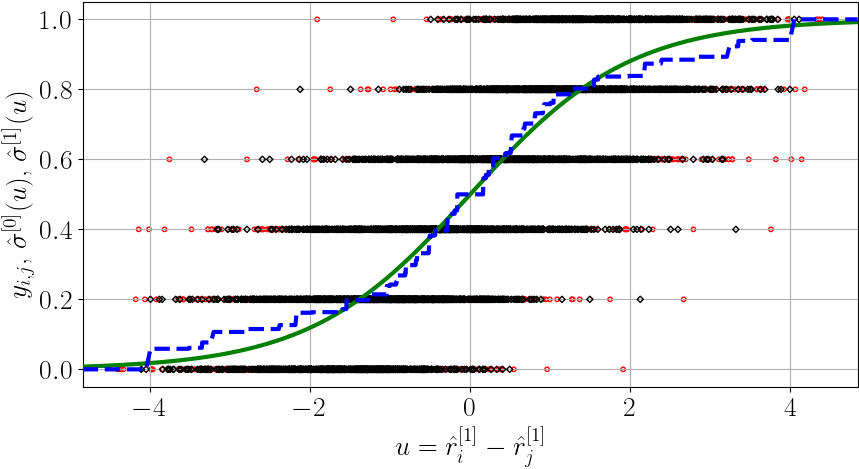}}&
{\includegraphics[width=2.0cm]{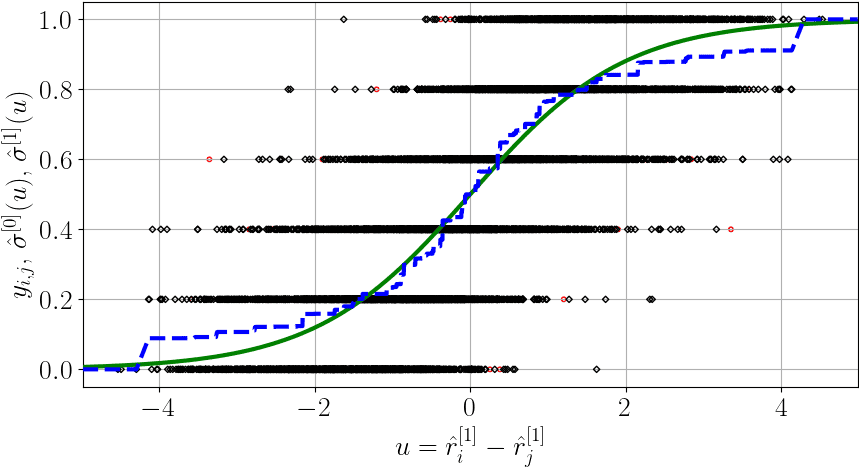}}&
{\includegraphics[width=2.0cm]{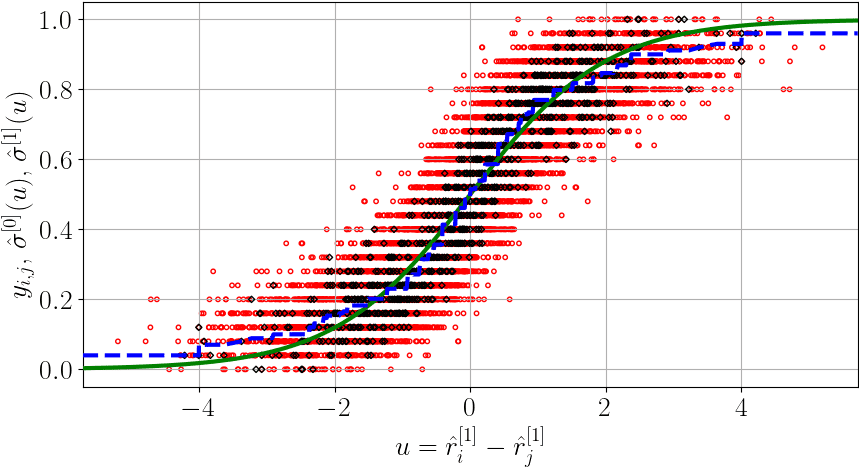}}&
{\includegraphics[width=2.0cm]{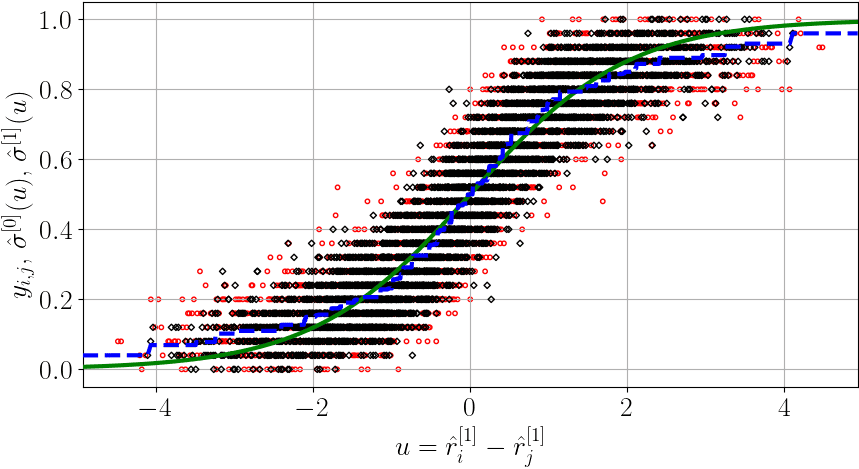}}&
{\includegraphics[width=2.0cm]{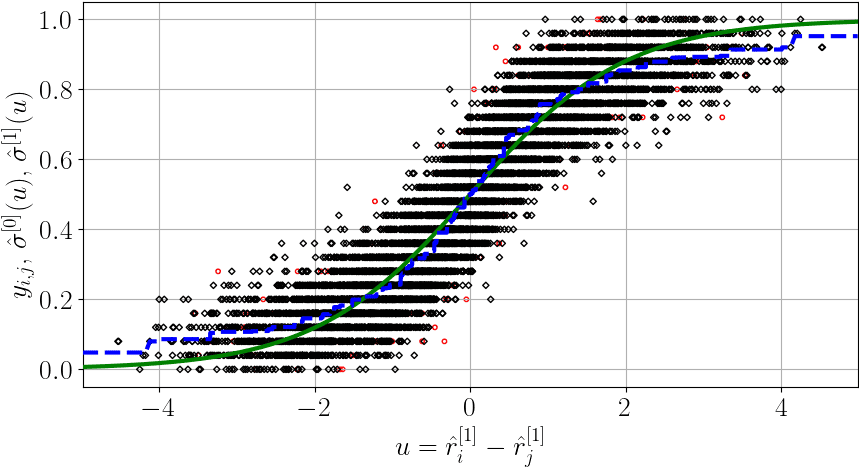}}\\
&\rotatebox{90}{\tiny\,~~~\,$400$}&
{\includegraphics[width=2.0cm]{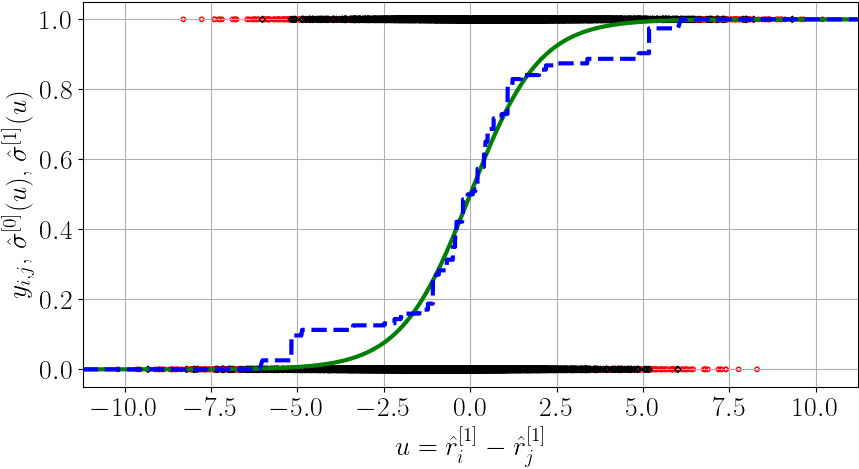}}&
{\includegraphics[width=2.0cm]{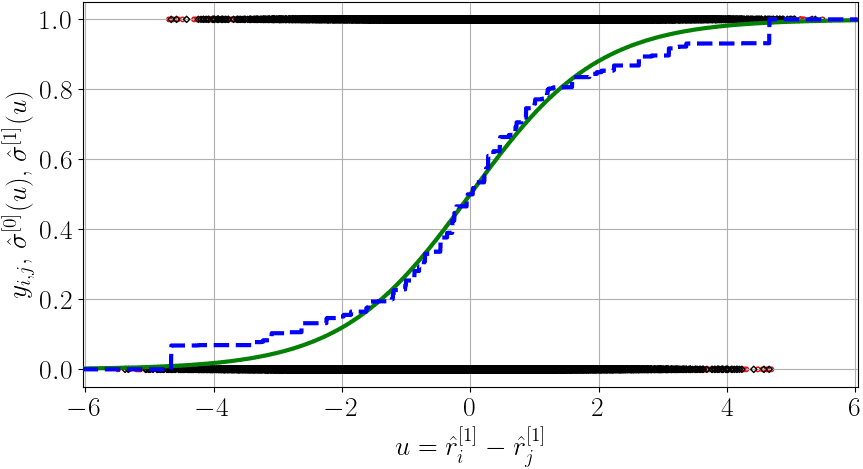}}&
{\includegraphics[width=2.0cm]{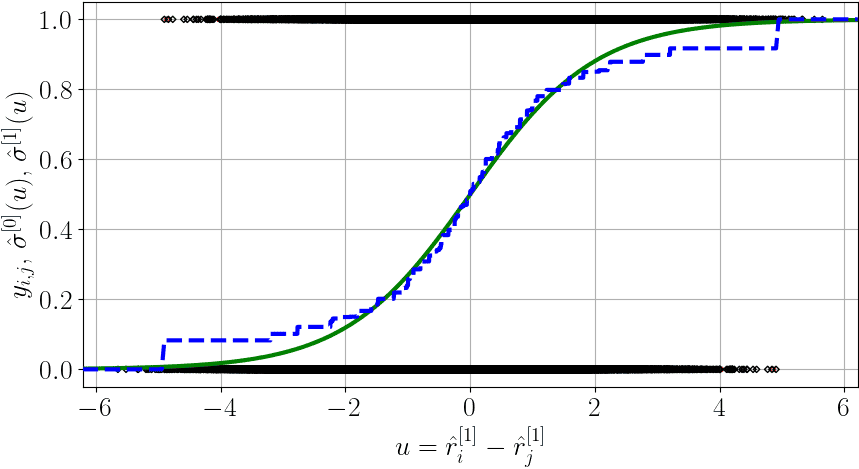}}&
{\includegraphics[width=2.0cm]{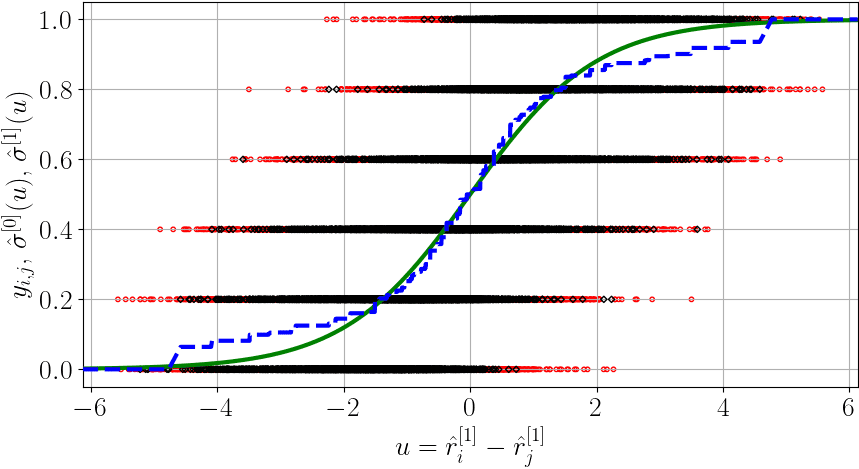}}&
{\includegraphics[width=2.0cm]{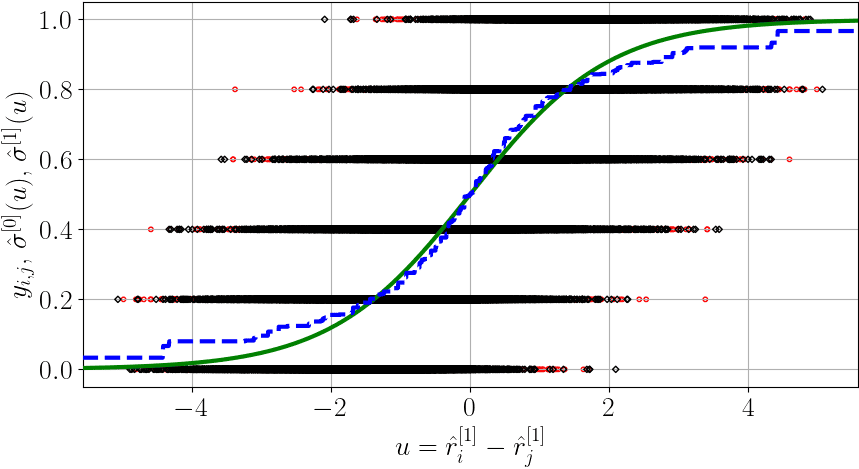}}&
{\includegraphics[width=2.0cm]{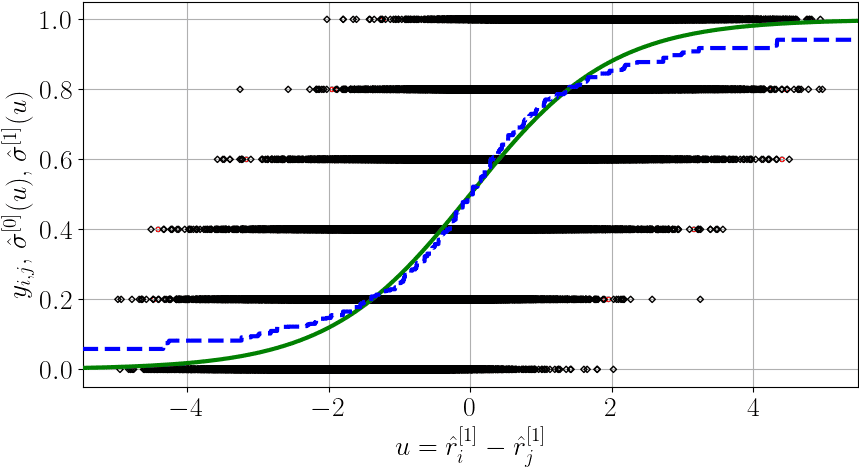}}&
{\includegraphics[width=2.0cm]{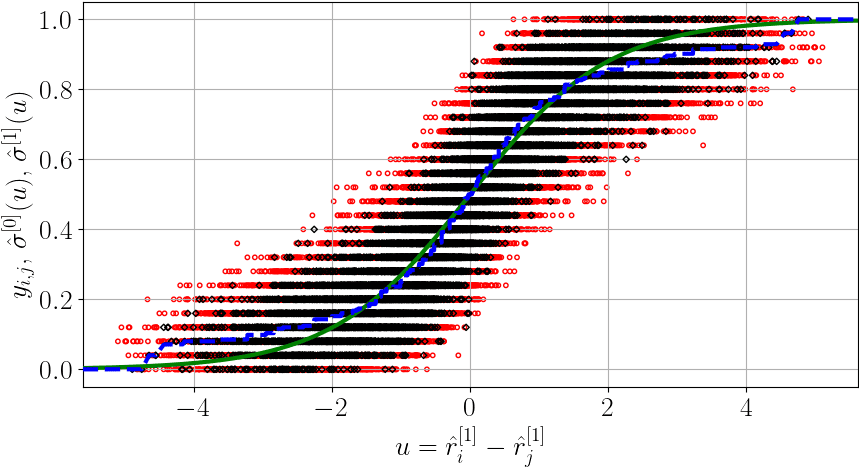}}&
{\includegraphics[width=2.0cm]{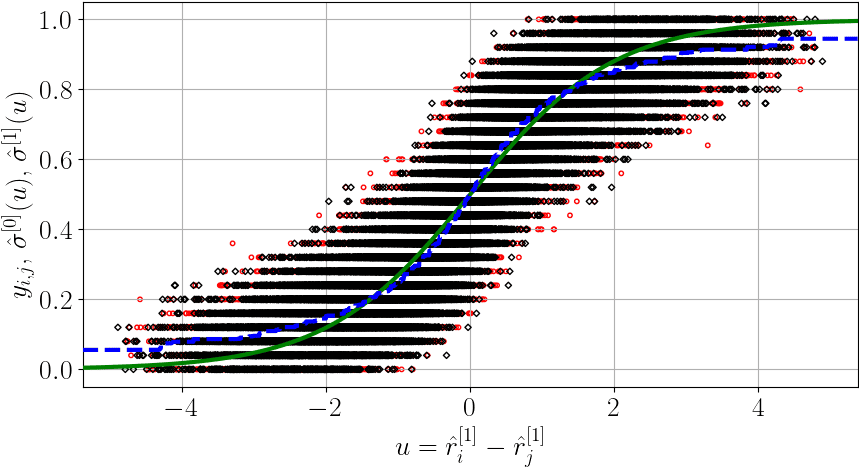}}&
{\includegraphics[width=2.0cm]{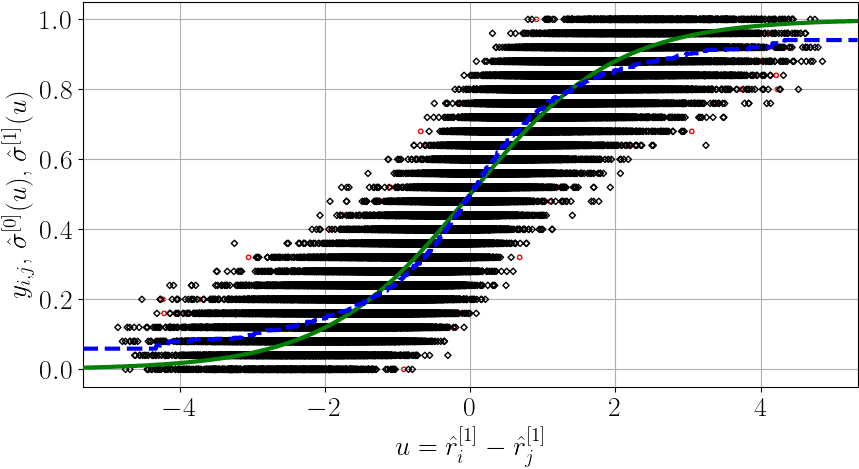}}
\end{tabular}
\begin{tabular}{ccccc}%
\multicolumn{5}{c}{\tiny Rescaled version: $n$, $N$, $|D_\tra|:|D_\tes|=$}\\
{\tiny25, 1, 1:9}&{\tiny25, 1, 5:5}&{\tiny25, 1, 9:1}&{\tiny25, 5, 1:9}&{\tiny100, 1, 1:9}\\
\midrule
{\includegraphics[width=2.0cm]{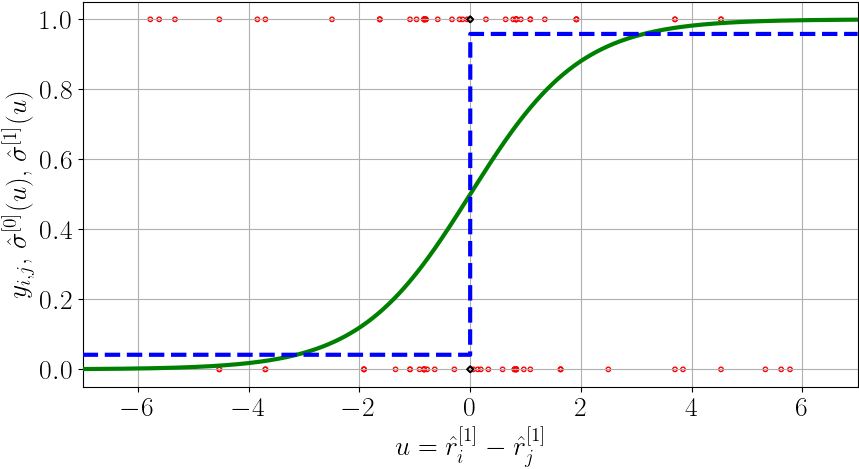}}&
{\includegraphics[width=2.0cm]{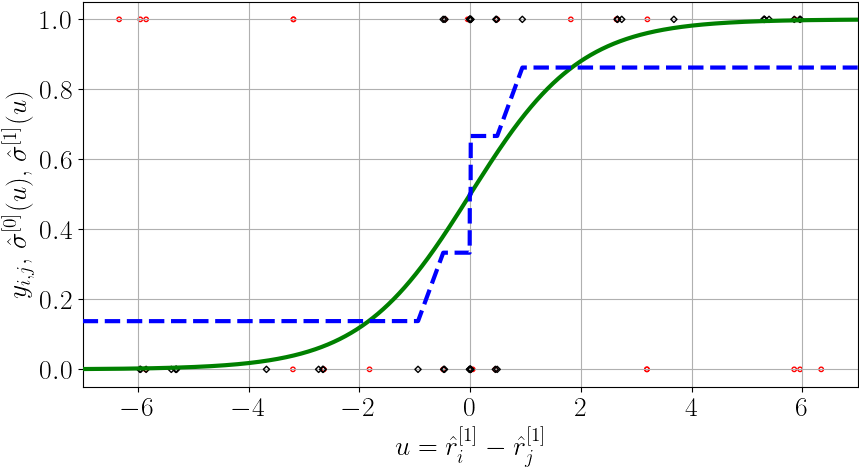}}&
{\includegraphics[width=2.0cm]{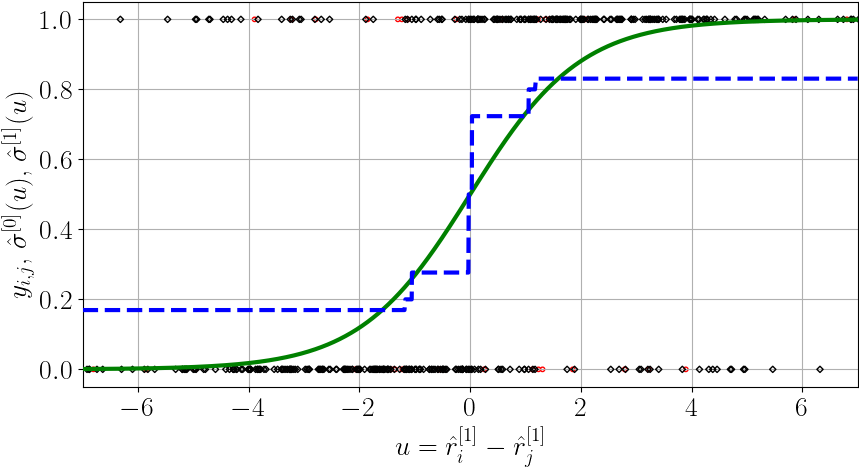}}&
{\includegraphics[width=2.0cm]{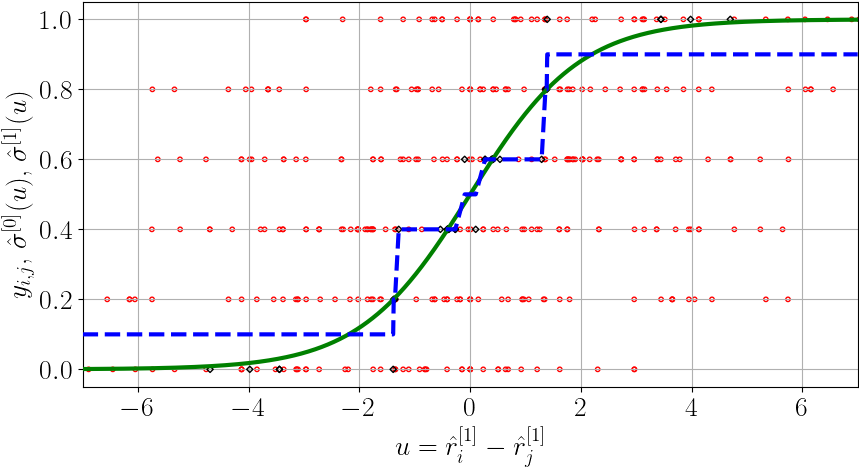}}&
{\includegraphics[width=2.0cm]{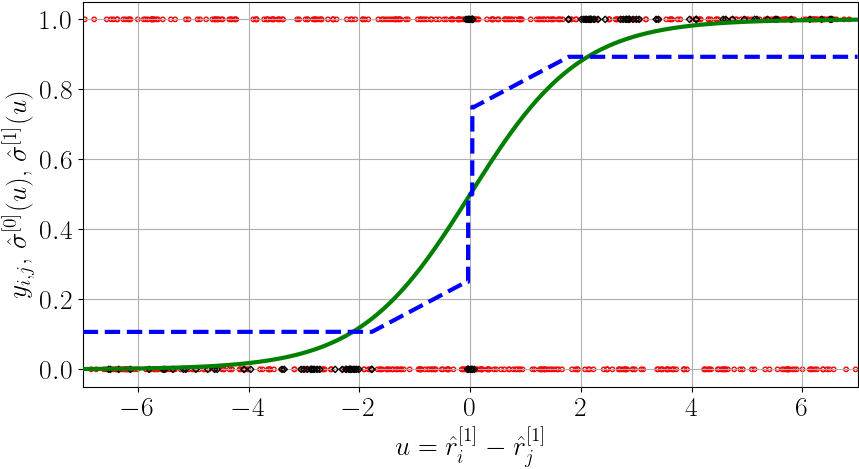}}
\end{tabular}
\caption{%
Part of results of synthetic data experiments, Procedure 1 in Section~\ref{sec:Synthetic}:
For Cauchy-$N$ synthetic data with $N=1,5,25$ (left to right),
black diamonds and red circles are training and test data $y_{i,j}$, 
and a green curve and a blue polyline are
the Bradley-Terry model $\hat{\sigma}^{[0]}(\hat{r}_i^{[1]}-\hat{r}_j^{[1]})$ and
isotonic Bradley-Terry model $\hat{\sigma}^{[1]}(\hat{r}_i^{[1]}-\hat{r}_j^{[1]})$
learned with the squared loss $\phi=\phi_\sq$ in a certain trial,
over the range $[-1.1\cdot\max_{i,j}|\hat{r}_i^{[1]}-\hat{r}_j^{[1]}|,1.1\cdot\max_{i,j}|\hat{r}_i^{[1]}-\hat{r}_j^{[1]}|]$
or $[-7,7]$ in the rescaled version.}
\label{fig:Res-Cauchy-SQ}
\end{sidewaysfigure}
\begin{sidewaystable}
\centering%
\renewcommand{\arraystretch}{0.5}%
\renewcommand{\tabcolsep}{0.5pt}%
\caption{%
Results of synthetic data experiments, Procedure 2 in Section~\ref{sec:Synthetic}:
For Cauchy-$N$ synthetic data with $N=1,5,25$ (left to right),
mean and standard deviation (std) ($\text{mean}_{\text{std}}$) of 1000 trial test evaluation of 
WPP error \eqref{eq:WPE} with the squared loss $\phi=\phi_\sq$
and Kendall's Tau \eqref{eq:Kendall}
for the Bradley-Terry model (upper) and model-selected isotonic 
Bradley-Terry model learned with the squared loss $\phi=\phi_\sq$ (lower).
Smaller \eqref{eq:WPE}, or larger \eqref{eq:Kendall} indicates a better model.
It also includes $p$-value for the Mann-Whitney U test in the parentheses:
A value below the significance level 0.05 implies 
that the isotonic Bradley-Terry model performs significantly better
and is highlighted in red bold.}
\label{tab:Cauchy-SQ}
\scalebox{0.9}{\begin{minipage}{25cm}
\begin{tabular}{cc|ccccc|ccccc|ccccc}%
&&\multicolumn{5}{c|}{\tiny $N=1$, $|D_\tra|:|D_\tes|=$}&\multicolumn{5}{c|}{\tiny $N=5$, $|D_\tra|:|D_\tes|=$}&\multicolumn{5}{c}{\tiny $N=25$, $|D_\tra|:|D_\tes|=$}\\
&&{\tiny$1:9$}&{\tiny$3:7$}&{\tiny$5:5$}&{\tiny$7:3$}&{\tiny$9:1$}&{\tiny$1:9$}&{\tiny$3:7$}&{\tiny$5:5$}&{\tiny$7:3$}&{\tiny$9:1$}&{\tiny$1:9$}&{\tiny$3:7$}&{\tiny$5:5$}&{\tiny$7:3$}&{\tiny$9:1$}
\\\midrule
\multirow{15}{*}[-6.75mm]{\rotatebox{90}{\tiny\eqref{eq:WPE}, $n=$}}
&\multirow{3}{*}{\rotatebox{90}{\tiny\,$25$~~}}
&\tcr{$\bf.3956_{.0474}$}&\tcr{$\bf.3531_{.0481}$}&\tcr{$\bf.2921_{.0540}$}&\tcr{$\bf.2361_{.0445}$}&$.2157_{.0531}$&\tcr{$\bf.1371_{.0350}$}&$.0563_{.0095}$&$.0465_{.0065}$&$.0435_{.0072}$&$.0419_{.0115}$&$.0361_{.0140}$&$.0115_{.0018}$&$.0098_{.0013}$&$.0092_{.0015}$&$.0089_{.0025}$\\&
&\tcr{$\bf.3689_{.0489}$}&\tcr{$\bf.3065_{.0393}$}&\tcr{$\bf.2576_{.0355}$}&\tcr{$\bf.2276_{.0356}$}&$.2147_{.0511}$&\tcr{$\bf.1332_{.0328}$}&$.0568_{.0095}$&$.0469_{.0067}$&$.0438_{.0074}$&$.0422_{.0116}$&$.0361_{.0138}$&$.0116_{.0019}$&$.0097_{.0014}$&$.0091_{.0016}$&$.0087_{.0025}$\\&
&(\tcr{$\bf.0000$})&(\tcr{$\bf.0000$})&(\tcr{$\bf.0000$})&(\tcr{$\bf.0004$})&($.4096$)&(\tcr{$\bf.0089$})&($.8939$)&($.9159$)&($.8419$)&($.7239$)&($.5418$)&($.9035$)&($.2069$)&($.1547$)&($.0842$)\\
&\multirow{3}{*}{\rotatebox{90}{\tiny\,$50$~~}}
&\tcr{$\bf.3830_{.0289}$}&\tcr{$\bf.2756_{.0402}$}&$.2104_{.0143}$&$.2010_{.0139}$&$.1968_{.0206}$&\tcr{$\bf.0774_{.0124}$}&$.0448_{.0029}$&$.0415_{.0028}$&$.0403_{.0032}$&$.0395_{.0050}$&$.0155_{.0026}$&\tcr{$\bf.0094_{.0006}$}&\tcr{$\bf.0088_{.0006}$}&\tcr{$\bf.0085_{.0007}$}&\tcr{$\bf.0084_{.0011}$}\\&
&\tcr{$\bf.3501_{.0265}$}&\tcr{$\bf.2459_{.0202}$}&$.2098_{.0133}$&$.2015_{.0138}$&$.1971_{.0208}$&\tcr{$\bf.0761_{.0110}$}&$.0450_{.0031}$&$.0416_{.0028}$&$.0403_{.0032}$&$.0395_{.0050}$&$.0155_{.0027}$&\tcr{$\bf.0092_{.0006}$}&\tcr{$\bf.0085_{.0006}$}&\tcr{$\bf.0081_{.0007}$}&\tcr{$\bf.0080_{.0011}$}\\&
&(\tcr{$\bf.0000$})&(\tcr{$\bf.0000$})&($.2574$)&($.7773$)&($.6132$)&(\tcr{$\bf.0423$})&($.9571$)&($.7823$)&($.5899$)&($.3755$)&($.6927$)&(\tcr{$\bf.0000$})&(\tcr{$\bf.0000$})&(\tcr{$\bf.0000$})&(\tcr{$\bf.0000$})\\
&\multirow{3}{*}{\rotatebox{90}{\tiny\,$100$~~}}
&\tcr{$\bf.3497_{.0163}$}&$.2042_{.0068}$&$.1952_{.0063}$&$.1921_{.0073}$&$.1904_{.0105}$&$.0499_{.0027}$&\tcr{$\bf.0407_{.0014}$}&\tcr{$\bf.0394_{.0014}$}&\tcr{$\bf.0388_{.0016}$}&\tcr{$\bf.0385_{.0026}$}&\tcr{$\bf.0104_{.0005}$}&\tcr{$\bf.0086_{.0003}$}&\tcr{$\bf.0084_{.0003}$}&\tcr{$\bf.0082_{.0003}$}&\tcr{$\bf.0082_{.0005}$}\\&
&\tcr{$\bf.2941_{.0123}$}&$.2038_{.0065}$&$.1953_{.0063}$&$.1922_{.0073}$&$.1904_{.0105}$&$.0500_{.0026}$&\tcr{$\bf.0406_{.0014}$}&\tcr{$\bf.0391_{.0015}$}&\tcr{$\bf.0385_{.0017}$}&\tcr{$\bf.0381_{.0026}$}&\tcr{$\bf.0101_{.0006}$}&\tcr{$\bf.0082_{.0003}$}&\tcr{$\bf.0079_{.0003}$}&\tcr{$\bf.0077_{.0004}$}&\tcr{$\bf.0077_{.0005}$}\\&
&(\tcr{$\bf.0000$})&($.1584$)&($.6476$)&($.5822$)&($.4895$)&($.7431$)&(\tcr{$\bf.0349$})&(\tcr{$\bf.0000$})&(\tcr{$\bf.0000$})&(\tcr{$\bf.0006$})&(\tcr{$\bf.0000$})&(\tcr{$\bf.0000$})&(\tcr{$\bf.0000$})&(\tcr{$\bf.0000$})&(\tcr{$\bf.0000$})\\
&\multirow{3}{*}{\rotatebox{90}{\tiny\,$200$~~}}
&\tcr{$\bf.2289_{.0133}$}&$.1934_{.0037}$&$.1901_{.0038}$&$.1887_{.0042}$&$.1879_{.0057}$&\tcr{$\bf.0426_{.0009}$}&\tcr{$\bf.0391_{.0008}$}&\tcr{$\bf.0385_{.0008}$}&\tcr{$\bf.0382_{.0009}$}&\tcr{$\bf.0381_{.0014}$}&\tcr{$\bf.0090_{.0002}$}&\tcr{$\bf.0083_{.0001}$}&\tcr{$\bf.0082_{.0001}$}&\tcr{$\bf.0081_{.0002}$}&\tcr{$\bf.0081_{.0003}$}\\&
&\tcr{$\bf.2202_{.0067}$}&$.1934_{.0037}$&$.1900_{.0039}$&$.1885_{.0042}$&$.1877_{.0058}$&\tcr{$\bf.0424_{.0010}$}&\tcr{$\bf.0387_{.0008}$}&\tcr{$\bf.0380_{.0009}$}&\tcr{$\bf.0377_{.0010}$}&\tcr{$\bf.0376_{.0014}$}&\tcr{$\bf.0086_{.0002}$}&\tcr{$\bf.0078_{.0002}$}&\tcr{$\bf.0076_{.0002}$}&\tcr{$\bf.0076_{.0002}$}&\tcr{$\bf.0075_{.0003}$}\\&
&(\tcr{$\bf.0000$})&($.4075$)&($.1961$)&($.1280$)&($.1423$)&(\tcr{$\bf.0000$})&(\tcr{$\bf.0000$})&(\tcr{$\bf.0000$})&(\tcr{$\bf.0000$})&(\tcr{$\bf.0000$})&(\tcr{$\bf.0000$})&(\tcr{$\bf.0000$})&(\tcr{$\bf.0000$})&(\tcr{$\bf.0000$})&(\tcr{$\bf.0000$})\\
&\multirow{3}{*}{\rotatebox{90}{\tiny\,$400$~~}}
&\tcr{$\bf.1984_{.0025}$}&\tcr{$\bf.1894_{.0024}$}&\tcr{$\bf.1879_{.0025}$}&\tcr{$\bf.1872_{.0026}$}&\tcr{$\bf.1869_{.0033}$}&\tcr{$\bf.0399_{.0005}$}&\tcr{$\bf.0383_{.0005}$}&\tcr{$\bf.0380_{.0005}$}&\tcr{$\bf.0379_{.0005}$}&\tcr{$\bf.0378_{.0008}$}&\tcr{$\bf.0085_{.0001}$}&\tcr{$\bf.0081_{.0001}$}&\tcr{$\bf.0081_{.0001}$}&\tcr{$\bf.0081_{.0001}$}&\tcr{$\bf.0080_{.0001}$}\\&
&\tcr{$\bf.1980_{.0026}$}&\tcr{$\bf.1890_{.0025}$}&\tcr{$\bf.1875_{.0025}$}&\tcr{$\bf.1868_{.0026}$}&\tcr{$\bf.1864_{.0033}$}&\tcr{$\bf.0395_{.0005}$}&\tcr{$\bf.0378_{.0005}$}&\tcr{$\bf.0375_{.0005}$}&\tcr{$\bf.0374_{.0006}$}&\tcr{$\bf.0373_{.0008}$}&\tcr{$\bf.0079_{.0001}$}&\tcr{$\bf.0076_{.0001}$}&\tcr{$\bf.0075_{.0001}$}&\tcr{$\bf.0075_{.0001}$}&\tcr{$\bf.0075_{.0002}$}\\&
&(\tcr{$\bf.0010$})&(\tcr{$\bf.0019$})&(\tcr{$\bf.0002$})&(\tcr{$\bf.0001$})&(\tcr{$\bf.0013$})&(\tcr{$\bf.0000$})&(\tcr{$\bf.0000$})&(\tcr{$\bf.0000$})&(\tcr{$\bf.0000$})&(\tcr{$\bf.0000$})&(\tcr{$\bf.0000$})&(\tcr{$\bf.0000$})&(\tcr{$\bf.0000$})&(\tcr{$\bf.0000$})&(\tcr{$\bf.0000$})\\
\midrule
\multirow{15}{*}[-6.75mm]{\rotatebox{90}{\tiny\eqref{eq:Kendall}, $n=$}}
&\multirow{3}{*}{\rotatebox{90}{\tiny\,$25$~~}}
&$.1684_{.0862}$&\tcr{$\bf.2556_{.0824}$}&\tcr{$\bf.2989_{.0775}$}&\tcr{$\bf.3345_{.0865}$}&\tcr{$\bf.3536_{.1330}$}&$.3217_{.0879}$&\tcr{$\bf.5378_{.0498}$}&\tcr{$\bf.5766_{.0454}$}&\tcr{$\bf.5912_{.0515}$}&\tcr{$\bf.6007_{.0780}$}&$.5712_{.0784}$&\tcr{$\bf.7535_{.0246}$}&\tcr{$\bf.7718_{.0219}$}&\tcr{$\bf.7788_{.0254}$}&\tcr{$\bf.7832_{.0415}$}\\&
&$.1707_{.0883}$&\tcr{$\bf.2653_{.0873}$}&\tcr{$\bf.3154_{.0842}$}&\tcr{$\bf.3534_{.0933}$}&\tcr{$\bf.3713_{.1445}$}&$.3270_{.0895}$&\tcr{$\bf.5530_{.0522}$}&\tcr{$\bf.5937_{.0473}$}&\tcr{$\bf.6082_{.0532}$}&\tcr{$\bf.6184_{.0800}$}&$.5751_{.0790}$&\tcr{$\bf.7652_{.0251}$}&\tcr{$\bf.7844_{.0223}$}&\tcr{$\bf.7913_{.0258}$}&\tcr{$\bf.7957_{.0413}$}\\&
&($.3049$)&(\tcr{$\bf.0051$})&(\tcr{$\bf.0000$})&(\tcr{$\bf.0000$})&(\tcr{$\bf.0020$})&($.0757$)&(\tcr{$\bf.0000$})&(\tcr{$\bf.0000$})&(\tcr{$\bf.0000$})&(\tcr{$\bf.0000$})&($.1161$)&(\tcr{$\bf.0000$})&(\tcr{$\bf.0000$})&(\tcr{$\bf.0000$})&(\tcr{$\bf.0000$})\\
&\multirow{3}{*}{\rotatebox{90}{\tiny\,$50$~~}}
&\tcr{$\bf.2147_{.0533}$}&\tcr{$\bf.3084_{.0393}$}&\tcr{$\bf.3602_{.0366}$}&\tcr{$\bf.3760_{.0405}$}&\tcr{$\bf.3851_{.0609}$}&\tcr{$\bf.4677_{.0431}$}&\tcr{$\bf.5844_{.0243}$}&\tcr{$\bf.6010_{.0235}$}&\tcr{$\bf.6076_{.0257}$}&\tcr{$\bf.6119_{.0368}$}&\tcr{$\bf.7166_{.0252}$}&\tcr{$\bf.7768_{.0115}$}&\tcr{$\bf.7848_{.0113}$}&\tcr{$\bf.7880_{.0125}$}&\tcr{$\bf.7900_{.0182}$}\\&
&\tcr{$\bf.2223_{.0555}$}&\tcr{$\bf.3277_{.0434}$}&\tcr{$\bf.3760_{.0393}$}&\tcr{$\bf.3918_{.0437}$}&\tcr{$\bf.4012_{.0659}$}&\tcr{$\bf.4798_{.0447}$}&\tcr{$\bf.6013_{.0259}$}&\tcr{$\bf.6184_{.0247}$}&\tcr{$\bf.6245_{.0265}$}&\tcr{$\bf.6284_{.0378}$}&\tcr{$\bf.7271_{.0256}$}&\tcr{$\bf.7894_{.0117}$}&\tcr{$\bf.7973_{.0114}$}&\tcr{$\bf.8000_{.0126}$}&\tcr{$\bf.8017_{.0181}$}\\&
&(\tcr{$\bf.0010$})&(\tcr{$\bf.0000$})&(\tcr{$\bf.0000$})&(\tcr{$\bf.0000$})&(\tcr{$\bf.0000$})&(\tcr{$\bf.0000$})&(\tcr{$\bf.0000$})&(\tcr{$\bf.0000$})&(\tcr{$\bf.0000$})&(\tcr{$\bf.0000$})&(\tcr{$\bf.0000$})&(\tcr{$\bf.0000$})&(\tcr{$\bf.0000$})&(\tcr{$\bf.0000$})&(\tcr{$\bf.0000$})\\
&\multirow{3}{*}{\rotatebox{90}{\tiny\,$100$~~}}
&\tcr{$\bf.2707_{.0302}$}&\tcr{$\bf.3710_{.0193}$}&\tcr{$\bf.3892_{.0196}$}&\tcr{$\bf.3965_{.0226}$}&\tcr{$\bf.4009_{.0323}$}&\tcr{$\bf.5620_{.0176}$}&\tcr{$\bf.6063_{.0138}$}&\tcr{$\bf.6138_{.0139}$}&\tcr{$\bf.6169_{.0151}$}&\tcr{$\bf.6187_{.0207}$}&\tcr{$\bf.7664_{.0085}$}&\tcr{$\bf.7873_{.0064}$}&\tcr{$\bf.7909_{.0066}$}&\tcr{$\bf.7923_{.0072}$}&\tcr{$\bf.7931_{.0099}$}\\&
&\tcr{$\bf.2854_{.0325}$}&\tcr{$\bf.3862_{.0208}$}&\tcr{$\bf.4046_{.0208}$}&\tcr{$\bf.4118_{.0238}$}&\tcr{$\bf.4158_{.0336}$}&\tcr{$\bf.5776_{.0186}$}&\tcr{$\bf.6229_{.0142}$}&\tcr{$\bf.6297_{.0141}$}&\tcr{$\bf.6319_{.0156}$}&\tcr{$\bf.6328_{.0212}$}&\tcr{$\bf.7786_{.0088}$}&\tcr{$\bf.7994_{.0065}$}&\tcr{$\bf.8019_{.0065}$}&\tcr{$\bf.8025_{.0072}$}&\tcr{$\bf.8027_{.0099}$}\\&
&(\tcr{$\bf.0000$})&(\tcr{$\bf.0000$})&(\tcr{$\bf.0000$})&(\tcr{$\bf.0000$})&(\tcr{$\bf.0000$})&(\tcr{$\bf.0000$})&(\tcr{$\bf.0000$})&(\tcr{$\bf.0000$})&(\tcr{$\bf.0000$})&(\tcr{$\bf.0000$})&(\tcr{$\bf.0000$})&(\tcr{$\bf.0000$})&(\tcr{$\bf.0000$})&(\tcr{$\bf.0000$})&(\tcr{$\bf.0000$})\\
&\multirow{3}{*}{\rotatebox{90}{\tiny\,$200$~~}}
&\tcr{$\bf.3366_{.0149}$}&\tcr{$\bf.3932_{.0120}$}&\tcr{$\bf.4017_{.0122}$}&\tcr{$\bf.4056_{.0132}$}&\tcr{$\bf.4077_{.0179}$}&\tcr{$\bf.5961_{.0100}$}&\tcr{$\bf.6158_{.0091}$}&\tcr{$\bf.6194_{.0092}$}&\tcr{$\bf.6210_{.0097}$}&\tcr{$\bf.6218_{.0119}$}&\tcr{$\bf.7828_{.0047}$}&\tcr{$\bf.7920_{.0043}$}&\tcr{$\bf.7937_{.0044}$}&\tcr{$\bf.7944_{.0046}$}&\tcr{$\bf.7948_{.0056}$}\\&
&\tcr{$\bf.3520_{.0154}$}&\tcr{$\bf.4072_{.0127}$}&\tcr{$\bf.4149_{.0126}$}&\tcr{$\bf.4183_{.0137}$}&\tcr{$\bf.4201_{.0184}$}&\tcr{$\bf.6118_{.0104}$}&\tcr{$\bf.6300_{.0092}$}&\tcr{$\bf.6322_{.0092}$}&\tcr{$\bf.6327_{.0098}$}&\tcr{$\bf.6329_{.0121}$}&\tcr{$\bf.7948_{.0048}$}&\tcr{$\bf.8021_{.0043}$}&\tcr{$\bf.8026_{.0044}$}&\tcr{$\bf.8025_{.0046}$}&\tcr{$\bf.8023_{.0056}$}\\&
&(\tcr{$\bf.0000$})&(\tcr{$\bf.0000$})&(\tcr{$\bf.0000$})&(\tcr{$\bf.0000$})&(\tcr{$\bf.0000$})&(\tcr{$\bf.0000$})&(\tcr{$\bf.0000$})&(\tcr{$\bf.0000$})&(\tcr{$\bf.0000$})&(\tcr{$\bf.0000$})&(\tcr{$\bf.0000$})&(\tcr{$\bf.0000$})&(\tcr{$\bf.0000$})&(\tcr{$\bf.0000$})&(\tcr{$\bf.0000$})\\
&\multirow{3}{*}{\rotatebox{90}{\tiny\,$400$~~}}
&\tcr{$\bf.3817_{.0083}$}&\tcr{$\bf.4038_{.0078}$}&\tcr{$\bf.4080_{.0079}$}&\tcr{$\bf.4099_{.0083}$}&\tcr{$\bf.4110_{.0104}$}&\tcr{$\bf.6110_{.0065}$}&\tcr{$\bf.6203_{.0063}$}&\tcr{$\bf.6220_{.0063}$}&\tcr{$\bf.6229_{.0064}$}&\tcr{$\bf.6233_{.0074}$}&\tcr{$\bf.7899_{.0030}$}&\tcr{$\bf.7941_{.0029}$}&\tcr{$\bf.7949_{.0029}$}&\tcr{$\bf.7953_{.0030}$}&\tcr{$\bf.7955_{.0035}$}\\&
&\tcr{$\bf.3942_{.0088}$}&\tcr{$\bf.4150_{.0080}$}&\tcr{$\bf.4183_{.0081}$}&\tcr{$\bf.4196_{.0085}$}&\tcr{$\bf.4202_{.0106}$}&\tcr{$\bf.6244_{.0066}$}&\tcr{$\bf.6311_{.0063}$}&\tcr{$\bf.6316_{.0063}$}&\tcr{$\bf.6316_{.0064}$}&\tcr{$\bf.6315_{.0074}$}&\tcr{$\bf.7999_{.0030}$}&\tcr{$\bf.8020_{.0029}$}&\tcr{$\bf.8018_{.0029}$}&\tcr{$\bf.8015_{.0029}$}&\tcr{$\bf.8012_{.0034}$}\\&
&(\tcr{$\bf.0000$})&(\tcr{$\bf.0000$})&(\tcr{$\bf.0000$})&(\tcr{$\bf.0000$})&(\tcr{$\bf.0000$})&(\tcr{$\bf.0000$})&(\tcr{$\bf.0000$})&(\tcr{$\bf.0000$})&(\tcr{$\bf.0000$})&(\tcr{$\bf.0000$})&(\tcr{$\bf.0000$})&(\tcr{$\bf.0000$})&(\tcr{$\bf.0000$})&(\tcr{$\bf.0000$})&(\tcr{$\bf.0000$})\\
\end{tabular}\end{minipage}}
\end{sidewaystable}

\begin{sidewaysfigure}
\centering%
\renewcommand{\arraystretch}{0.5}%
\renewcommand{\tabcolsep}{0.5pt}%
\begin{tabular}{c|ccc|ccc|ccc}%
&\multicolumn{3}{c|}{\tiny PL, $|D_\tra|:|D_\tes|=$}&\multicolumn{3}{c|}{\tiny MLB, $|D_\tra|:|D_\tes|=$}&\multicolumn{3}{c}{\tiny ATP, $|D_\tra|:|D_\tes|=$}\\
&{\tiny$1:9$}&{\tiny$5:5$}&{\tiny$9:1$}&{\tiny$1:9$}&{\tiny$5:5$}&{\tiny$9:1$}&{\tiny$1:9$}&{\tiny$5:5$}&{\tiny$9:1$}
\\\midrule
\raisebox{1.75ex}{\rotatebox{90}{\tiny\eqref{eq:WPE}}}&
\CF{\includegraphics[width=2.0cm]{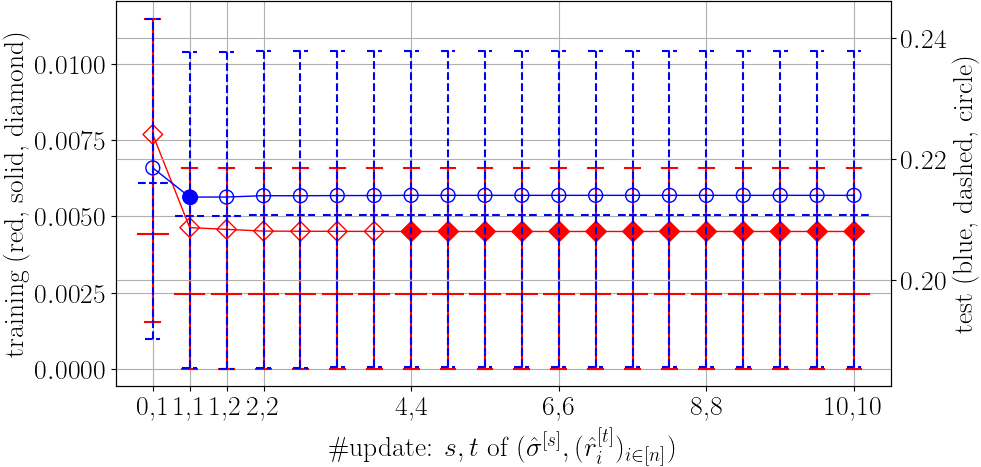}}&
{\includegraphics[width=2.0cm]{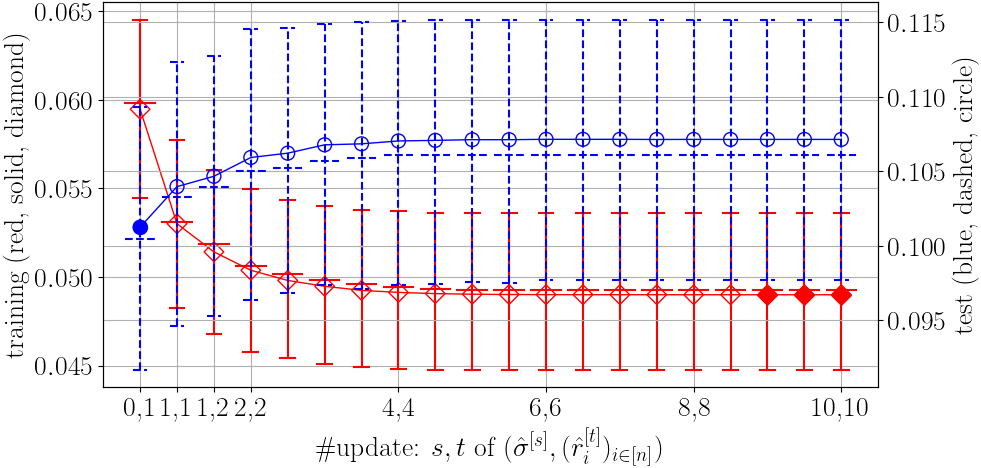}}&
{\includegraphics[width=2.0cm]{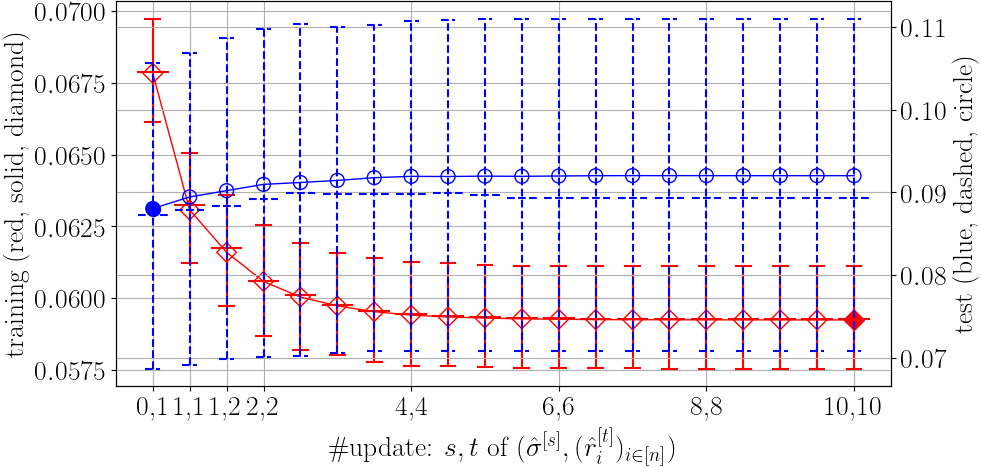}}&
\CF{\includegraphics[width=2.0cm]{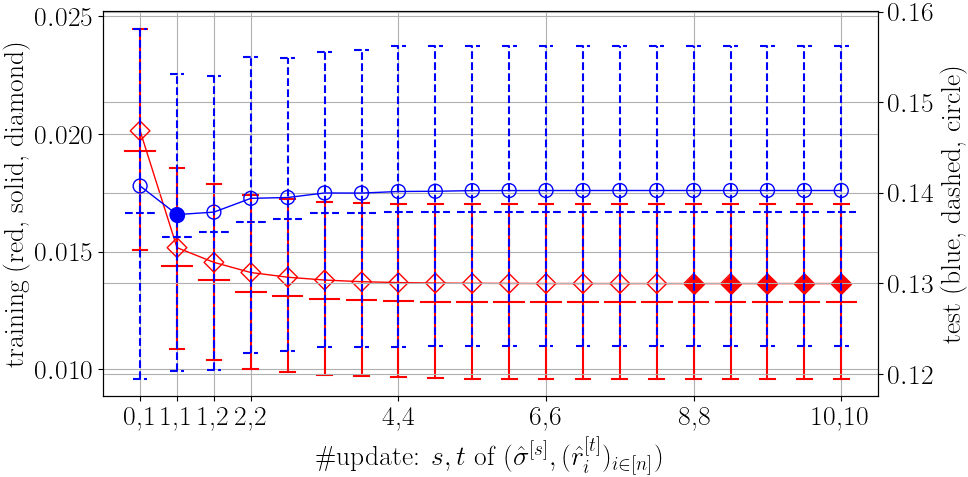}}&
{\includegraphics[width=2.0cm]{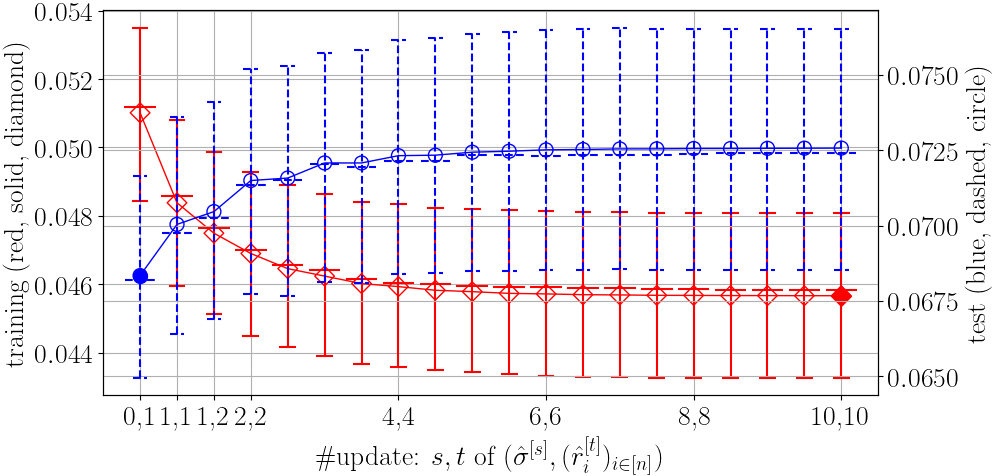}}&
{\includegraphics[width=2.0cm]{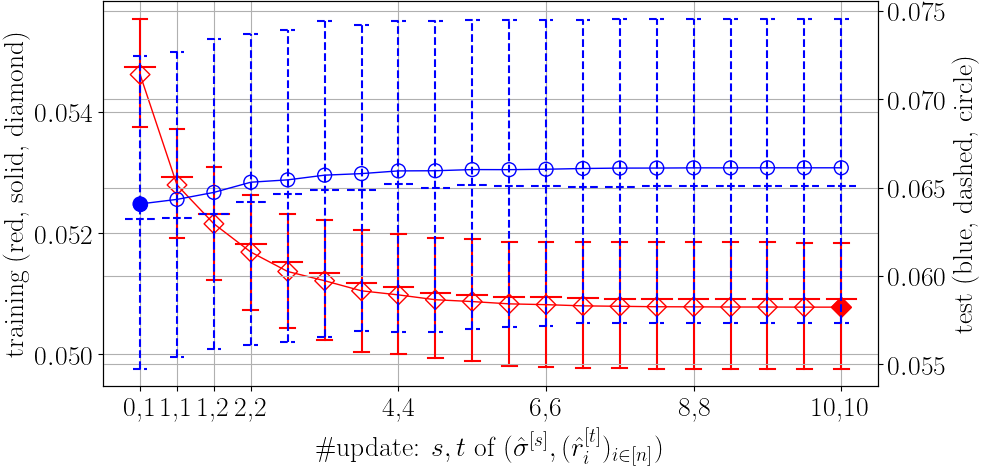}}&
\CF{\includegraphics[width=2.0cm]{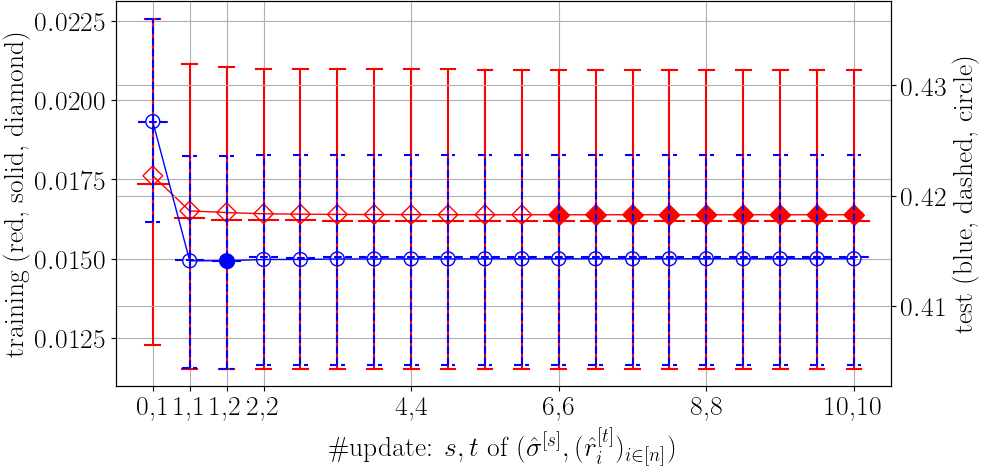}}&
\CF{\includegraphics[width=2.0cm]{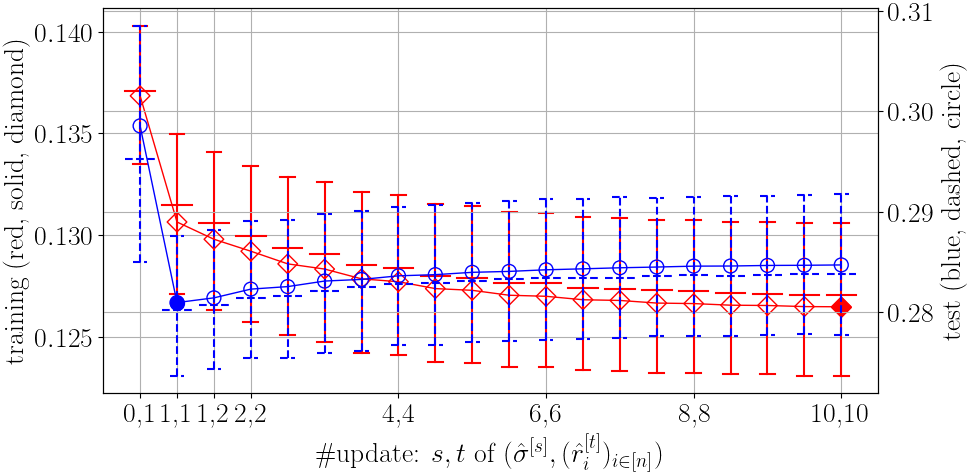}}&
\CF{\includegraphics[width=2.0cm]{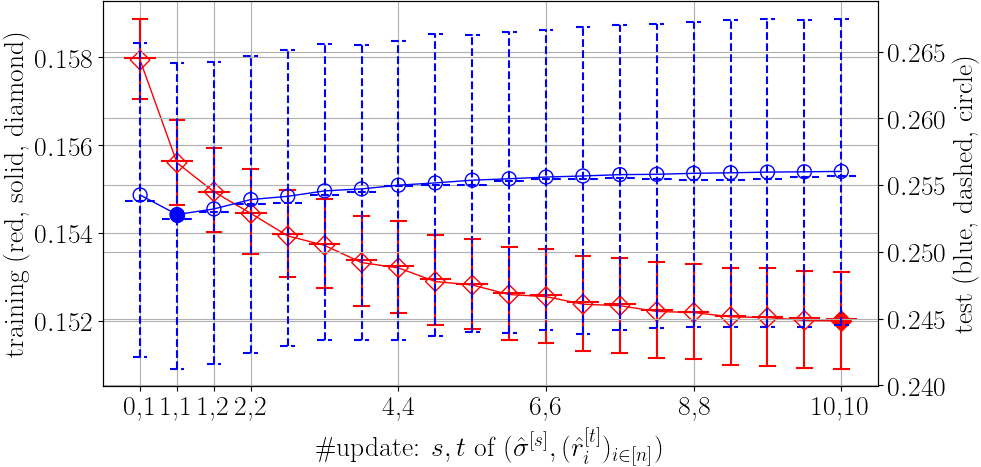}}
\\\midrule
\raisebox{1.75ex}{\rotatebox{90}{\tiny\eqref{eq:Kendall}}}&
{\includegraphics[width=2.0cm]{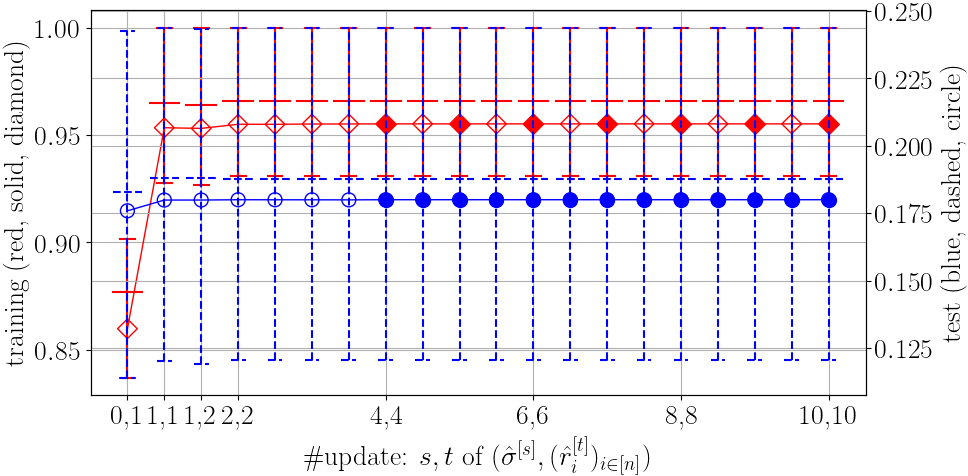}}&
\CF{\includegraphics[width=2.0cm]{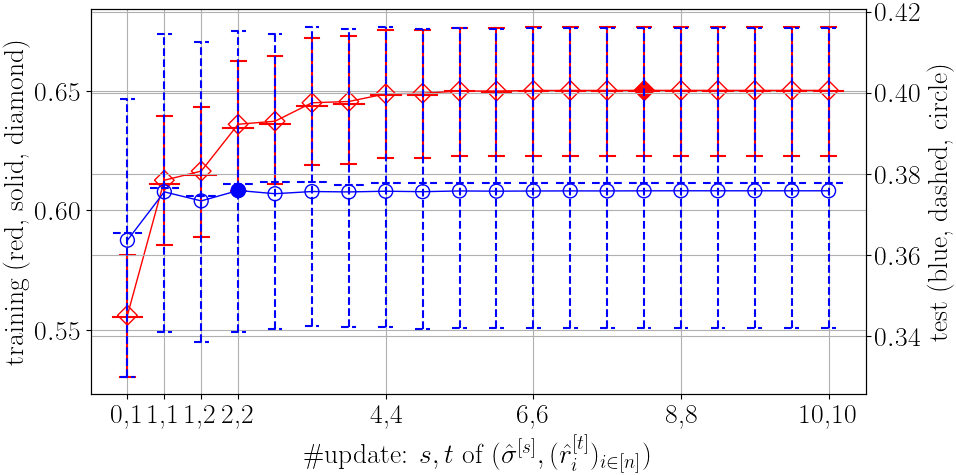}}&
\CF{\includegraphics[width=2.0cm]{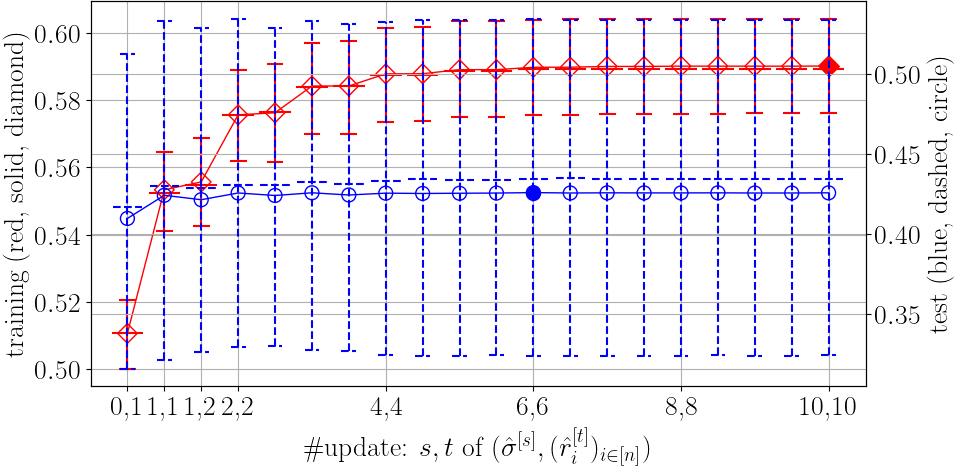}}&
{\includegraphics[width=2.0cm]{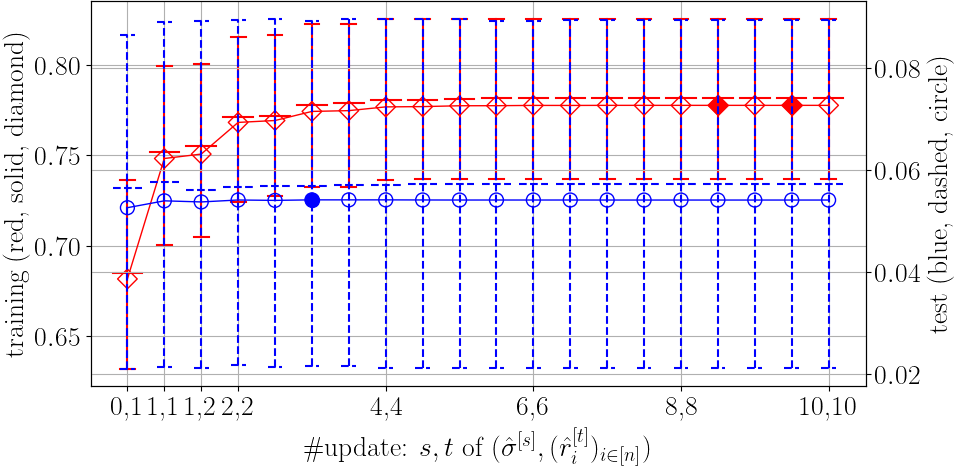}}&
\CF{\includegraphics[width=2.0cm]{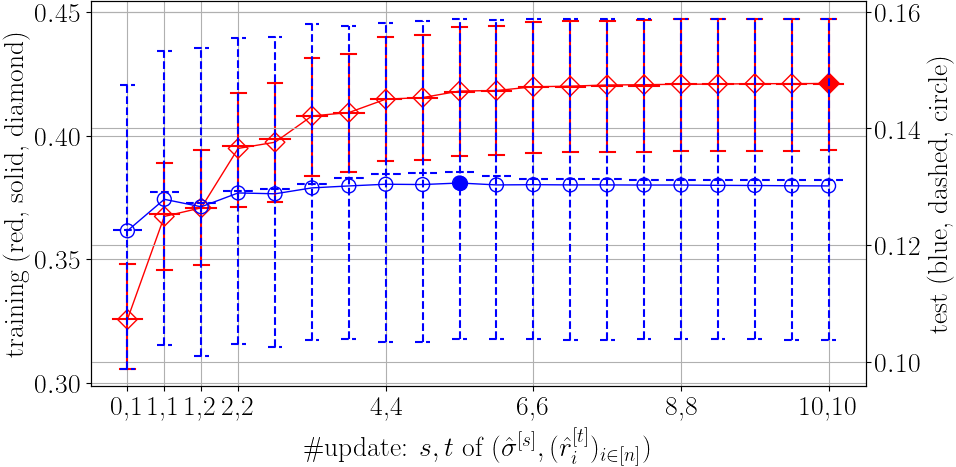}}&
\CF{\includegraphics[width=2.0cm]{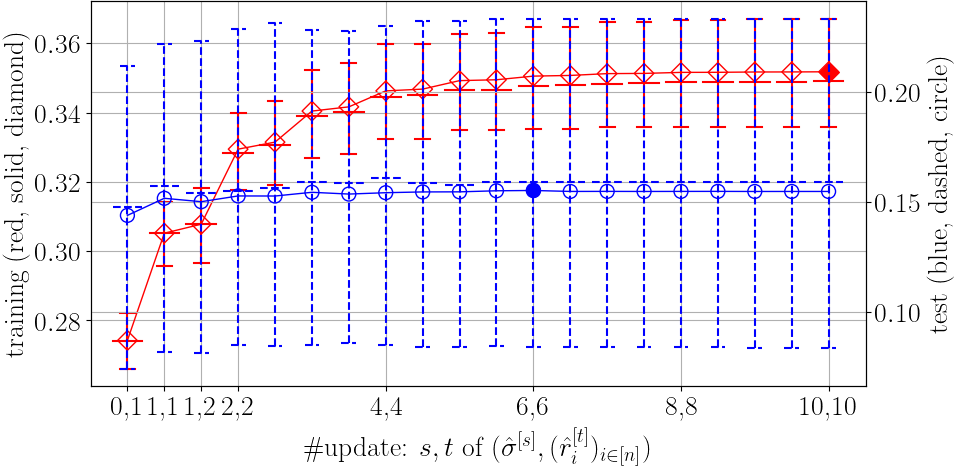}}&
\CF{\includegraphics[width=2.0cm]{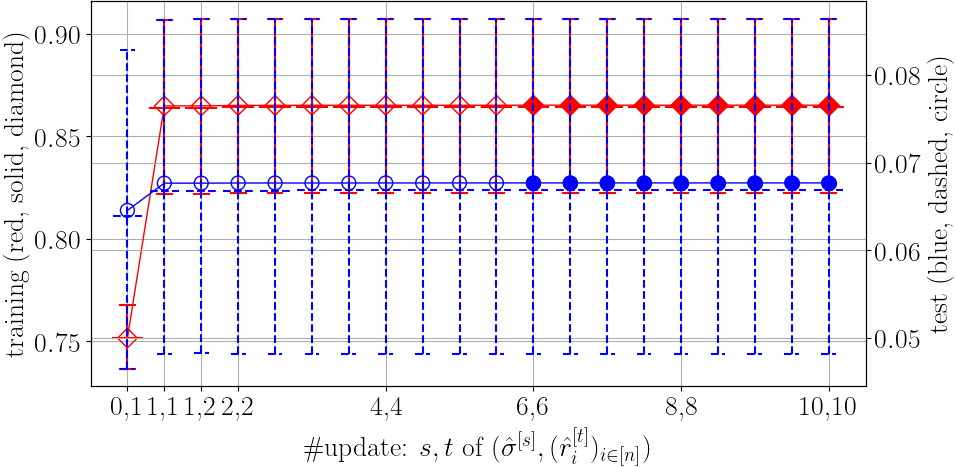}}&
\CF{\includegraphics[width=2.0cm]{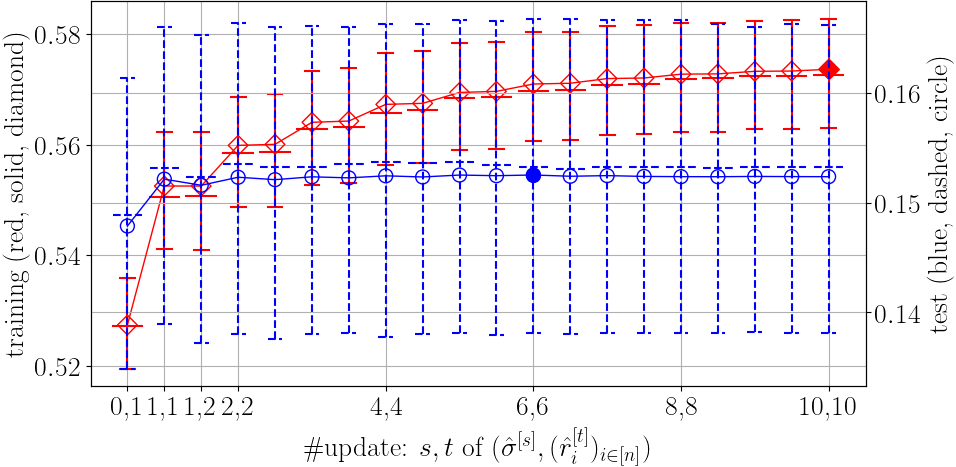}}&
\CF{\includegraphics[width=2.0cm]{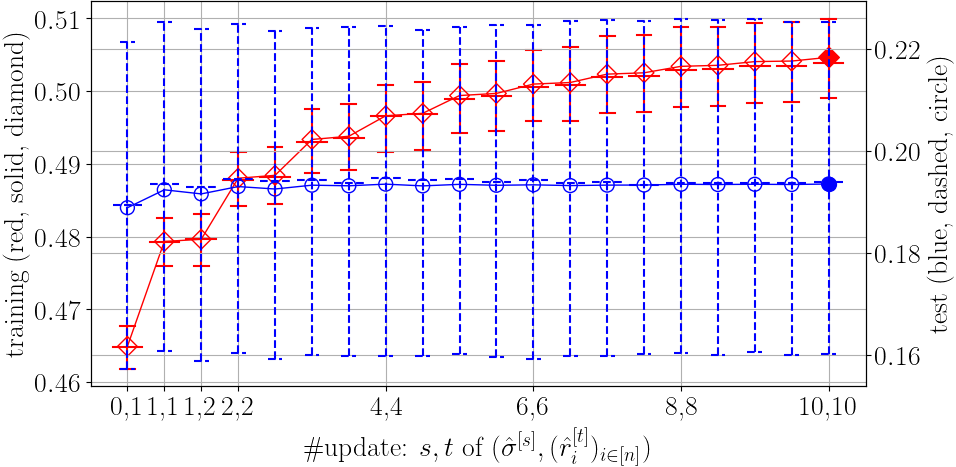}}
\\\midrule
\raisebox{1.75ex}{\rotatebox{90}{\tiny\eqref{eq:TIE}}}&
{\includegraphics[width=2.0cm]{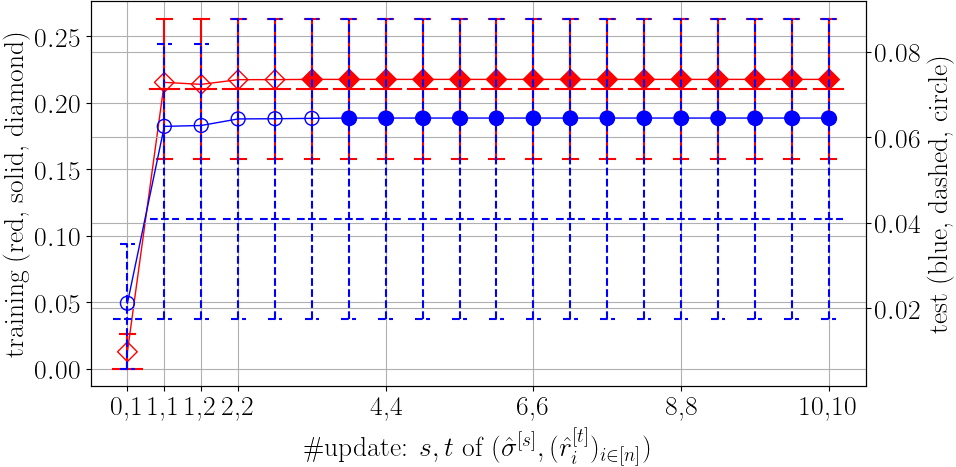}}&
{\includegraphics[width=2.0cm]{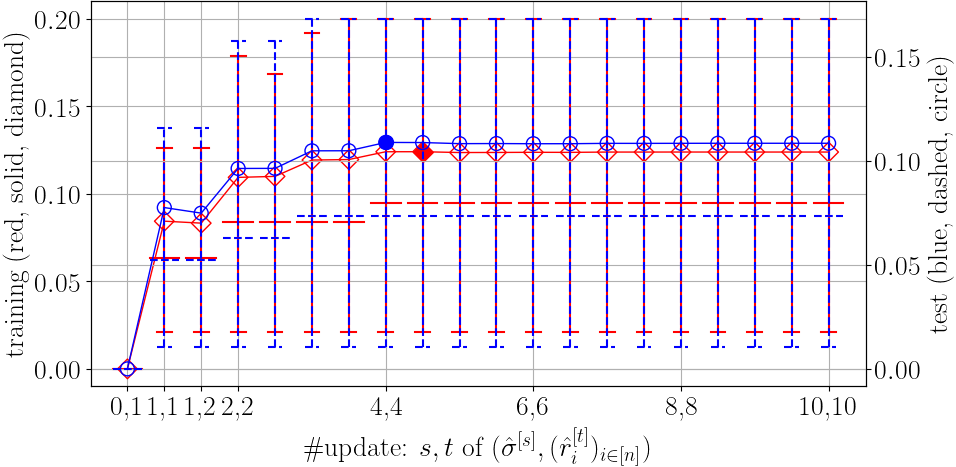}}&
{\includegraphics[width=2.0cm]{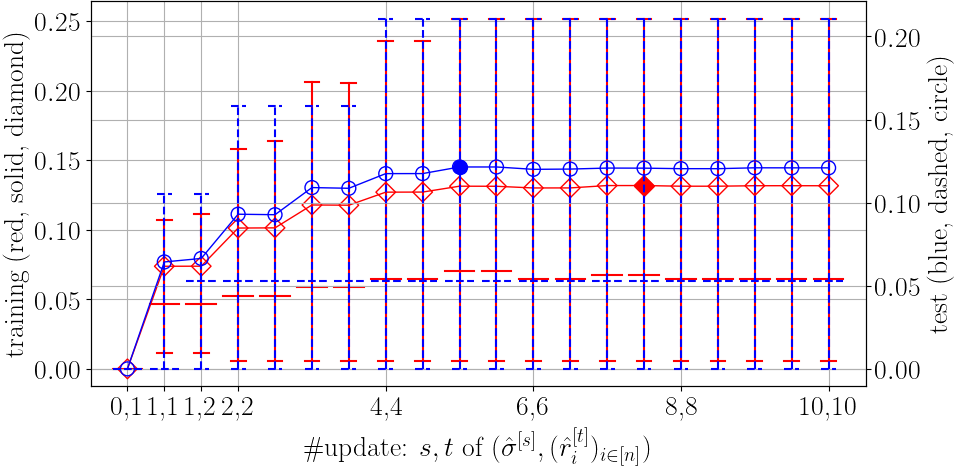}}&
{\includegraphics[width=2.0cm]{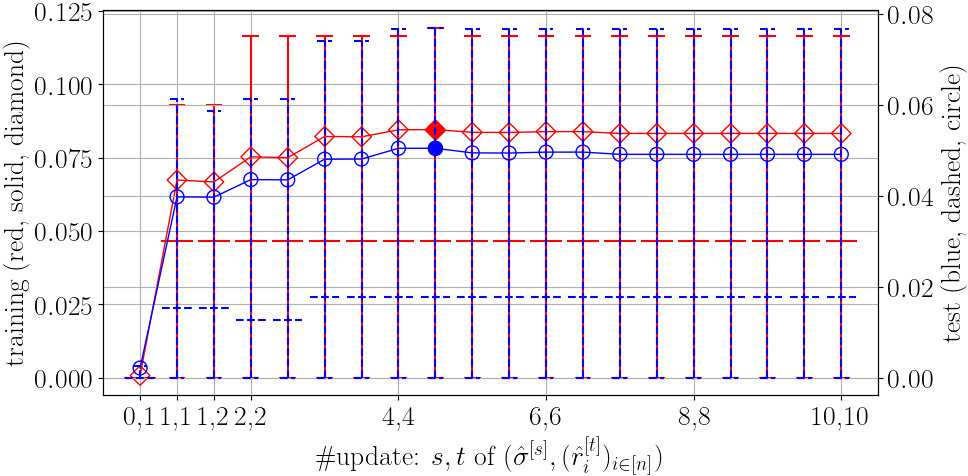}}&
{\includegraphics[width=2.0cm]{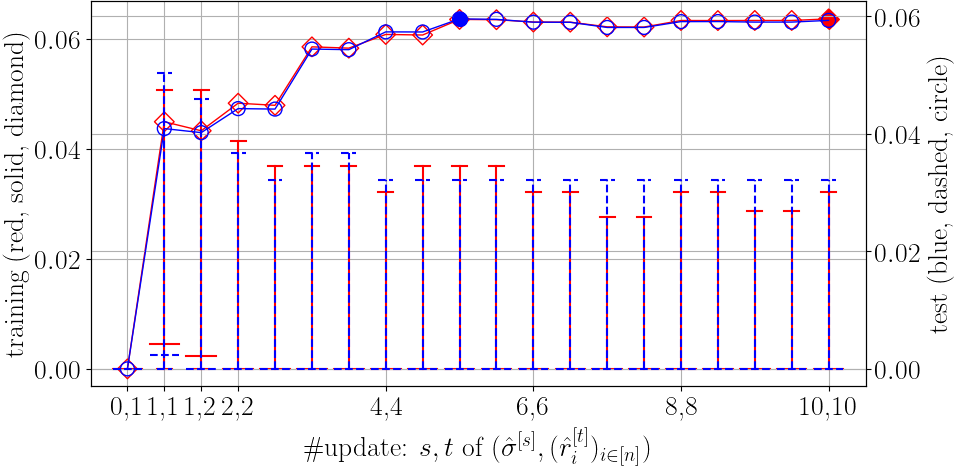}}&
{\includegraphics[width=2.0cm]{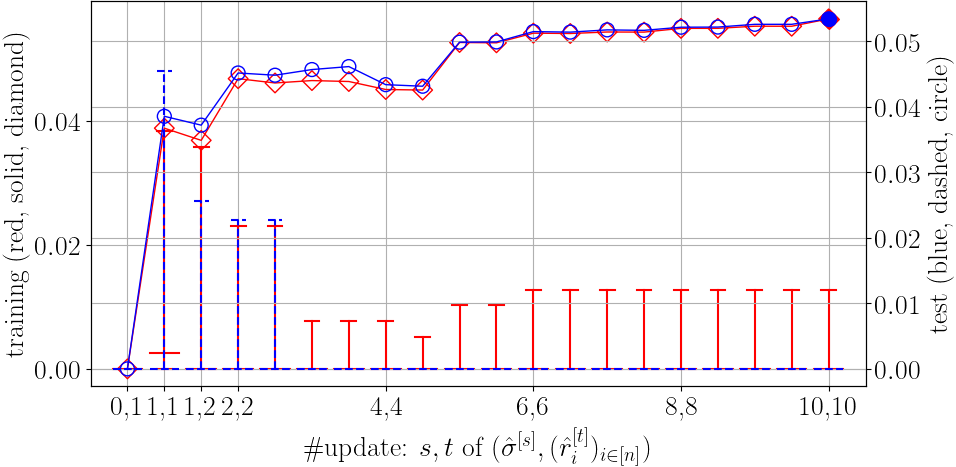}}&
{\includegraphics[width=2.0cm]{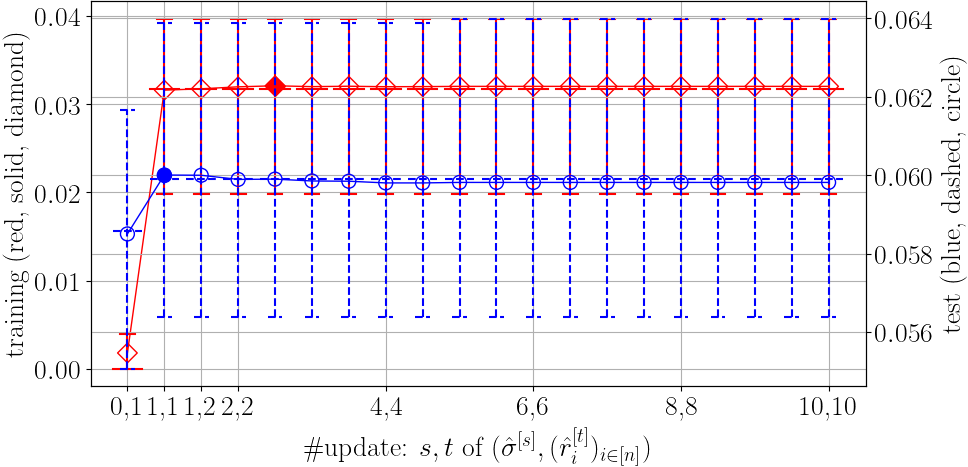}}&
{\includegraphics[width=2.0cm]{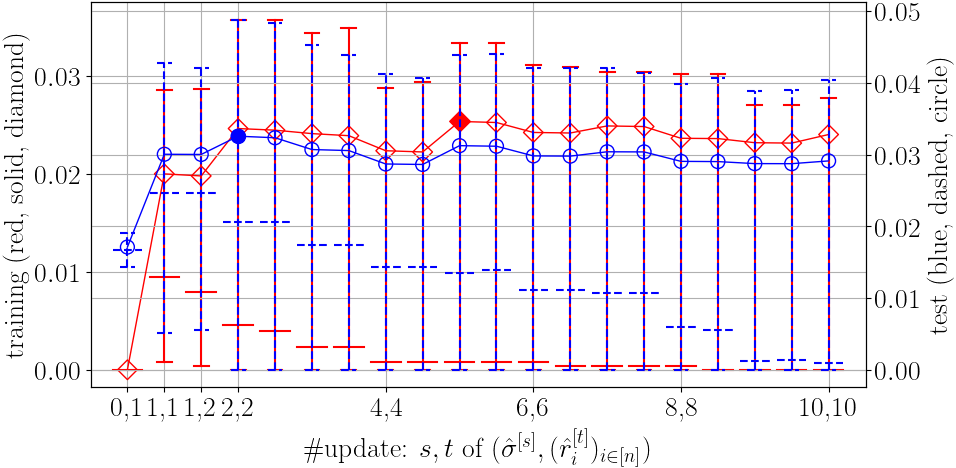}}&
{\includegraphics[width=2.0cm]{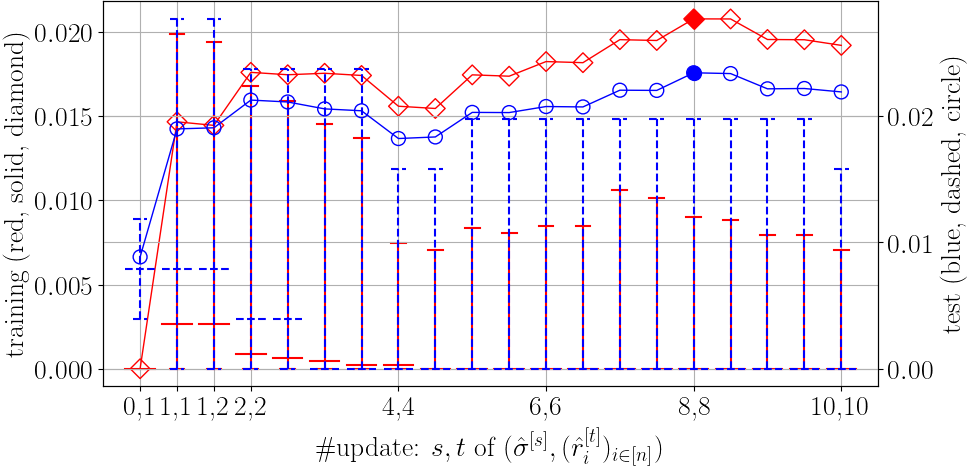}}
\end{tabular}
\caption{%
Part of results of real-world data experiments, Procedure 1 in Section~\ref{sec:RealWorld}:
For PL, MLB, and ATP data (left to right),
mean (marker) and 0.25, 0.5, and 0.75 quantiles (lower, middle, and upper bars) of 
1000 trial training (red, solid, diamond) and test (blue, dashed, circle) evaluation of 
WPP error \eqref{eq:WPE} with the squared loss $\phi=\phi_\sq$, 
Kendall's Tau \eqref{eq:Kendall}, and tie rate \eqref{eq:TIE} (top to bottom)
for the isotonic Bradley-Terry model learned with the squared loss $\phi=\phi_\sq$.
Smaller \eqref{eq:WPE}, or larger \eqref{eq:Kendall} indicates a better model.
The marker for the best model and model with the most ties was filled in, 
and the outer frame of the figure was highlighted in gray
if the best model was significantly better than the Bradley-Terry model
with respect to Mann-Whitney U test of the significance level 0.05.}
\label{fig:Real-SQ}
\centering%
\renewcommand{\arraystretch}{0.5}%
\renewcommand{\tabcolsep}{0.5pt}%
\begin{tabular}{ccc|ccc|ccc}%
\multicolumn{9}{c}{~}\\
\multicolumn{9}{c}{~}\\
\multicolumn{3}{c|}{\tiny PL, $|D_\tra|:|D_\tes|=$}&\multicolumn{3}{c|}{\tiny MLB, $|D_\tra|:|D_\tes|=$}&\multicolumn{3}{c}{\tiny ATP, $|D_\tra|:|D_\tes|=$}\\
{\tiny$1:9$}&{\tiny$5:5$}&{\tiny$9:1$}&{\tiny$1:9$}&{\tiny$5:5$}&{\tiny$9:1$}&{\tiny$1:9$}&{\tiny$5:5$}&{\tiny$9:1$}\\
\midrule
{\includegraphics[width=2.0cm]{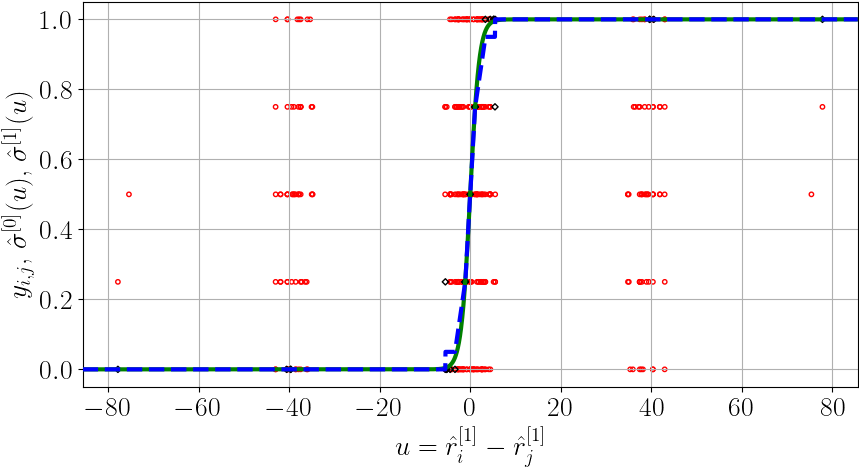}}&
{\includegraphics[width=2.0cm]{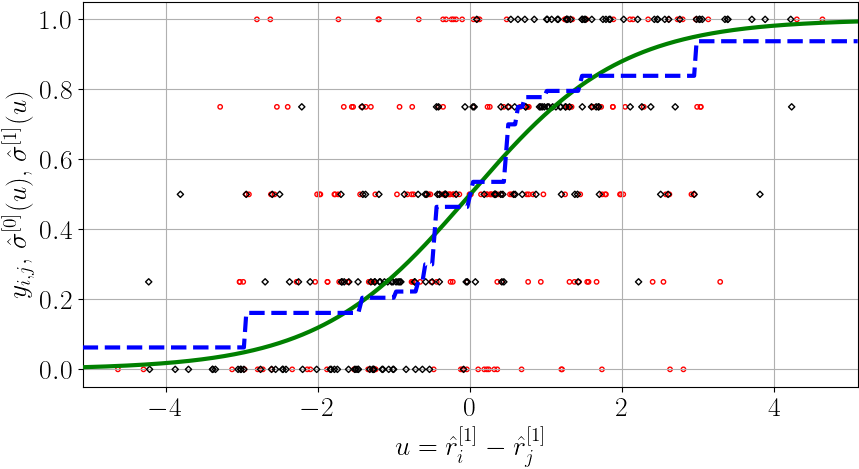}}&
{\includegraphics[width=2.0cm]{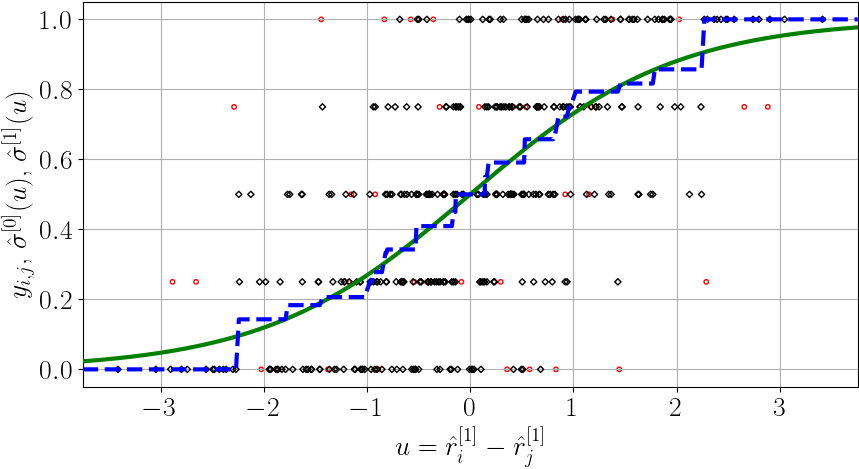}}&
{\includegraphics[width=2.0cm]{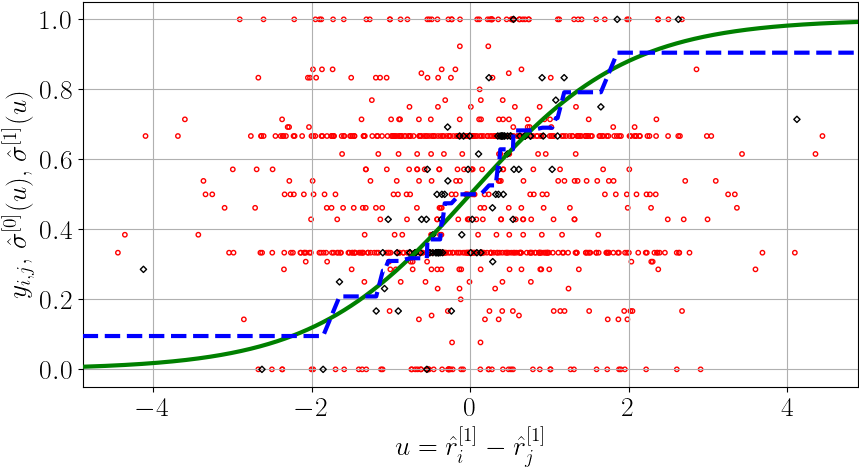}}&
{\includegraphics[width=2.0cm]{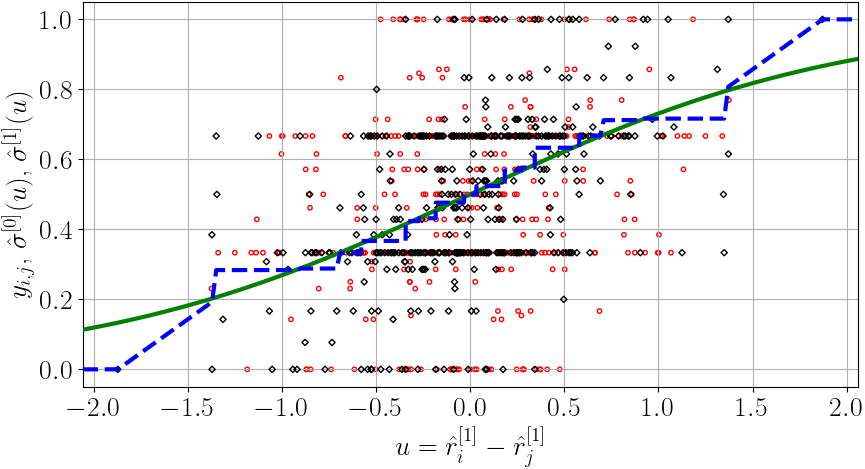}}&
{\includegraphics[width=2.0cm]{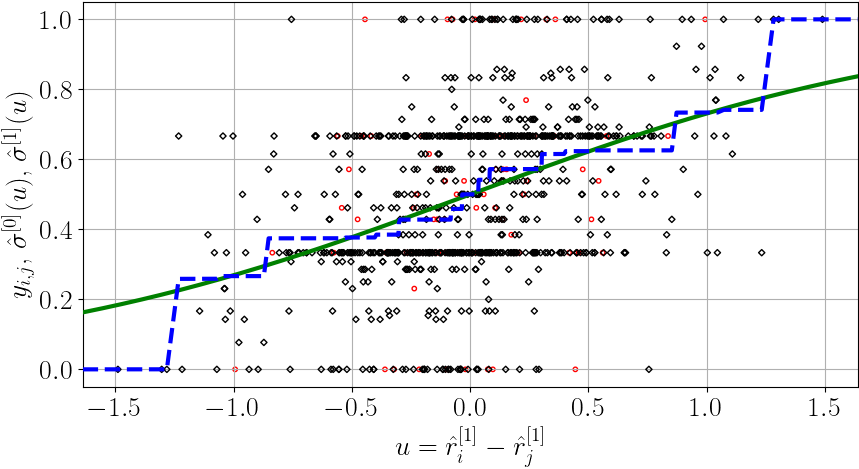}}&
{\includegraphics[width=2.0cm]{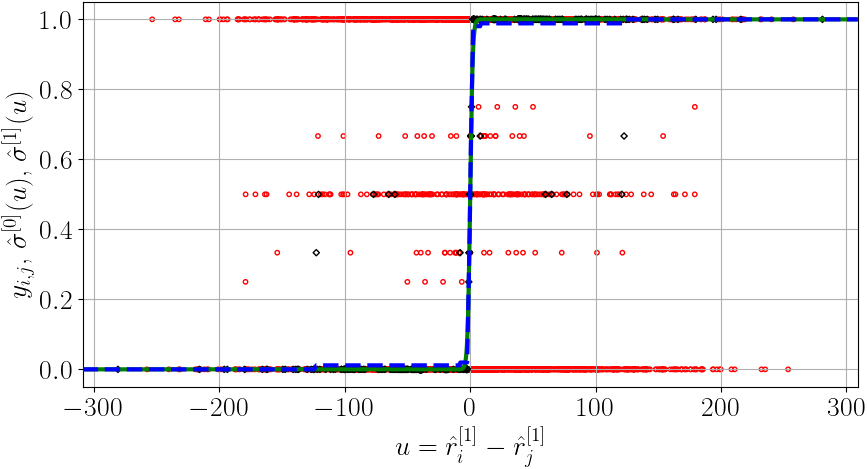}}&
{\includegraphics[width=2.0cm]{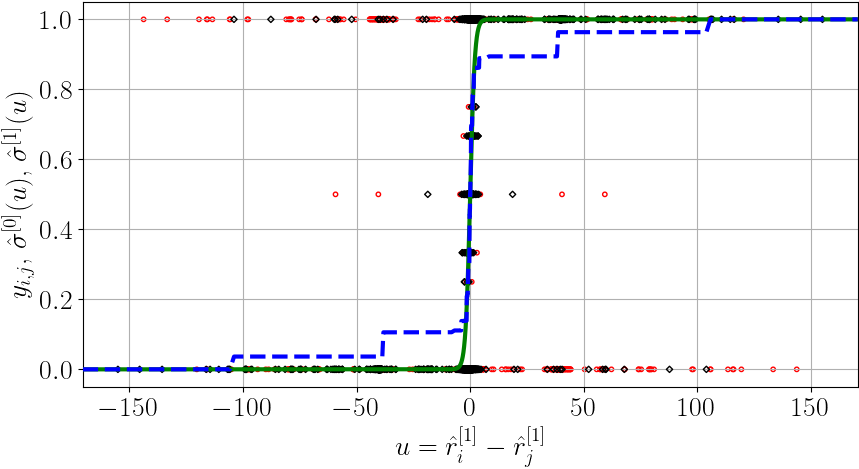}}&
{\includegraphics[width=2.0cm]{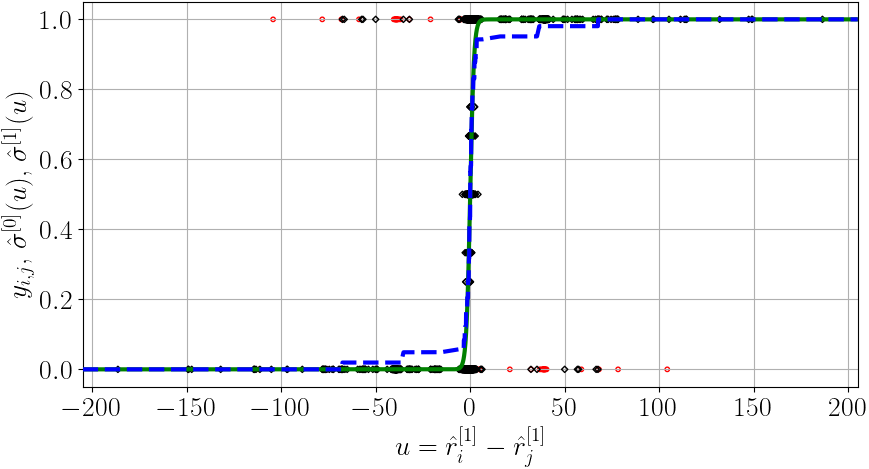}}
\end{tabular}
\begin{tabular}{cccc}%
\multicolumn{4}{c}{\tiny Rescaled version: Data, $|D_\tra|:|D_\tes|=$}\\
{\tiny PL, 1:9}&{\tiny ATP, 1:9}&{\tiny ATP, 5:5}&{\tiny ATP, 9:1}\\
\midrule
{\includegraphics[width=2.0cm]{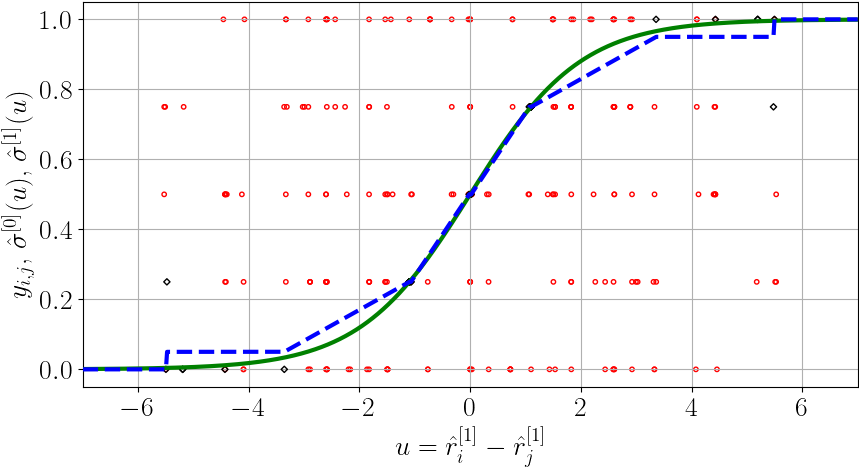}}&
{\includegraphics[width=2.0cm]{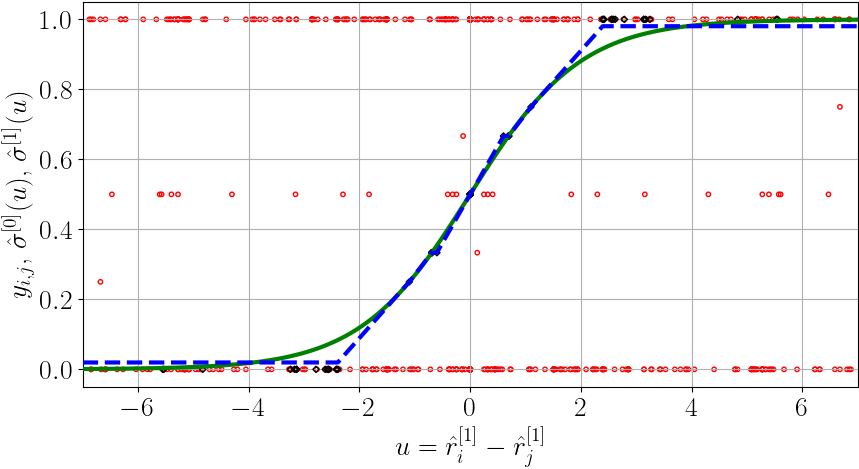}}&
{\includegraphics[width=2.0cm]{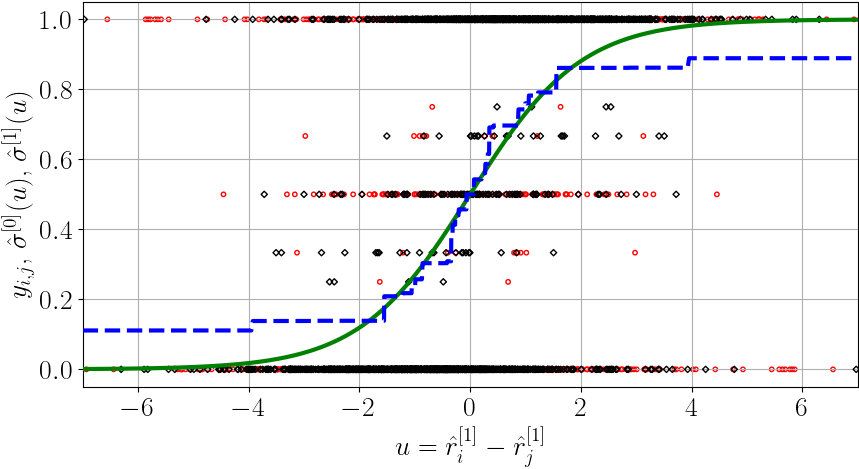}}&
{\includegraphics[width=2.0cm]{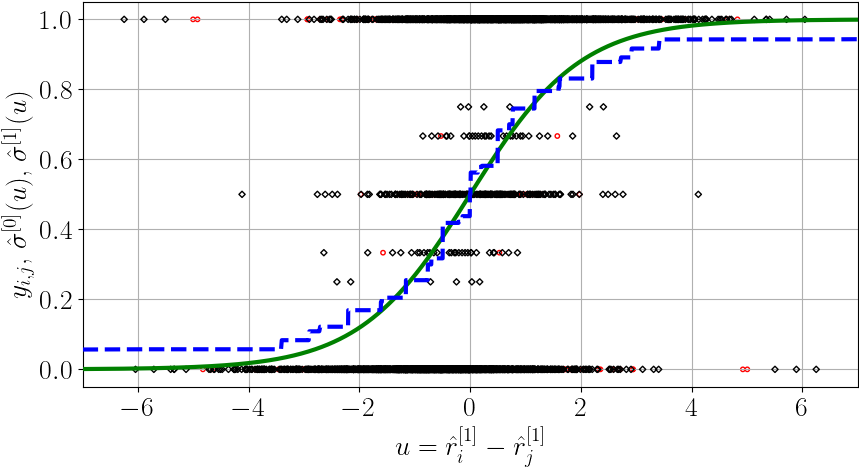}}
\end{tabular}
\caption{%
Part of results of real-world data experiments, Procedure 1 in Section~\ref{sec:RealWorld}:
For PL, MLB, and ATP data (left to right),
black diamonds and red circles are training and test data $y_{i,j}$, 
and a green curve and a blue polyline are
the Bradley-Terry model $\hat{\sigma}^{[0]}(\hat{r}_i^{[1]}-\hat{r}_j^{[1]})$ and
isotonic Bradley-Terry model $\hat{\sigma}^{[1]}(\hat{r}_i^{[1]}-\hat{r}_j^{[1]})$
learned with the squared loss $\phi=\phi_\sq$ in a certain trial.
over the range $[-1.1\cdot\max_{i,j}|\hat{r}_i^{[1]}-\hat{r}_j^{[1]}|,1.1\cdot\max_{i,j}|\hat{r}_i^{[1]}-\hat{r}_j^{[1]}|]$
or $[-7,7]$ in the rescaled version.}
\label{fig:Real-Res-Cauchy-SQ}
\end{sidewaysfigure}
\begin{sidewaystable}
\centering%
\renewcommand{\arraystretch}{0.5}%
\renewcommand{\tabcolsep}{0.5pt}%
\caption{%
Results of real-world data experiments, Procedure 2 in Section~\ref{sec:RealWorld}:
For PL, MLB, and ATP data (left to right),
mean and std ($\text{mean}_{\text{std}}$) of 1000 trial test evaluation of 
WPP error \eqref{eq:WPE} with the squared loss $\phi=\phi_\sq$
and Kendall's Tau \eqref{eq:Kendall}
for the Bradley-Terry model (upper) and model-selected isotonic 
Bradley-Terry model learned with the squared loss $\phi=\phi_\sq$ (lower).
Smaller \eqref{eq:WPE}, or larger \eqref{eq:Kendall} indicates a better model.
It also includes $p$-value for the Mann-Whitney U test in the parentheses:
A value below the significance level 0.05 implies 
that the isotonic Bradley-Terry model performs significantly better
and is highlighted in red bold.}
\label{tab:Real-SQ}
\scalebox{0.9}{\begin{minipage}{25cm}
\begin{tabular}{c|ccccc|ccccc|ccccc}%
&\multicolumn{5}{c|}{\tiny PL, $|D_\tra|:|D_\tes|=$}&\multicolumn{5}{c|}{\tiny MLB, $|D_\tra|:|D_\tes|=$}&\multicolumn{5}{c}{\tiny ATP, $|D_\tra|:|D_\tes|=$}\\
&{\tiny$1:9$}&{\tiny$3:7$}&{\tiny$5:5$}&{\tiny$7:3$}&{\tiny$9:1$}&{\tiny$1:9$}&{\tiny$3:7$}&{\tiny$5:5$}&{\tiny$7:3$}&{\tiny$9:1$}&{\tiny$1:9$}&{\tiny$3:7$}&{\tiny$5:5$}&{\tiny$7:3$}&{\tiny$9:1$}
\\\midrule
\multirow{3}{*}[-1.375mm]{\rotatebox{90}{\tiny\eqref{eq:WPE}}}
&\tcr{$\bf.2186_{.0420}$}&$.1336_{.0274}$&$.1012_{.0136}$&$.0932_{.0153}$&$.0881_{.0274}$
&\tcr{$\bf.1407_{.0285}$}&$.0765_{.0062}$&$.0683_{.0050}$&$.0654_{.0068}$&$.0641_{.0131}$
&\tcr{$\bf.4267_{.0139}$}&\tcr{$\bf.4015_{.0139}$}&\tcr{$\bf.2986_{.0201}$}&\tcr{$\bf.2673_{.0115}$}&$.2543_{.0177}$\\
&\tcr{$\bf.2125_{.0403}$}&$.1310_{.0224}$&$.1027_{.0138}$&$.0942_{.0156}$&$.0889_{.0276}$
&\tcr{$\bf.1359_{.0235}$}&$.0770_{.0064}$&$.0686_{.0051}$&$.0657_{.0069}$&$.0643_{.0131}$
&\tcr{$\bf.4141_{.0143}$}&\tcr{$\bf.3396_{.0115}$}&\tcr{$\bf.2813_{.0107}$}&\tcr{$\bf.2628_{.0099}$}&$.2529_{.0171}$\\
&(\tcr{$\bf.0004$})&($.1613$)&($.9892$)&($.9223$)&($.7480$)
&(\tcr{$\bf.0010$})&($.9485$)&($.9023$)&($.8122$)&($.6737$)
&(\tcr{$\bf.0000$})&(\tcr{$\bf.0000$})&(\tcr{$\bf.0000$})&(\tcr{$\bf.0000$})&($.0593$)\\
\midrule
\multirow{3}{*}[-1.375mm]{\rotatebox{90}{\tiny\eqref{eq:Kendall}}}
&$.1759_{.0914}$&\tcr{$\bf.3048_{.0643}$}&\tcr{$\bf.3637_{.0512}$}&\tcr{$\bf.3872_{.0701}$}&\tcr{$\bf.4096_{.1441}$}
&$.0526_{.0484}$&\tcr{$\bf.1017_{.0341}$}&\tcr{$\bf.1224_{.0356}$}&\tcr{$\bf.1358_{.0507}$}&\tcr{$\bf.1438_{.0991}$}
&$.0646_{.0266}$&\tcr{$\bf.1095_{.0209}$}&\tcr{$\bf.1479_{.0204}$}&\tcr{$\bf.1720_{.0240}$}&\tcr{$\bf.1889_{.0470}$}\\
&$.1800_{.0939}$&\tcr{$\bf.3127_{.0671}$}&\tcr{$\bf.3745_{.0541}$}&\tcr{$\bf.3989_{.0740}$}&\tcr{$\bf.4243_{.1499}$}
&$.0537_{.0495}$&\tcr{$\bf.1057_{.0373}$}&\tcr{$\bf.1283_{.0396}$}&\tcr{$\bf.1426_{.0566}$}&\tcr{$\bf.1527_{.1093}$}
&$.0664_{.0273}$&\tcr{$\bf.1152_{.0226}$}&\tcr{$\bf.1521_{.0209}$}&\tcr{$\bf.1760_{.0250}$}&\tcr{$\bf.1932_{.0489}$}\\
&($.1398$)&(\tcr{$\bf.0016$})&(\tcr{$\bf.0000$})&(\tcr{$\bf.0002$})&(\tcr{$\bf.0149$})
&($.3086$)&(\tcr{$\bf.0049$})&(\tcr{$\bf.0003$})&(\tcr{$\bf.0015$})&(\tcr{$\bf.0182$})
&($.0697$)&(\tcr{$\bf.0000$})&(\tcr{$\bf.0000$})&(\tcr{$\bf.0002$})&(\tcr{$\bf.0190$})\\
\end{tabular}\end{minipage}}
\end{sidewaystable}

\section{Conclusion and Future Works}
\label{sec:Conclusion}
In this study, because conventional usages of generalized Bradley-Terry model 
may suffer from model misspecification 
owing to the selection of a fixed inverse link function,
we proposed the isotonic Bradley-Terry model that 
learns the rate parameters by a (sub-)gradient method and 
the inverse link function by an isotonic regression technique alternately.
Through numerical experiments, we confirmed that 
the proposed model could improve the win probability prediction and ranking.

Generalized Bradley-Terry model is also a foundation for analysis of 
a variety of advanced prediction problems for paired comparison data.
When we have access to covariates that characterize each player, 
we can model the rate parameters in terms of those covariates
\citep{springall1973response,de1993thurstonian}.
We can consider home-field (or first-move) advantage by using different 
rate parameters for home and away settings \citep{david1963method}.
When paired comparisons are performed in prolonged time periods, 
it may be better to incorporate time-dependence 
into the model of the rate parameters;
for example, Elo rating system \citep{elo1978rating,glickman1995comprehensive}
is a representative time-dependent extension of the Bradley-Terry model.
A multi-stage expansion \citep{glenn1960ties,rao1967ties,davidson1970extending}, 
for example, for distinguishing draws, and Plackett-Luce model 
\citep{plackett1975analysis,luce1959individual,marden1996analyzing} for 
multi-player matches (e.g., horse racing) with 3 or more players are also closely related.
It is one promising direction for future research to apply the idea of 
the isotonic Bradley-Terry model that also learns the inverse link function 
for these advanced problems.

\backmatter
\bmhead{Acknowledgements}
This work was supported by JSPS KAKENHI Grant Number JP24K23856 and 
ISM Cooperative Research Program 2025-ISMCRP-1001 and 2026-ISMCRP-1006.
The author would like to express gratitude to the providers of the open-access datasets used in this study. 
The football Premier League data were sourced from football-data.co.uk maintained by Joseph Buchdahl, 
the baseball MLB data were from retrosheet.org by Tom Thress, and
the tennis ATP tour data were from the Github repository by Jeff Sackmann 
(all data were accessed on 14 May 2026).

\begin{appendices}
\section{Experiments with Specified Model}
\label{sec:Specified}
We here show the experimental results with specified model:
We generated synthetic data (Logistic-$N$) consisted of $n=25,50,100,200,400$ players and
win probabilities $y_{i,j}=1-y_{j,i}\sim\Binomial(N,\sigma_\Logistic(\tilde{r}_i-\tilde{r}_j))/N$
with the number of matches $N=1,5,25$, 
and underlying players'\;strengths $\tilde{r}_i\sim\Normal(0,1)$.
We performed additional experiments keeping all other settings identical to 
Procedures 1 and 2 of the synthetic data experiments in Section~\ref{sec:Synthetic}.
Figures~\ref{fig:Logistic-SQ} and \ref{fig:Res-Logistic-SQ}
and Table~\ref{tab:Logistic-SQ} are specified-model versions of 
Figures~\ref{fig:Cauchy-SQ} and \ref{fig:Res-Cauchy-SQ}
and Table~\ref{tab:Cauchy-SQ}.
The isotonic Bradley-Terry model improved 
the test WPP error with $\phi=\phi_\sq$ 
when $n$, $N$, and the ratio of $|D_\tra|$ were small
(not when $n$ or $N$ was large unlike cases with model misspecification),
and improved Kendall's Tau in most cases.

\section{Experiments with NLL Loss}
\label{sec:NLL}
We here show the experimental results with NLL loss $\phi=\phi_\nll$.
Figures~\ref{fig:Cauchy-NLL}--\ref{fig:Real-Res-Cauchy-NLL}
and Tables~\ref{tab:Cauchy-NLL} and Table~\ref{tab:Real-NLL} 
are NLL versions of 
Figures~\ref{fig:Cauchy-SQ}--\ref{fig:Real-Res-Cauchy-SQ}
and Tables~\ref{tab:Cauchy-SQ} and Table~\ref{tab:Real-SQ}.
The test WPP error with NLL loss $\phi=\phi_\nll$ frequently yields NaN values. 
Thus, even if rank tests indicate statistical significance, 
the isotonic Bradley-Terry model would still be considered unsuitable 
for win probability prediction in this setting.
On the other hand, ranking performance could be calculated in all cases, 
and it was confirmed that that performance was improved 
even with the isotonic Bradley-Terry model using NLL loss.

\begin{sidewaysfigure}
\centering%
\renewcommand{\arraystretch}{0.5}%
\renewcommand{\tabcolsep}{0.5pt}%
\begin{tabular}{cc|ccc|ccc|ccc}%
&&\multicolumn{3}{c|}{\tiny$N=1$, $|D_\tra|:|D_\tes|=$}&\multicolumn{3}{c|}{\tiny$N=5$, $|D_\tra|:|D_\tes|=$}&\multicolumn{3}{c}{\tiny$N=25$, $|D_\tra|:|D_\tes|=$}\\
&&{\tiny$1:9$}&{\tiny$5:5$}&{\tiny$9:1$}&{\tiny$1:9$}&{\tiny$5:5$}&{\tiny$9:1$}&{\tiny$1:9$}&{\tiny$5:5$}&{\tiny$9:1$}\\
\midrule
\multirow{3}{*}[-2.5mm]{\rotatebox{90}{\tiny\eqref{eq:WPE}, $n=$}}
&\rotatebox{90}{\tiny\,~~~\,$25$}&
\CF{\includegraphics[width=2.0cm]{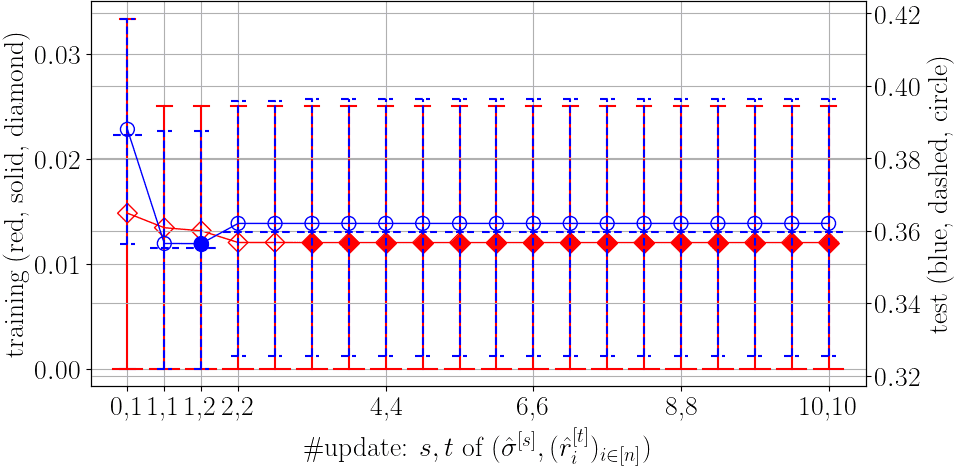}}&
\CF{\includegraphics[width=2.0cm]{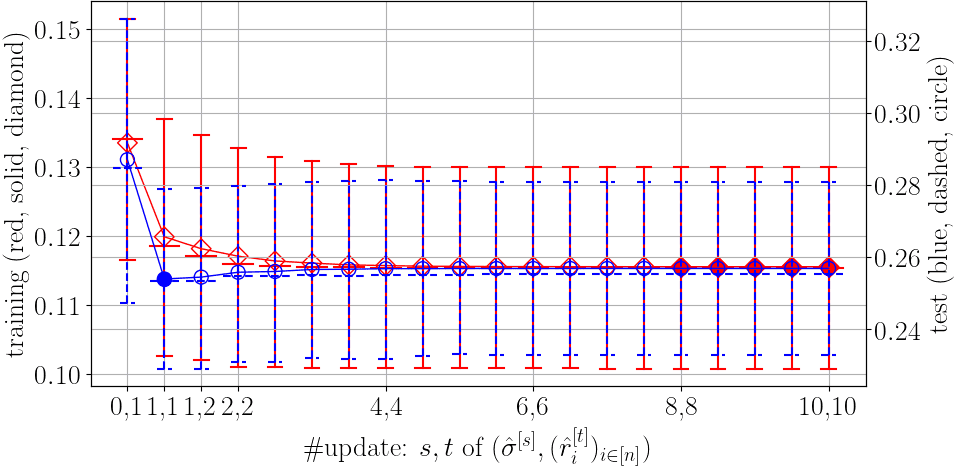}}&
{\includegraphics[width=2.0cm]{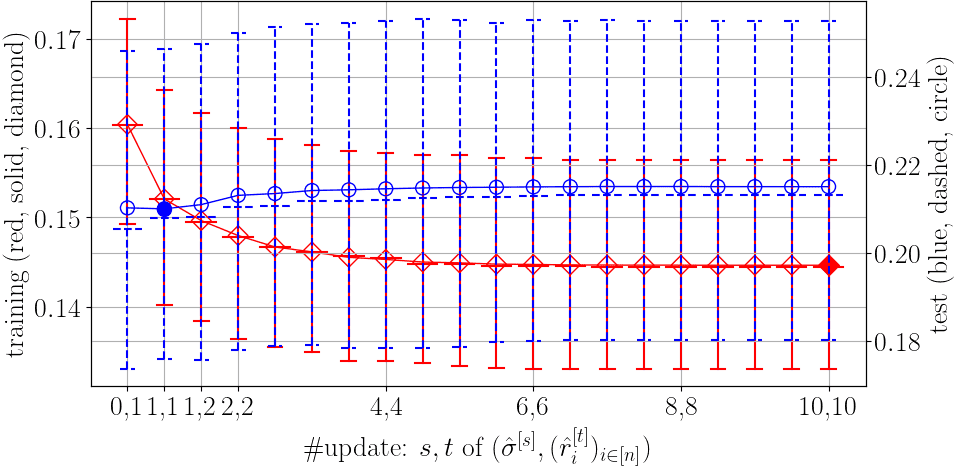}}&
{\includegraphics[width=2.0cm]{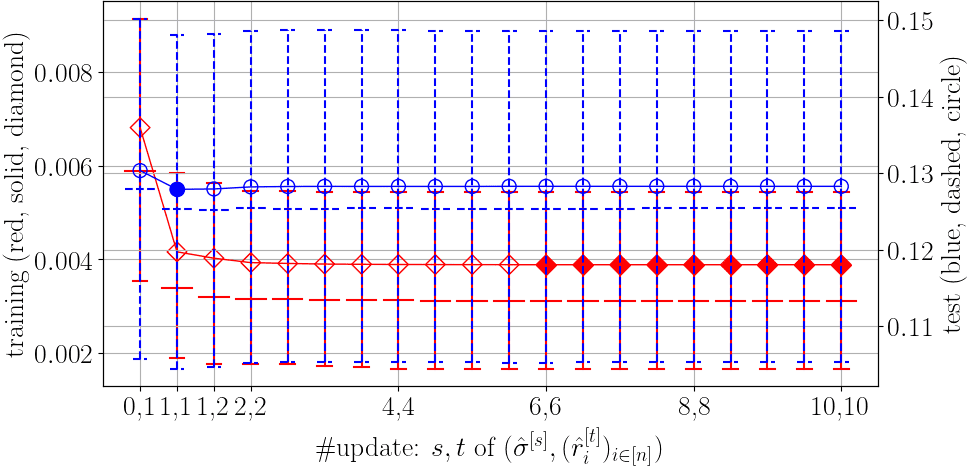}}&
{\includegraphics[width=2.0cm]{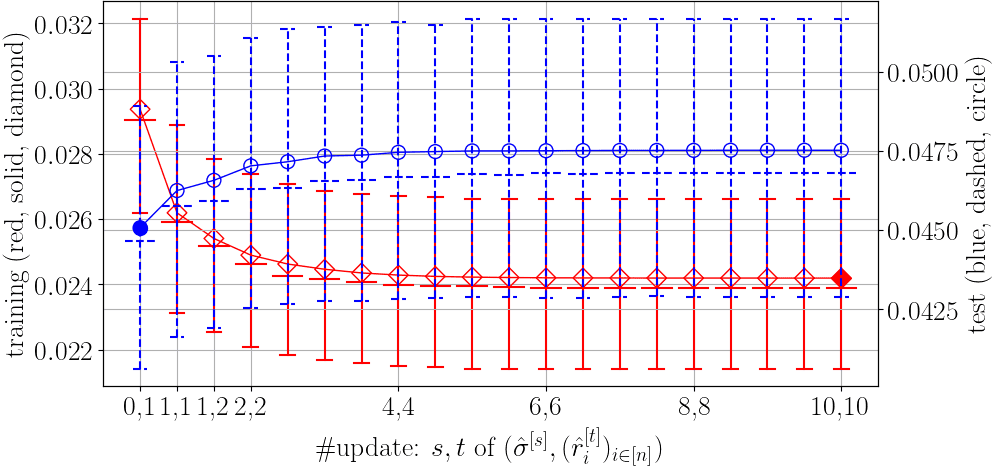}}&
{\includegraphics[width=2.0cm]{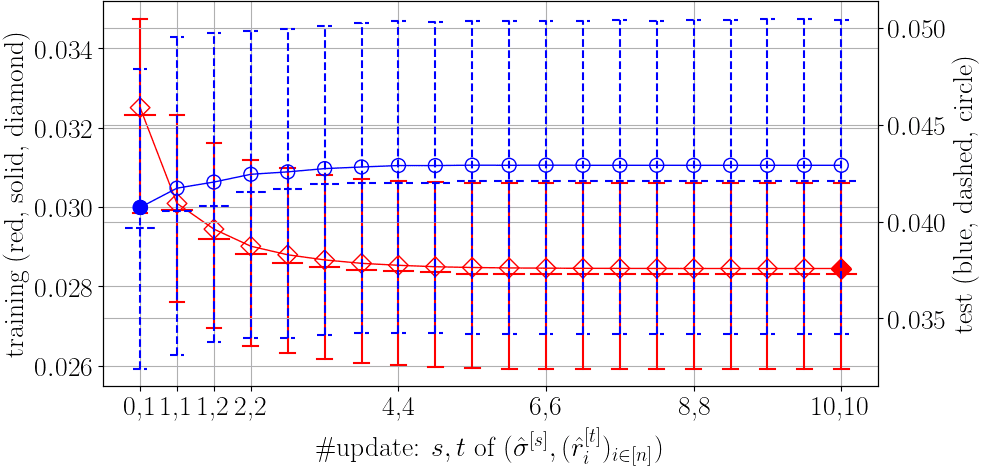}}&
{\includegraphics[width=2.0cm]{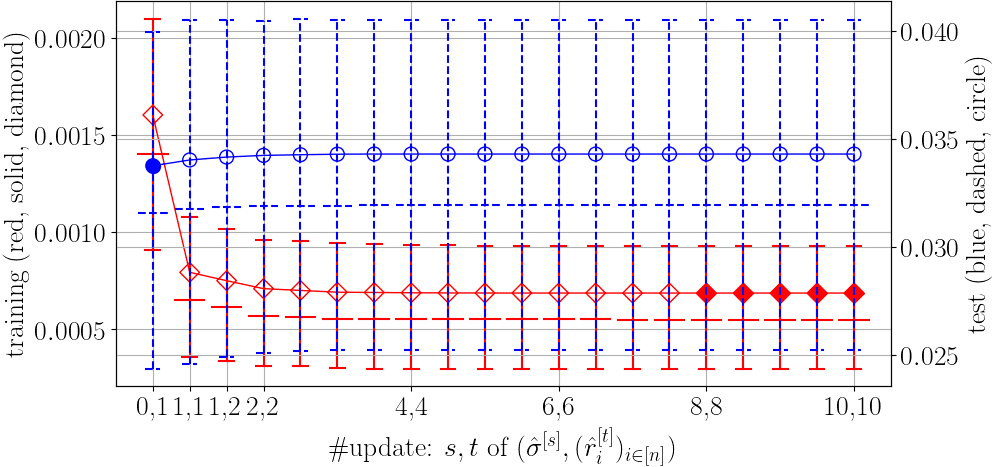}}&
{\includegraphics[width=2.0cm]{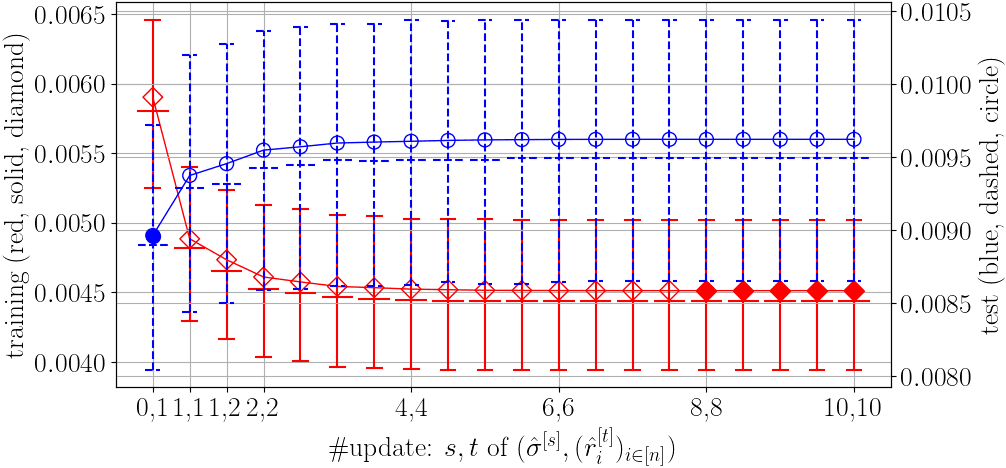}}&
{\includegraphics[width=2.0cm]{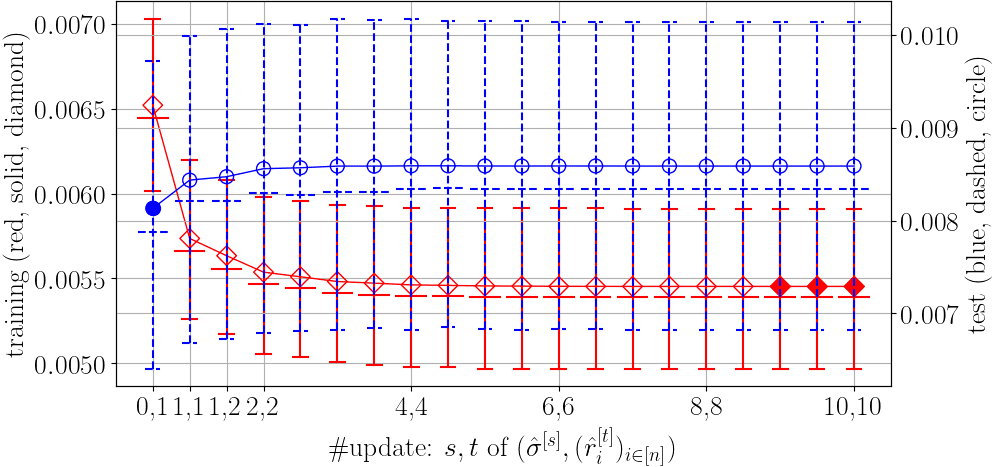}}\\
&\rotatebox{90}{\tiny\,~~\,$100$}&
\CF{\includegraphics[width=2.0cm]{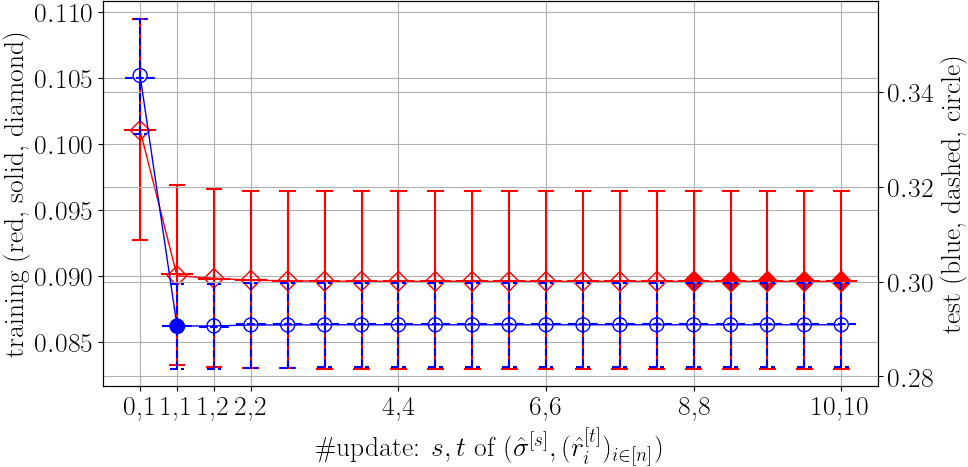}}&
{\includegraphics[width=2.0cm]{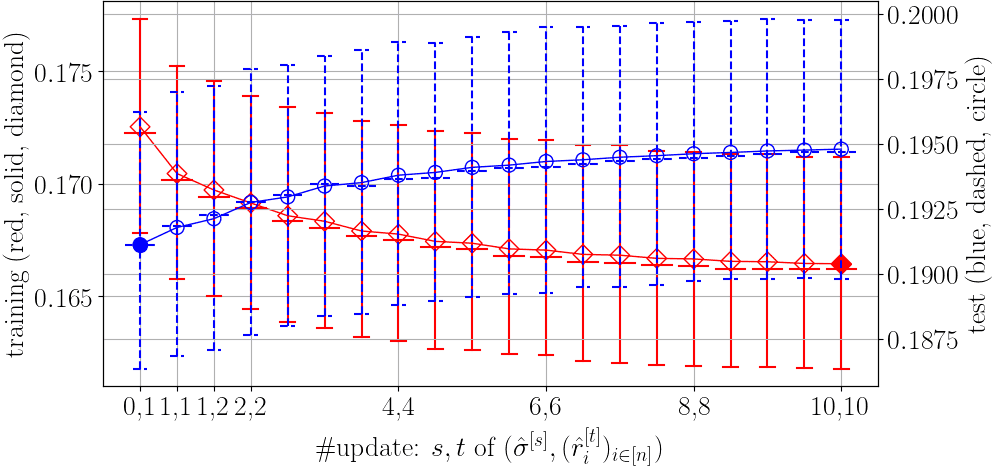}}&
{\includegraphics[width=2.0cm]{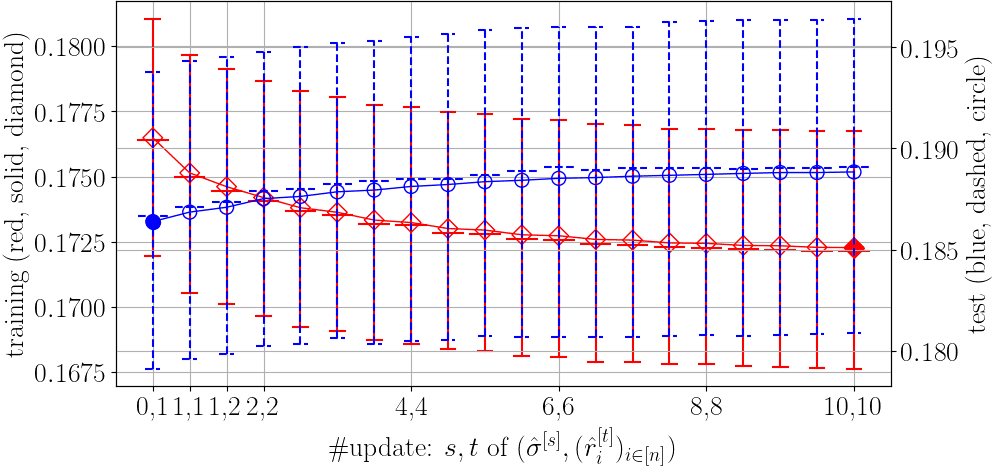}}&
{\includegraphics[width=2.0cm]{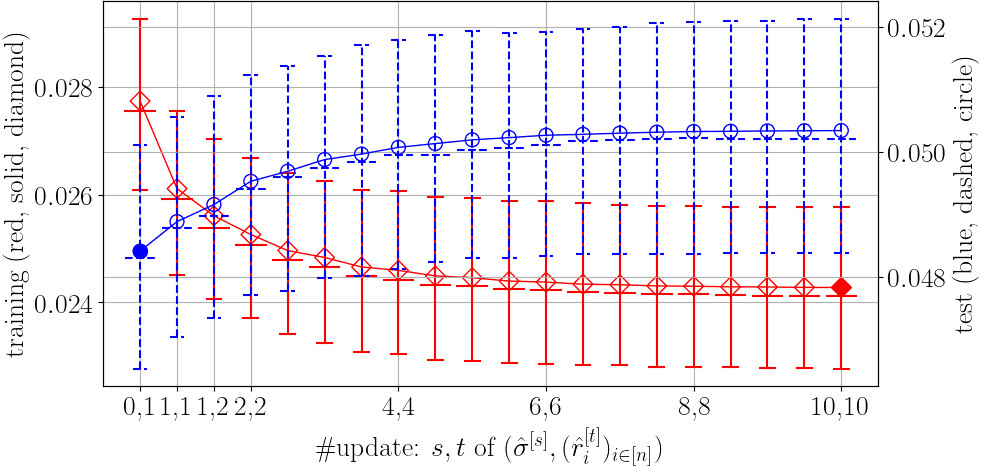}}&
{\includegraphics[width=2.0cm]{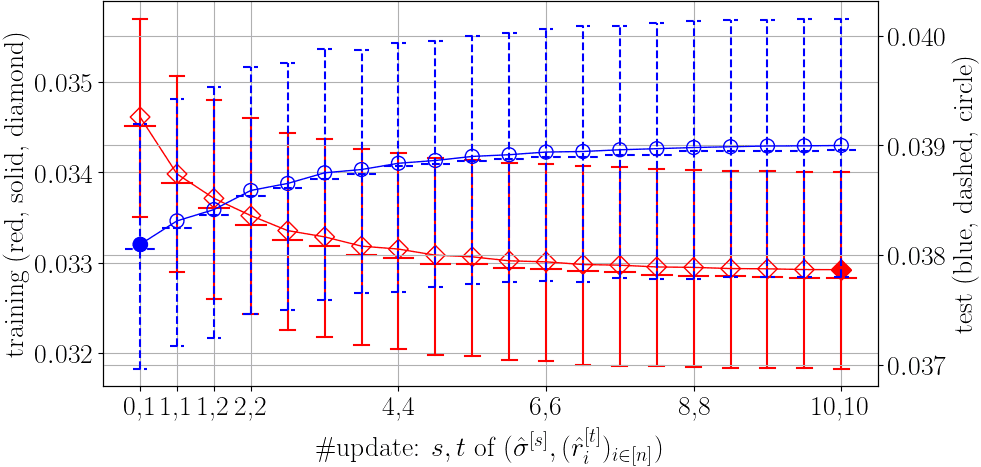}}&
{\includegraphics[width=2.0cm]{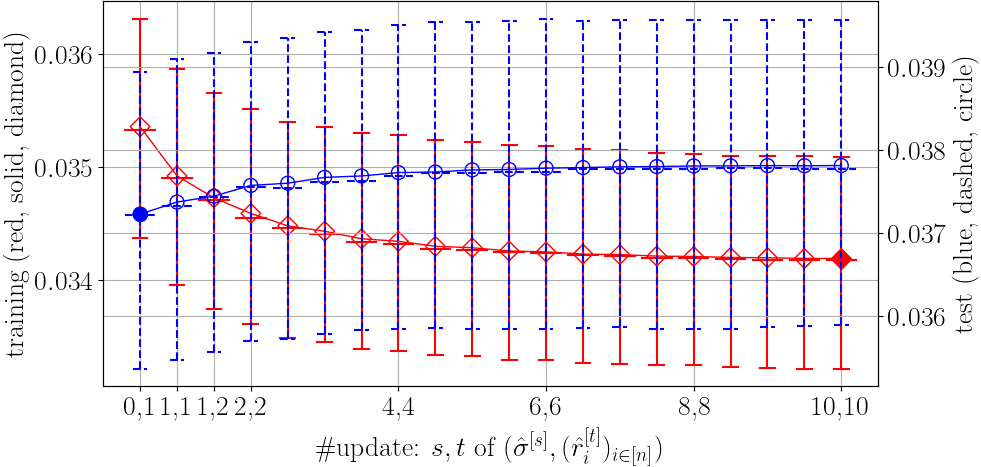}}&
{\includegraphics[width=2.0cm]{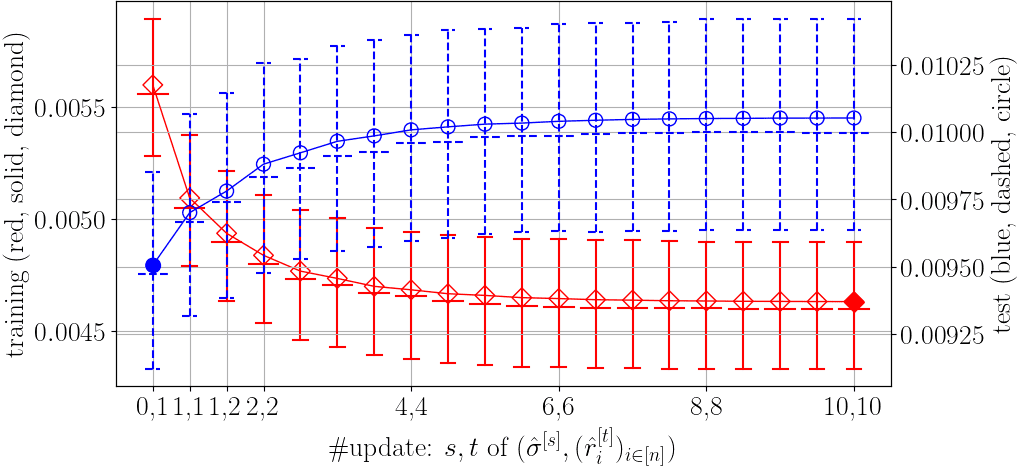}}&
{\includegraphics[width=2.0cm]{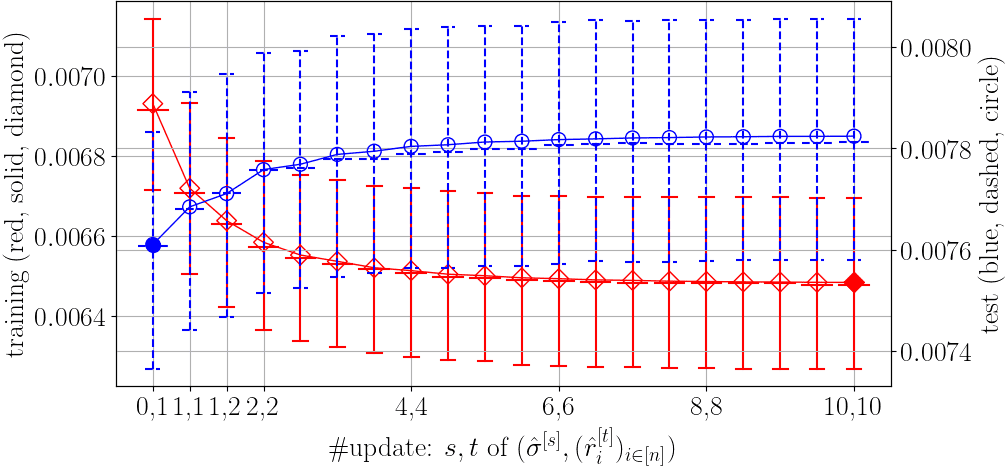}}&
{\includegraphics[width=2.0cm]{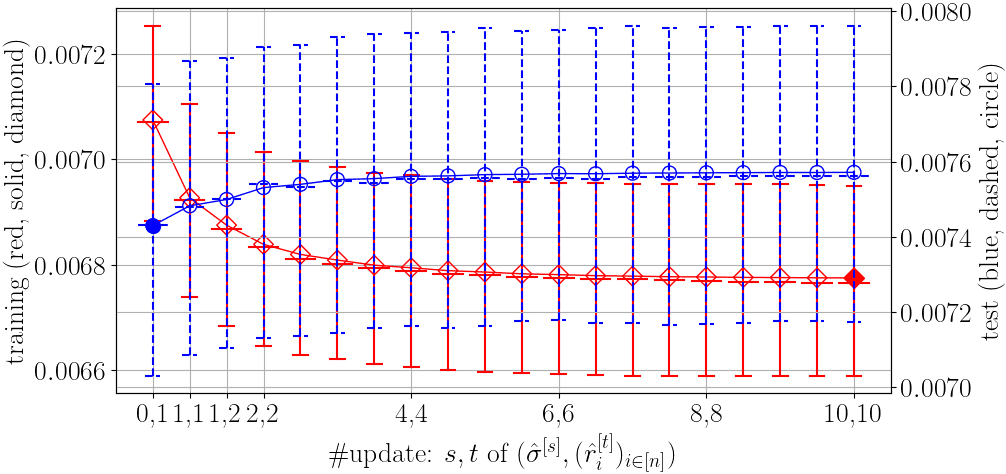}}\\
&\rotatebox{90}{\tiny\,~~\,$400$}&
{\includegraphics[width=2.0cm]{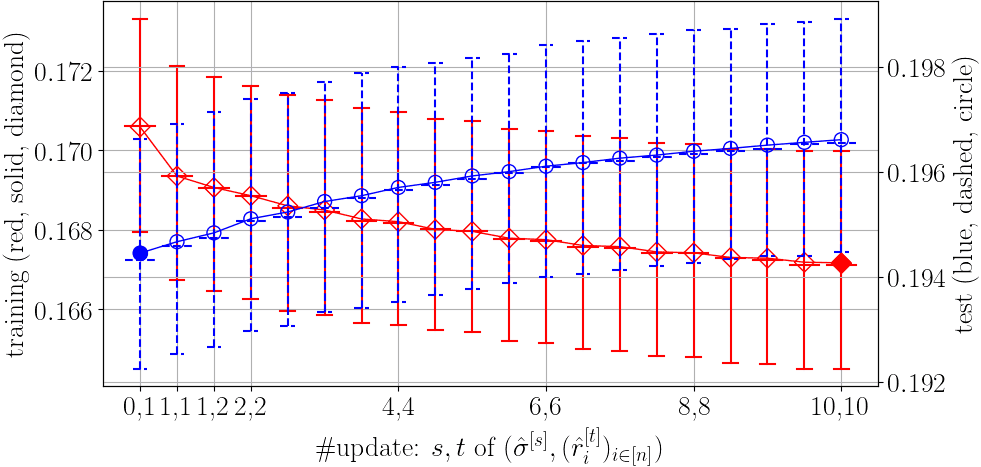}}&
{\includegraphics[width=2.0cm]{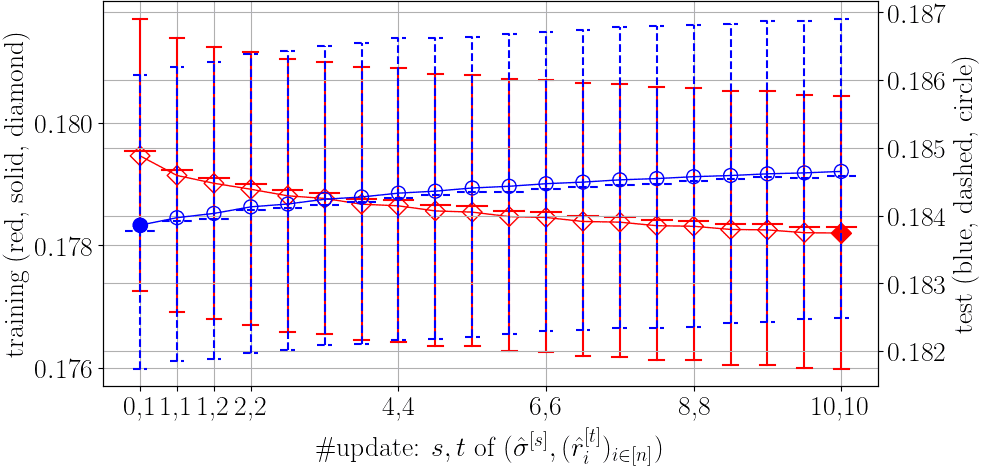}}&
{\includegraphics[width=2.0cm]{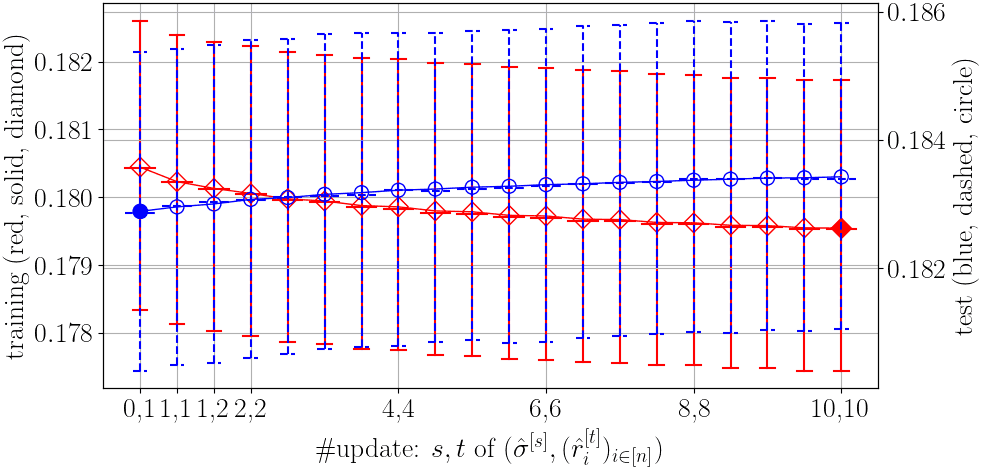}}&
{\includegraphics[width=2.0cm]{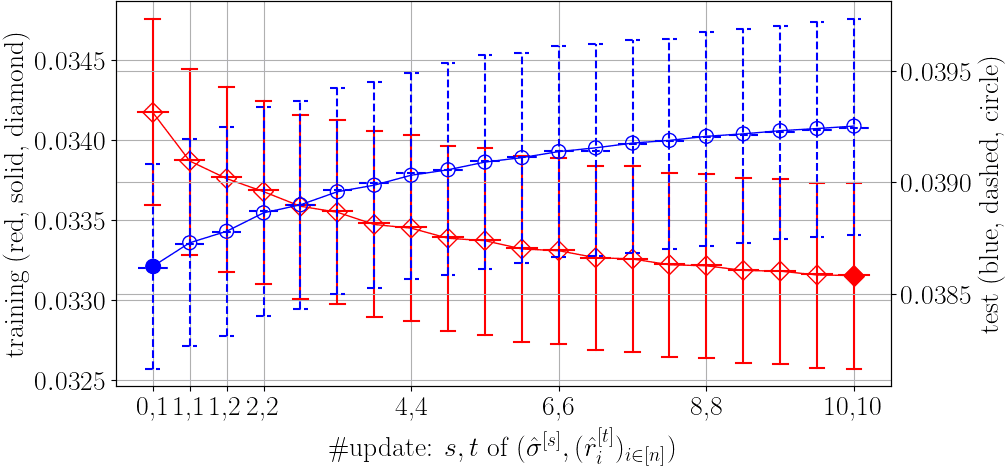}}&
{\includegraphics[width=2.0cm]{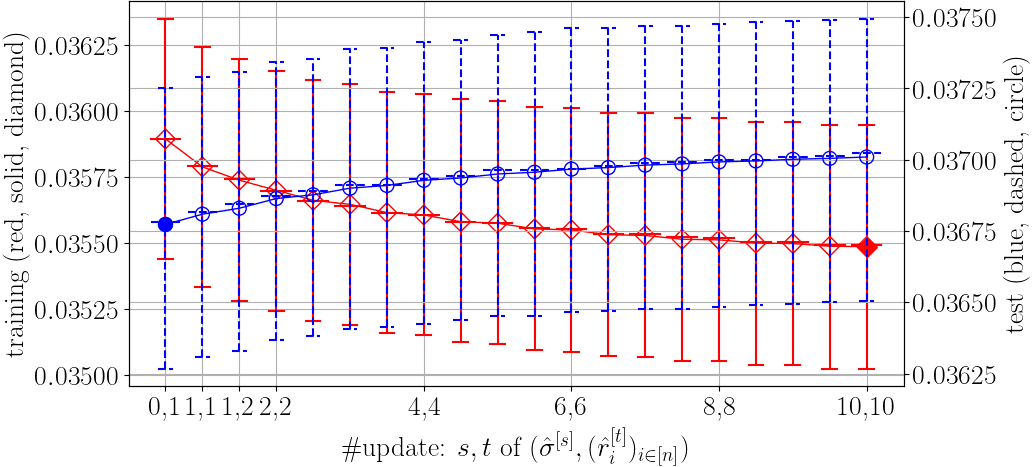}}&
{\includegraphics[width=2.0cm]{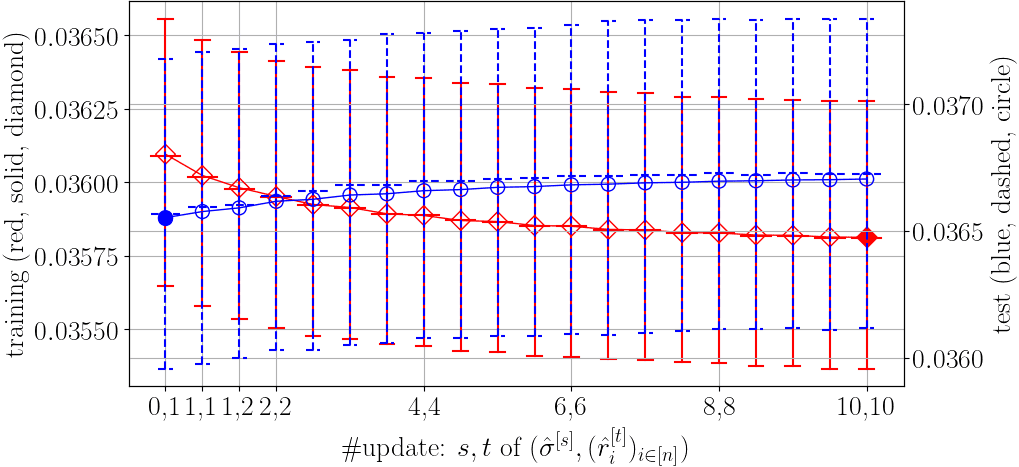}}&
{\includegraphics[width=2.0cm]{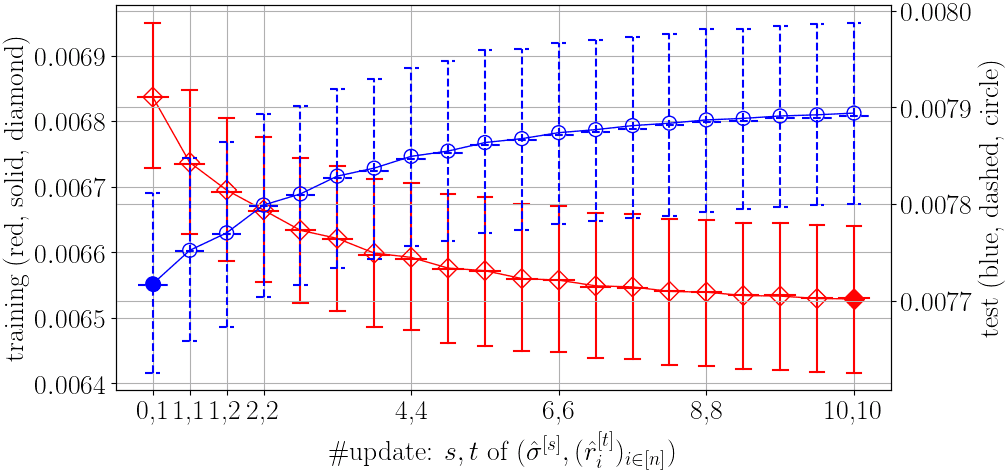}}&
{\includegraphics[width=2.0cm]{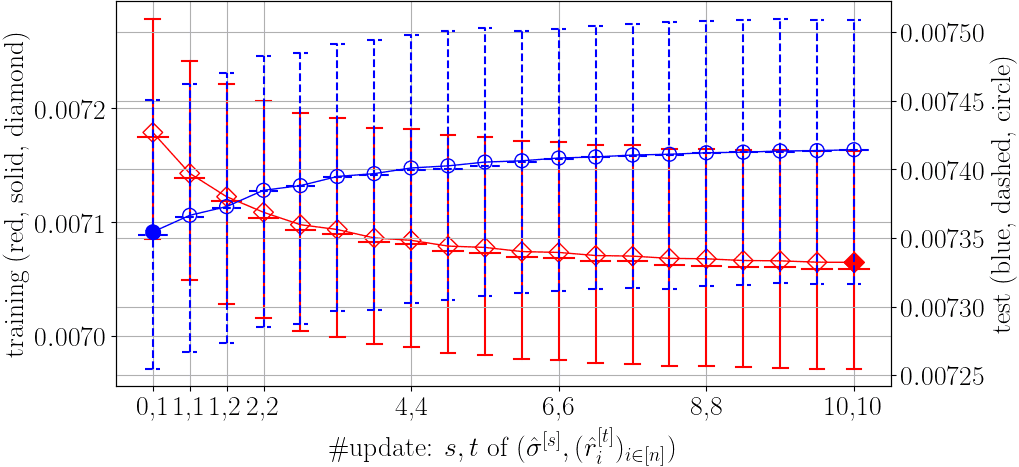}}&
{\includegraphics[width=2.0cm]{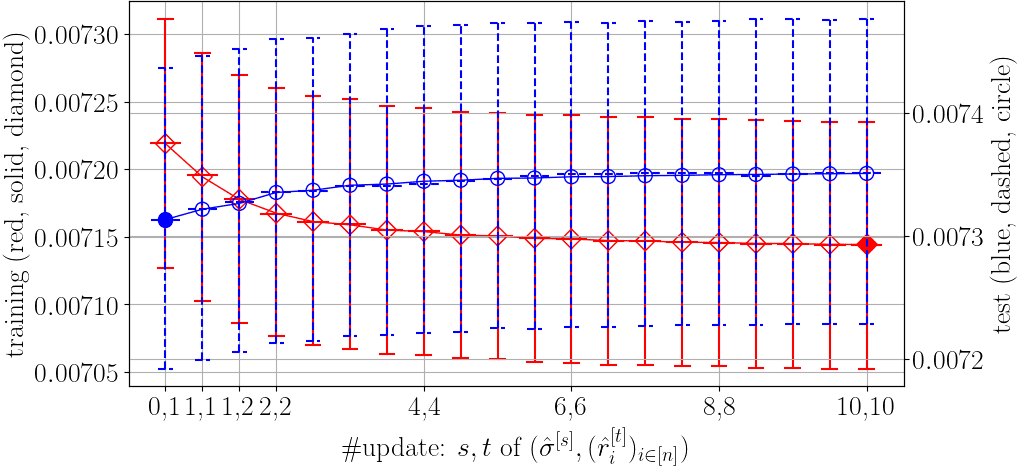}}
\\\midrule
\multirow{3}{*}[-2.5mm]{\rotatebox{90}{\tiny\eqref{eq:Kendall}, $n=$}}
&\rotatebox{90}{\tiny\,~~~\,$25$}&
{\includegraphics[width=2.0cm]{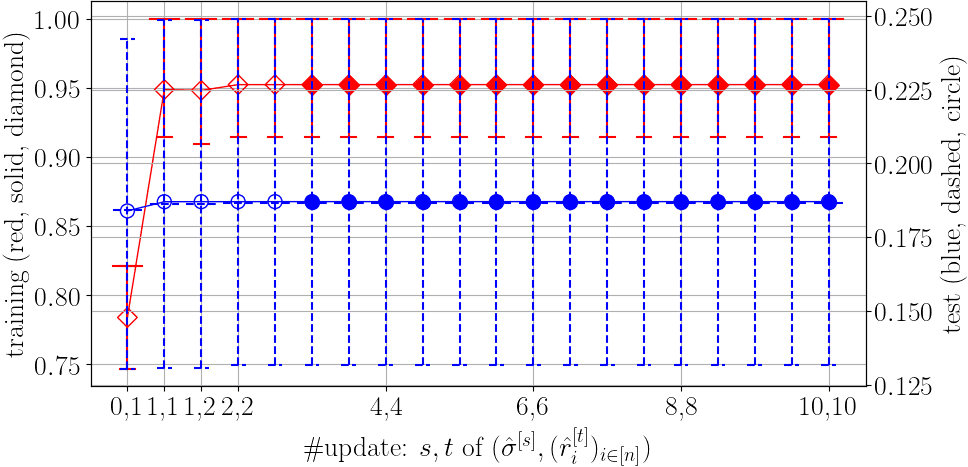}}&
\CF{\includegraphics[width=2.0cm]{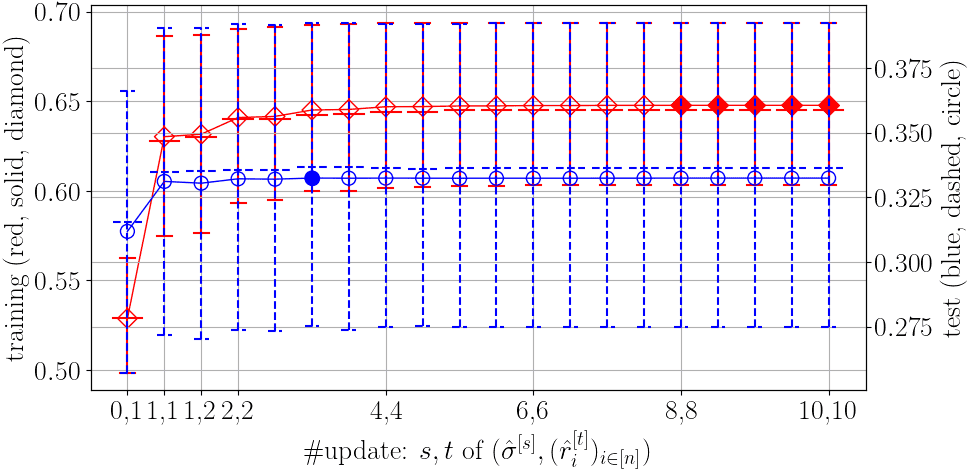}}&
\CF{\includegraphics[width=2.0cm]{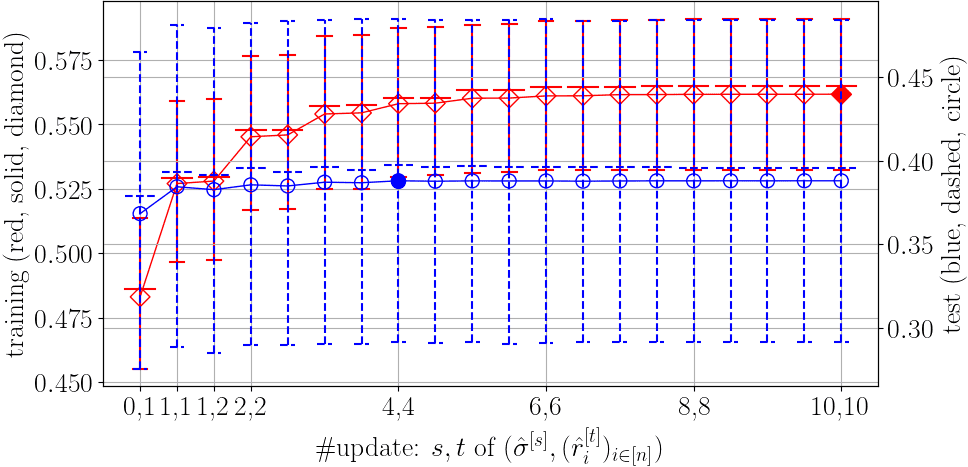}}&
\CF{\includegraphics[width=2.0cm]{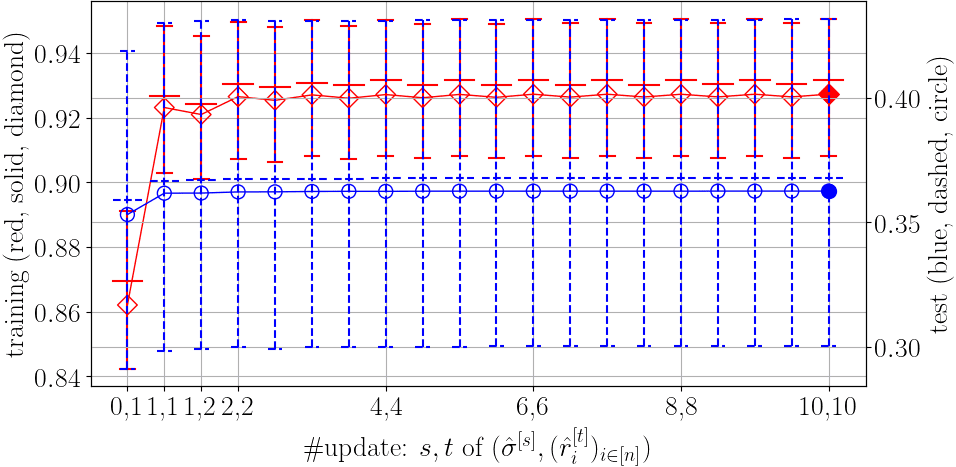}}&
\CF{\includegraphics[width=2.0cm]{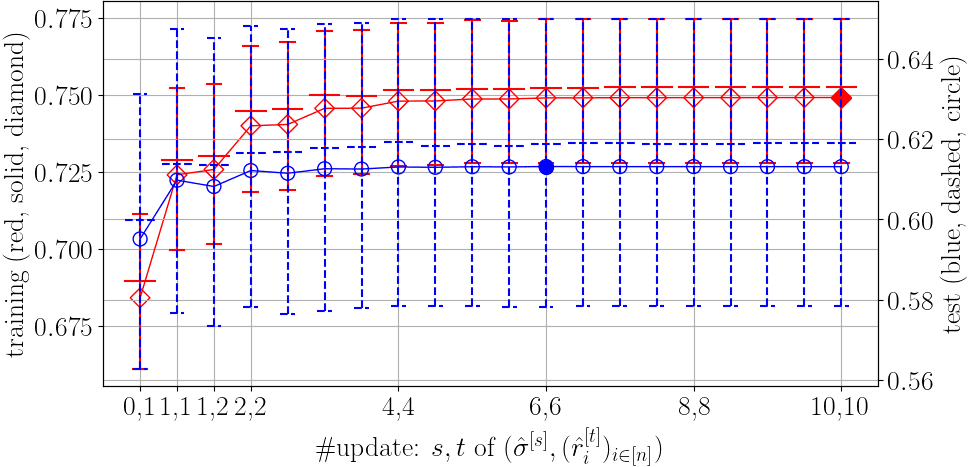}}&
\CF{\includegraphics[width=2.0cm]{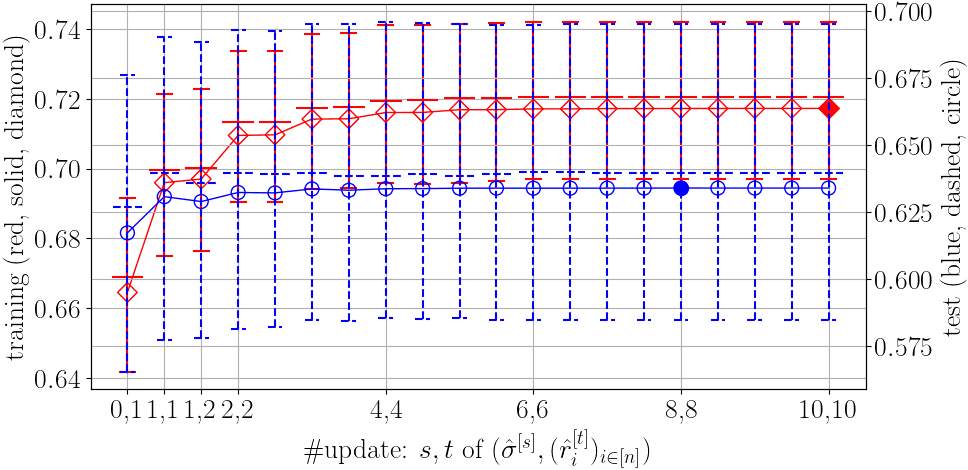}}&
{\includegraphics[width=2.0cm]{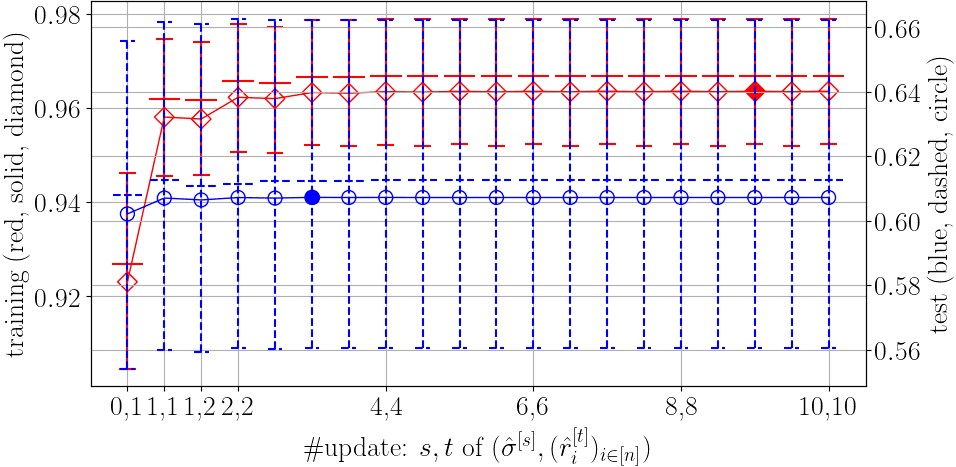}}&
\CF{\includegraphics[width=2.0cm]{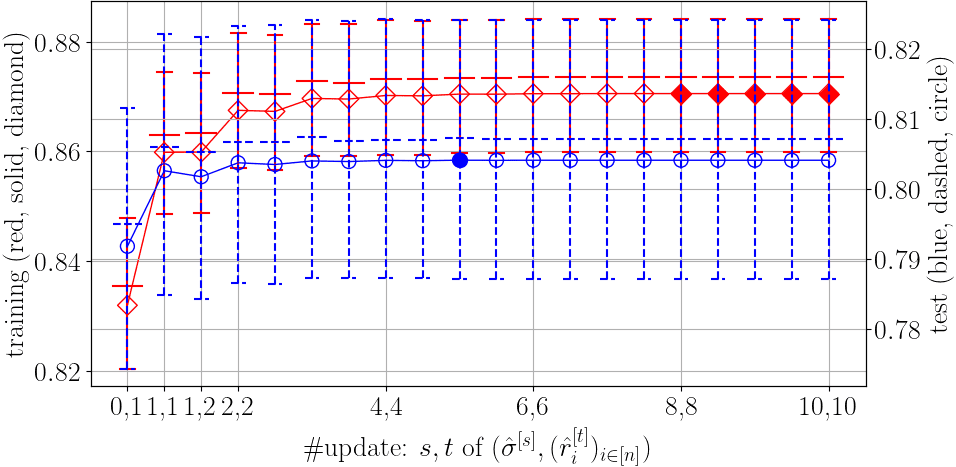}}&
\CF{\includegraphics[width=2.0cm]{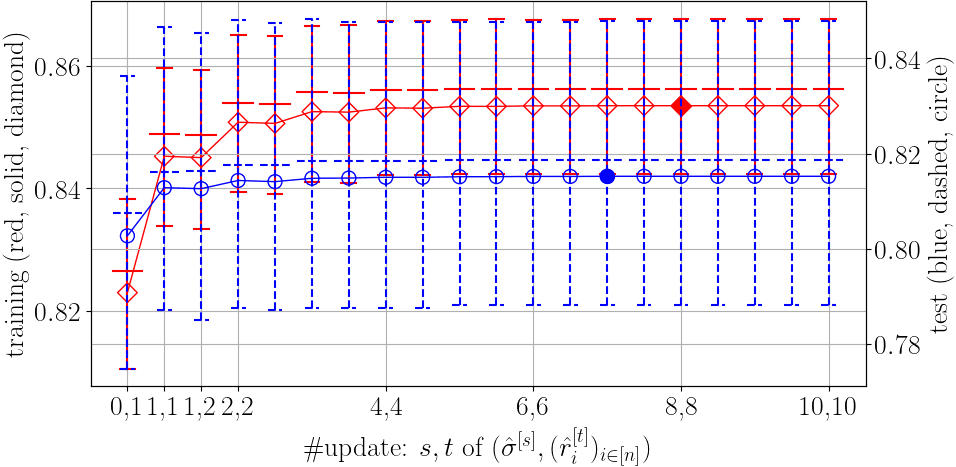}}\\
&\rotatebox{90}{\tiny\,~~\,$100$}&
\CF{\includegraphics[width=2.0cm]{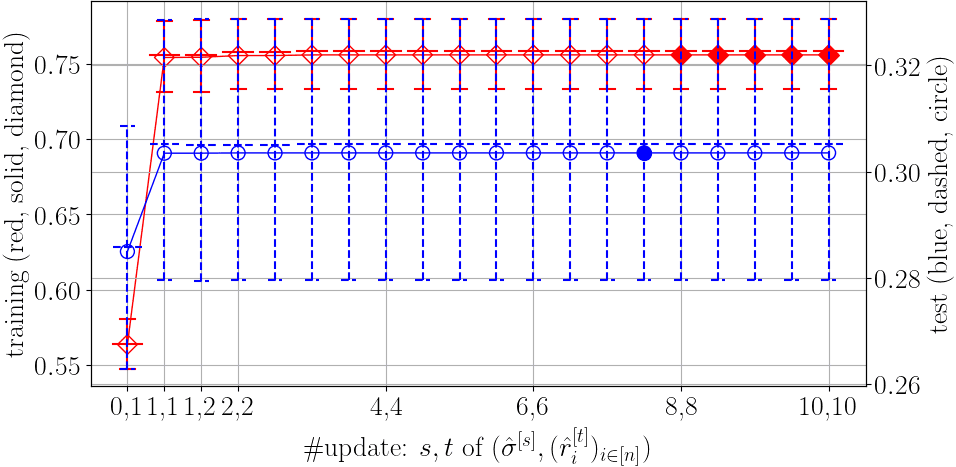}}&
\CF{\includegraphics[width=2.0cm]{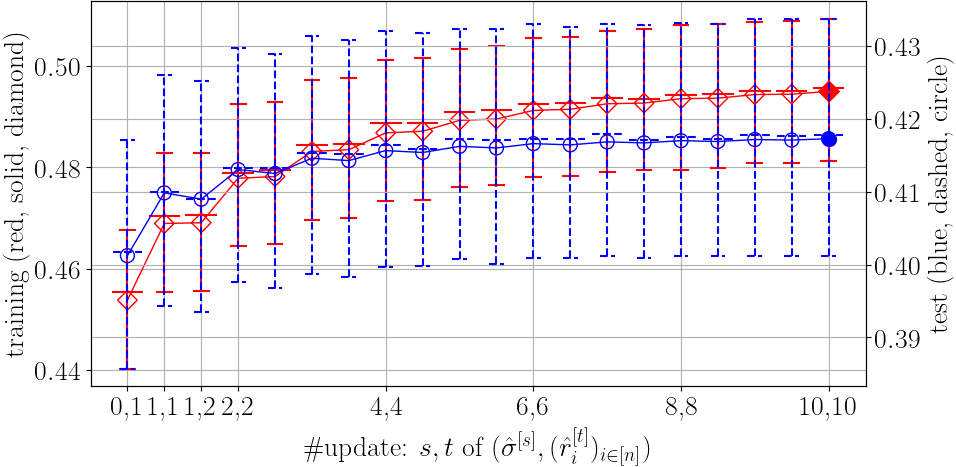}}&
\CF{\includegraphics[width=2.0cm]{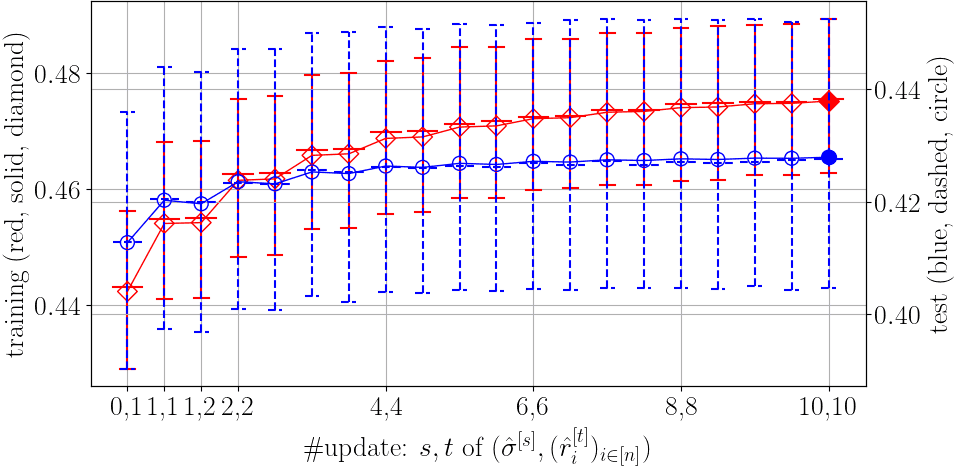}}&
\CF{\includegraphics[width=2.0cm]{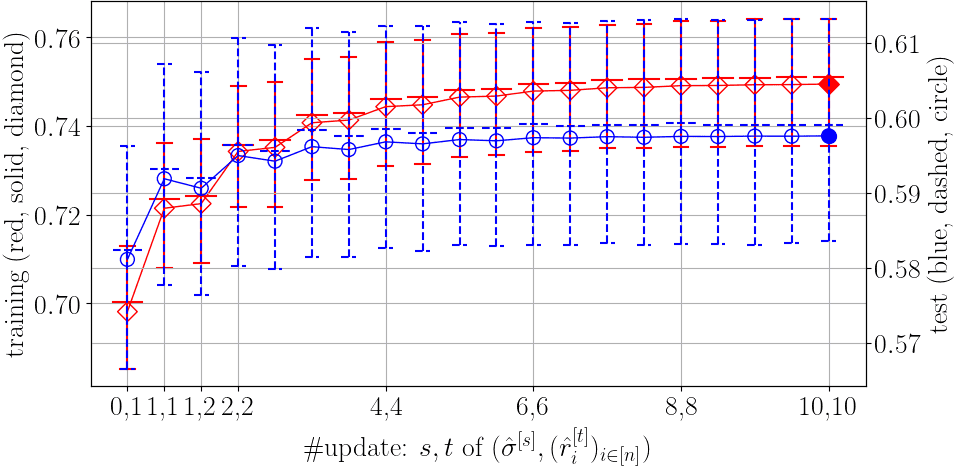}}&
\CF{\includegraphics[width=2.0cm]{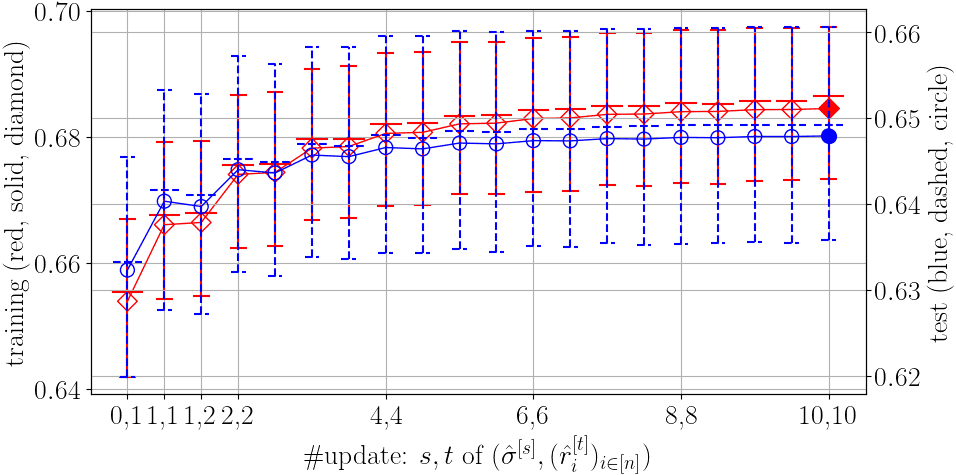}}&
\CF{\includegraphics[width=2.0cm]{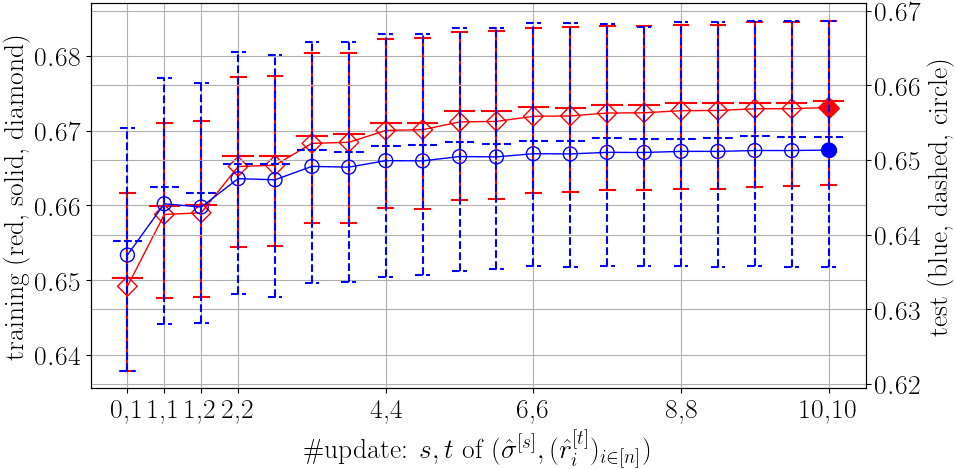}}&
\CF{\includegraphics[width=2.0cm]{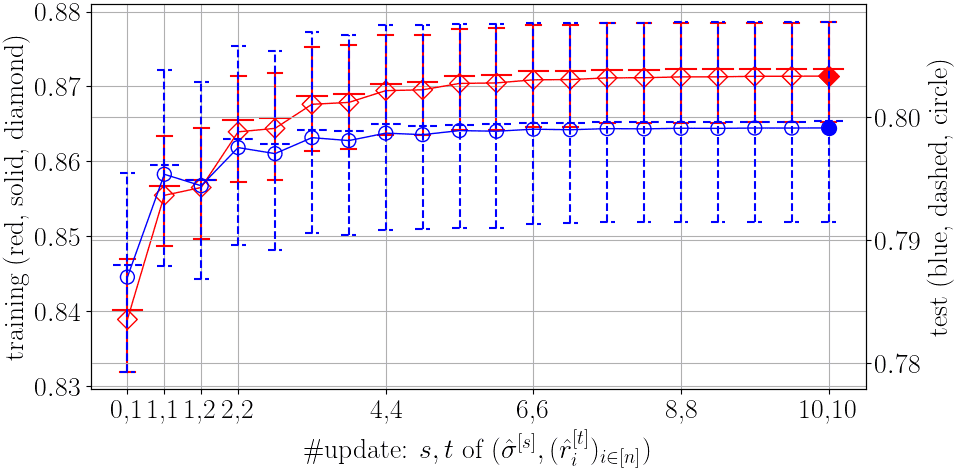}}&
\CF{\includegraphics[width=2.0cm]{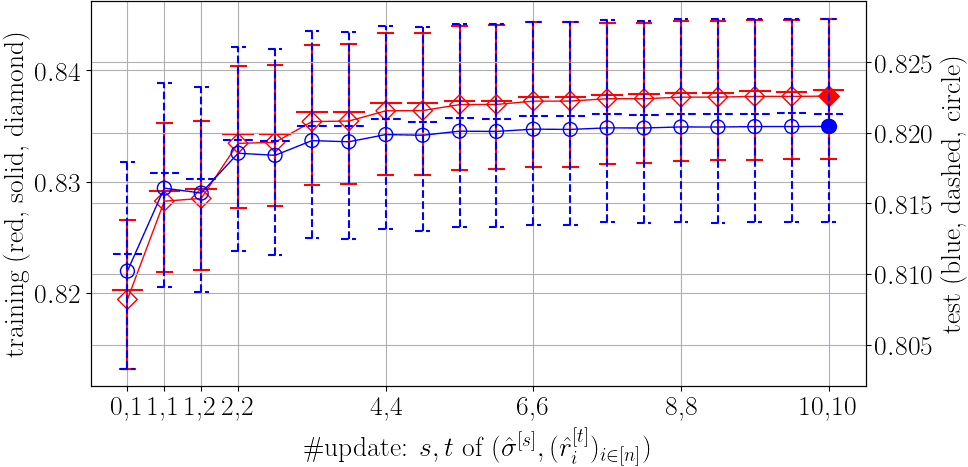}}&
\CF{\includegraphics[width=2.0cm]{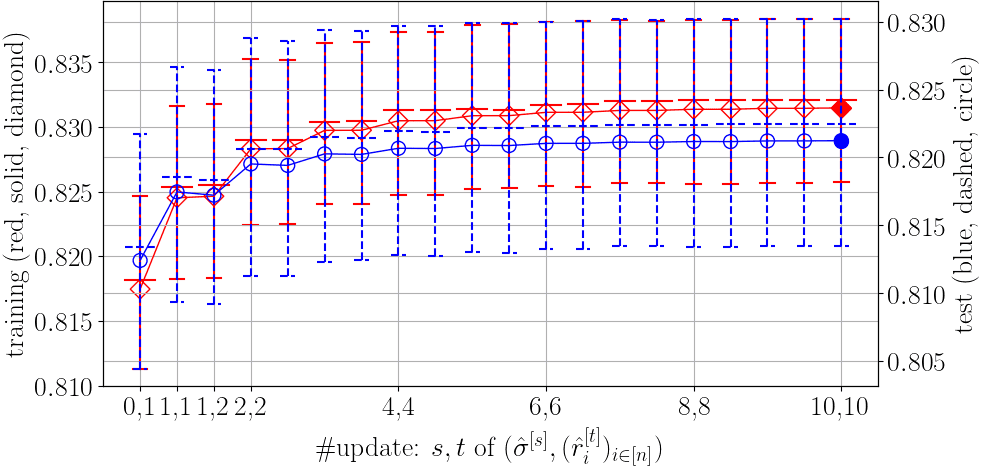}}\\
&\rotatebox{90}{\tiny\,~~\,$400$}&
\CF{\includegraphics[width=2.0cm]{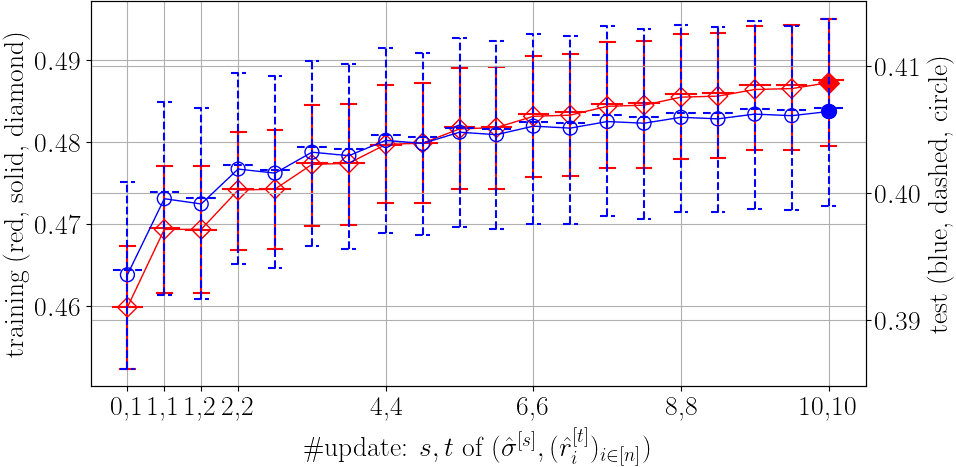}}&
\CF{\includegraphics[width=2.0cm]{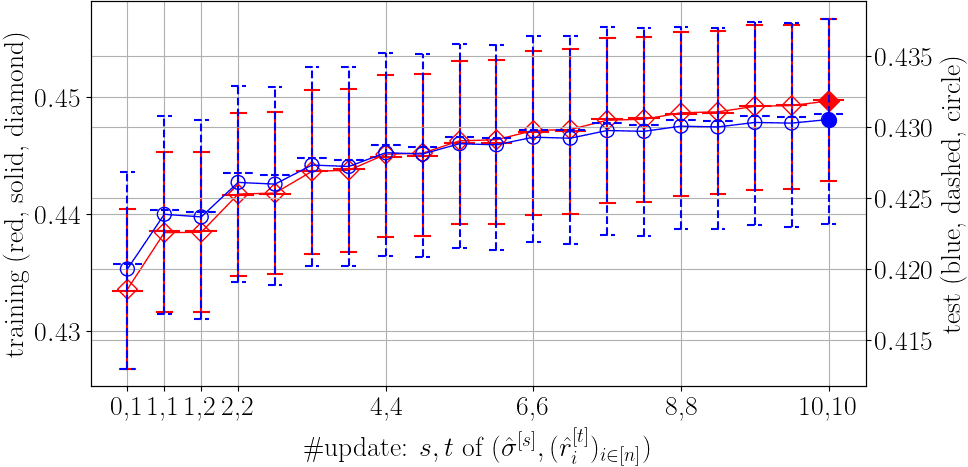}}&
\CF{\includegraphics[width=2.0cm]{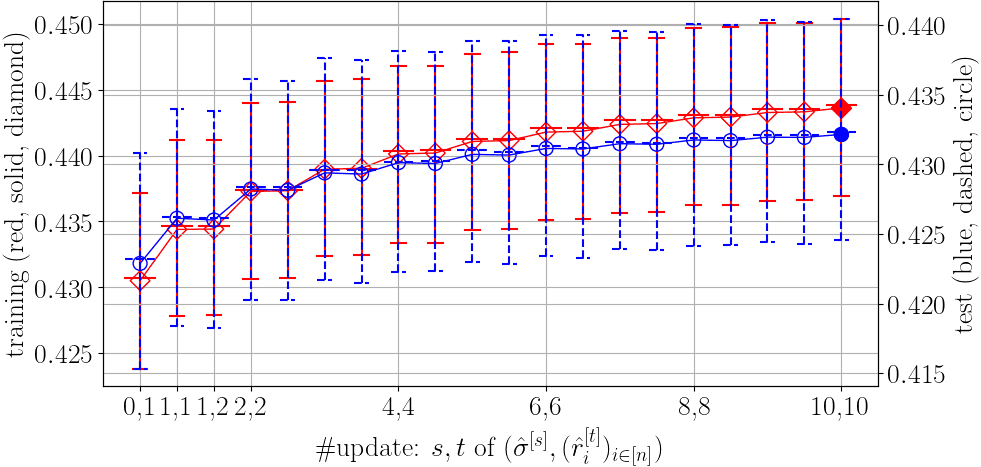}}&
\CF{\includegraphics[width=2.0cm]{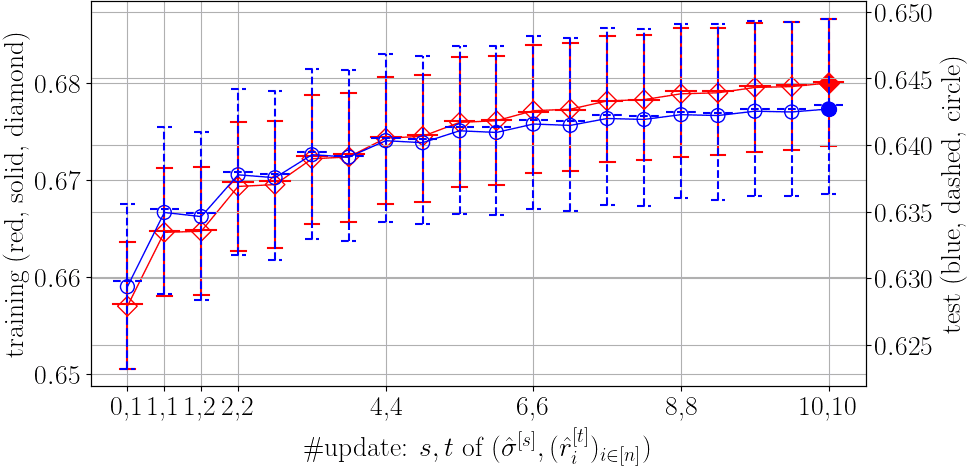}}&
\CF{\includegraphics[width=2.0cm]{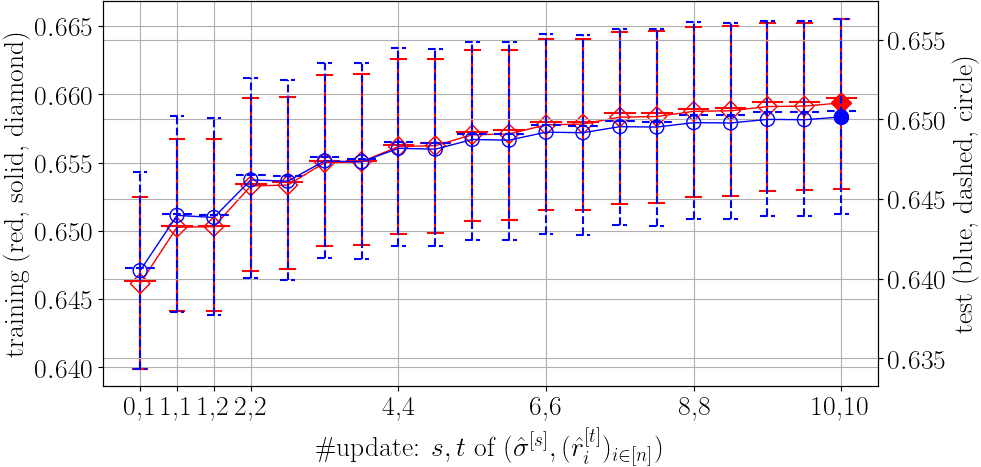}}&
\CF{\includegraphics[width=2.0cm]{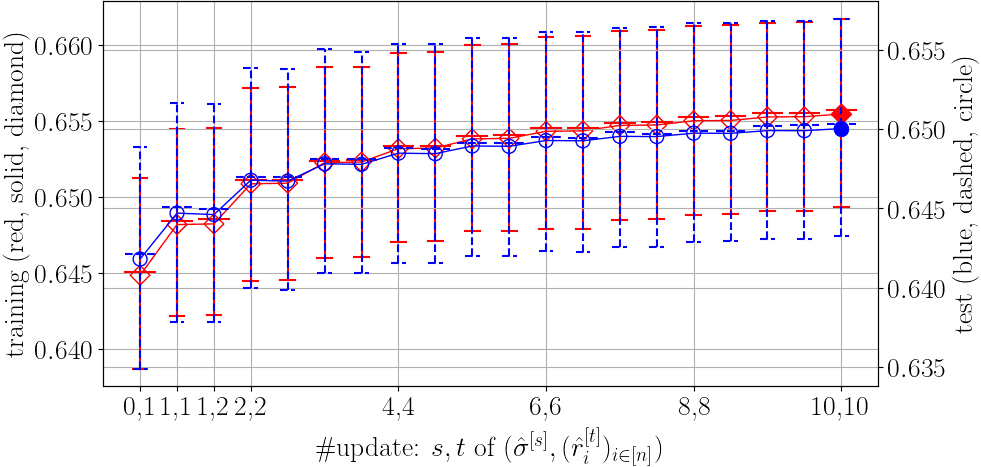}}&
\CF{\includegraphics[width=2.0cm]{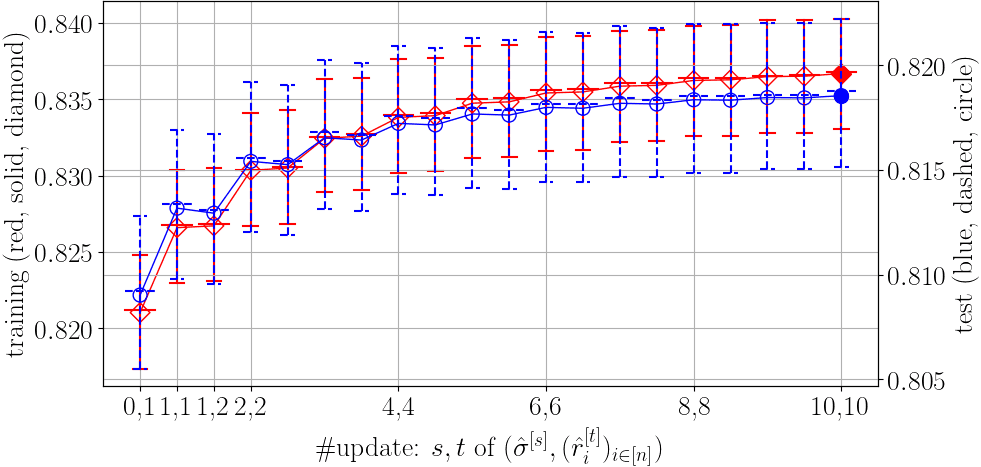}}&
\CF{\includegraphics[width=2.0cm]{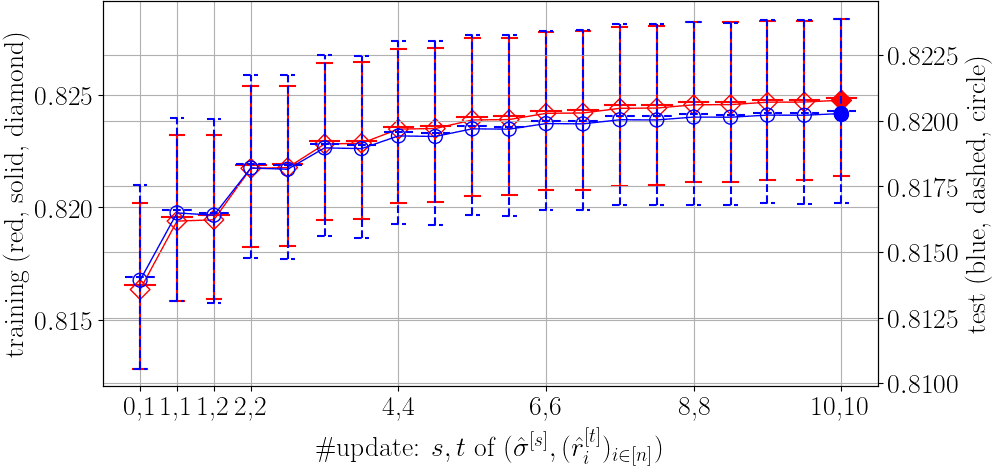}}&
\CF{\includegraphics[width=2.0cm]{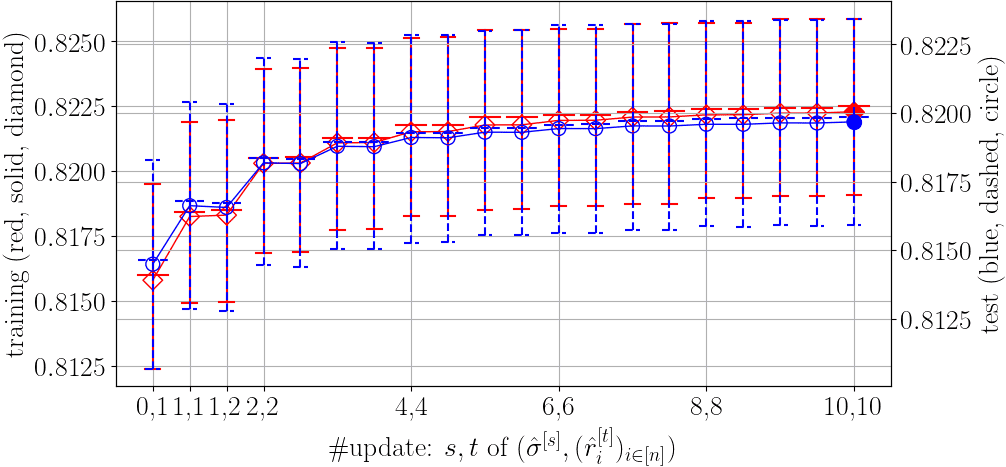}}
\\\midrule
\multirow{3}{*}[-2.5mm]{\rotatebox{90}{\tiny\eqref{eq:TIE}, $n=$}}
&\rotatebox{90}{\tiny\,~~~\,$25$}&
{\includegraphics[width=2.0cm]{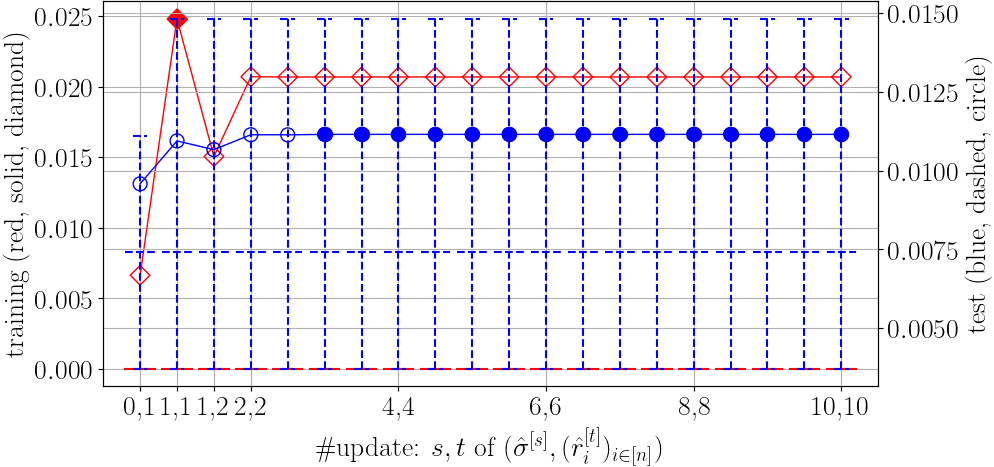}}&
{\includegraphics[width=2.0cm]{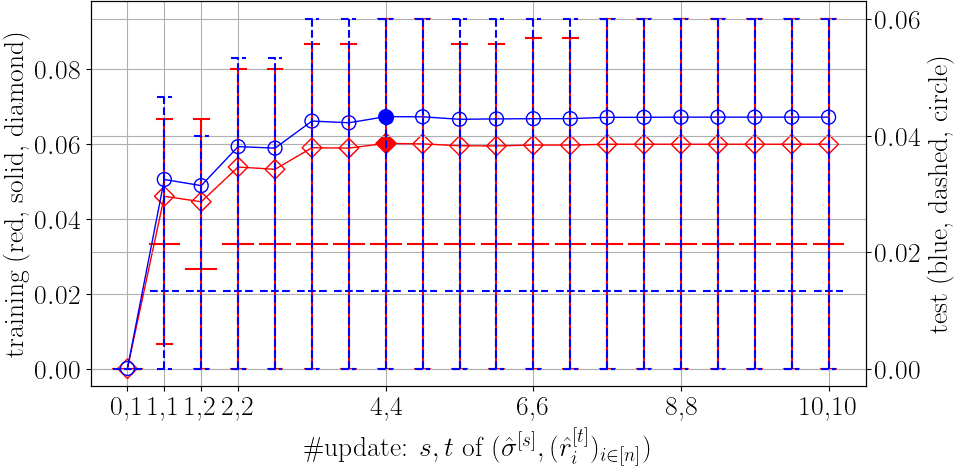}}&
{\includegraphics[width=2.0cm]{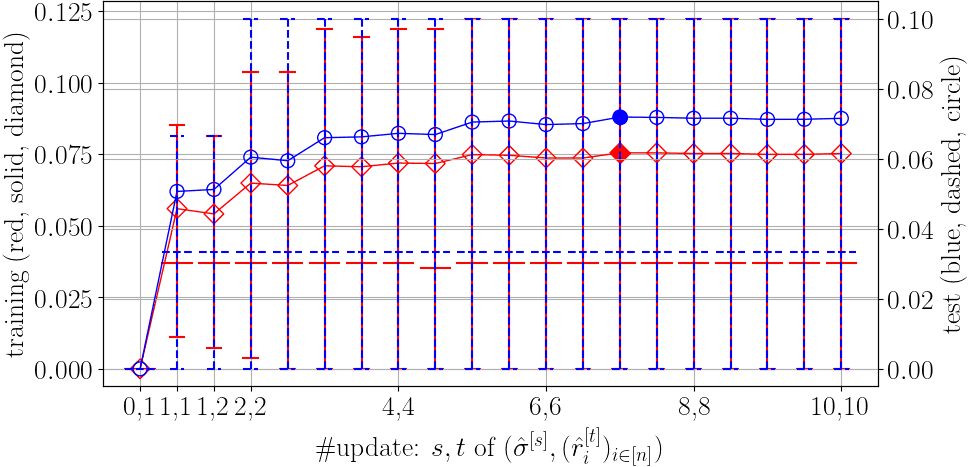}}&
{\includegraphics[width=2.0cm]{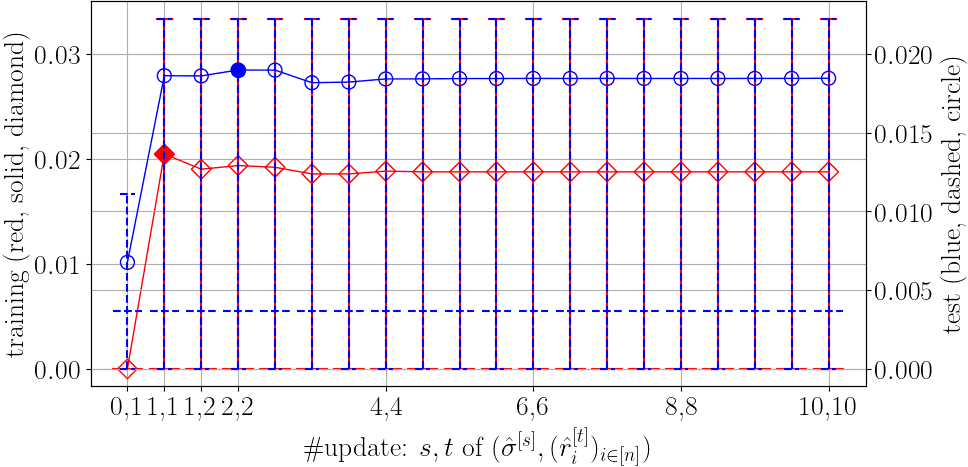}}&
{\includegraphics[width=2.0cm]{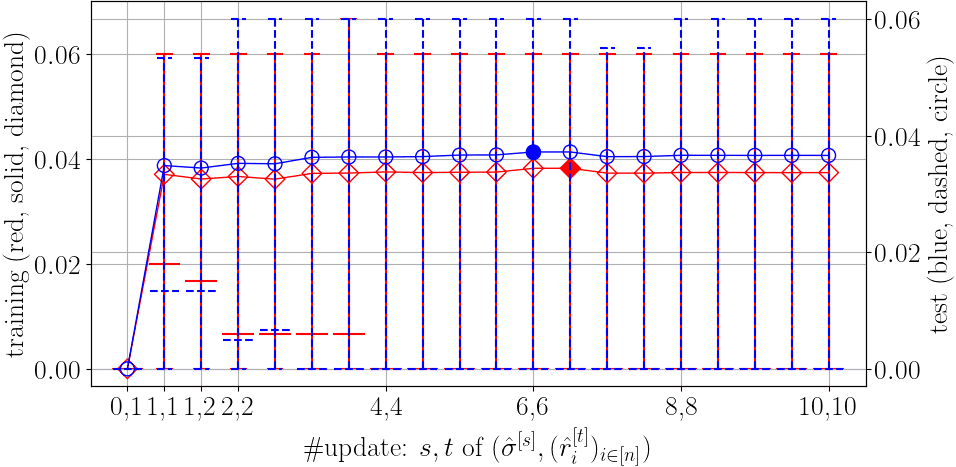}}&
{\includegraphics[width=2.0cm]{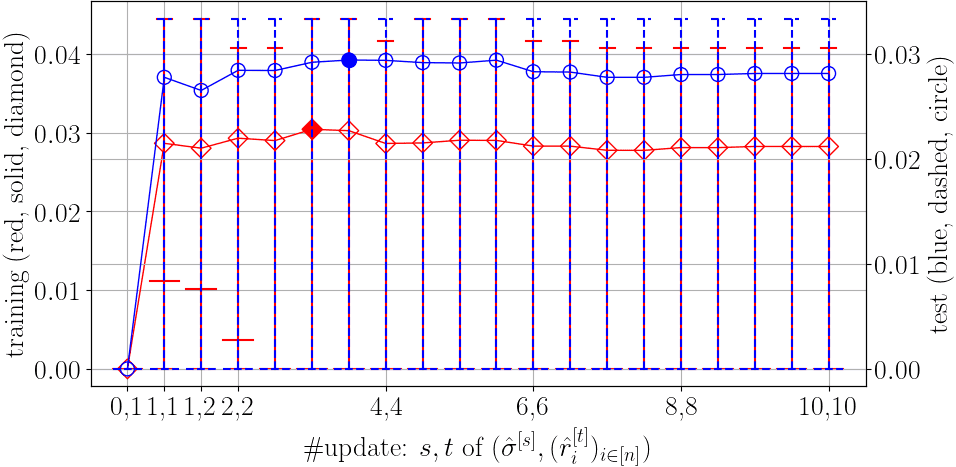}}&
{\includegraphics[width=2.0cm]{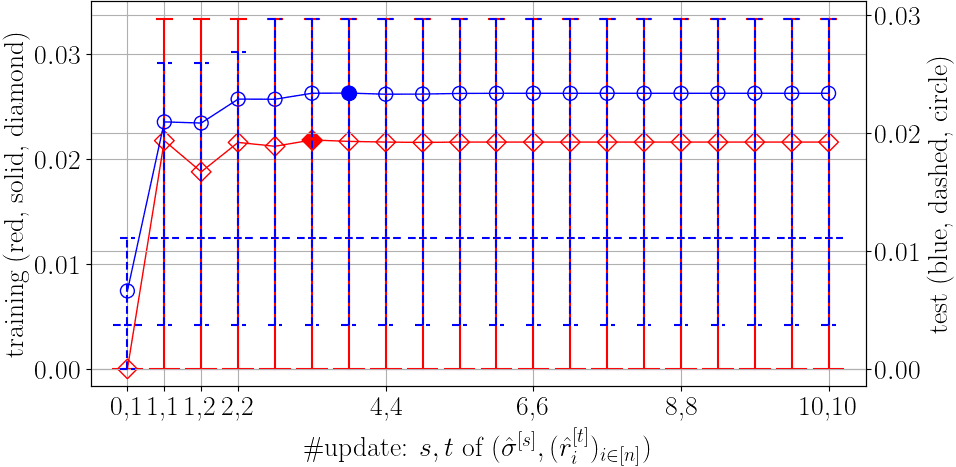}}&
{\includegraphics[width=2.0cm]{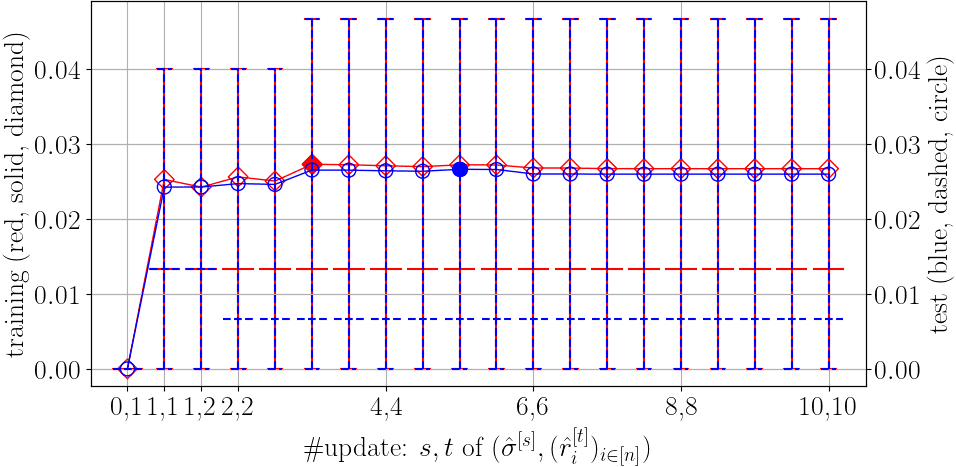}}&
{\includegraphics[width=2.0cm]{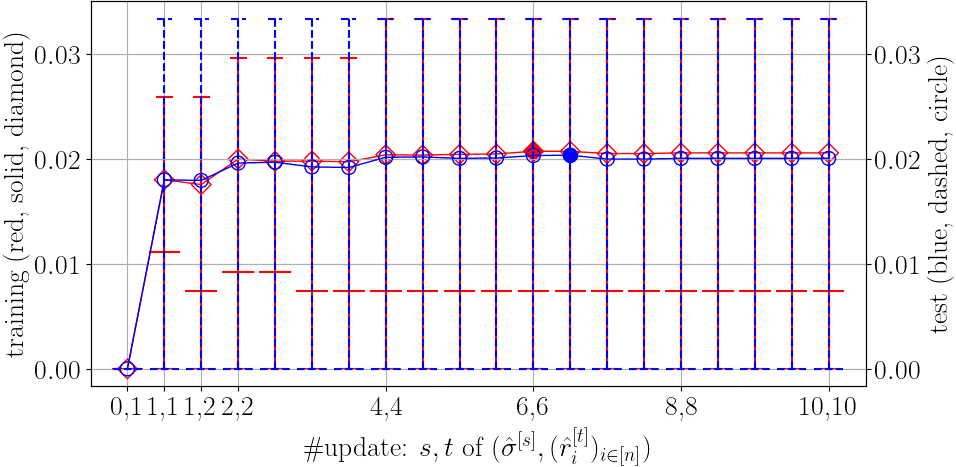}}\\
&\rotatebox{90}{\tiny\,~~\,$100$}&
{\includegraphics[width=2.0cm]{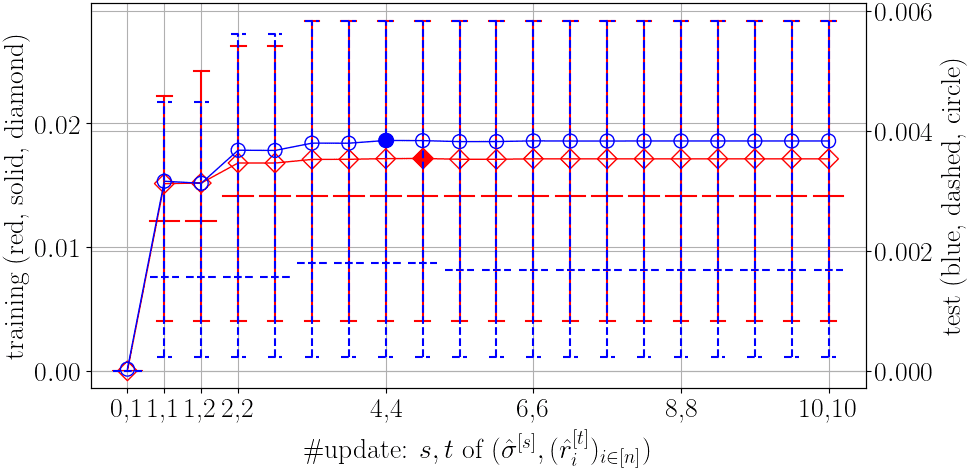}}&
{\includegraphics[width=2.0cm]{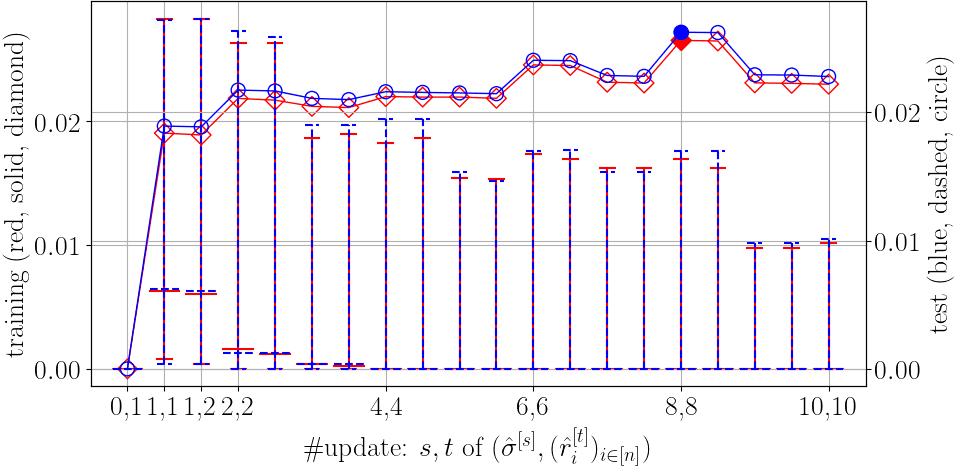}}&
{\includegraphics[width=2.0cm]{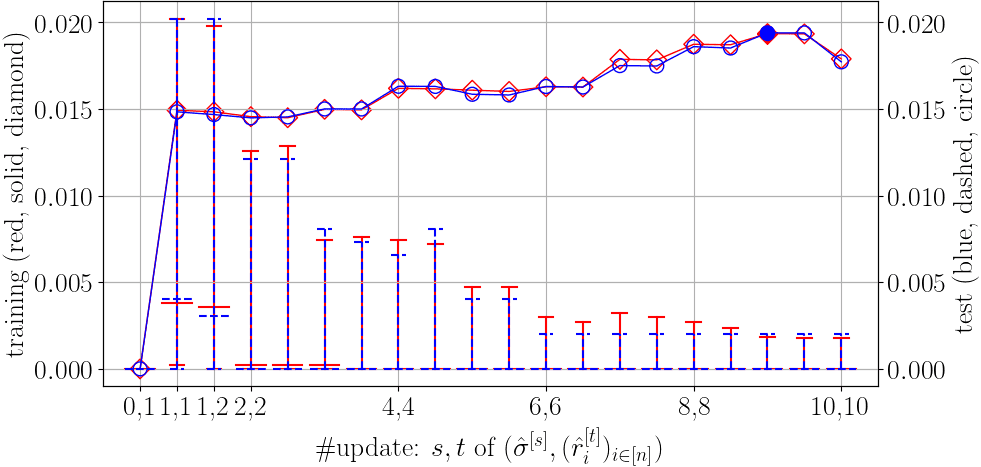}}&
{\includegraphics[width=2.0cm]{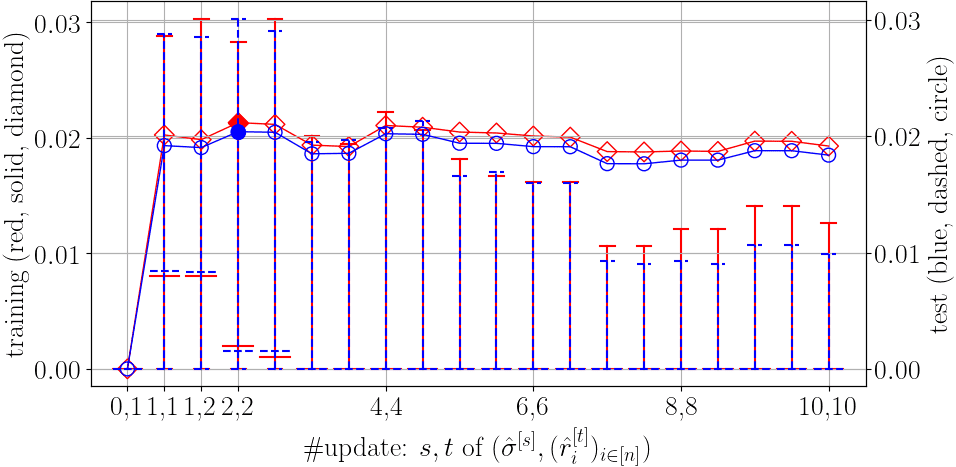}}&
{\includegraphics[width=2.0cm]{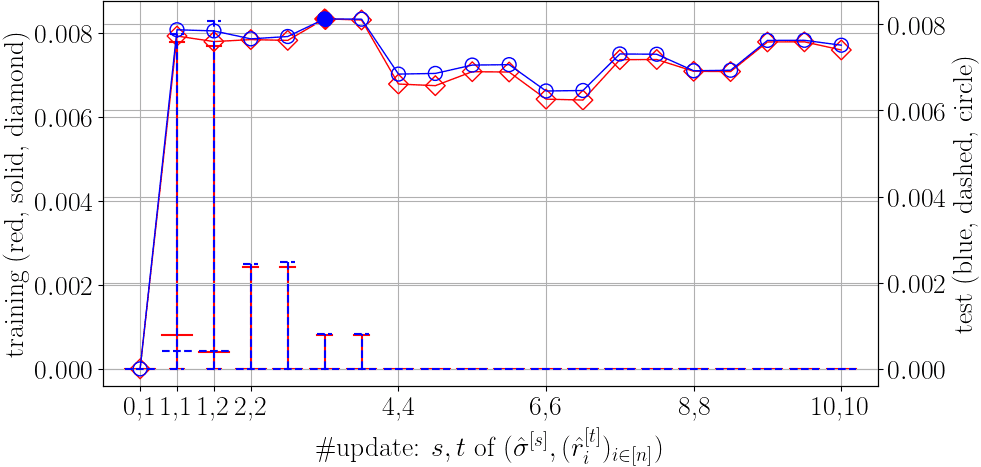}}&
{\includegraphics[width=2.0cm]{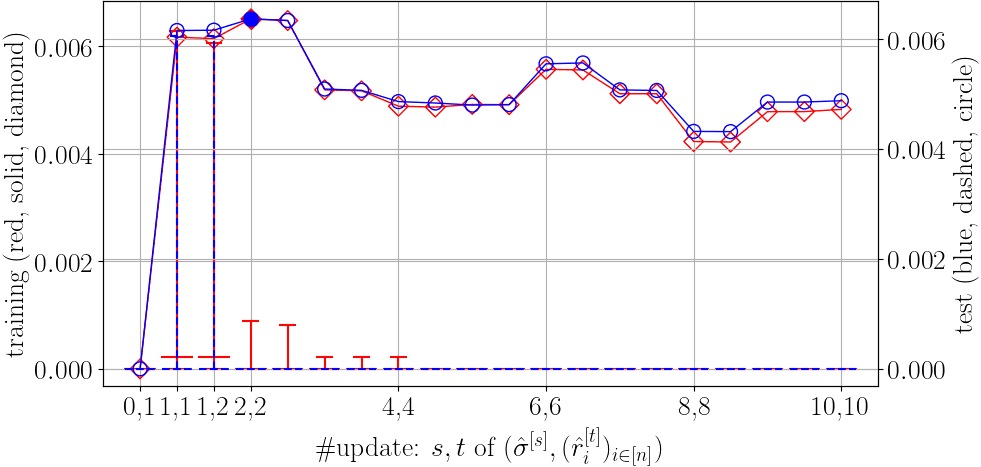}}&
{\includegraphics[width=2.0cm]{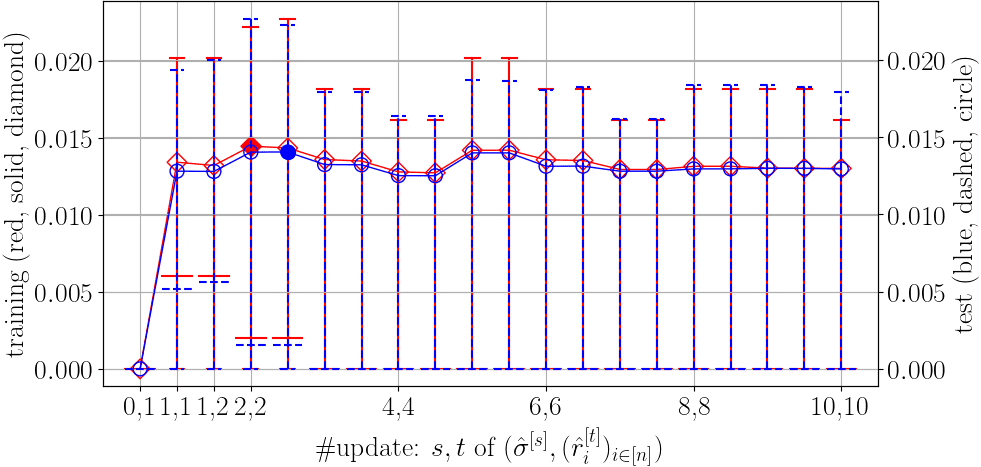}}&
{\includegraphics[width=2.0cm]{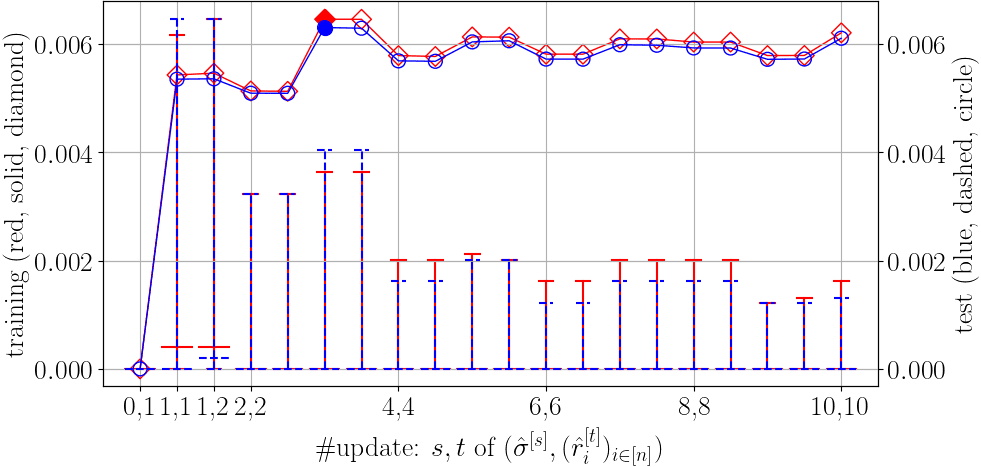}}&
{\includegraphics[width=2.0cm]{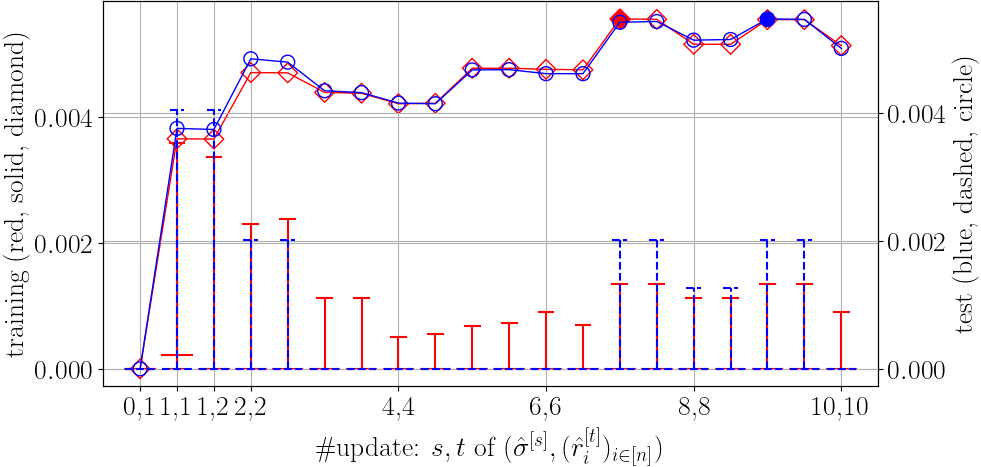}}\\
&\rotatebox{90}{\tiny\,~~\,$400$}&
{\includegraphics[width=2.0cm]{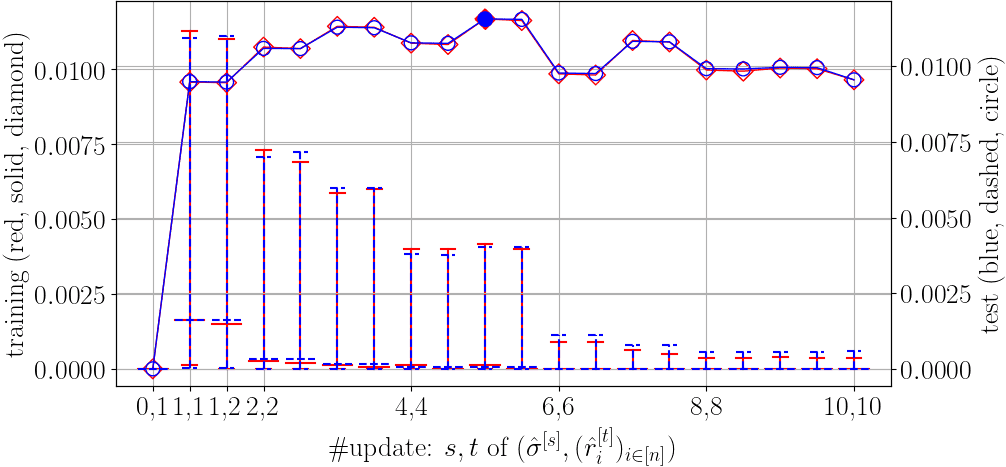}}&
{\includegraphics[width=2.0cm]{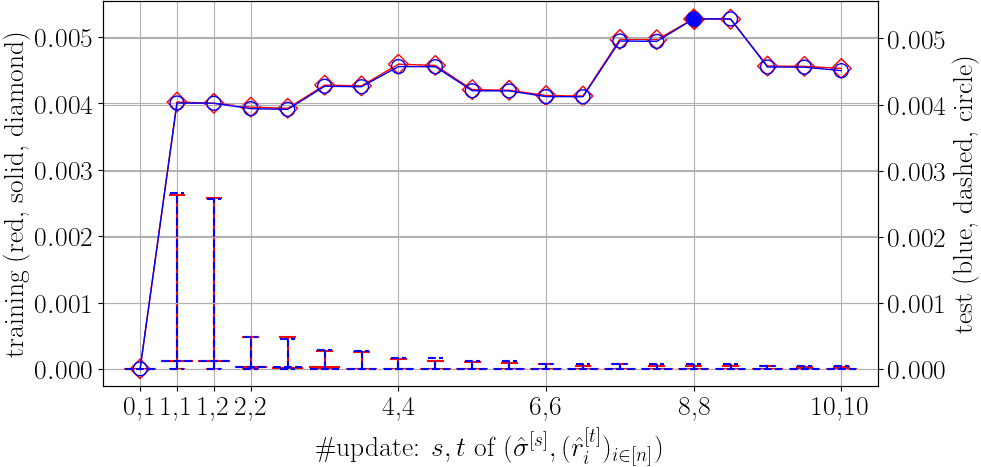}}&
{\includegraphics[width=2.0cm]{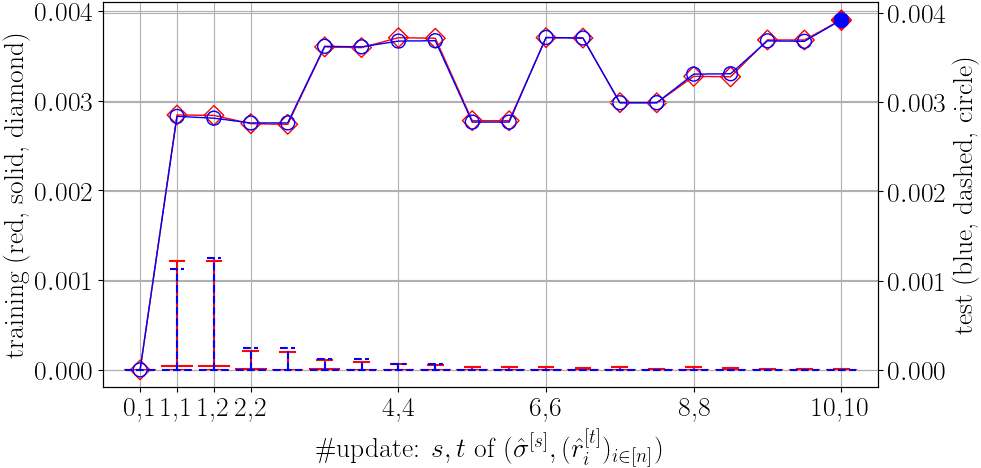}}&
{\includegraphics[width=2.0cm]{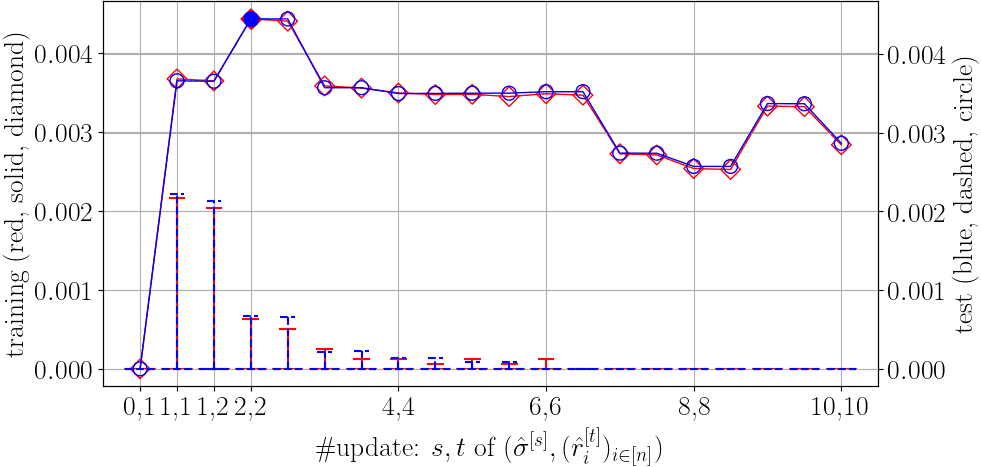}}&
{\includegraphics[width=2.0cm]{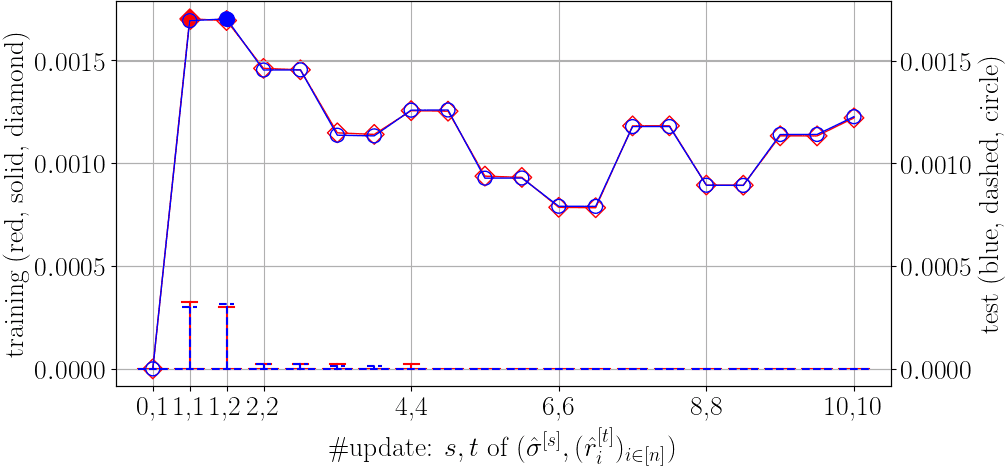}}&
{\includegraphics[width=2.0cm]{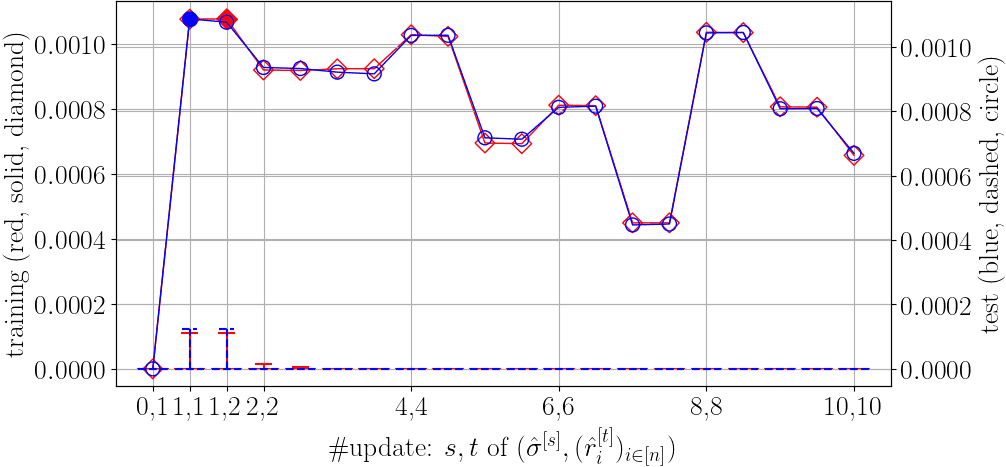}}&
{\includegraphics[width=2.0cm]{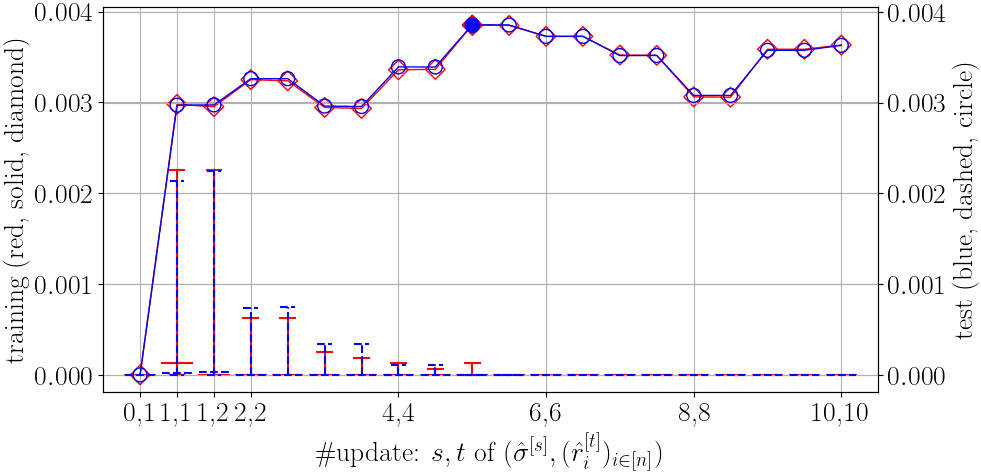}}&
{\includegraphics[width=2.0cm]{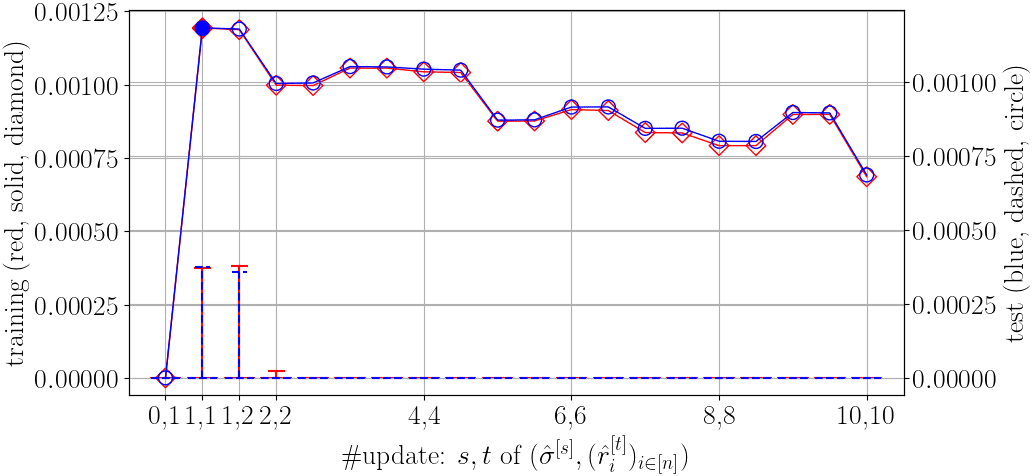}}&
{\includegraphics[width=2.0cm]{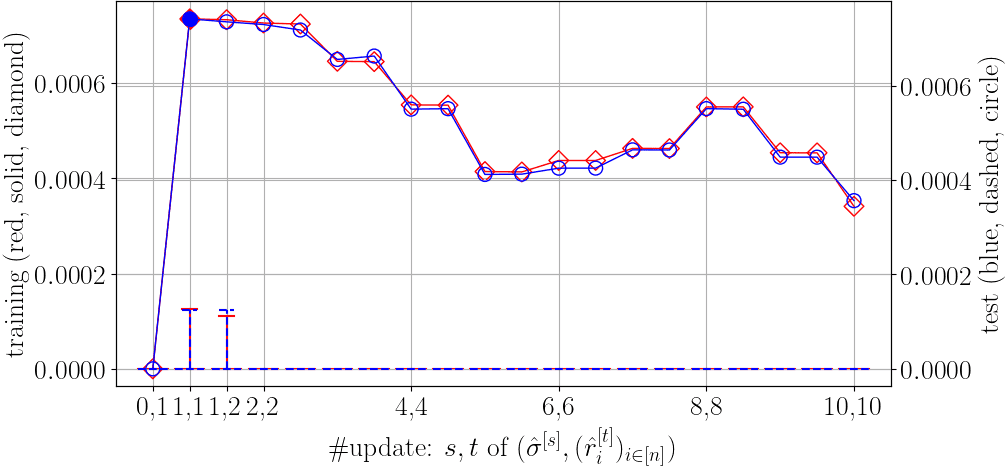}}
\end{tabular}
\caption{%
Part of results of synthetic data experiments, Procedure 1 in Section~\ref{sec:Synthetic} and Appendix~\ref{sec:Specified}:
For Logistic-$N$ synthetic data with $N=1,5,25$ (left to right),
mean (marker) and 0.25, 0.5, and 0.75 quantiles (lower, middle, and upper bars) of 
1000 trial training (red, solid, diamond) and test (blue, dashed, circle) evaluation of 
WPP error \eqref{eq:WPE} with the squared loss $\phi=\phi_\sq$, 
Kendall's Tau \eqref{eq:Kendall}, and tie rate \eqref{eq:TIE} (top to bottom)
for the isotonic Bradley-Terry model learned with the squared loss $\phi=\phi_\sq$.
Smaller \eqref{eq:WPE}, or larger \eqref{eq:Kendall} indicates a better model.
The marker for the best model and model with the most ties was filled in, 
and the outer frame of the figure was highlighted in gray
if the best model was significantly better than the Bradley-Terry model
with respect to Mann-Whitney U test of the significance level 0.05.}
\label{fig:Logistic-SQ}
\end{sidewaysfigure}
\begin{sidewaysfigure}
\centering%
\renewcommand{\arraystretch}{0.5}%
\renewcommand{\tabcolsep}{0.5pt}%
\begin{tabular}{cc|ccc|ccc|ccc}%
&&\multicolumn{3}{c|}{\tiny$N=1$, $|D_\tra|:|D_\tes|=$}&\multicolumn{3}{c|}{\tiny$N=5$, $|D_\tra|:|D_\tes|=$}&\multicolumn{3}{c}{\tiny$N=25$, $|D_\tra|:|D_\tes|=$}\\
&&{\tiny$1:9$}&{\tiny$5:5$}&{\tiny$9:1$}&{\tiny$1:9$}&{\tiny$5:5$}&{\tiny$9:1$}&{\tiny$1:9$}&{\tiny$5:5$}&{\tiny$9:1$}\\
\midrule
\multirow{3}{*}[-3.4mm]{\rotatebox{90}{\tiny $n=$}}
&\rotatebox{90}{\tiny\,~~~~\,$25$}&
{\includegraphics[width=2.0cm]{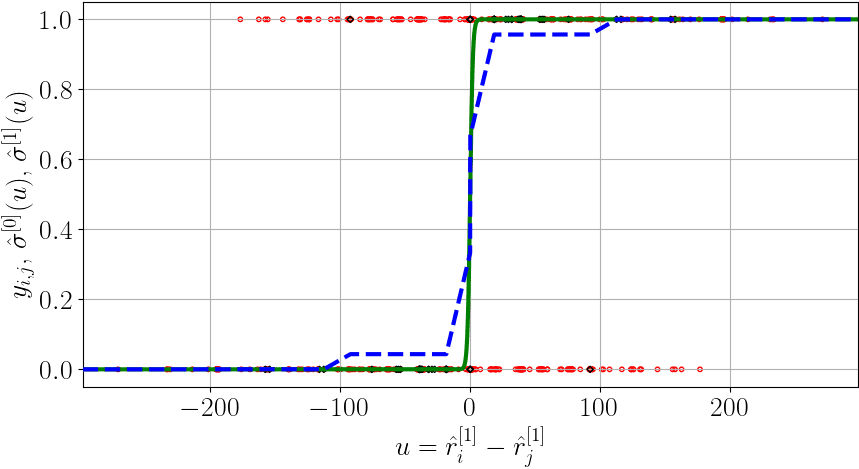}}&
{\includegraphics[width=2.0cm]{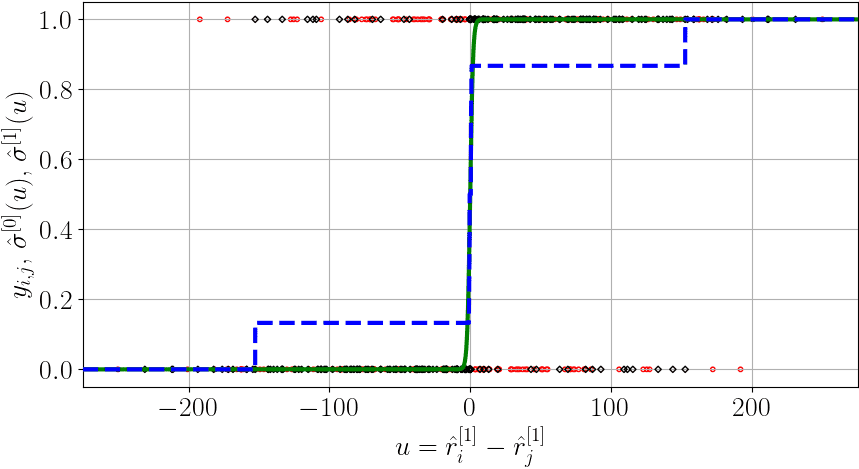}}&
{\includegraphics[width=2.0cm]{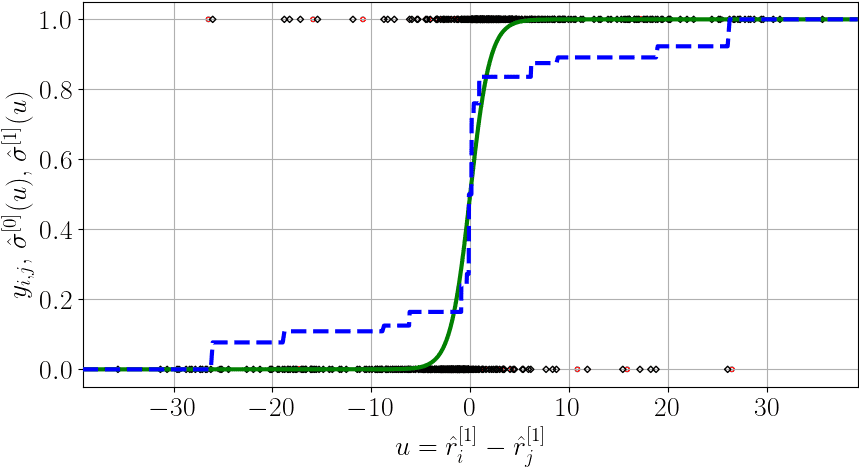}}&
{\includegraphics[width=2.0cm]{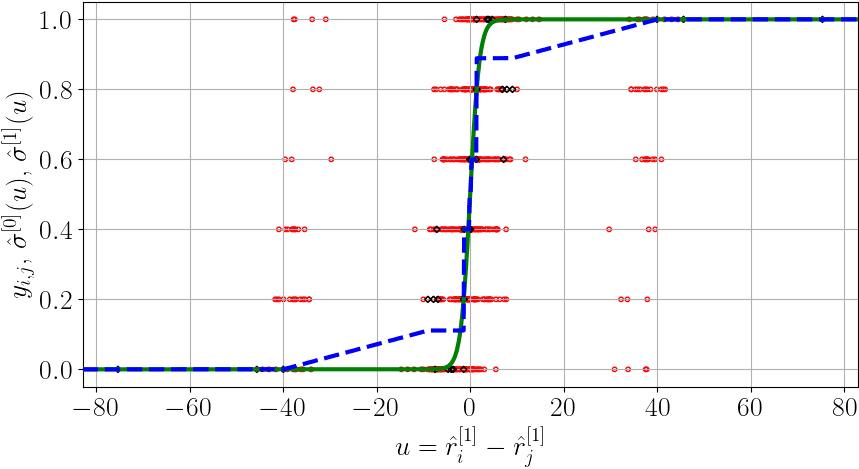}}&
{\includegraphics[width=2.0cm]{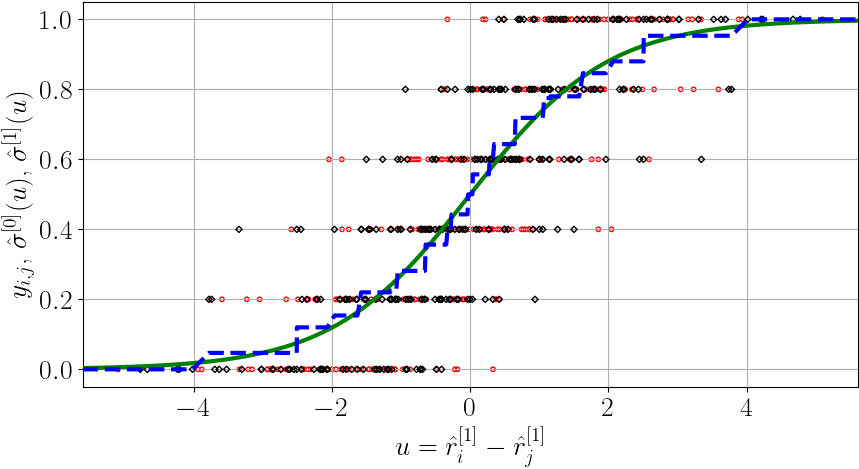}}&
{\includegraphics[width=2.0cm]{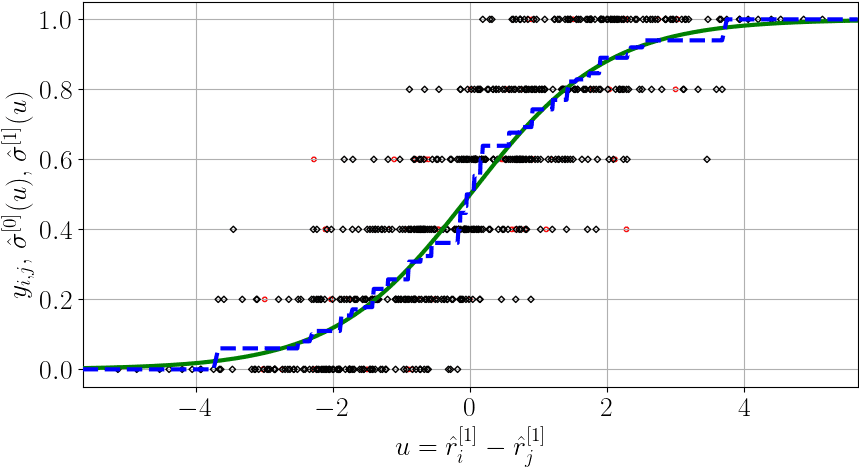}}&
{\includegraphics[width=2.0cm]{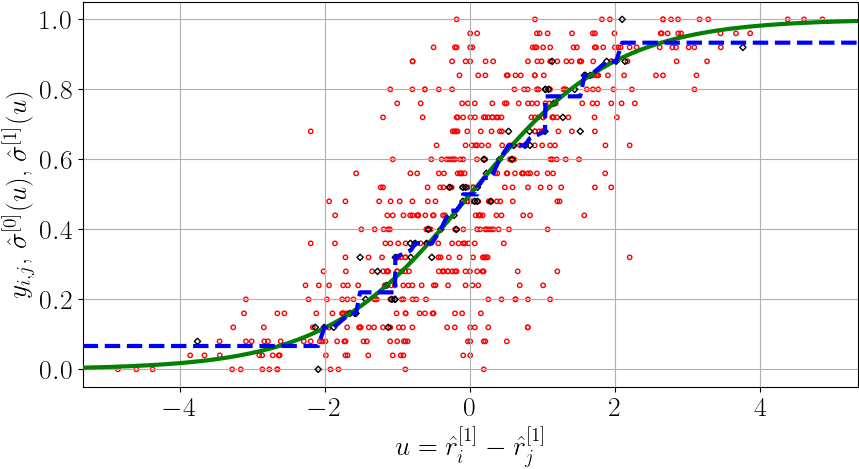}}&
{\includegraphics[width=2.0cm]{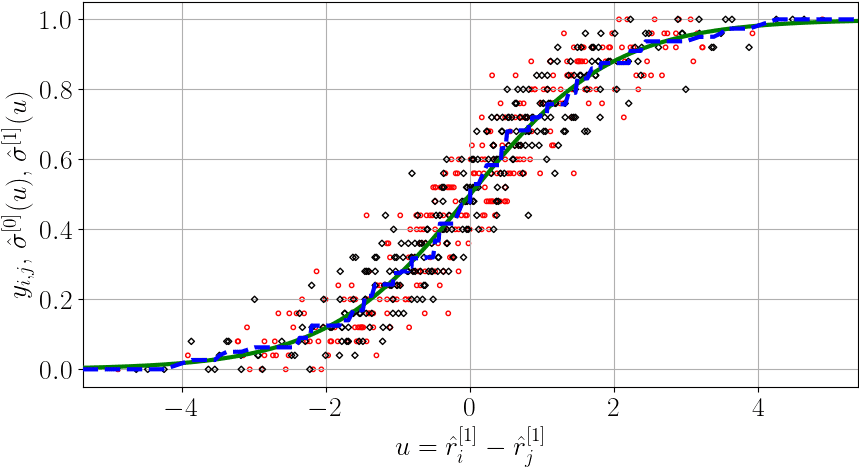}}&
{\includegraphics[width=2.0cm]{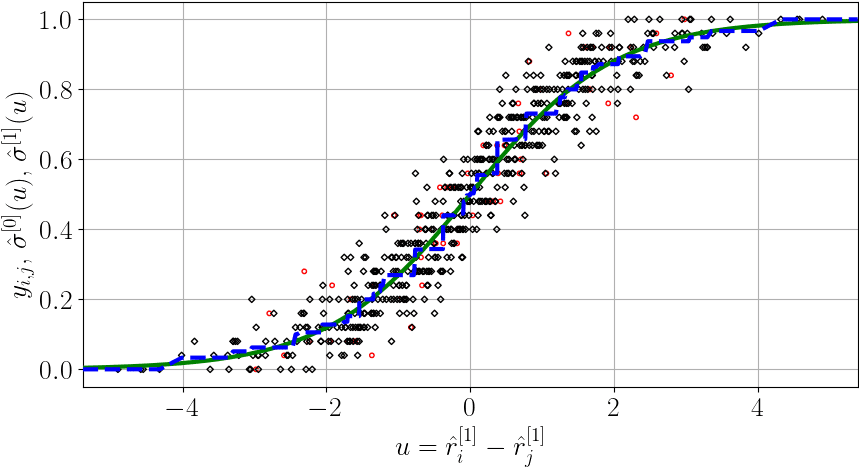}}\\
&\rotatebox{90}{\tiny\,~~~\,$100$}&
{\includegraphics[width=2.0cm]{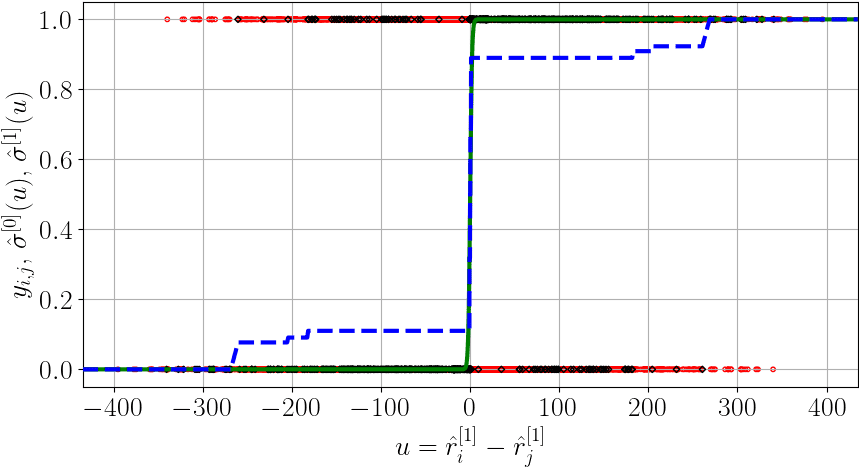}}&
{\includegraphics[width=2.0cm]{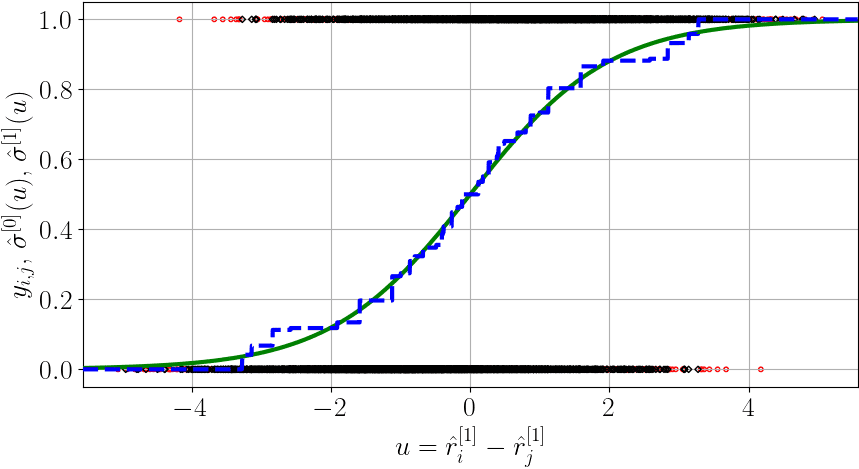}}&
{\includegraphics[width=2.0cm]{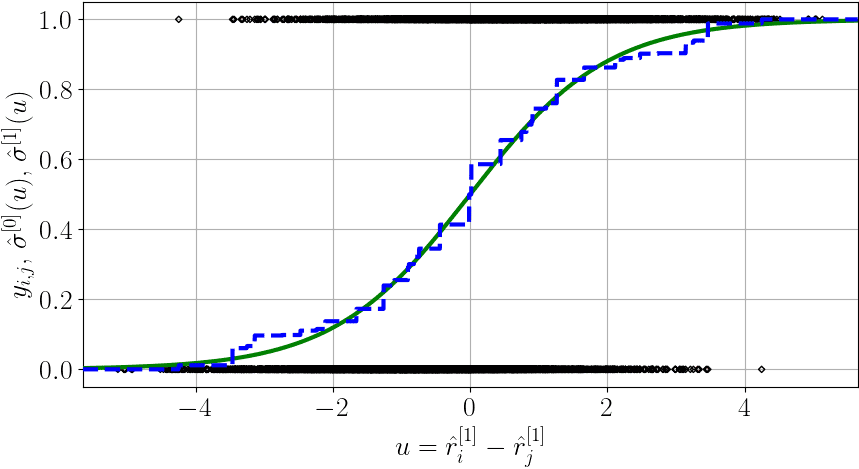}}&
{\includegraphics[width=2.0cm]{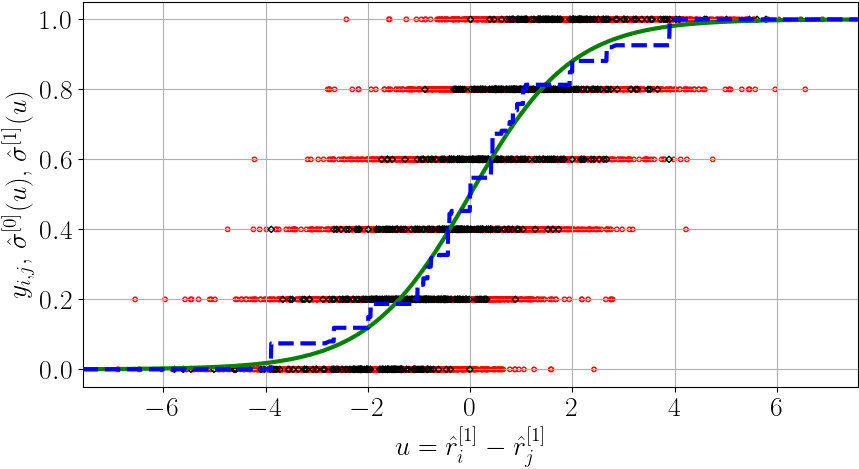}}&
{\includegraphics[width=2.0cm]{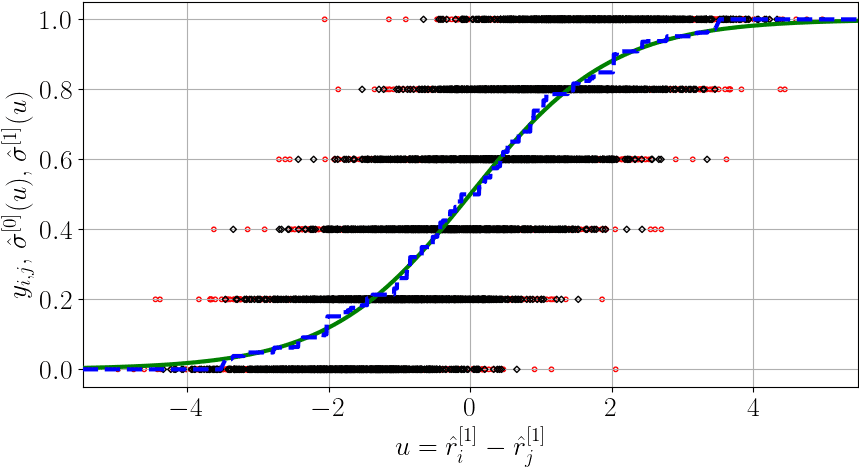}}&
{\includegraphics[width=2.0cm]{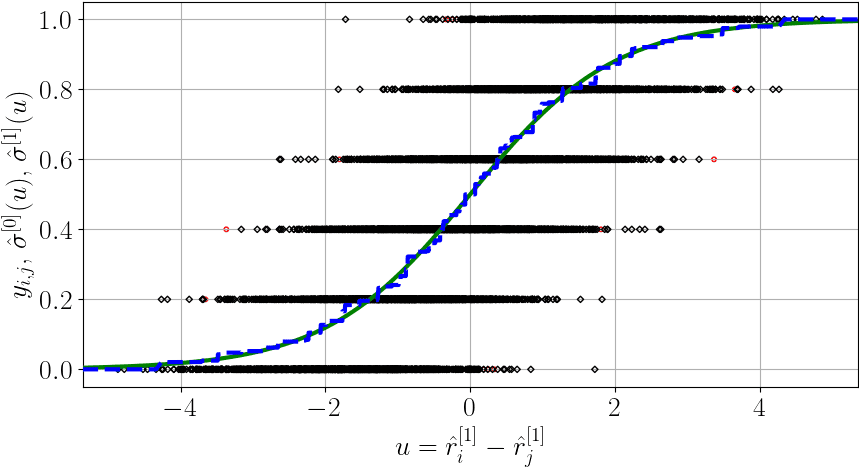}}&
{\includegraphics[width=2.0cm]{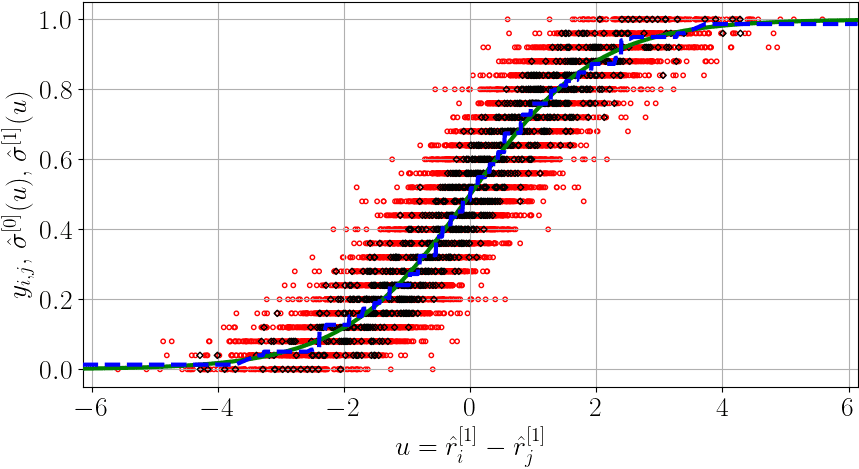}}&
{\includegraphics[width=2.0cm]{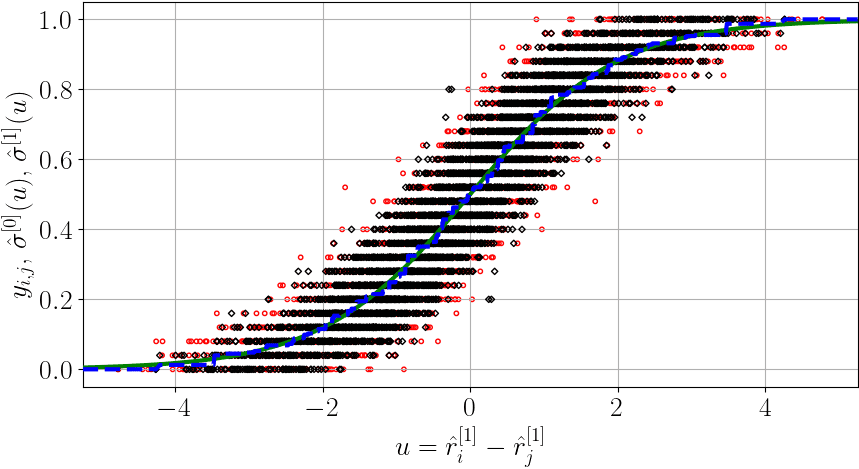}}&
{\includegraphics[width=2.0cm]{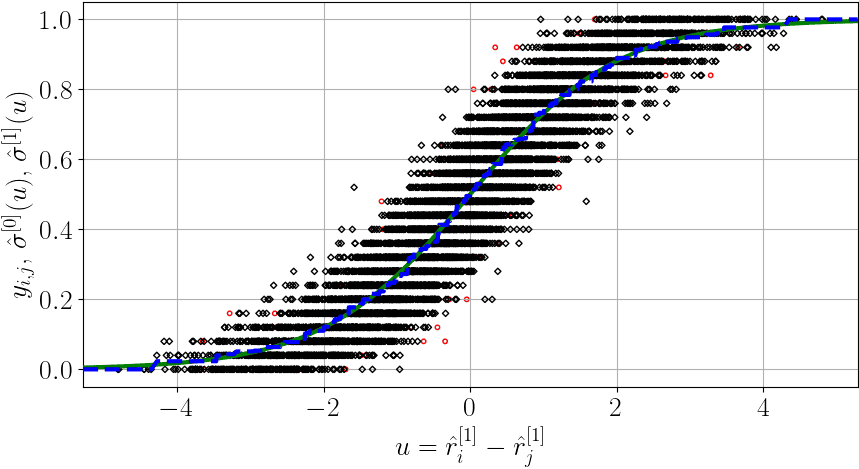}}\\
&\rotatebox{90}{\tiny\,~~~\,$400$}&
{\includegraphics[width=2.0cm]{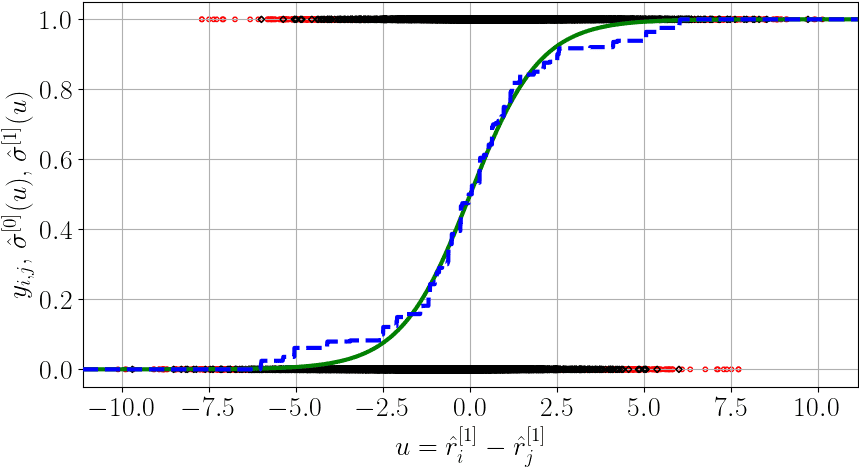}}&
{\includegraphics[width=2.0cm]{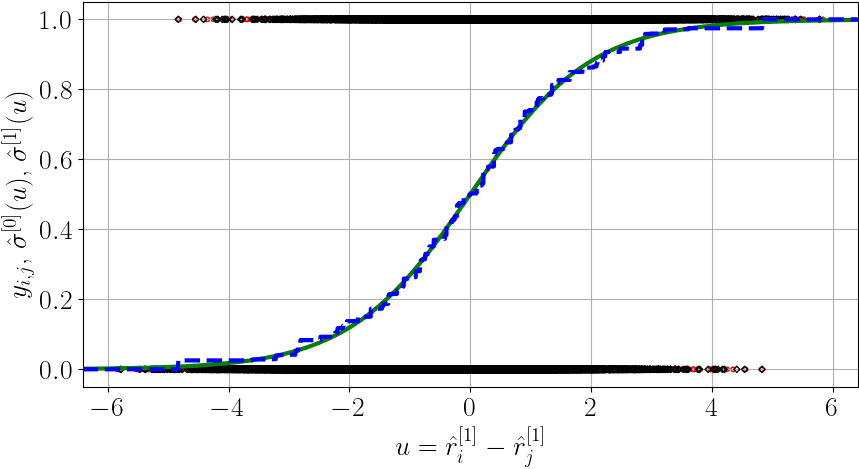}}&
{\includegraphics[width=2.0cm]{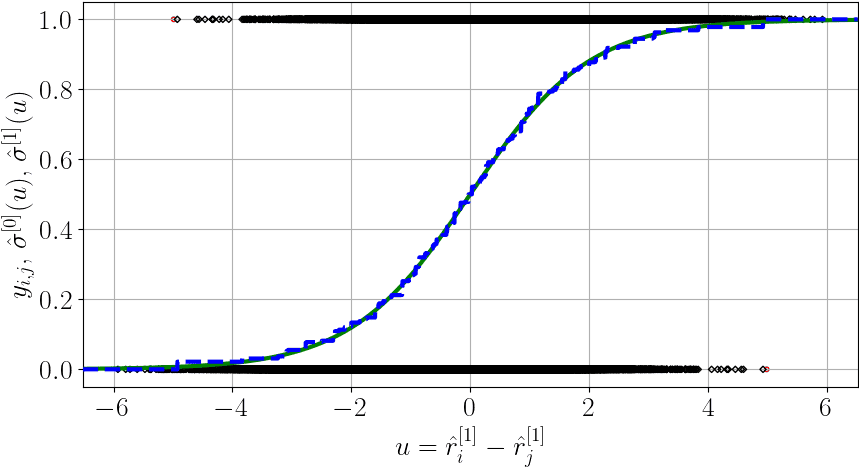}}&
{\includegraphics[width=2.0cm]{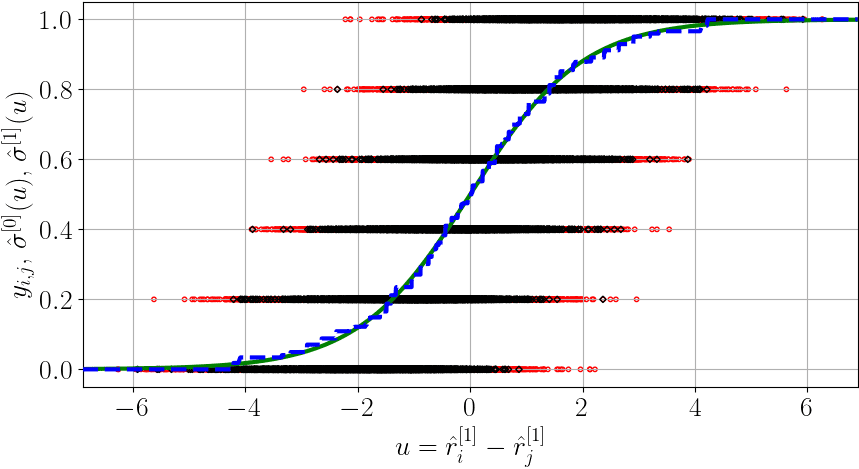}}&
{\includegraphics[width=2.0cm]{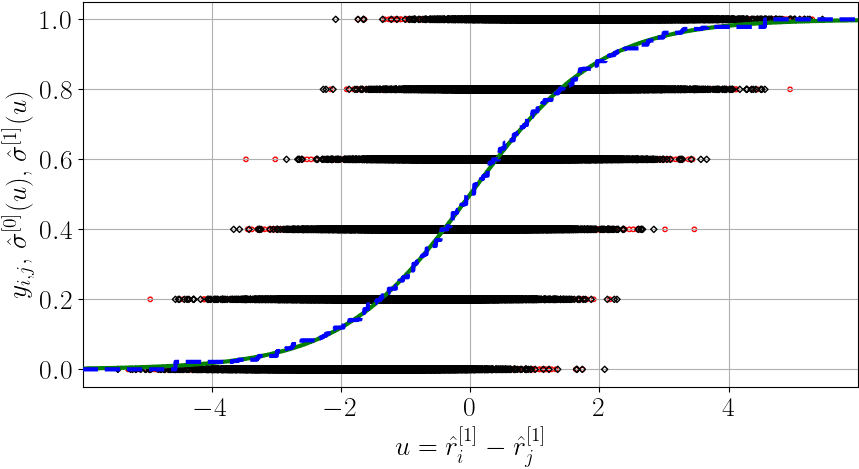}}&
{\includegraphics[width=2.0cm]{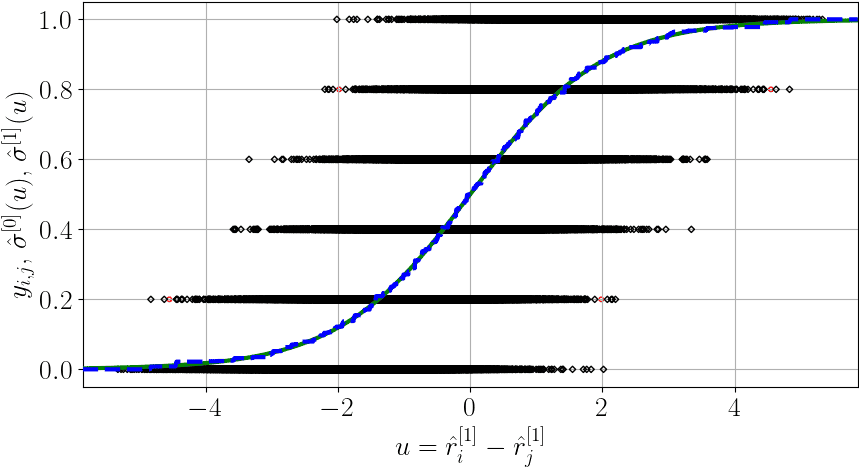}}&
{\includegraphics[width=2.0cm]{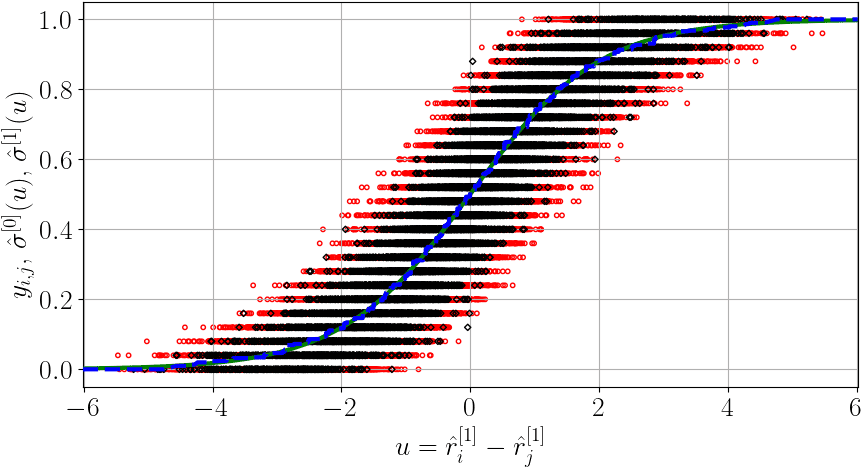}}&
{\includegraphics[width=2.0cm]{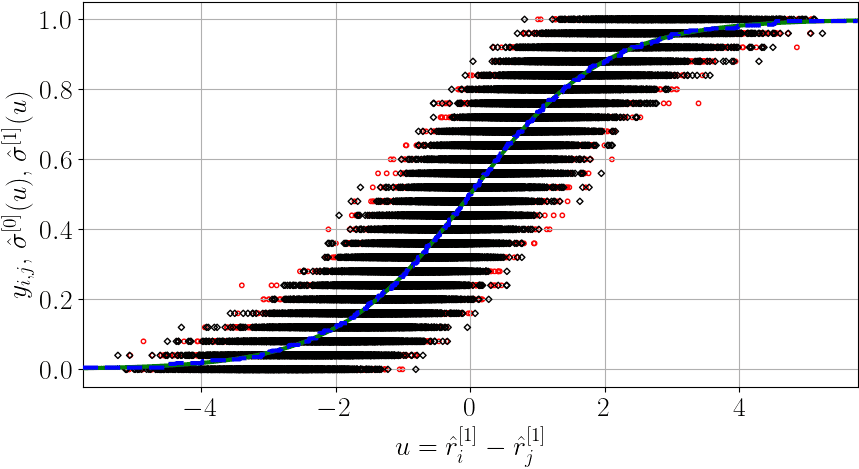}}&
{\includegraphics[width=2.0cm]{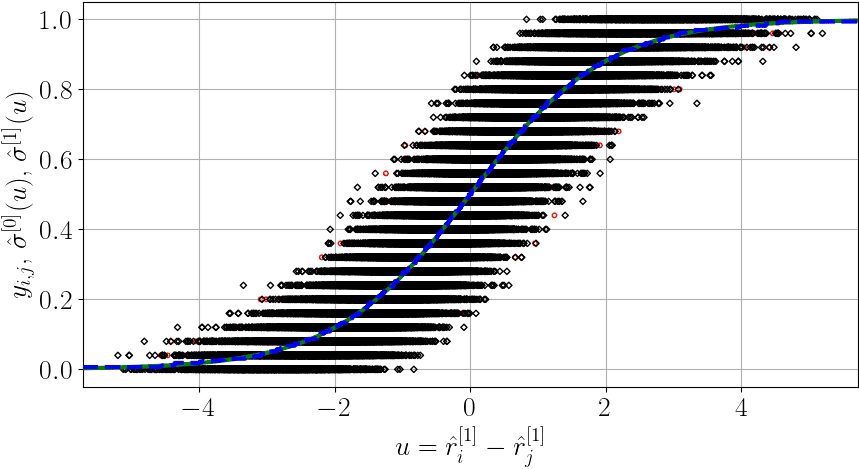}}
\end{tabular}
\begin{tabular}{ccccc}%
\multicolumn{5}{c}{\tiny Rescaled version: $n$, $N$, $|D_\tra|:|D_\tes|=$}\\
{\tiny25, 1, 1:9}&{\tiny25, 1, 5:5}&{\tiny25, 1, 9:1}&{\tiny25, 5, 1:9}&{\tiny100, 1, 1:9}\\
\midrule
{\includegraphics[width=2.0cm]{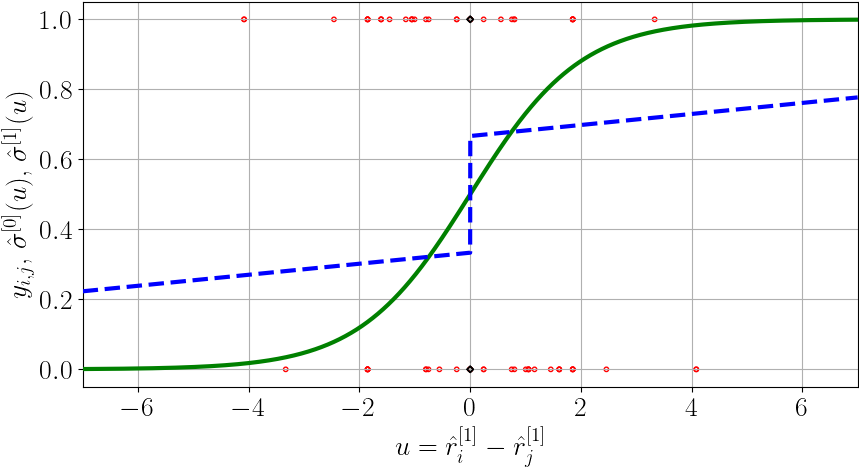}}&
{\includegraphics[width=2.0cm]{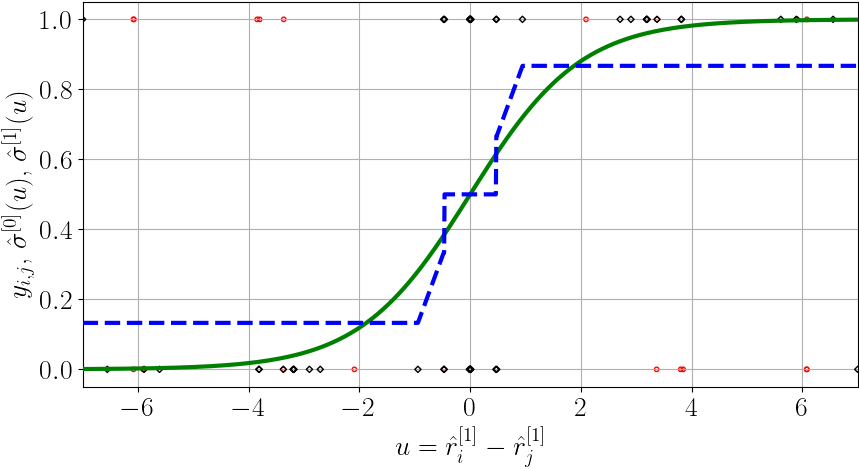}}&
{\includegraphics[width=2.0cm]{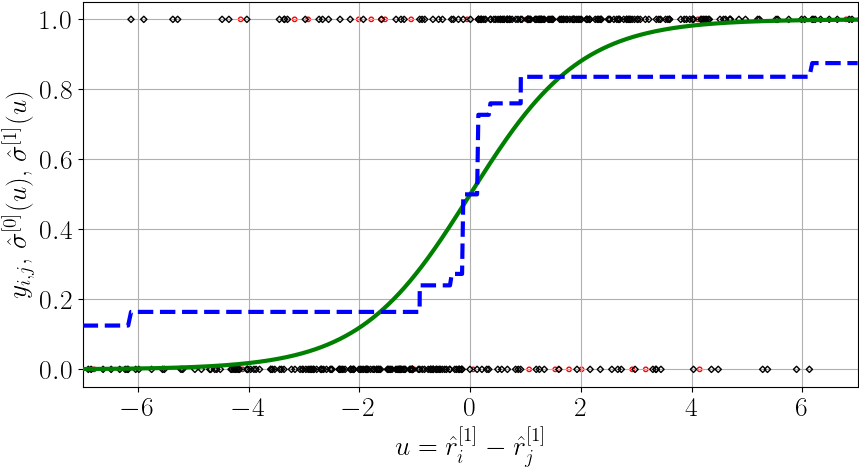}}&
{\includegraphics[width=2.0cm]{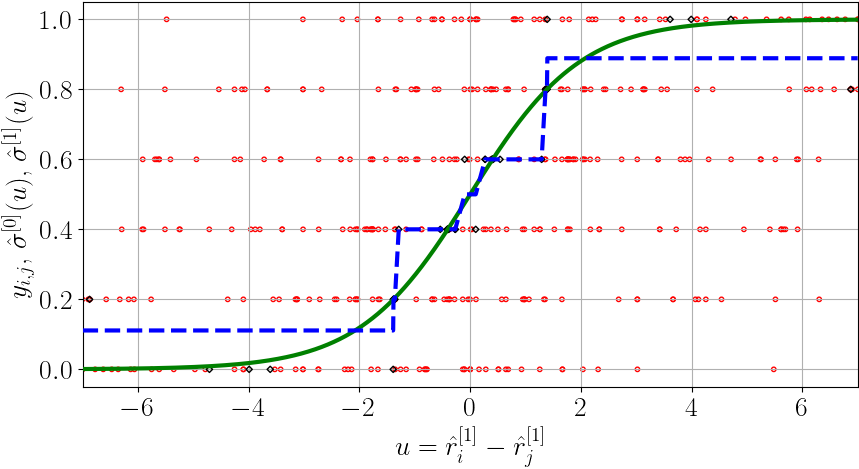}}&
{\includegraphics[width=2.0cm]{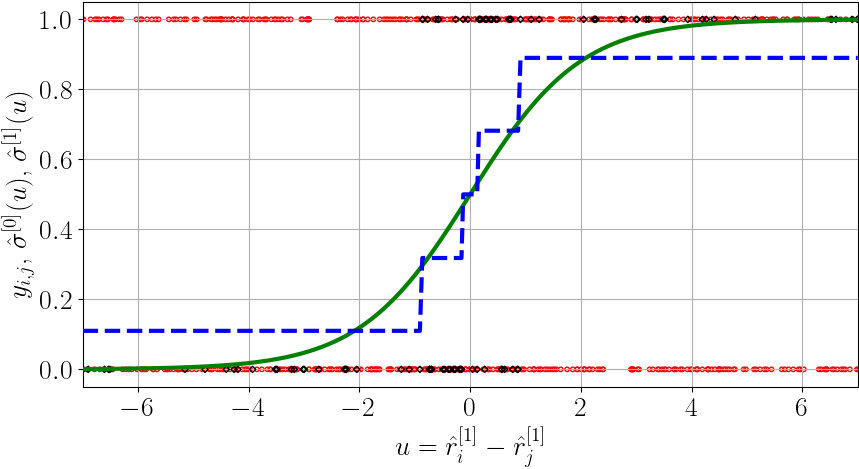}}
\end{tabular}
\caption{%
Part of results of synthetic data experiments, Procedure 1 in Section~\ref{sec:Synthetic} and Appendix~\ref{sec:Specified}:
For Logistic-$N$ synthetic data with $N=1,5,25$ (left to right),
black diamonds and red circles are training and test data $y_{i,j}$, 
and a green curve and a blue polyline are
the Bradley-Terry model $\hat{\sigma}^{[0]}(\hat{r}_i^{[1]}-\hat{r}_j^{[1]})$ and
isotonic Bradley-Terry model $\hat{\sigma}^{[1]}(\hat{r}_i^{[1]}-\hat{r}_j^{[1]})$
learned with the squared loss $\phi=\phi_\sq$ in a certain trial,
over the range $[-1.1\cdot\max_{i,j}|\hat{r}_i^{[1]}-\hat{r}_j^{[1]}|,1.1\cdot\max_{i,j}|\hat{r}_i^{[1]}-\hat{r}_j^{[1]}|]$
or $[-7,7]$ in the rescaled version.}
\label{fig:Res-Logistic-SQ}
\end{sidewaysfigure}
\begin{sidewaystable}
\centering%
\renewcommand{\arraystretch}{0.5}%
\renewcommand{\tabcolsep}{0.5pt}%
\caption{%
Results of real-world data experiments, Procedure 2 in Section~\ref{sec:Synthetic} and Appendix~\ref{sec:Specified}:
For Logistic-$N$ synthetic data with $N=1,5,25$ (left to right),
mean and std ($\text{mean}_{\text{std}}$) of 1000 trial test evaluation of 
WPP error \eqref{eq:WPE} with the squared loss $\phi=\phi_\sq$
and Kendall's Tau \eqref{eq:Kendall}
for the Bradley-Terry model (upper) and model-selected isotonic 
Bradley-Terry model learned with the squared loss $\phi=\phi_\sq$ (lower).
Smaller \eqref{eq:WPE}, or larger \eqref{eq:Kendall} indicates a better model.
It also includes $p$-value for the Mann-Whitney U test in the parentheses:
A value below the significance level 0.05 implies 
that the isotonic Bradley-Terry model performs significantly better
and is highlighted in red bold.}
\label{tab:Logistic-SQ}
\scalebox{0.9}{\begin{minipage}{25cm}
\begin{tabular}{cc|ccccc|ccccc|ccccc}%
&&\multicolumn{5}{c|}{\tiny $N=1$, $|D_\tra|:|D_\tes|=$}&\multicolumn{5}{c|}{\tiny $N=5$, $|D_\tra|:|D_\tes|=$}&\multicolumn{5}{c}{\tiny $N=25$, $|D_\tra|:|D_\tes|=$}\\
&&{\tiny$1:9$}&{\tiny$3:7$}&{\tiny$5:5$}&{\tiny$7:3$}&{\tiny$9:1$}&{\tiny$1:9$}&{\tiny$3:7$}&{\tiny$5:5$}&{\tiny$7:3$}&{\tiny$9:1$}&{\tiny$1:9$}&{\tiny$3:7$}&{\tiny$5:5$}&{\tiny$7:3$}&{\tiny$9:1$}
\\\midrule
\multirow{15}{*}[-6.75mm]{\rotatebox{90}{\tiny\eqref{eq:WPE}, $n=$}}
&\multirow{3}{*}{\rotatebox{90}{\tiny\,$25$~~}}
&\tcr{$\bf.3880_{.0477}$}&\tcr{$\bf.3464_{.0475}$}&\tcr{$\bf.2870_{.0541}$}&\tcr{$\bf.2315_{.0433}$}&$.2102_{.0519}$&\tcr{$\bf.1303_{.0331}$}&$.0544_{.0094}$&$.0451_{.0067}$&$.0422_{.0073}$&$.0407_{.0114}$&$.0337_{.0128}$&$.0105_{.0017}$&$.0090_{.0013}$&$.0084_{.0015}$&$.0081_{.0024}$\\&
&\tcr{$\bf.3611_{.0504}$}&\tcr{$\bf.3022_{.0392}$}&\tcr{$\bf.2545_{.0361}$}&\tcr{$\bf.2240_{.0356}$}&$.2100_{.0508}$&\tcr{$\bf.1274_{.0310}$}&$.0551_{.0093}$&$.0455_{.0068}$&$.0425_{.0075}$&$.0411_{.0115}$&$.0339_{.0127}$&$.0107_{.0017}$&$.0091_{.0014}$&$.0085_{.0015}$&$.0082_{.0024}$\\&
&(\tcr{$\bf.0000$})&(\tcr{$\bf.0000$})&(\tcr{$\bf.0000$})&(\tcr{$\bf.0010$})&($.5361$)&(\tcr{$\bf.0362$})&($.9767$)&($.9472$)&($.8325$)&($.7538$)&($.6652$)&($.9977$)&($.9642$)&($.8482$)&($.6980$)\\
&\multirow{3}{*}{\rotatebox{90}{\tiny\,$50$~~}}
&\tcr{$\bf.3752_{.0298}$}&\tcr{$\bf.2710_{.0398}$}&$.2068_{.0152}$&$.1970_{.0144}$&$.1928_{.0205}$&$.0750_{.0122}$&$.0434_{.0031}$&$.0402_{.0030}$&$.0390_{.0034}$&$.0383_{.0051}$&$.0141_{.0027}$&$.0086_{.0006}$&$.0080_{.0006}$&$.0078_{.0007}$&$.0076_{.0011}$\\&
&\tcr{$\bf.3450_{.0267}$}&\tcr{$\bf.2434_{.0204}$}&$.2068_{.0142}$&$.1976_{.0146}$&$.1930_{.0205}$&$.0745_{.0111}$&$.0437_{.0032}$&$.0403_{.0030}$&$.0391_{.0034}$&$.0383_{.0051}$&$.0143_{.0027}$&$.0087_{.0007}$&$.0080_{.0006}$&$.0078_{.0007}$&$.0077_{.0011}$\\&
&(\tcr{$\bf.0000$})&(\tcr{$\bf.0000$})&($.6370$)&($.7658$)&($.5799$)&($.3816$)&($.9603$)&($.8317$)&($.7519$)&($.6307$)&($.9572$)&($.9445$)&($.8182$)&($.6749$)&($.5883$)\\
&\multirow{3}{*}{\rotatebox{90}{\tiny\,$100$~~}}
&\tcr{$\bf.3435_{.0173}$}&$.2001_{.0077}$&$.1911_{.0074}$&$.1880_{.0082}$&$.1864_{.0109}$&$.0484_{.0027}$&$.0394_{.0016}$&$.0381_{.0016}$&$.0376_{.0019}$&$.0372_{.0027}$&$.0095_{.0005}$&$.0079_{.0003}$&$.0076_{.0003}$&$.0075_{.0004}$&$.0074_{.0006}$\\&
&\tcr{$\bf.2907_{.0126}$}&$.2003_{.0076}$&$.1913_{.0074}$&$.1881_{.0082}$&$.1865_{.0110}$&$.0486_{.0028}$&$.0395_{.0016}$&$.0381_{.0016}$&$.0376_{.0019}$&$.0372_{.0027}$&$.0095_{.0006}$&$.0079_{.0003}$&$.0076_{.0003}$&$.0075_{.0004}$&$.0074_{.0006}$\\&
&(\tcr{$\bf.0000$})&($.7224$)&($.7049$)&($.5934$)&($.5593$)&($.9628$)&($.6839$)&($.6029$)&($.5398$)&($.5284$)&($.8941$)&($.6309$)&($.5824$)&($.5240$)&($.4992$)\\
&\multirow{3}{*}{\rotatebox{90}{\tiny\,$200$~~}}
&\tcr{$\bf.2241_{.0122}$}&$.1895_{.0047}$&$.1861_{.0048}$&$.1847_{.0051}$&$.1840_{.0063}$&$.0413_{.0011}$&$.0378_{.0010}$&$.0372_{.0010}$&$.0369_{.0011}$&$.0368_{.0015}$&$.0082_{.0002}$&$.0076_{.0002}$&$.0074_{.0002}$&$.0074_{.0002}$&$.0074_{.0003}$\\&
&\tcr{$\bf.2176_{.0069}$}&$.1895_{.0047}$&$.1861_{.0048}$&$.1847_{.0051}$&$.1840_{.0063}$&$.0413_{.0011}$&$.0378_{.0010}$&$.0372_{.0010}$&$.0369_{.0011}$&$.0368_{.0015}$&$.0082_{.0002}$&$.0076_{.0002}$&$.0074_{.0002}$&$.0074_{.0002}$&$.0074_{.0003}$\\&
&(\tcr{$\bf.0000$})&($.5886$)&($.5340$)&($.5075$)&($.5082$)&($.6519$)&($.5377$)&($.5147$)&($.5053$)&($.5029$)&($.5746$)&($.5215$)&($.5021$)&($.5075$)&($.5000$)\\
&\multirow{3}{*}{\rotatebox{90}{\tiny\,$400$~~}}
&$.1945_{.0032}$&$.1854_{.0032}$&$.1839_{.0032}$&$.1832_{.0033}$&$.1829_{.0039}$&$.0386_{.0007}$&$.0371_{.0007}$&$.0368_{.0007}$&$.0366_{.0007}$&$.0366_{.0009}$&$.0077_{.0001}$&$.0074_{.0001}$&$.0074_{.0001}$&$.0073_{.0001}$&$.0073_{.0002}$\\&
&$.1945_{.0033}$&$.1854_{.0032}$&$.1839_{.0032}$&$.1832_{.0033}$&$.1829_{.0039}$&$.0386_{.0007}$&$.0371_{.0007}$&$.0368_{.0007}$&$.0366_{.0007}$&$.0366_{.0009}$&$.0077_{.0001}$&$.0074_{.0001}$&$.0074_{.0001}$&$.0073_{.0001}$&$.0073_{.0002}$\\&
&($.6775$)&($.5125$)&($.5060$)&($.5042$)&($.5016$)&($.5271$)&($.5066$)&($.5015$)&($.5000$)&($.5022$)&($.5110$)&($.5000$)&($.5000$)&($.5000$)&($.5000$)\\
\midrule
\multirow{15}{*}[-6.75mm]{\rotatebox{90}{\tiny\eqref{eq:Kendall}, $n=$}}
&\multirow{3}{*}{\rotatebox{90}{\tiny\,$25$~~}}
&$.1839_{.0879}$&\tcr{$\bf.2705_{.0834}$}&\tcr{$\bf.3118_{.0800}$}&\tcr{$\bf.3464_{.0870}$}&\tcr{$\bf.3683_{.1309}$}&$.3532_{.0893}$&\tcr{$\bf.5582_{.0553}$}&\tcr{$\bf.5951_{.0526}$}&\tcr{$\bf.6093_{.0579}$}&\tcr{$\bf.6173_{.0834}$}&$.6021_{.0770}$&\tcr{$\bf.7749_{.0303}$}&\tcr{$\bf.7918_{.0287}$}&\tcr{$\bf.7986_{.0314}$}&\tcr{$\bf.8027_{.0454}$}\\&
&$.1858_{.0898}$&\tcr{$\bf.2810_{.0867}$}&\tcr{$\bf.3287_{.0864}$}&\tcr{$\bf.3651_{.0936}$}&\tcr{$\bf.3861_{.1414}$}&$.3594_{.0913}$&\tcr{$\bf.5730_{.0570}$}&\tcr{$\bf.6115_{.0533}$}&\tcr{$\bf.6262_{.0589}$}&\tcr{$\bf.6328_{.0851}$}&$.6059_{.0771}$&\tcr{$\bf.7857_{.0300}$}&\tcr{$\bf.8033_{.0286}$}&\tcr{$\bf.8101_{.0310}$}&\tcr{$\bf.8147_{.0440}$}\\&
&($.3055$)&(\tcr{$\bf.0036$})&(\tcr{$\bf.0000$})&(\tcr{$\bf.0000$})&(\tcr{$\bf.0024$})&($.0595$)&(\tcr{$\bf.0000$})&(\tcr{$\bf.0000$})&(\tcr{$\bf.0000$})&(\tcr{$\bf.0000$})&($.1291$)&(\tcr{$\bf.0000$})&(\tcr{$\bf.0000$})&(\tcr{$\bf.0000$})&(\tcr{$\bf.0000$})\\
&\multirow{3}{*}{\rotatebox{90}{\tiny\,$50$~~}}
&\tcr{$\bf.2300_{.0545}$}&\tcr{$\bf.3201_{.0422}$}&\tcr{$\bf.3713_{.0400}$}&\tcr{$\bf.3878_{.0435}$}&\tcr{$\bf.3965_{.0613}$}&\tcr{$\bf.4892_{.0463}$}&\tcr{$\bf.6027_{.0312}$}&\tcr{$\bf.6191_{.0303}$}&\tcr{$\bf.6256_{.0320}$}&\tcr{$\bf.6297_{.0424}$}&\tcr{$\bf.7412_{.0287}$}&\tcr{$\bf.7962_{.0171}$}&\tcr{$\bf.8037_{.0167}$}&\tcr{$\bf.8068_{.0175}$}&\tcr{$\bf.8085_{.0225}$}\\&
&\tcr{$\bf.2386_{.0572}$}&\tcr{$\bf.3395_{.0453}$}&\tcr{$\bf.3869_{.0423}$}&\tcr{$\bf.4037_{.0464}$}&\tcr{$\bf.4128_{.0645}$}&\tcr{$\bf.5014_{.0475}$}&\tcr{$\bf.6193_{.0317}$}&\tcr{$\bf.6360_{.0309}$}&\tcr{$\bf.6418_{.0324}$}&\tcr{$\bf.6455_{.0432}$}&\tcr{$\bf.7512_{.0287}$}&\tcr{$\bf.8079_{.0170}$}&\tcr{$\bf.8151_{.0164}$}&\tcr{$\bf.8179_{.0171}$}&\tcr{$\bf.8189_{.0222}$}\\&
&(\tcr{$\bf.0004$})&(\tcr{$\bf.0000$})&(\tcr{$\bf.0000$})&(\tcr{$\bf.0000$})&(\tcr{$\bf.0000$})&(\tcr{$\bf.0000$})&(\tcr{$\bf.0000$})&(\tcr{$\bf.0000$})&(\tcr{$\bf.0000$})&(\tcr{$\bf.0000$})&(\tcr{$\bf.0000$})&(\tcr{$\bf.0000$})&(\tcr{$\bf.0000$})&(\tcr{$\bf.0000$})&(\tcr{$\bf.0000$})\\
&\multirow{3}{*}{\rotatebox{90}{\tiny\,$100$~~}}
&\tcr{$\bf.2849_{.0318}$}&\tcr{$\bf.3830_{.0230}$}&\tcr{$\bf.4012_{.0232}$}&\tcr{$\bf.4086_{.0256}$}&\tcr{$\bf.4127_{.0339}$}&\tcr{$\bf.5812_{.0219}$}&\tcr{$\bf.6246_{.0190}$}&\tcr{$\bf.6323_{.0191}$}&\tcr{$\bf.6353_{.0201}$}&\tcr{$\bf.6372_{.0250}$}&\tcr{$\bf.7870_{.0121}$}&\tcr{$\bf.8067_{.0107}$}&\tcr{$\bf.8102_{.0108}$}&\tcr{$\bf.8116_{.0113}$}&\tcr{$\bf.8124_{.0135}$}\\&
&\tcr{$\bf.3002_{.0348}$}&\tcr{$\bf.3979_{.0242}$}&\tcr{$\bf.4165_{.0242}$}&\tcr{$\bf.4237_{.0268}$}&\tcr{$\bf.4274_{.0355}$}&\tcr{$\bf.5966_{.0224}$}&\tcr{$\bf.6410_{.0191}$}&\tcr{$\bf.6476_{.0191}$}&\tcr{$\bf.6499_{.0201}$}&\tcr{$\bf.6511_{.0252}$}&\tcr{$\bf.7985_{.0121}$}&\tcr{$\bf.8179_{.0104}$}&\tcr{$\bf.8204_{.0105}$}&\tcr{$\bf.8209_{.0110}$}&\tcr{$\bf.8212_{.0133}$}\\&
&(\tcr{$\bf.0000$})&(\tcr{$\bf.0000$})&(\tcr{$\bf.0000$})&(\tcr{$\bf.0000$})&(\tcr{$\bf.0000$})&(\tcr{$\bf.0000$})&(\tcr{$\bf.0000$})&(\tcr{$\bf.0000$})&(\tcr{$\bf.0000$})&(\tcr{$\bf.0000$})&(\tcr{$\bf.0000$})&(\tcr{$\bf.0000$})&(\tcr{$\bf.0000$})&(\tcr{$\bf.0000$})&(\tcr{$\bf.0000$})\\
&\multirow{3}{*}{\rotatebox{90}{\tiny\,$200$~~}}
&\tcr{$\bf.3483_{.0165}$}&\tcr{$\bf.4050_{.0152}$}&\tcr{$\bf.4136_{.0154}$}&\tcr{$\bf.4174_{.0162}$}&\tcr{$\bf.4196_{.0198}$}&\tcr{$\bf.6145_{.0139}$}&\tcr{$\bf.6342_{.0130}$}&\tcr{$\bf.6378_{.0131}$}&\tcr{$\bf.6395_{.0135}$}&\tcr{$\bf.6405_{.0153}$}&\tcr{$\bf.8023_{.0078}$}&\tcr{$\bf.8111_{.0074}$}&\tcr{$\bf.8127_{.0075}$}&\tcr{$\bf.8135_{.0076}$}&\tcr{$\bf.8139_{.0084}$}\\&
&\tcr{$\bf.3626_{.0170}$}&\tcr{$\bf.4191_{.0157}$}&\tcr{$\bf.4269_{.0158}$}&\tcr{$\bf.4302_{.0164}$}&\tcr{$\bf.4318_{.0202}$}&\tcr{$\bf.6302_{.0142}$}&\tcr{$\bf.6481_{.0130}$}&\tcr{$\bf.6504_{.0131}$}&\tcr{$\bf.6511_{.0135}$}&\tcr{$\bf.6513_{.0153}$}&\tcr{$\bf.8135_{.0077}$}&\tcr{$\bf.8205_{.0072}$}&\tcr{$\bf.8210_{.0073}$}&\tcr{$\bf.8209_{.0075}$}&\tcr{$\bf.8207_{.0082}$}\\&
&(\tcr{$\bf.0000$})&(\tcr{$\bf.0000$})&(\tcr{$\bf.0000$})&(\tcr{$\bf.0000$})&(\tcr{$\bf.0000$})&(\tcr{$\bf.0000$})&(\tcr{$\bf.0000$})&(\tcr{$\bf.0000$})&(\tcr{$\bf.0000$})&(\tcr{$\bf.0000$})&(\tcr{$\bf.0000$})&(\tcr{$\bf.0000$})&(\tcr{$\bf.0000$})&(\tcr{$\bf.0000$})&(\tcr{$\bf.0000$})\\
&\multirow{3}{*}{\rotatebox{90}{\tiny\,$400$~~}}
&\tcr{$\bf.3935_{.0107}$}&\tcr{$\bf.4157_{.0103}$}&\tcr{$\bf.4200_{.0103}$}&\tcr{$\bf.4219_{.0106}$}&\tcr{$\bf.4229_{.0122}$}&\tcr{$\bf.6294_{.0094}$}&\tcr{$\bf.6388_{.0092}$}&\tcr{$\bf.6405_{.0092}$}&\tcr{$\bf.6413_{.0093}$}&\tcr{$\bf.6418_{.0101}$}&\tcr{$\bf.8090_{.0053}$}&\tcr{$\bf.8132_{.0052}$}&\tcr{$\bf.8139_{.0052}$}&\tcr{$\bf.8143_{.0053}$}&\tcr{$\bf.8145_{.0056}$}\\&
&\tcr{$\bf.4061_{.0111}$}&\tcr{$\bf.4272_{.0105}$}&\tcr{$\bf.4305_{.0105}$}&\tcr{$\bf.4317_{.0107}$}&\tcr{$\bf.4321_{.0124}$}&\tcr{$\bf.6426_{.0094}$}&\tcr{$\bf.6495_{.0091}$}&\tcr{$\bf.6501_{.0092}$}&\tcr{$\bf.6501_{.0092}$}&\tcr{$\bf.6500_{.0100}$}&\tcr{$\bf.8185_{.0052}$}&\tcr{$\bf.8206_{.0051}$}&\tcr{$\bf.8203_{.0051}$}&\tcr{$\bf.8199_{.0052}$}&\tcr{$\bf.8197_{.0056}$}\\&
&(\tcr{$\bf.0000$})&(\tcr{$\bf.0000$})&(\tcr{$\bf.0000$})&(\tcr{$\bf.0000$})&(\tcr{$\bf.0000$})&(\tcr{$\bf.0000$})&(\tcr{$\bf.0000$})&(\tcr{$\bf.0000$})&(\tcr{$\bf.0000$})&(\tcr{$\bf.0000$})&(\tcr{$\bf.0000$})&(\tcr{$\bf.0000$})&(\tcr{$\bf.0000$})&(\tcr{$\bf.0000$})&(\tcr{$\bf.0000$})\\
\end{tabular}\end{minipage}}
\end{sidewaystable}

\begin{sidewaysfigure}
\centering%
\renewcommand{\arraystretch}{0.5}%
\renewcommand{\tabcolsep}{0.5pt}%
\begin{tabular}{cc|ccc|ccc|ccc}%
&&\multicolumn{3}{c|}{\tiny$N=1$, $|D_\tra|:|D_\tes|=$}&\multicolumn{3}{c|}{\tiny$N=5$, $|D_\tra|:|D_\tes|=$}&\multicolumn{3}{c}{\tiny$N=25$, $|D_\tra|:|D_\tes|=$}\\
&&{\tiny$1:9$}&{\tiny$5:5$}&{\tiny$9:1$}&{\tiny$1:9$}&{\tiny$5:5$}&{\tiny$9:1$}&{\tiny$1:9$}&{\tiny$5:5$}&{\tiny$9:1$}\\
\midrule
\multirow{3}{*}[-2.5mm]{\rotatebox{90}{\tiny\eqref{eq:WPE}, $n=$}}
&\rotatebox{90}{\tiny\,~~~\,$25$}&
{\includegraphics[width=2.0cm]{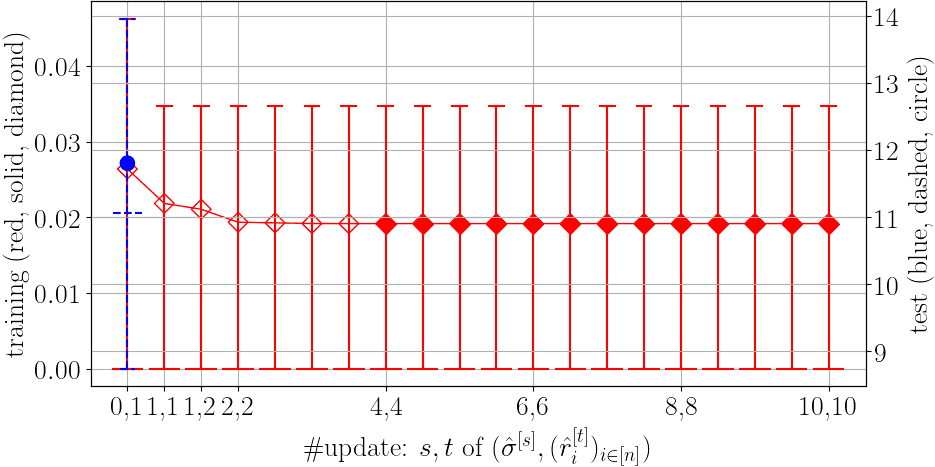}}&
{\includegraphics[width=2.0cm]{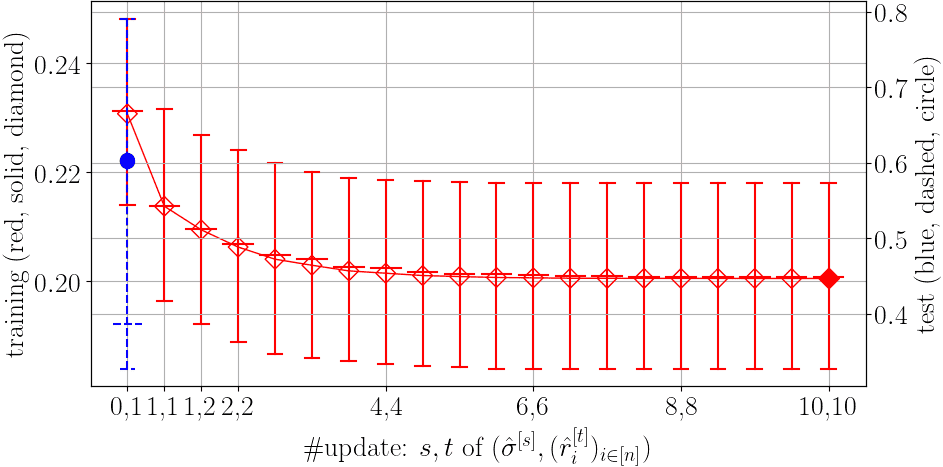}}&
{\includegraphics[width=2.0cm]{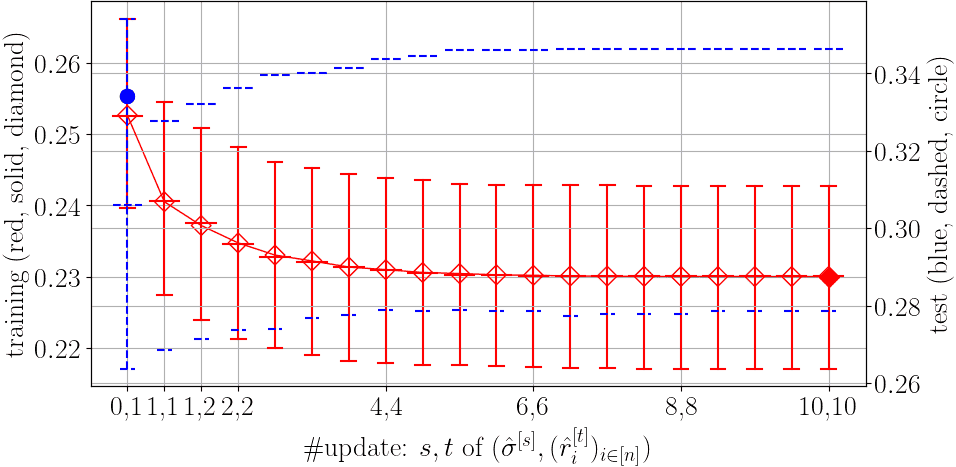}}&
{\includegraphics[width=2.0cm]{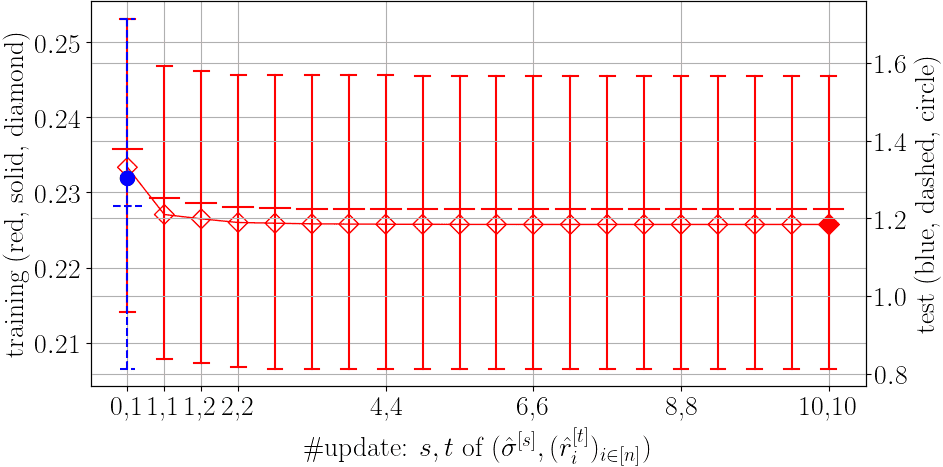}}&
{\includegraphics[width=2.0cm]{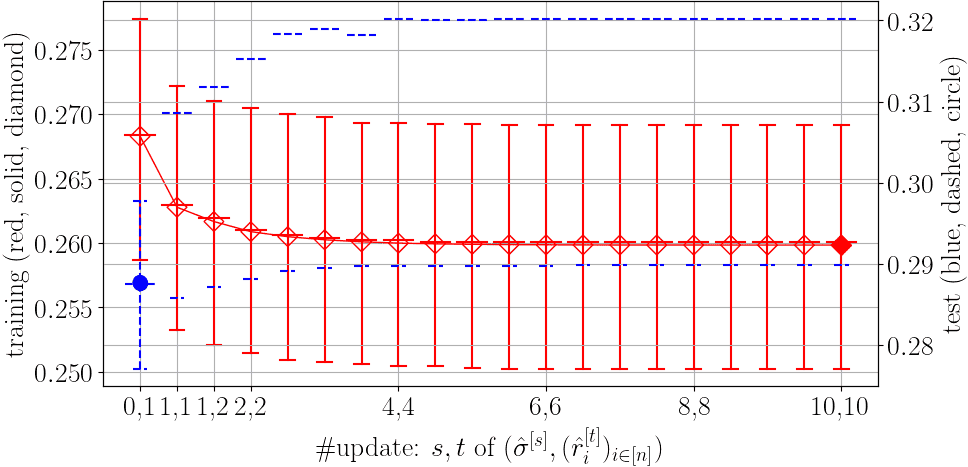}}&
{\includegraphics[width=2.0cm]{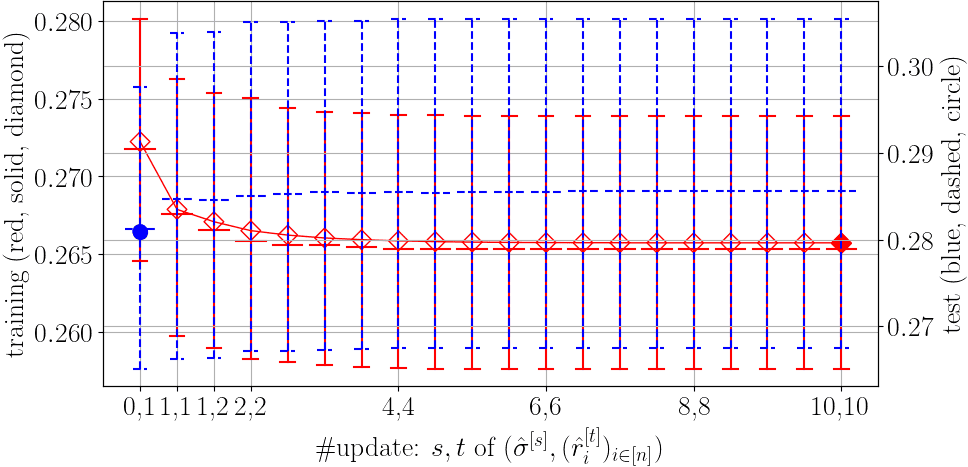}}&
{\includegraphics[width=2.0cm]{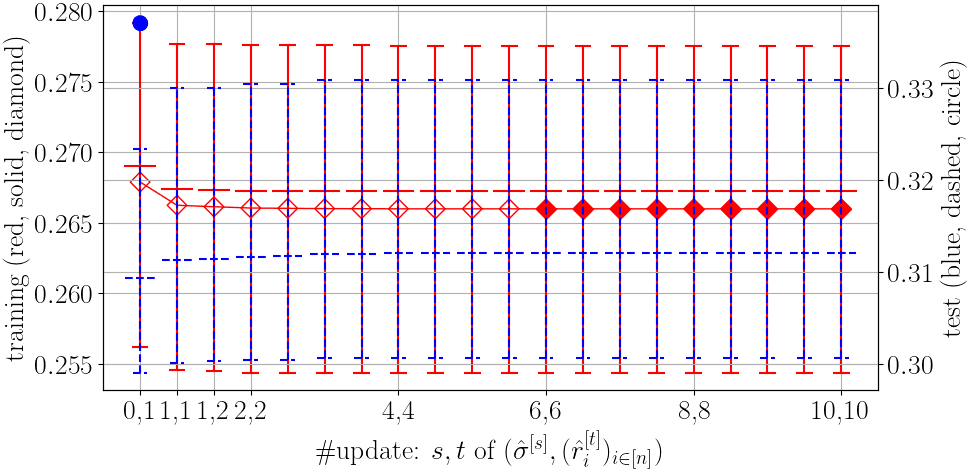}}&
{\includegraphics[width=2.0cm]{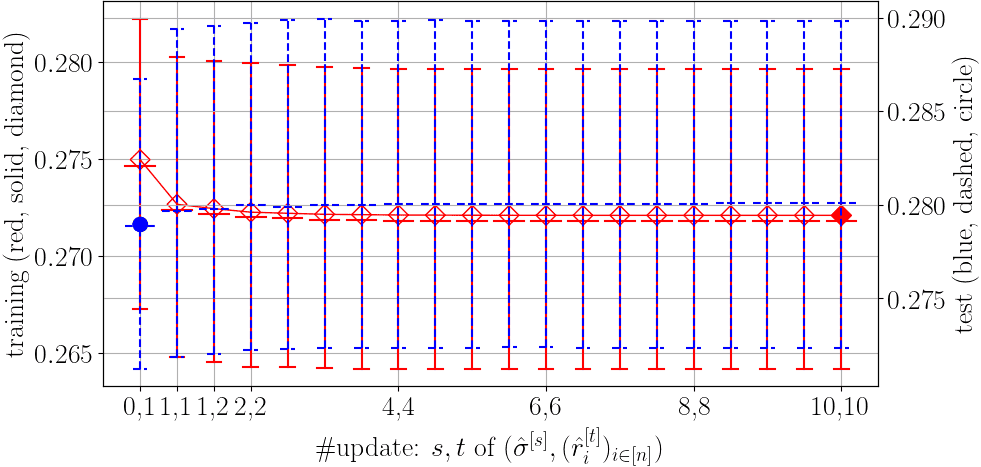}}&
{\includegraphics[width=2.0cm]{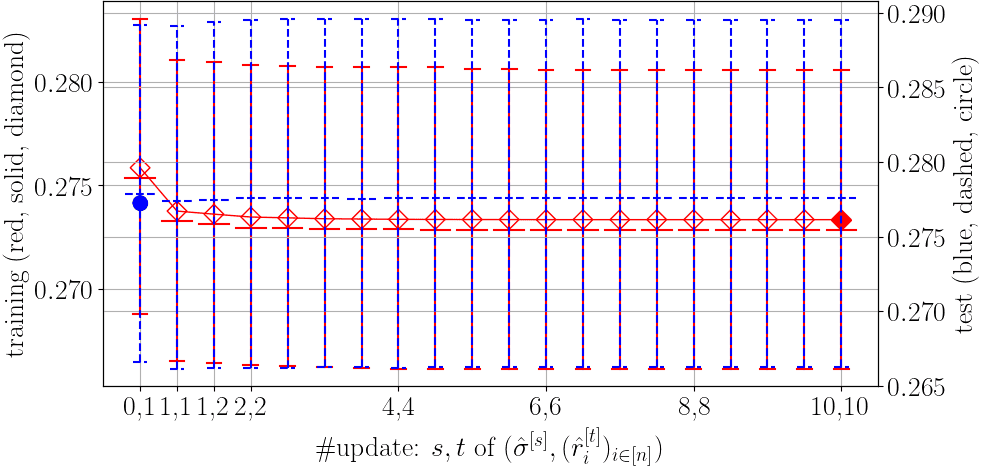}}\\
&\rotatebox{90}{\tiny\,~~\,$100$}&
{\includegraphics[width=2.0cm]{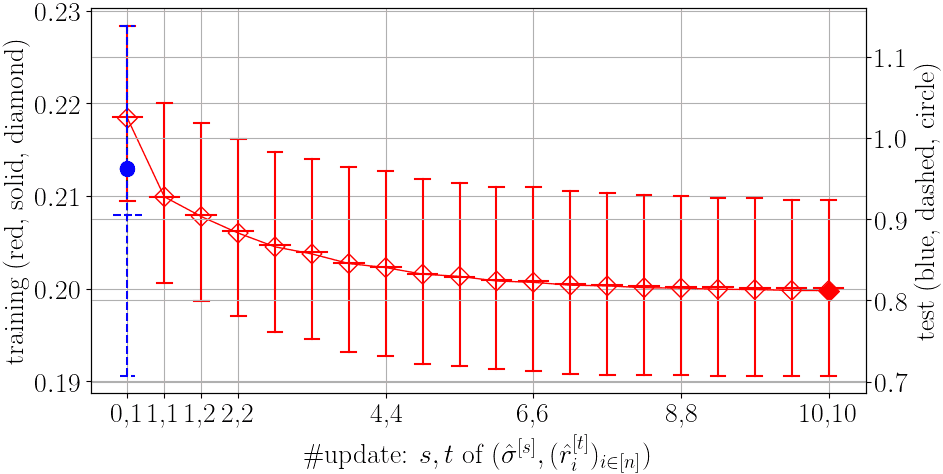}}&
{\includegraphics[width=2.0cm]{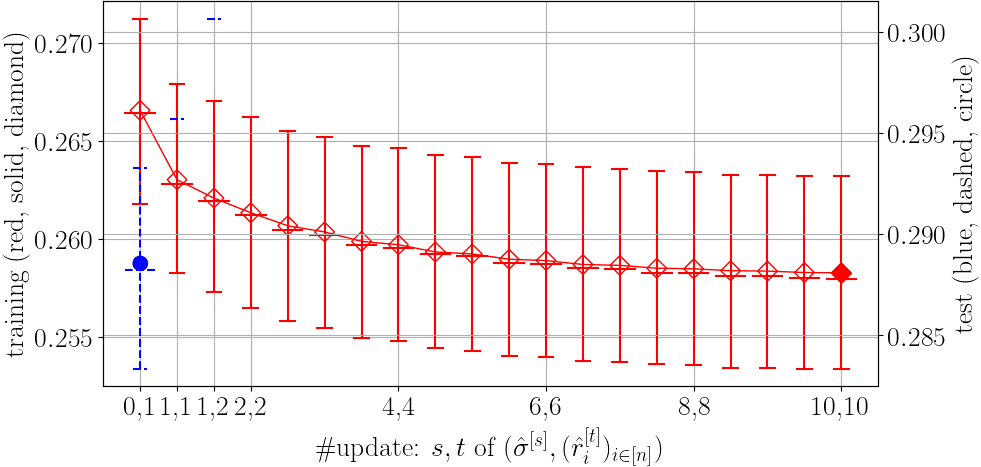}}&
{\includegraphics[width=2.0cm]{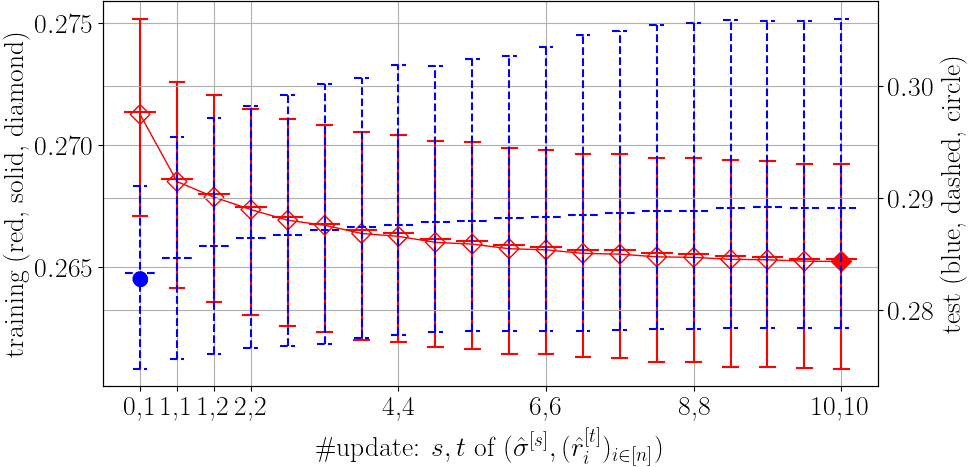}}&
{\includegraphics[width=2.0cm]{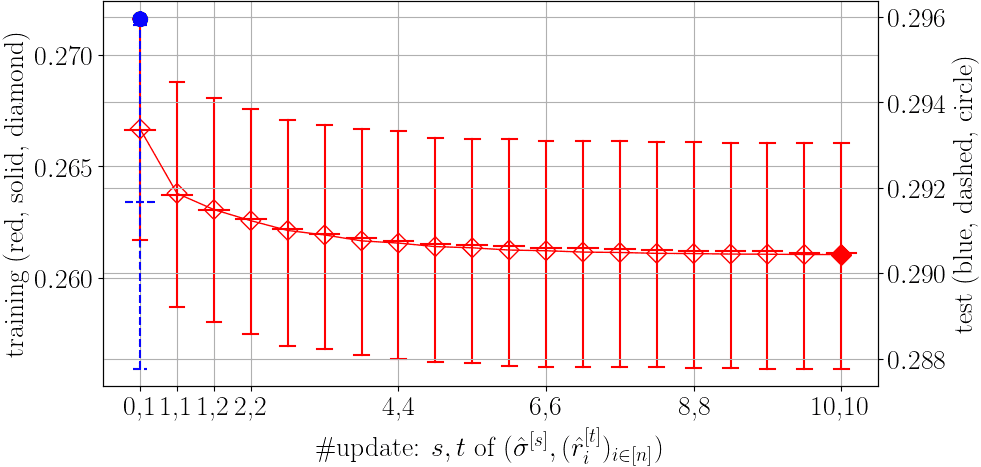}}&
{\includegraphics[width=2.0cm]{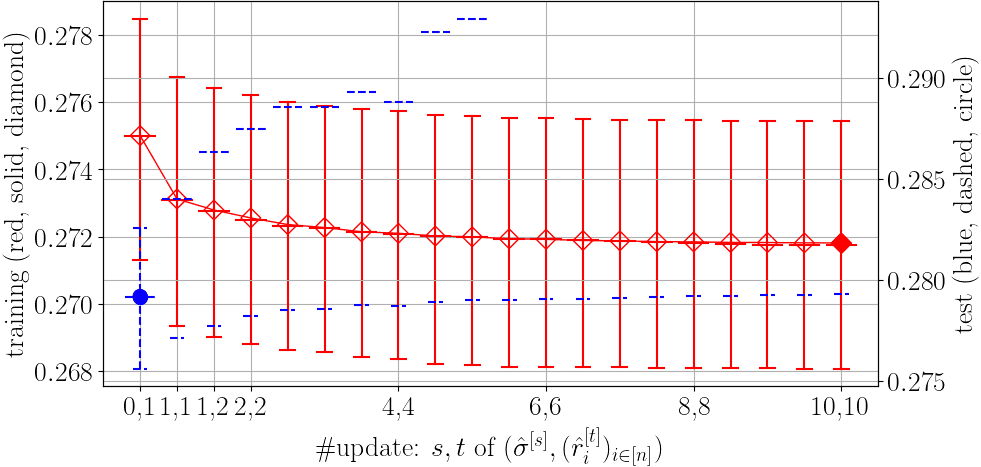}}&
{\includegraphics[width=2.0cm]{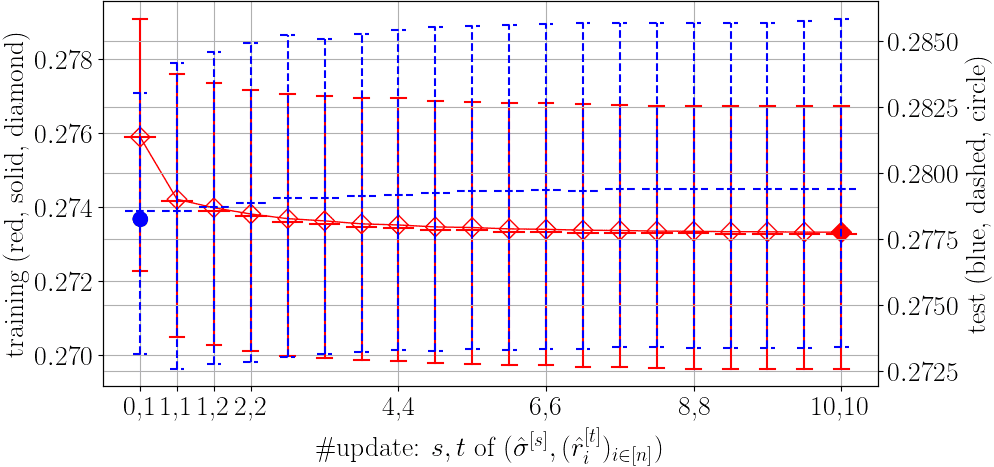}}&
{\includegraphics[width=2.0cm]{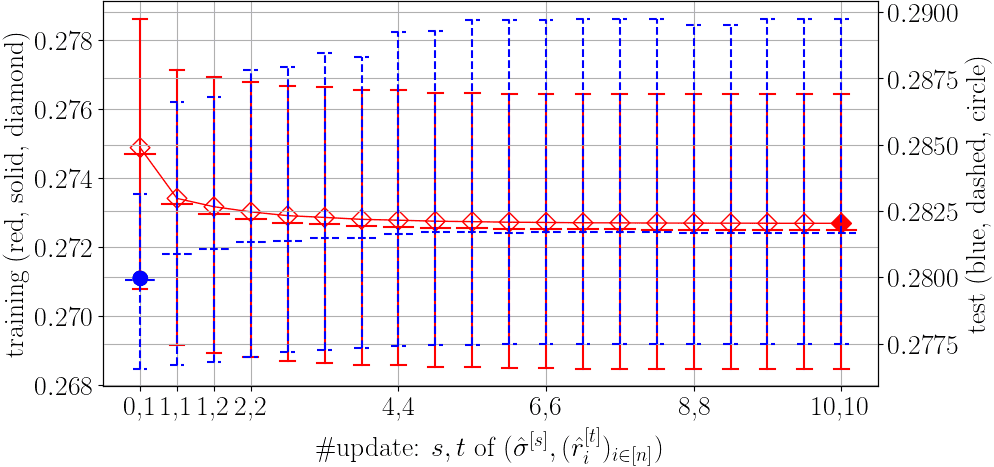}}&
{\includegraphics[width=2.0cm]{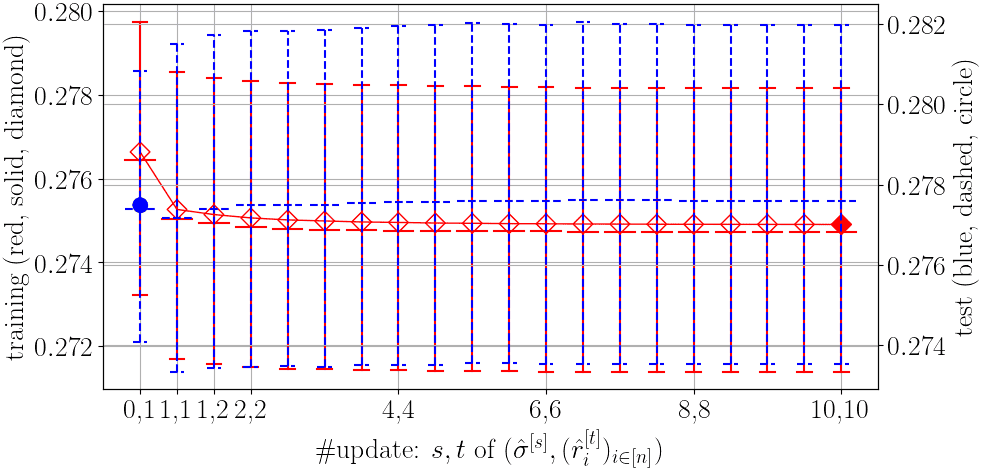}}&
{\includegraphics[width=2.0cm]{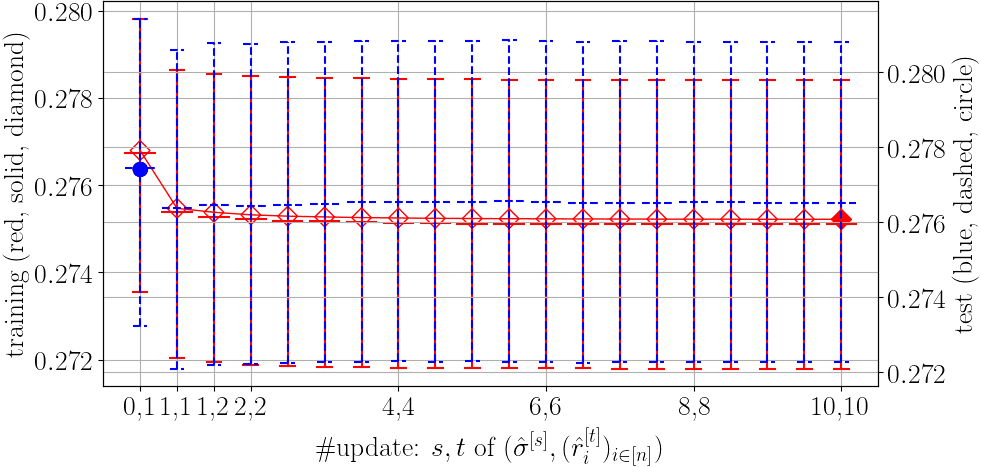}}\\
&\rotatebox{90}{\tiny\,~~\,$400$}&
{\includegraphics[width=2.0cm]{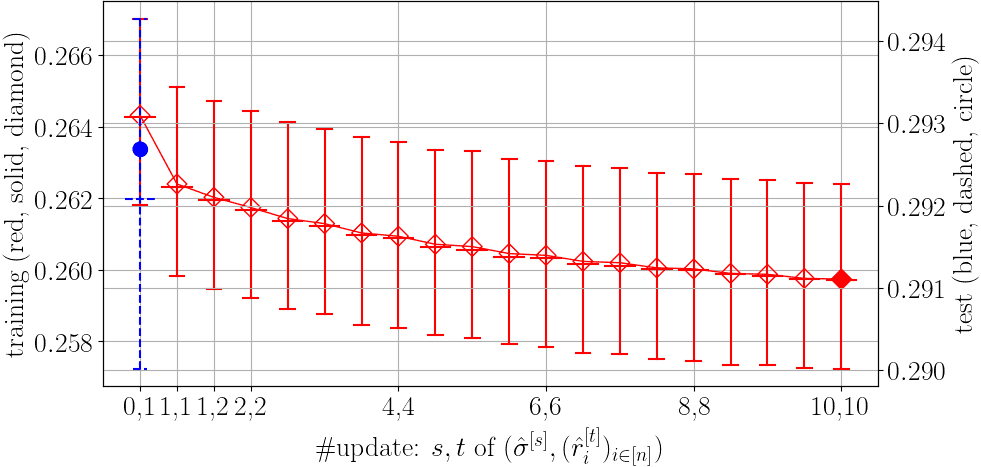}}&
{\includegraphics[width=2.0cm]{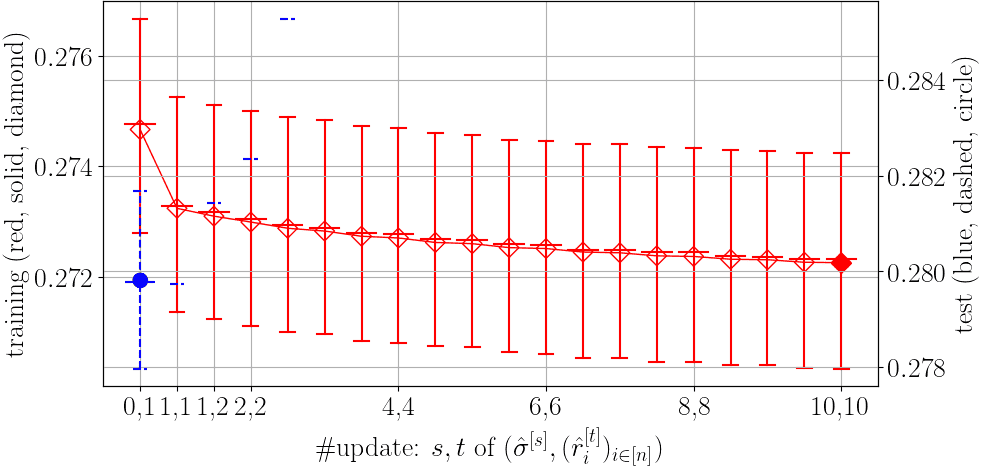}}&
{\includegraphics[width=2.0cm]{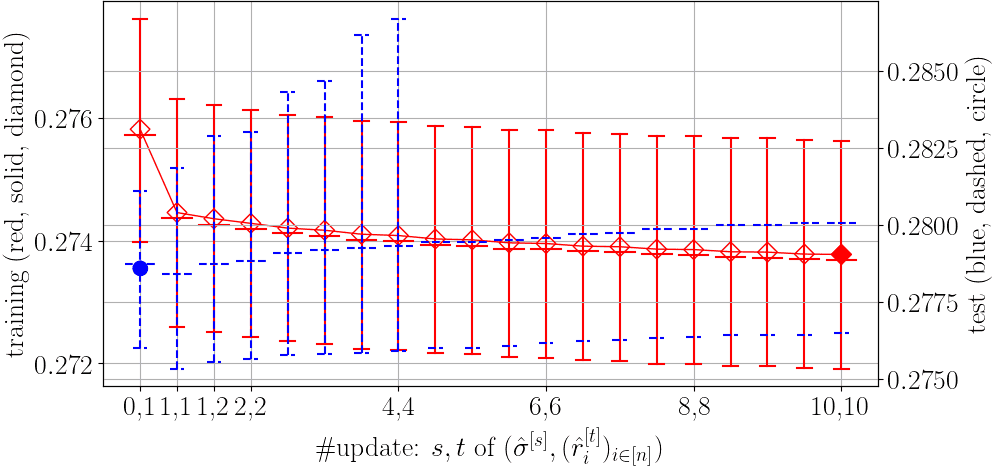}}&
{\includegraphics[width=2.0cm]{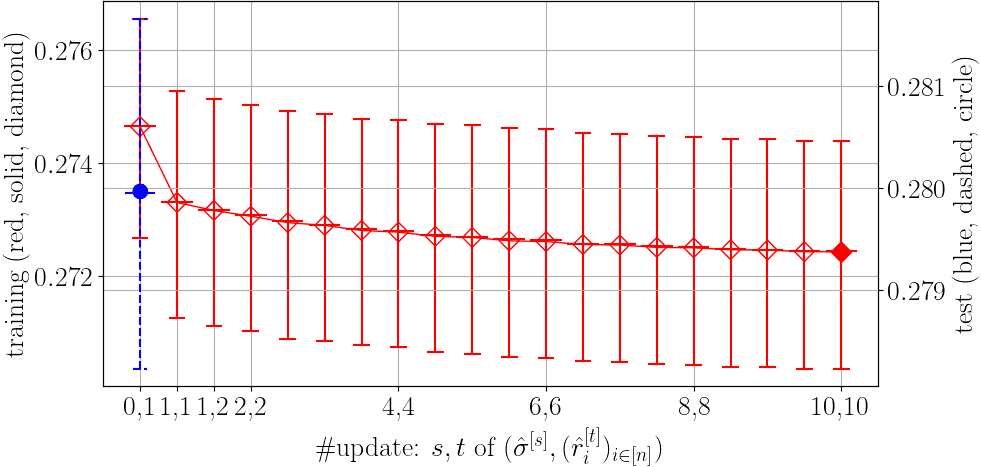}}&
{\includegraphics[width=2.0cm]{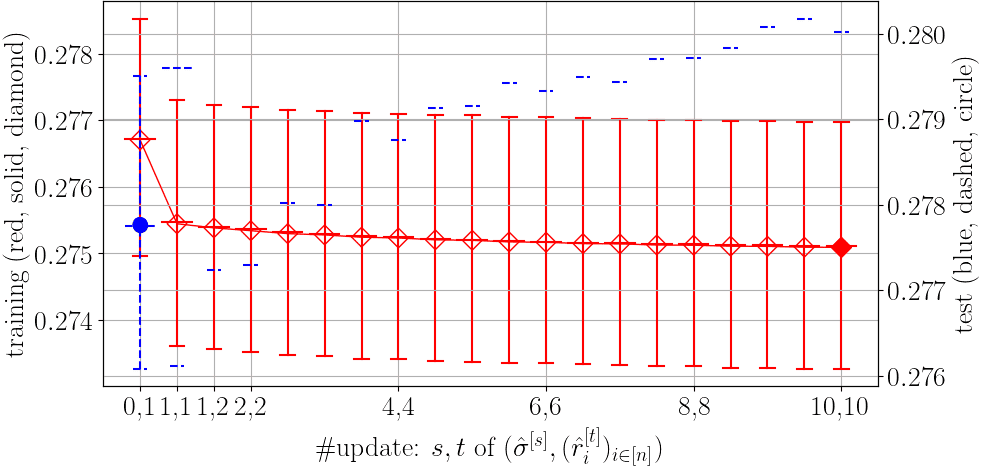}}&
{\includegraphics[width=2.0cm]{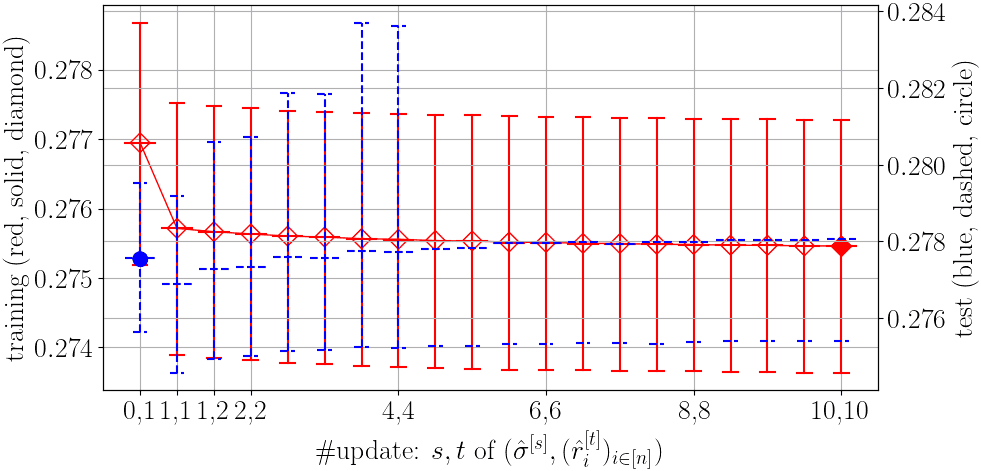}}&
{\includegraphics[width=2.0cm]{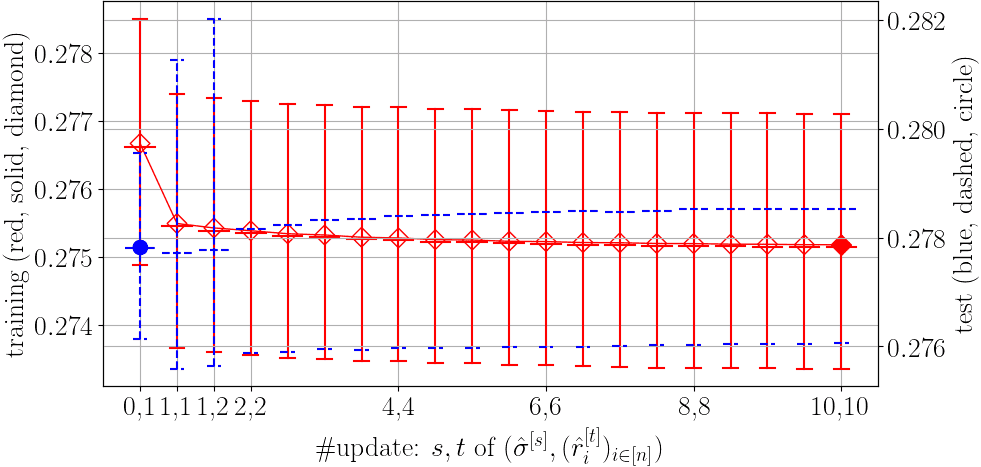}}&
{\includegraphics[width=2.0cm]{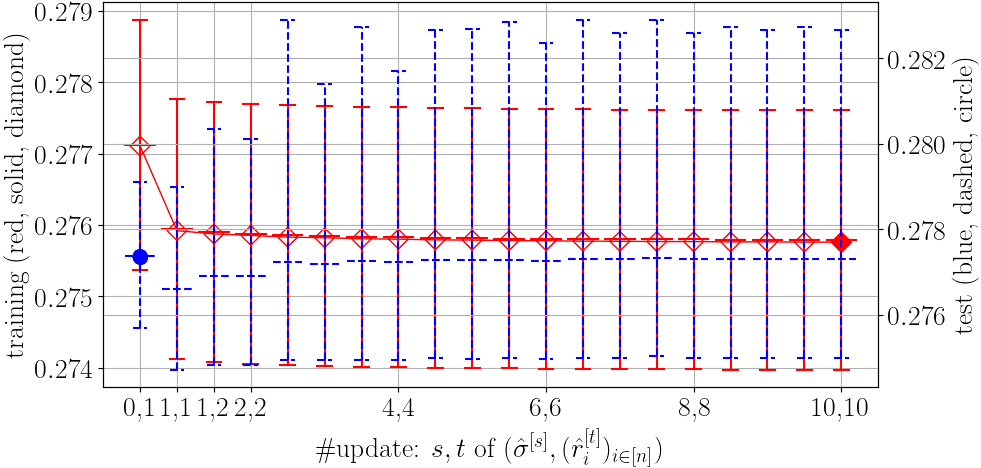}}&
{\includegraphics[width=2.0cm]{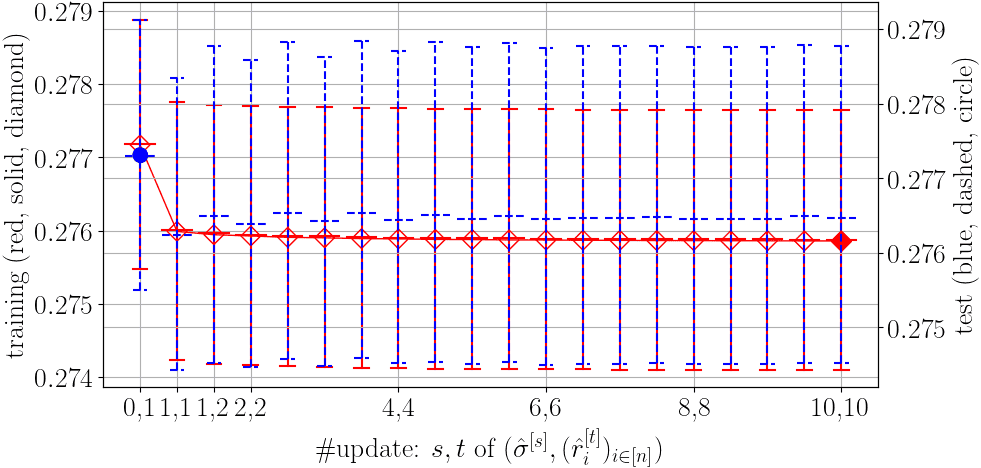}}
\\\midrule
\multirow{3}{*}[-2.5mm]{\rotatebox{90}{\tiny\eqref{eq:Kendall}, $n=$}}
&\rotatebox{90}{\tiny\,~~~\,$25$}&
{\includegraphics[width=2.0cm]{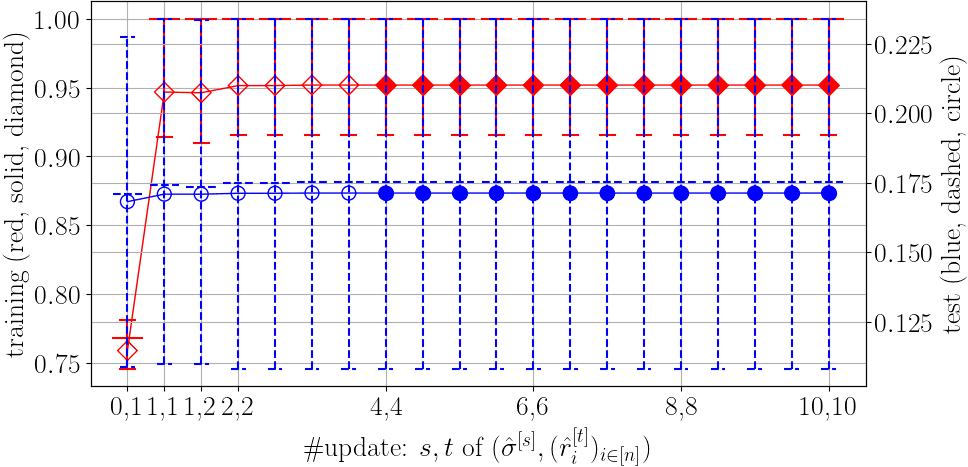}}&
\CF{\includegraphics[width=2.0cm]{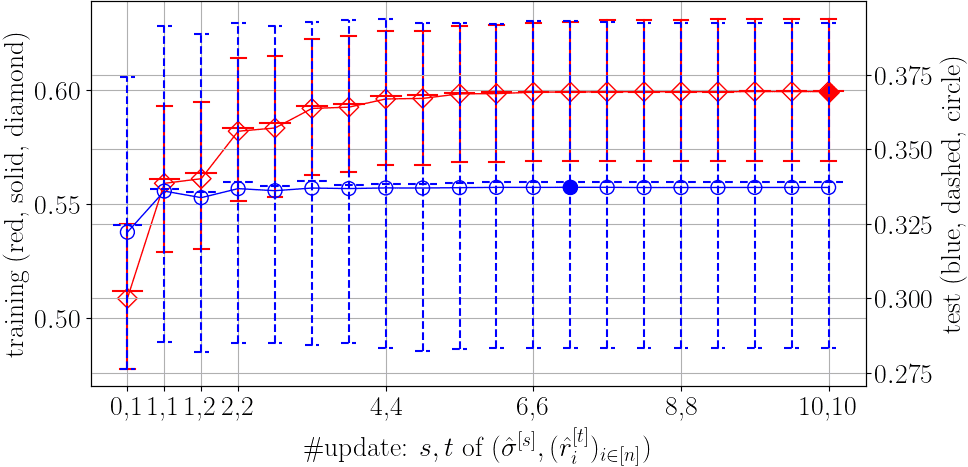}}&
\CF{\includegraphics[width=2.0cm]{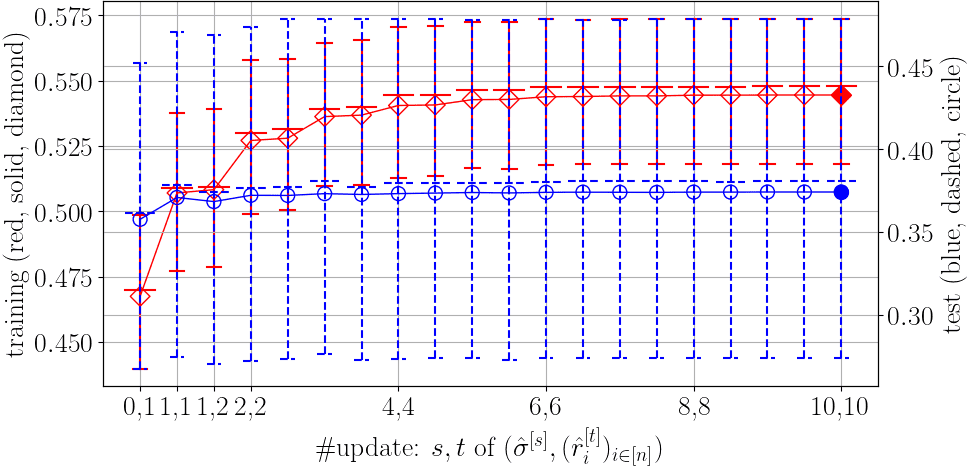}}&
\CF{\includegraphics[width=2.0cm]{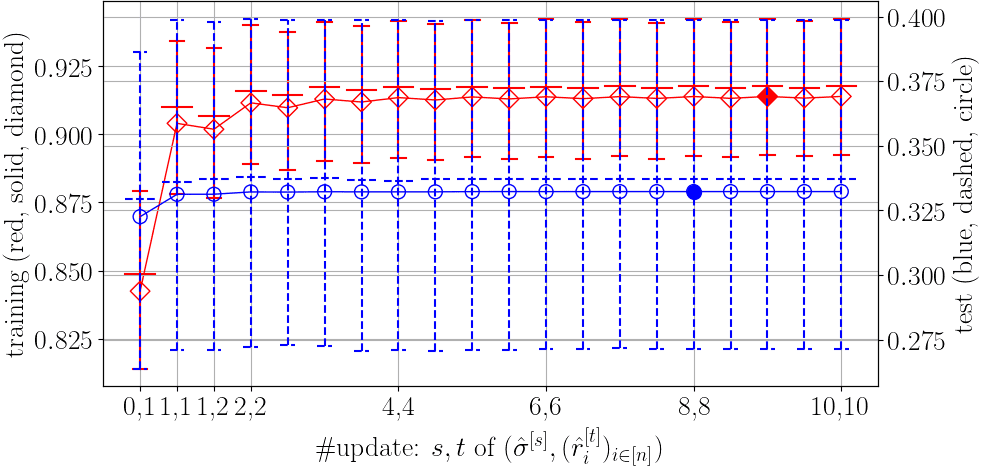}}&
\CF{\includegraphics[width=2.0cm]{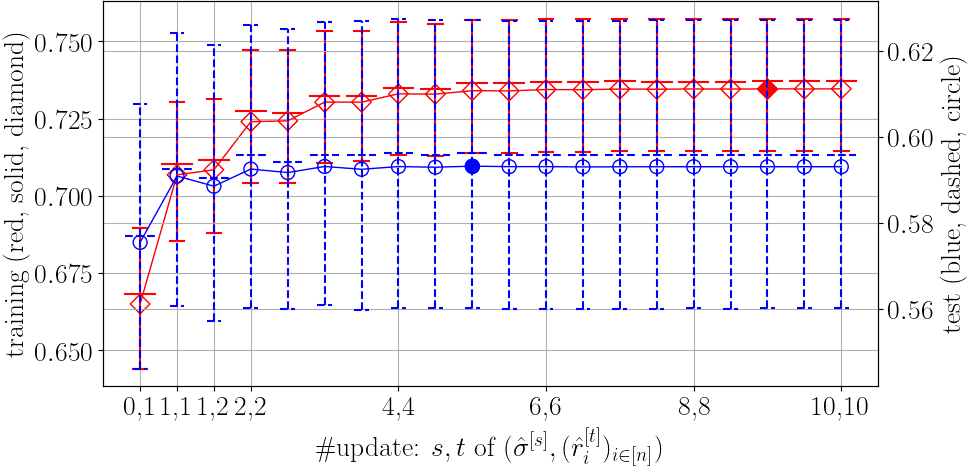}}&
\CF{\includegraphics[width=2.0cm]{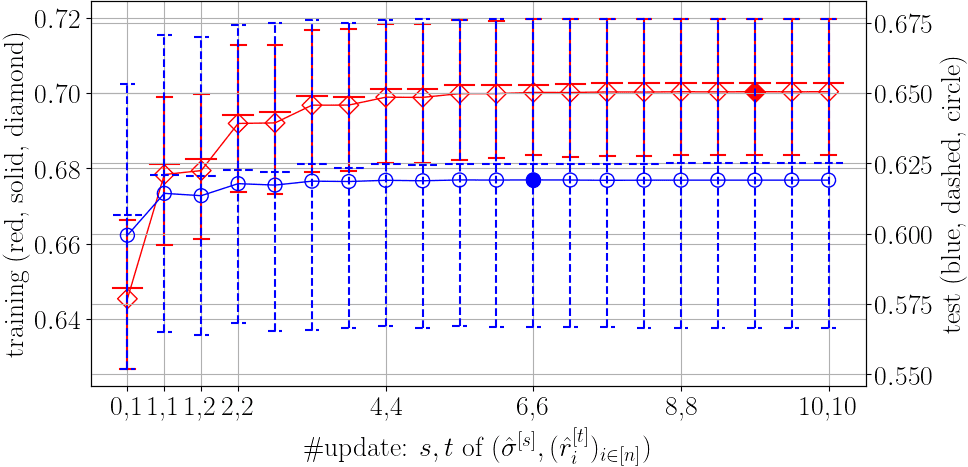}}&
\CF{\includegraphics[width=2.0cm]{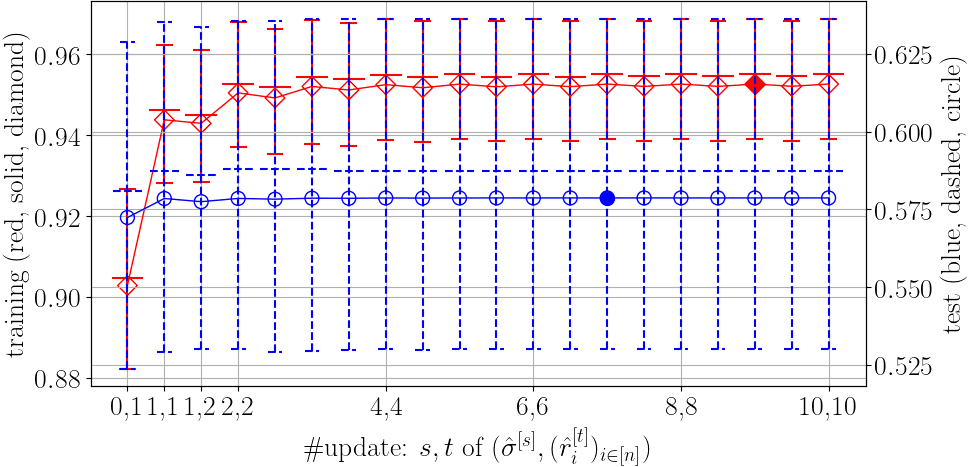}}&
\CF{\includegraphics[width=2.0cm]{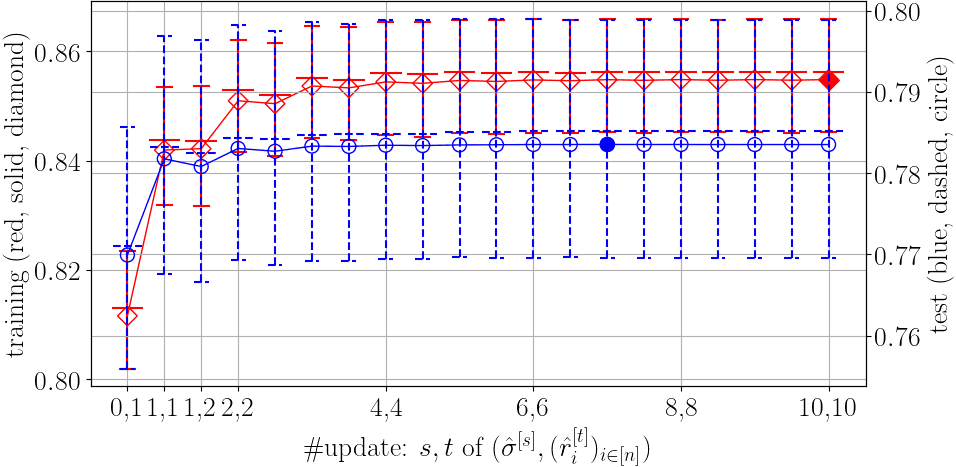}}&
\CF{\includegraphics[width=2.0cm]{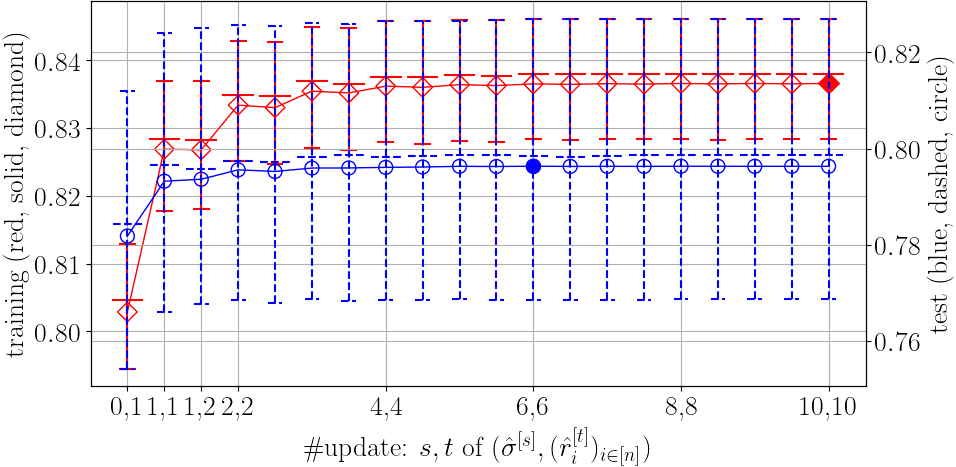}}\\
&\rotatebox{90}{\tiny\,~~\,$100$}&
\CF{\includegraphics[width=2.0cm]{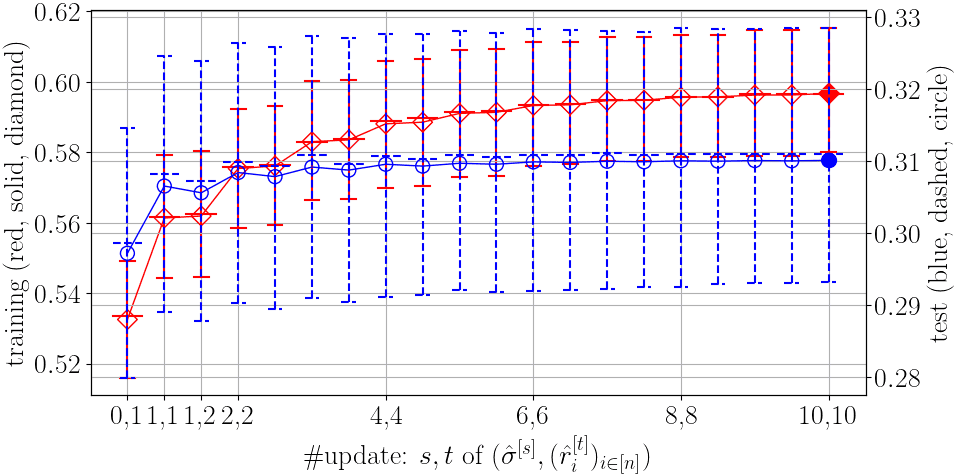}}&
\CF{\includegraphics[width=2.0cm]{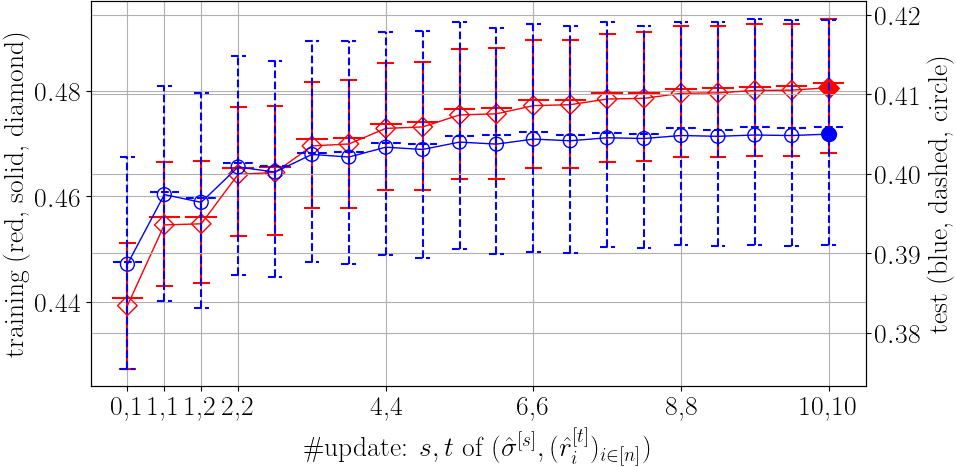}}&
\CF{\includegraphics[width=2.0cm]{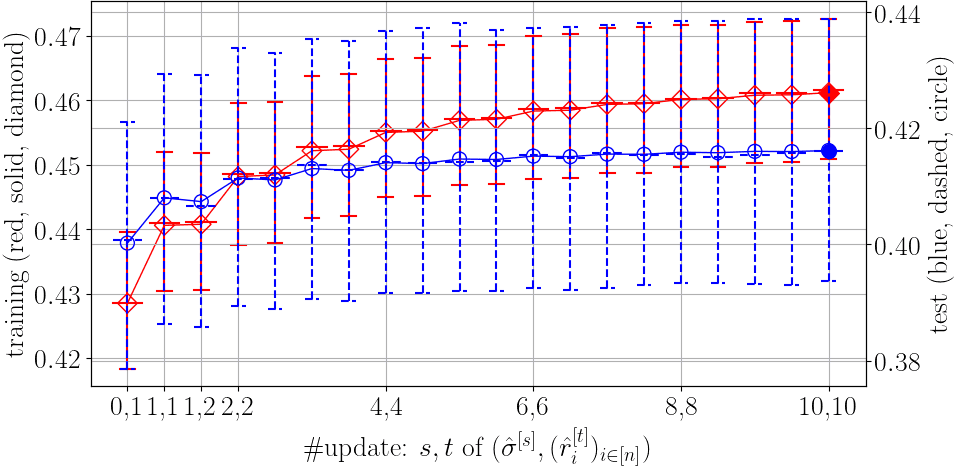}}&
\CF{\includegraphics[width=2.0cm]{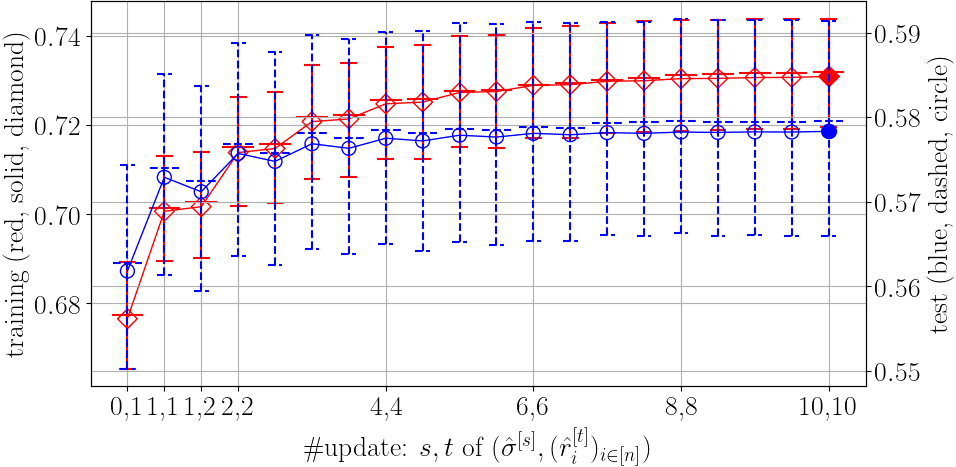}}&
\CF{\includegraphics[width=2.0cm]{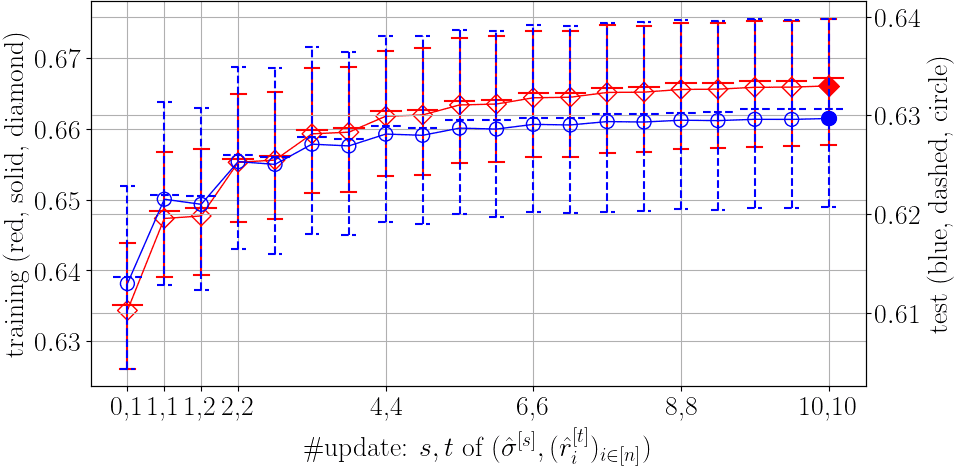}}&
\CF{\includegraphics[width=2.0cm]{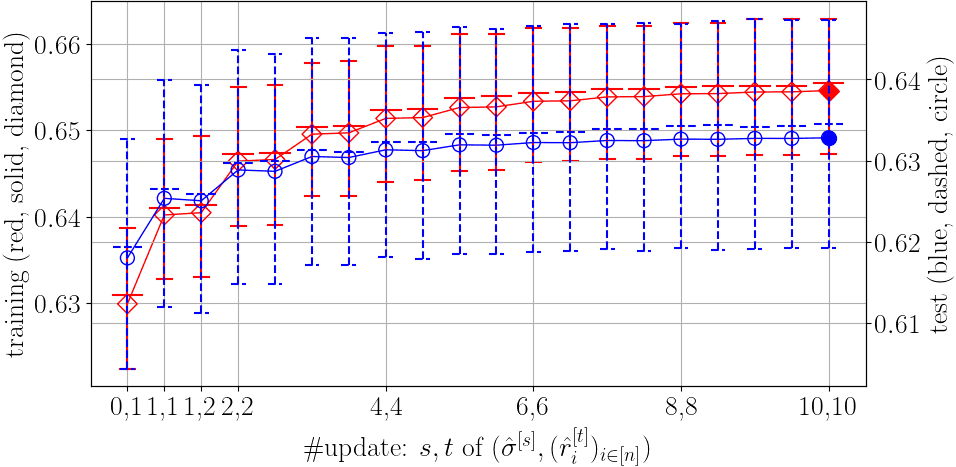}}&
\CF{\includegraphics[width=2.0cm]{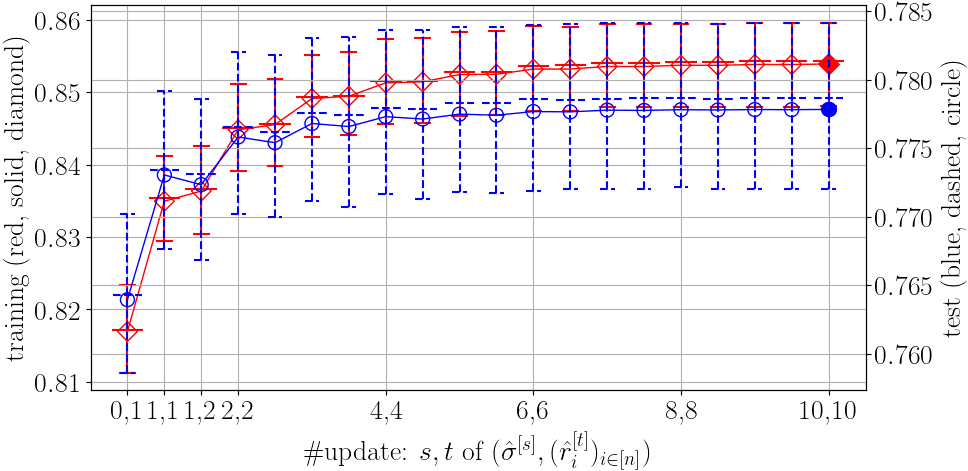}}&
\CF{\includegraphics[width=2.0cm]{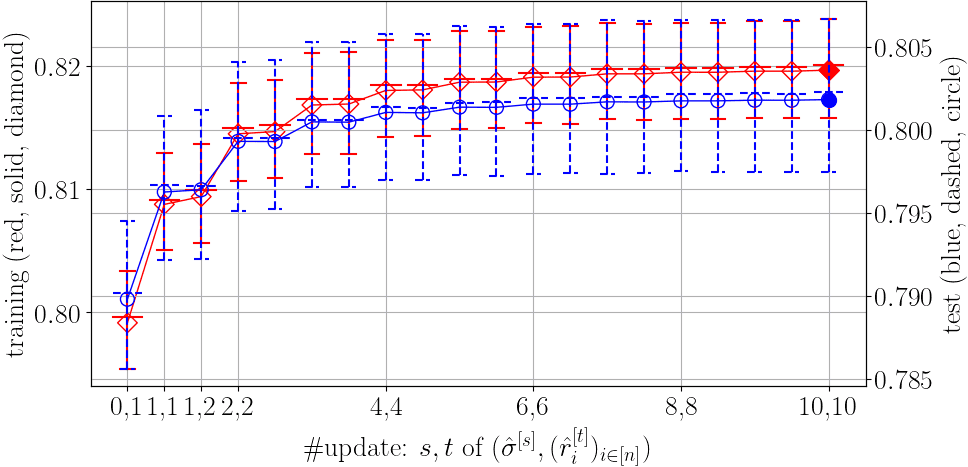}}&
\CF{\includegraphics[width=2.0cm]{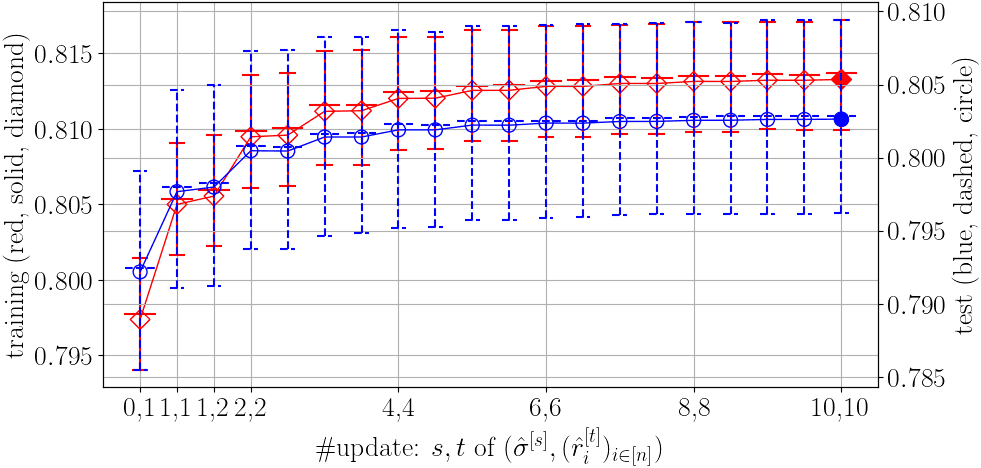}}\\
&\rotatebox{90}{\tiny\,~~\,$400$}&
\CF{\includegraphics[width=2.0cm]{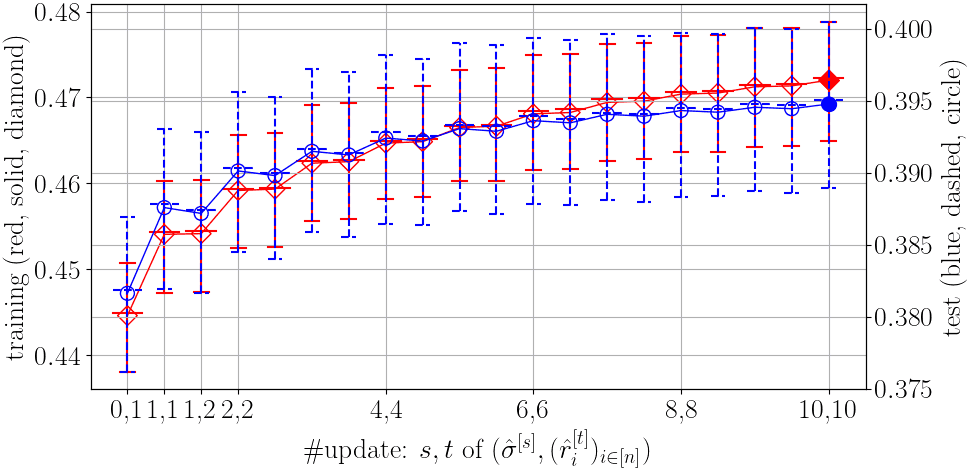}}&
\CF{\includegraphics[width=2.0cm]{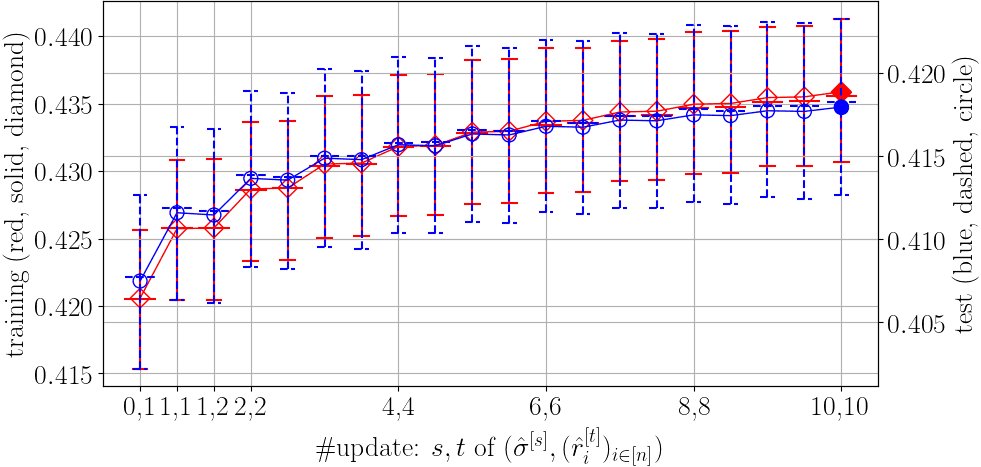}}&
\CF{\includegraphics[width=2.0cm]{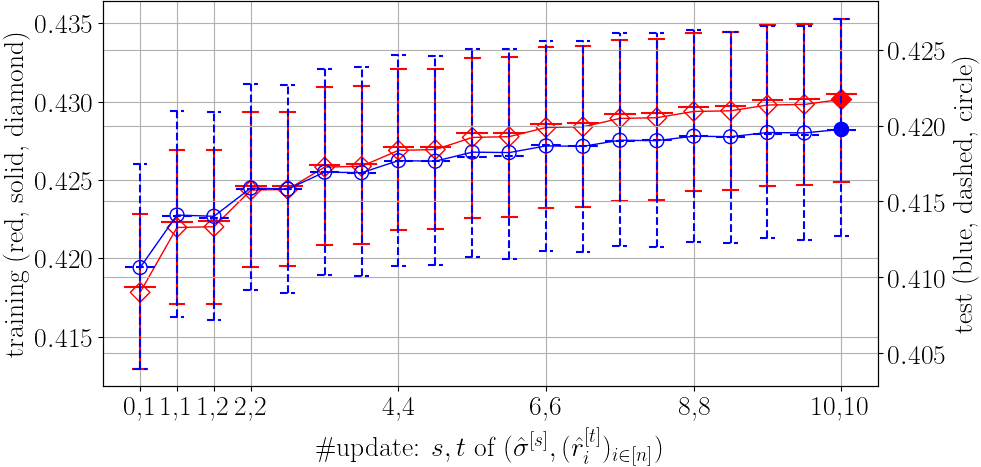}}&
\CF{\includegraphics[width=2.0cm]{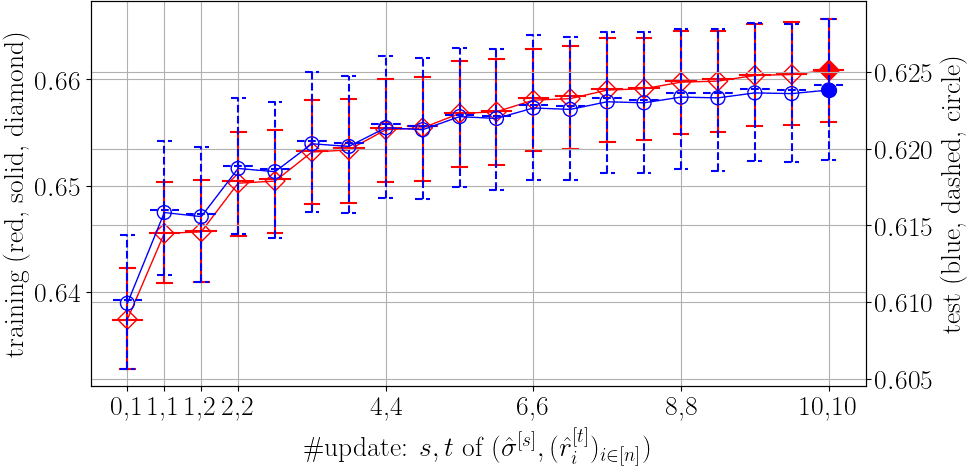}}&
\CF{\includegraphics[width=2.0cm]{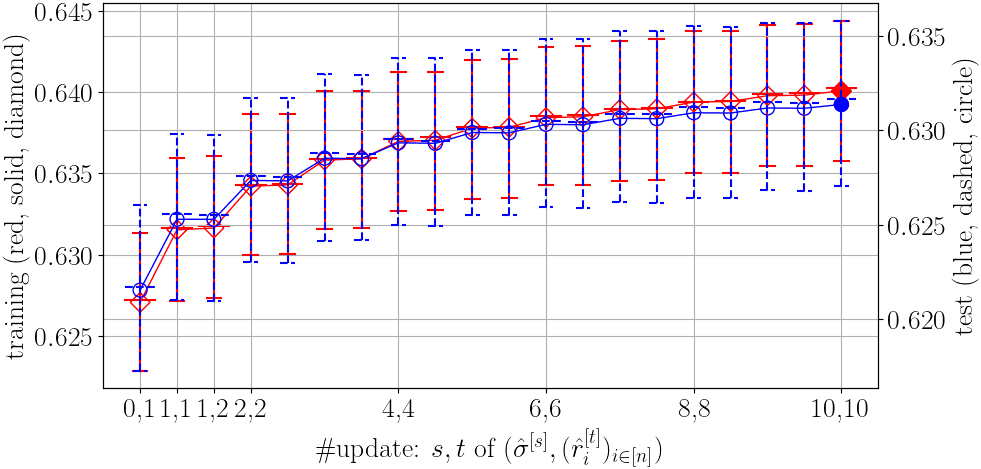}}&
\CF{\includegraphics[width=2.0cm]{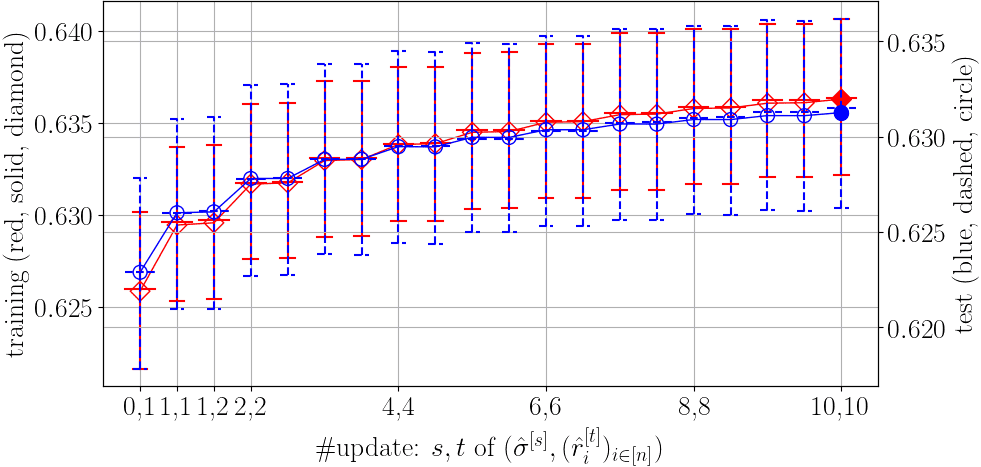}}&
\CF{\includegraphics[width=2.0cm]{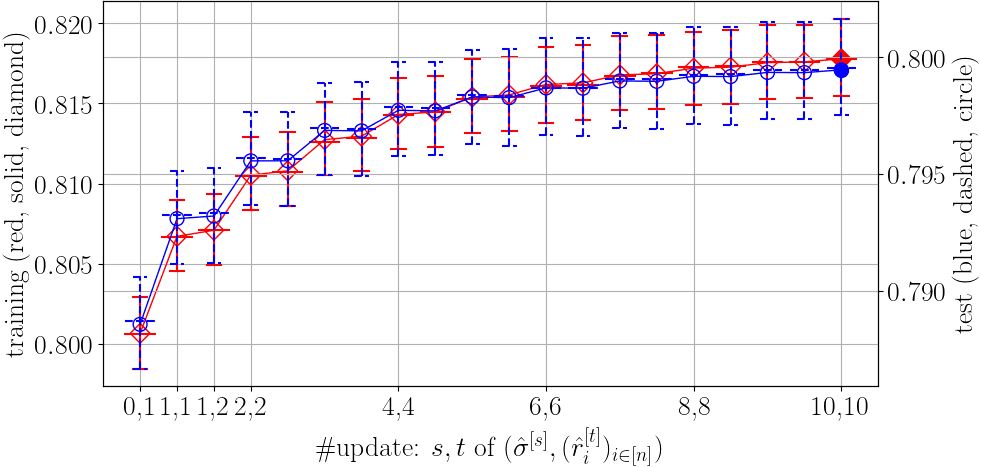}}&
\CF{\includegraphics[width=2.0cm]{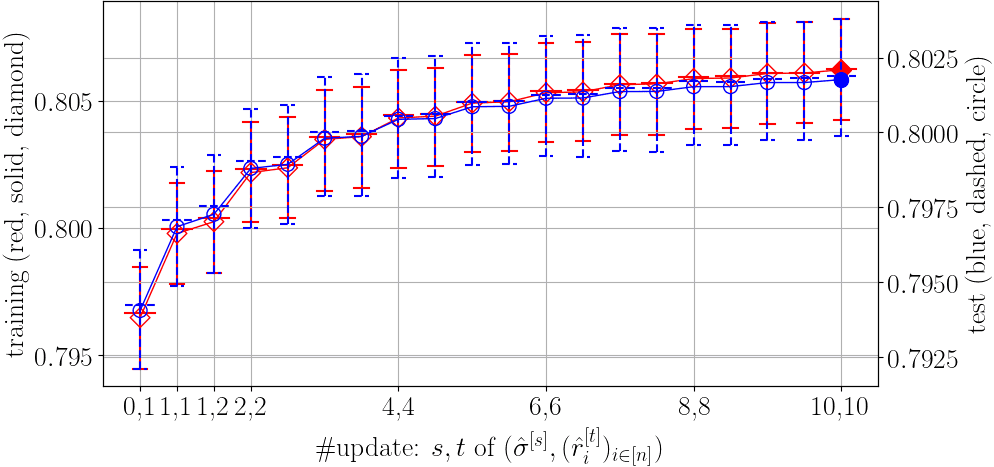}}&
\CF{\includegraphics[width=2.0cm]{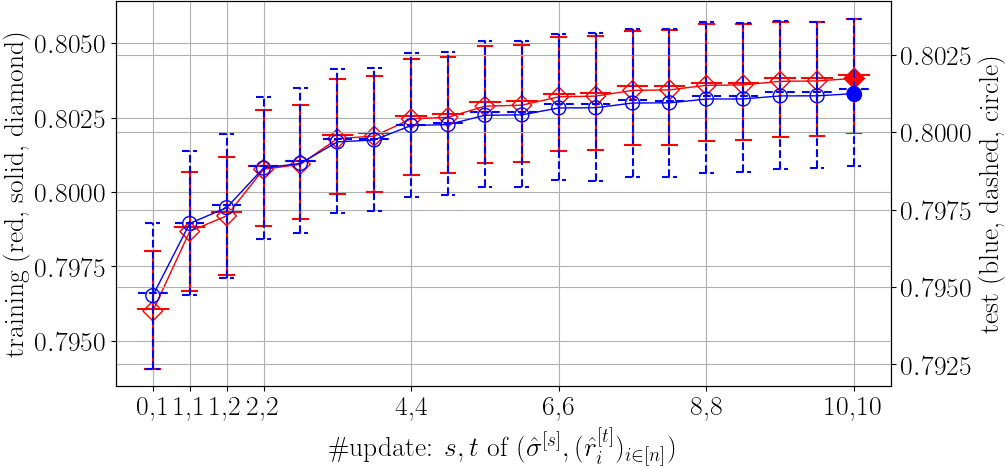}}
\\\midrule
\multirow{3}{*}[-2.5mm]{\rotatebox{90}{\tiny\eqref{eq:TIE}, $n=$}}
&\rotatebox{90}{\tiny\,~~~\,$25$}&
{\includegraphics[width=2.0cm]{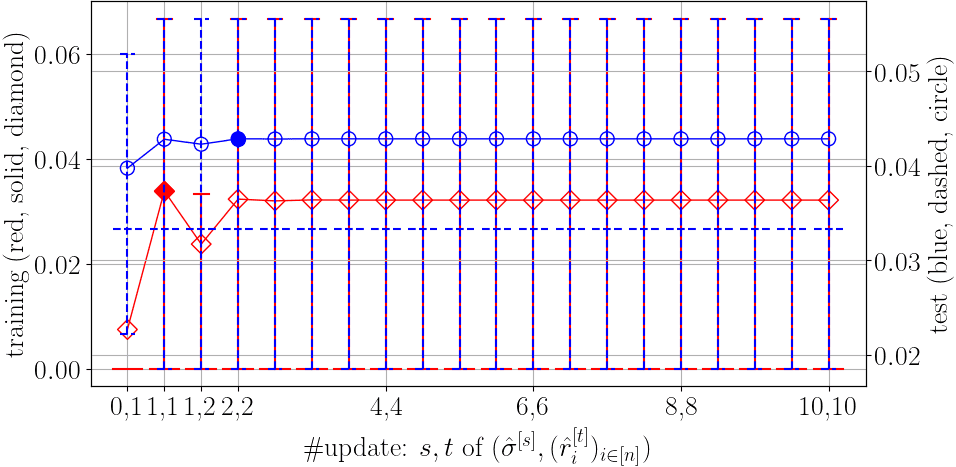}}&
{\includegraphics[width=2.0cm]{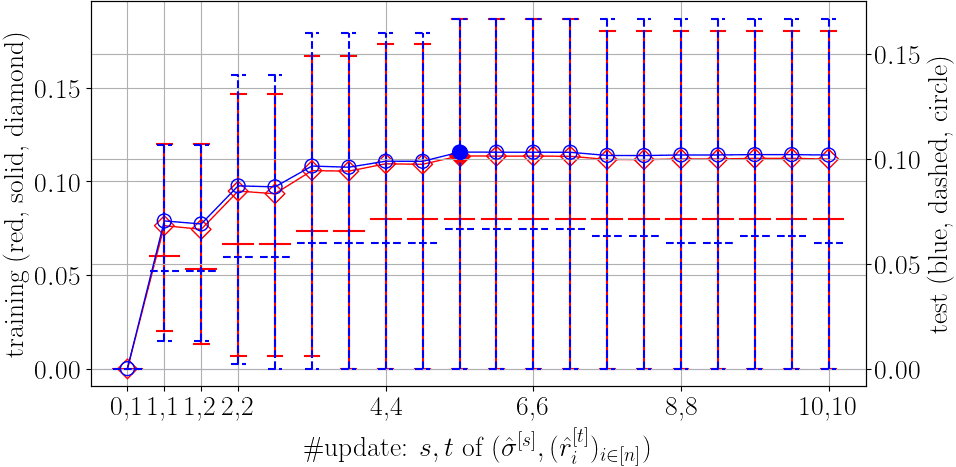}}&
{\includegraphics[width=2.0cm]{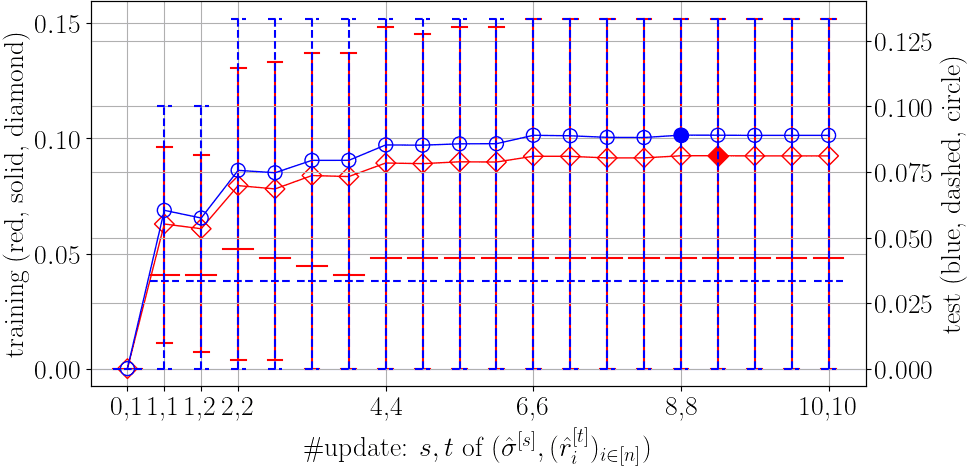}}&
{\includegraphics[width=2.0cm]{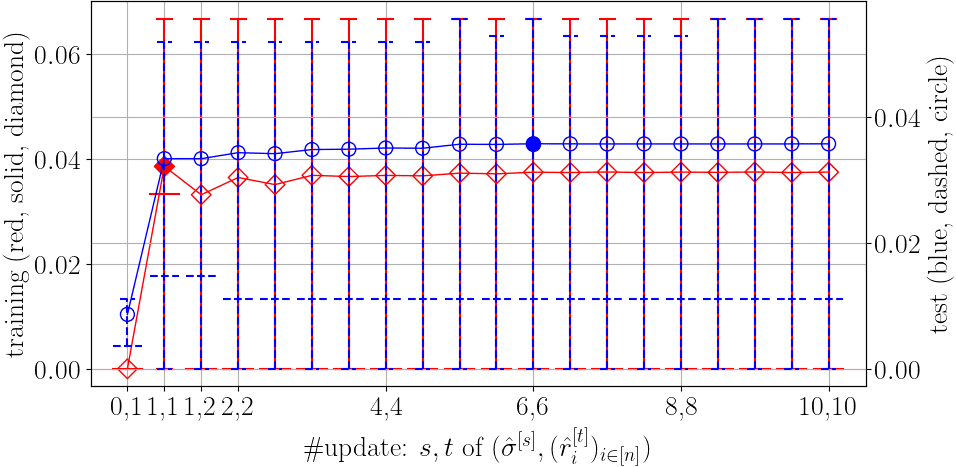}}&
{\includegraphics[width=2.0cm]{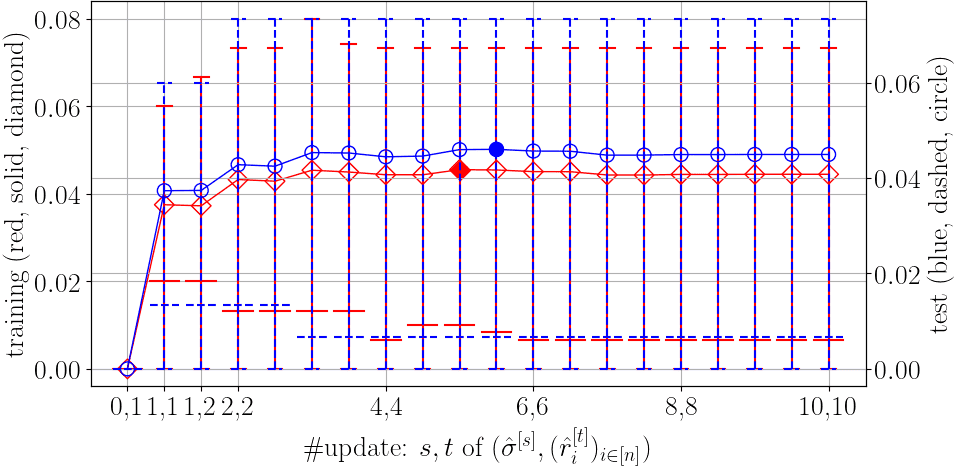}}&
{\includegraphics[width=2.0cm]{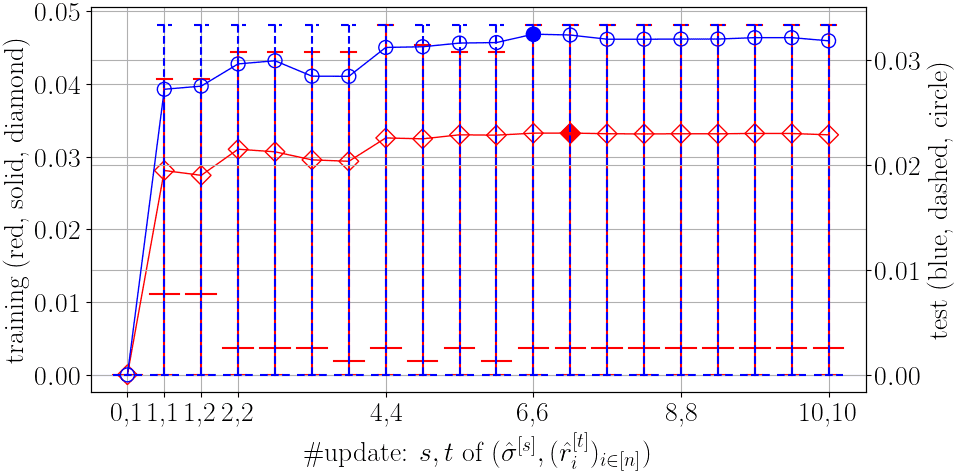}}&
{\includegraphics[width=2.0cm]{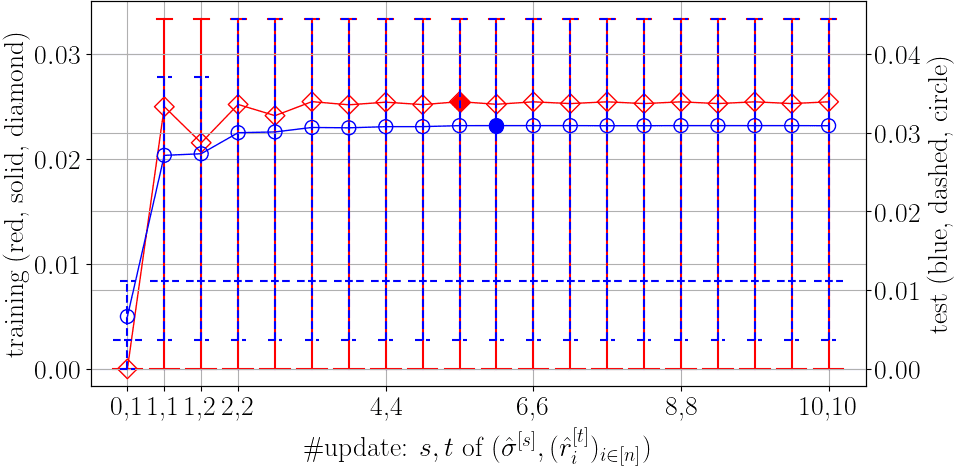}}&
{\includegraphics[width=2.0cm]{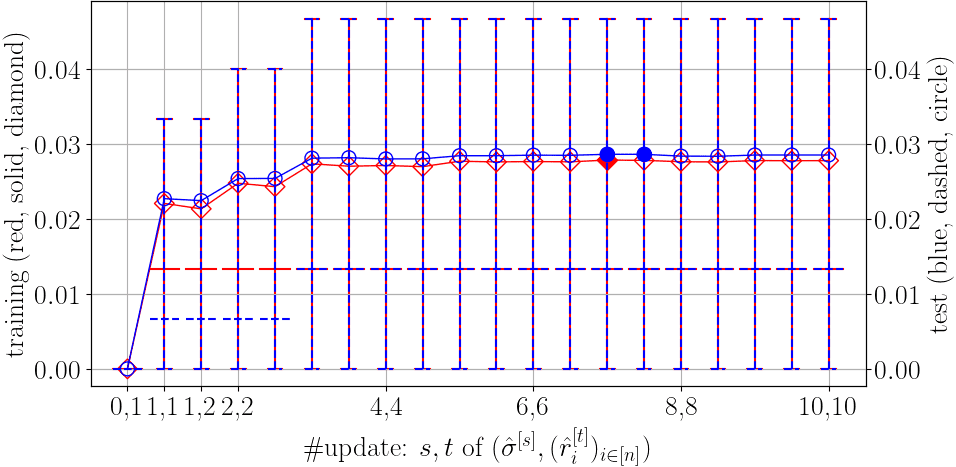}}&
{\includegraphics[width=2.0cm]{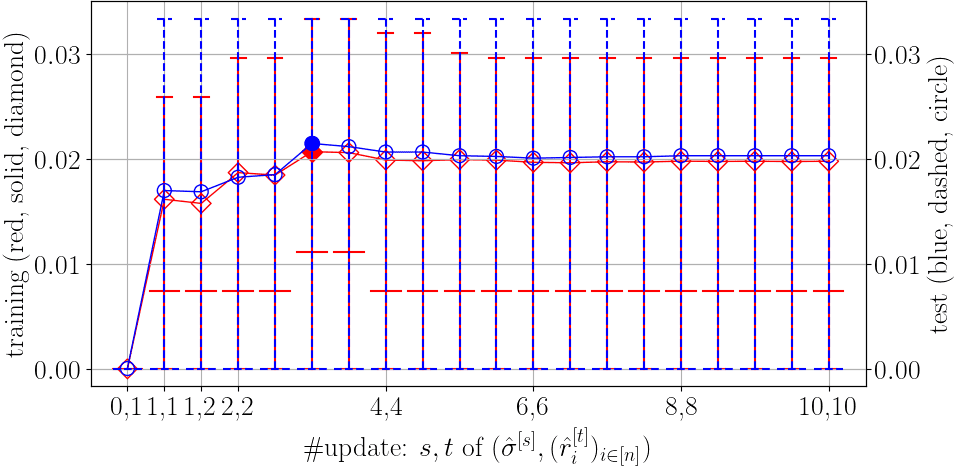}}\\
&\rotatebox{90}{\tiny\,~~\,$100$}&
{\includegraphics[width=2.0cm]{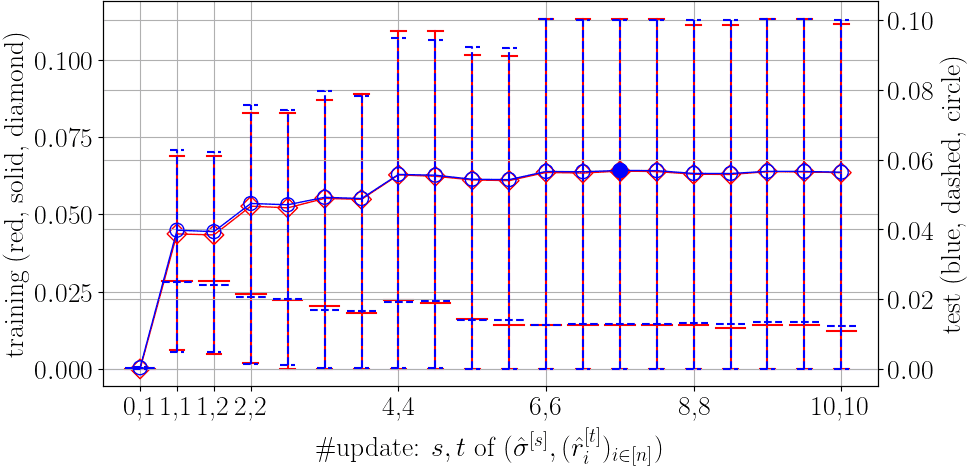}}&
{\includegraphics[width=2.0cm]{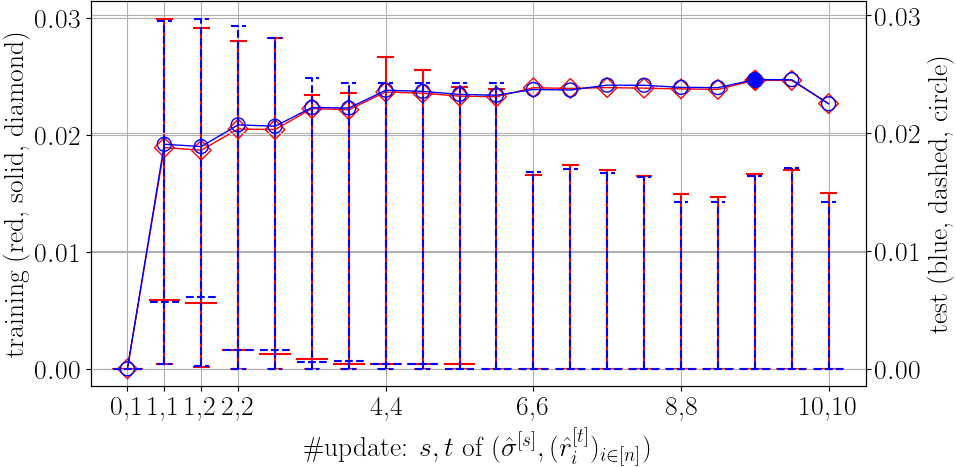}}&
{\includegraphics[width=2.0cm]{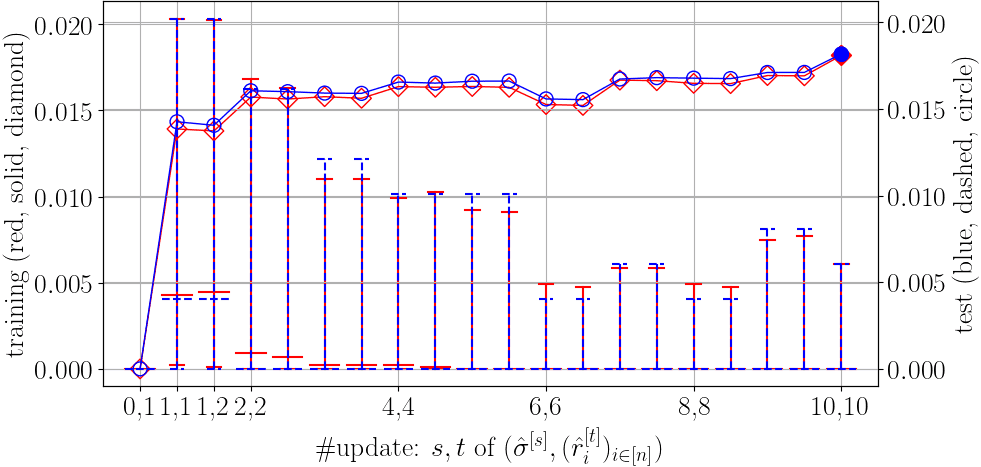}}&
{\includegraphics[width=2.0cm]{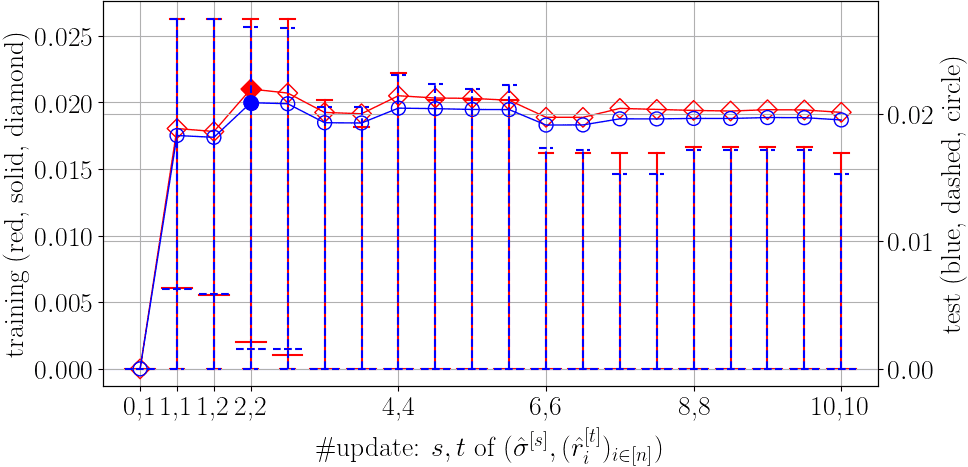}}&
{\includegraphics[width=2.0cm]{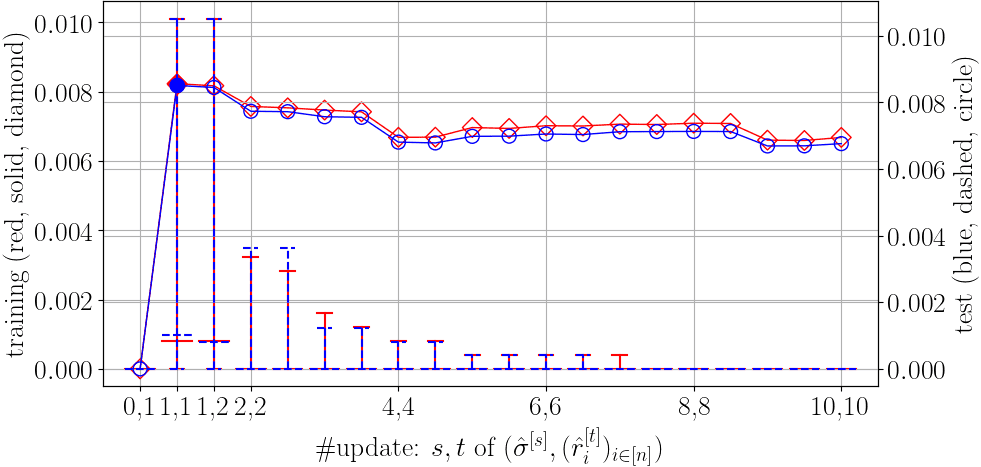}}&
{\includegraphics[width=2.0cm]{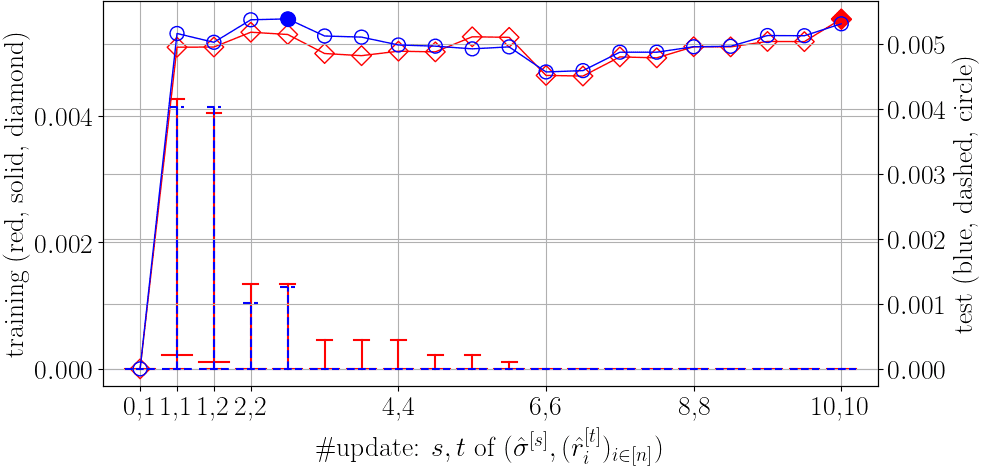}}&
{\includegraphics[width=2.0cm]{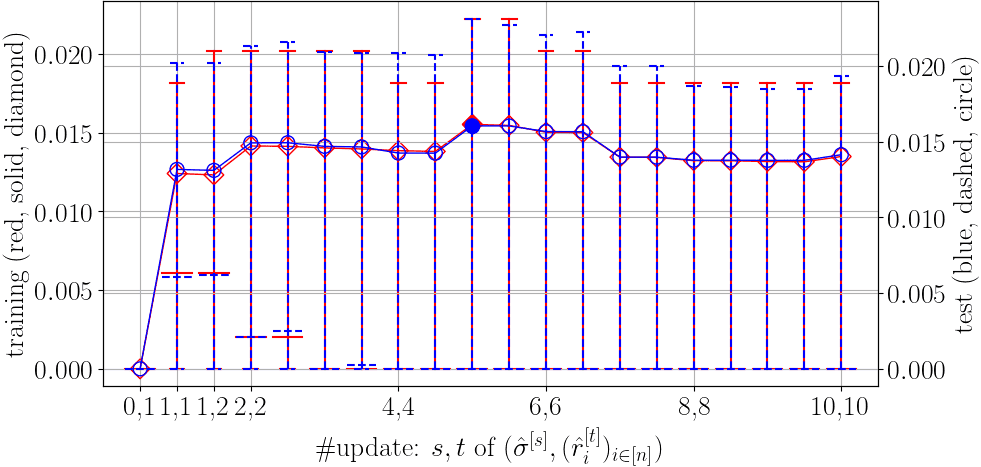}}&
{\includegraphics[width=2.0cm]{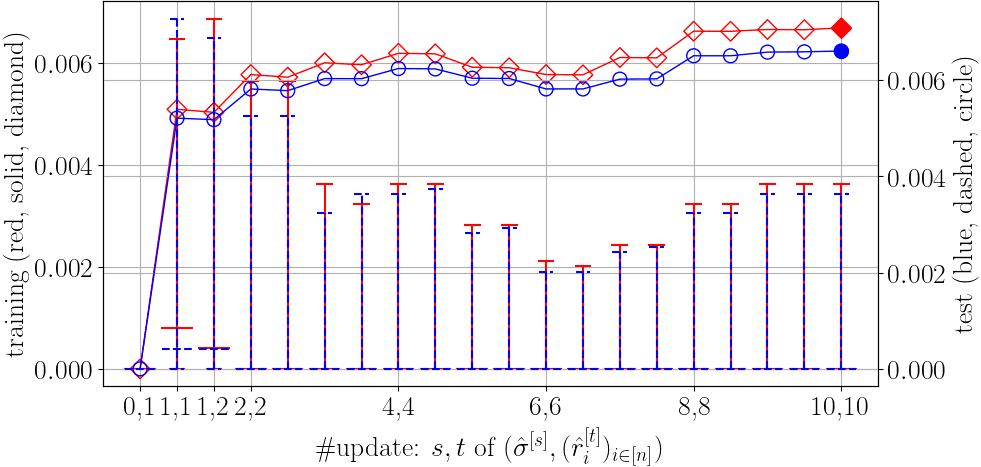}}&
{\includegraphics[width=2.0cm]{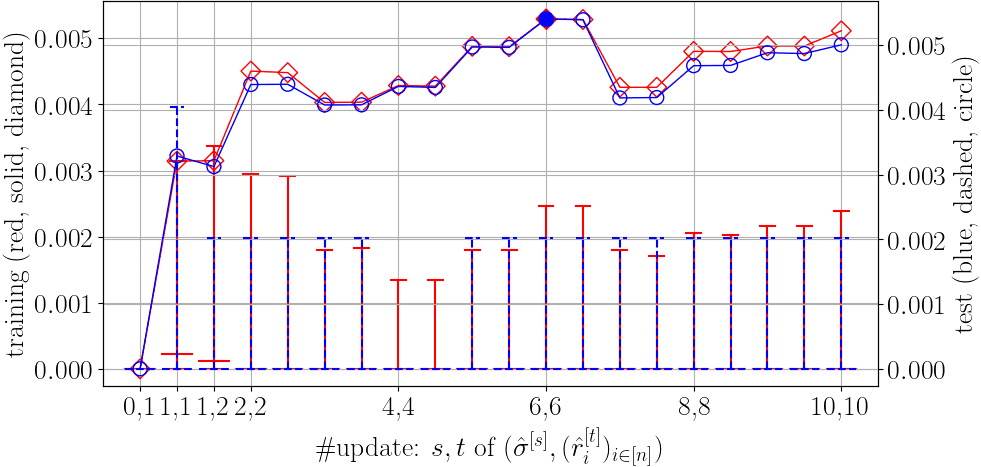}}\\
&\rotatebox{90}{\tiny\,~~\,$400$}&
{\includegraphics[width=2.0cm]{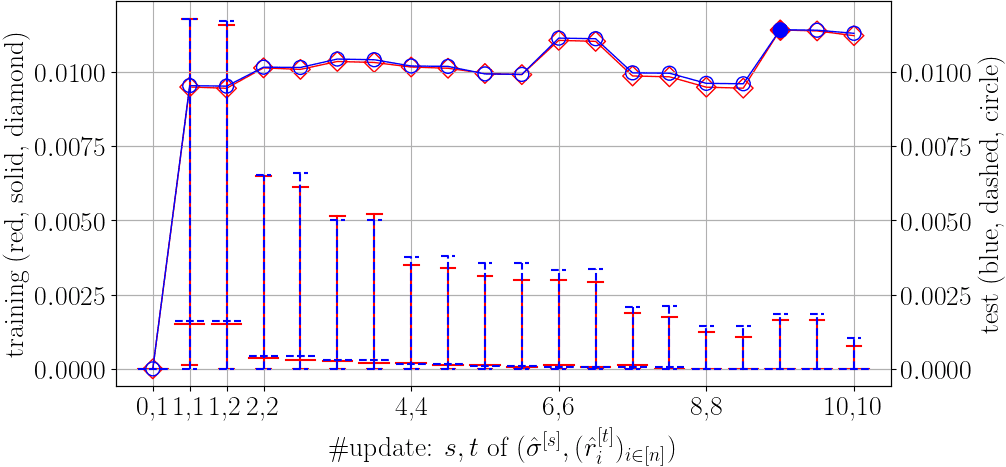}}&
{\includegraphics[width=2.0cm]{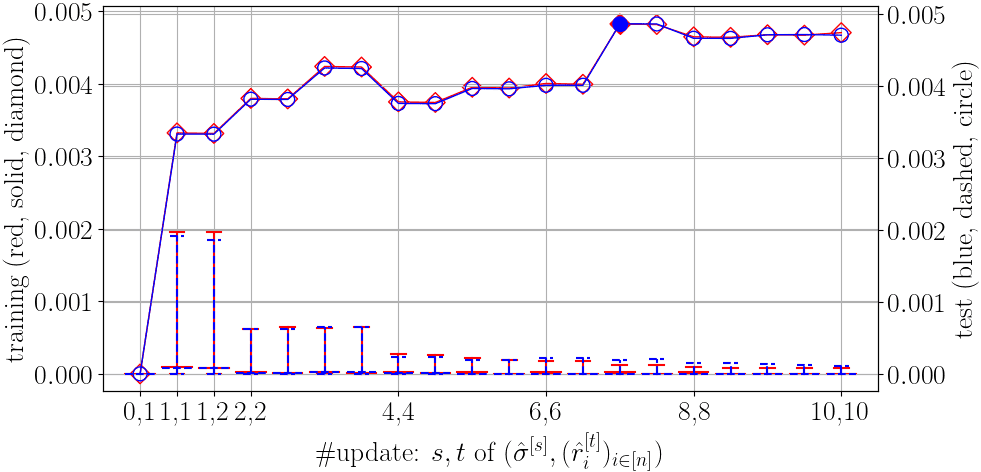}}&
{\includegraphics[width=2.0cm]{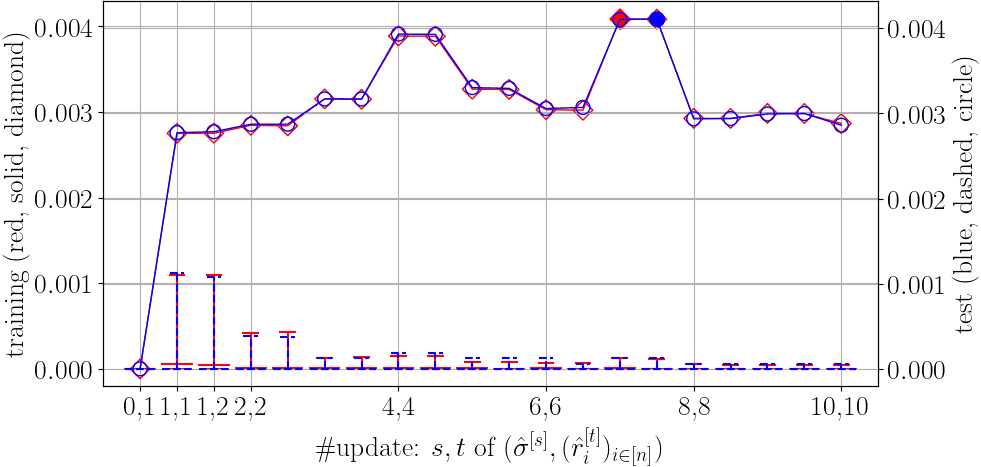}}&
{\includegraphics[width=2.0cm]{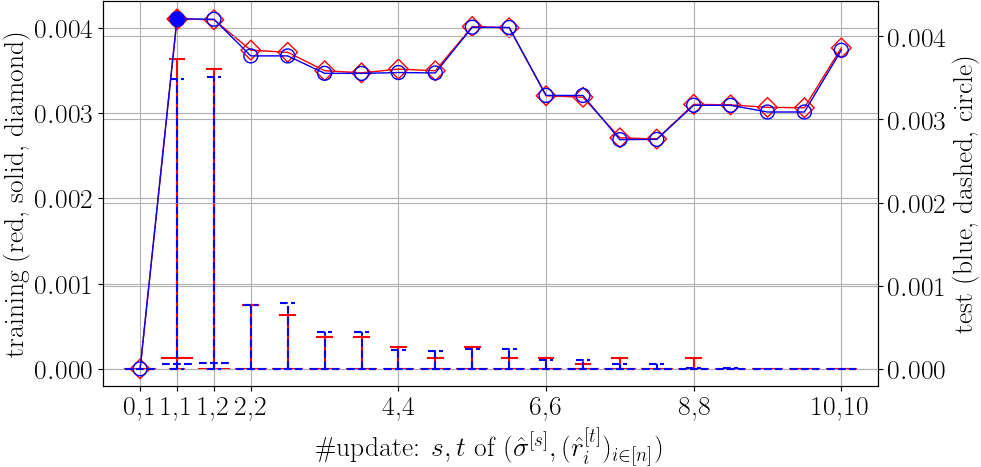}}&
{\includegraphics[width=2.0cm]{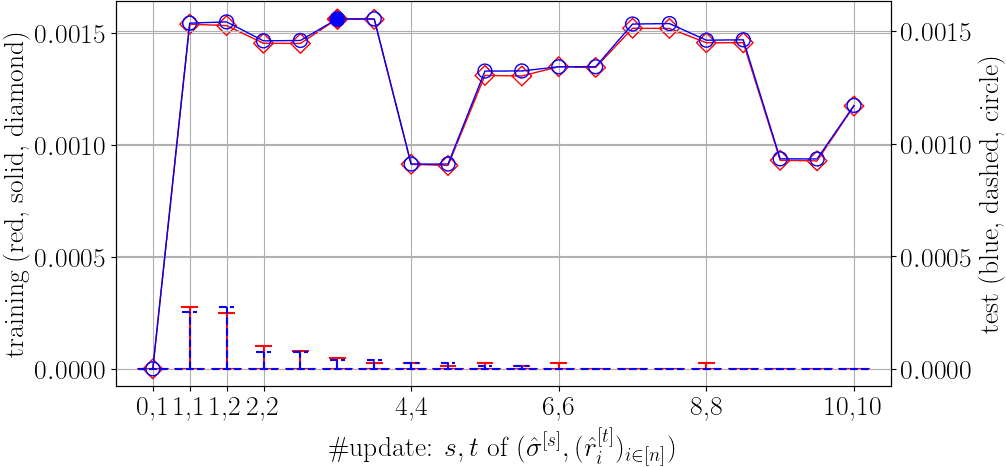}}&
{\includegraphics[width=2.0cm]{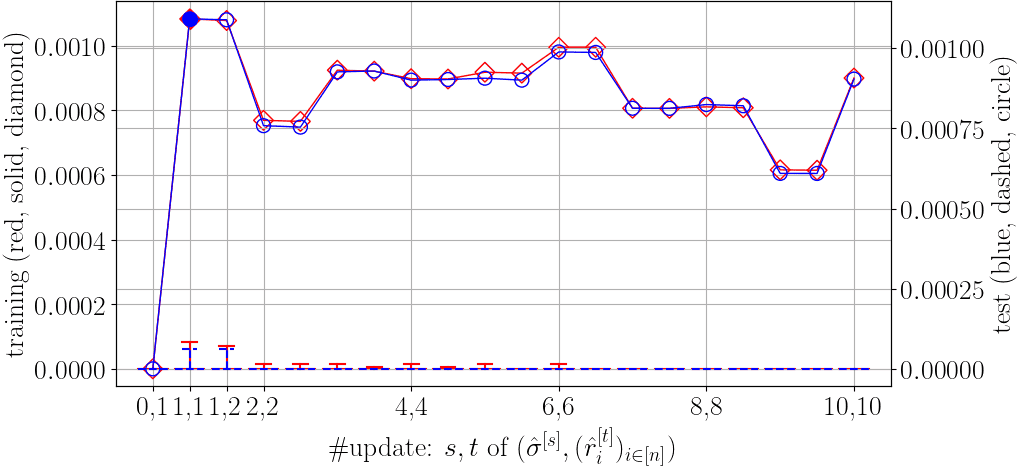}}&
{\includegraphics[width=2.0cm]{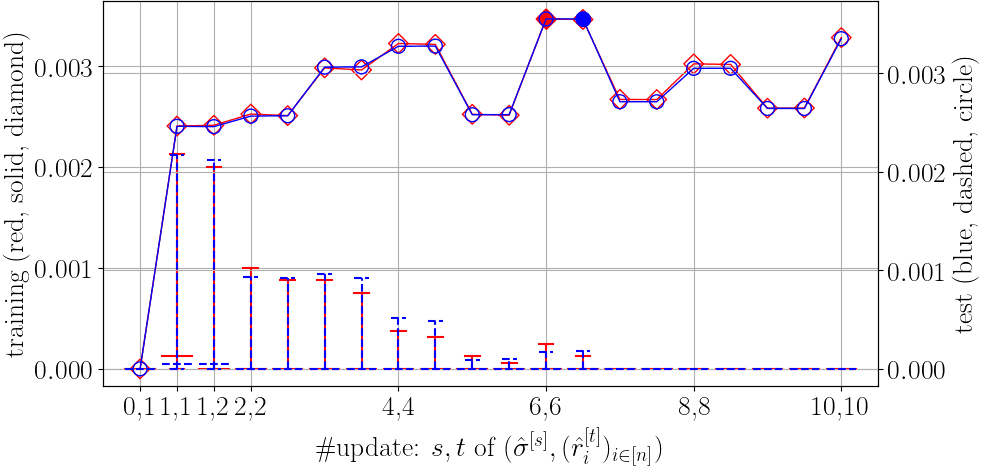}}&
{\includegraphics[width=2.0cm]{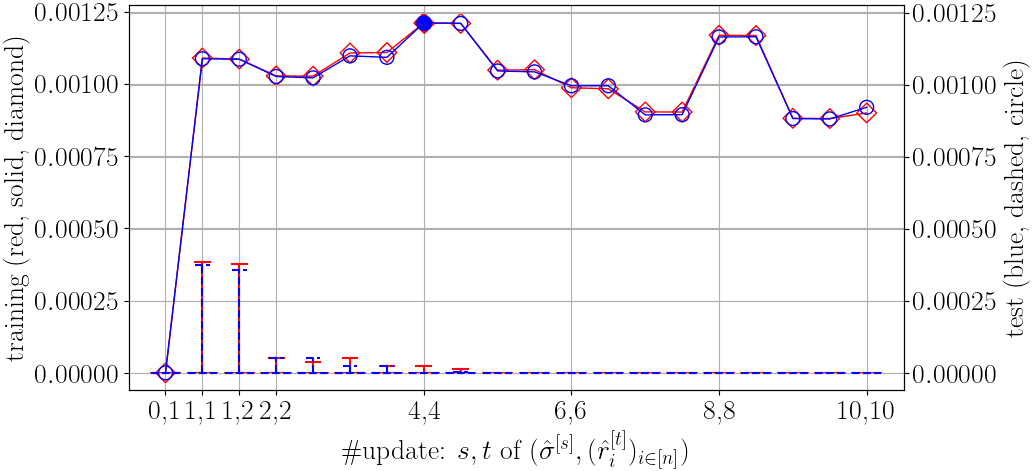}}&
{\includegraphics[width=2.0cm]{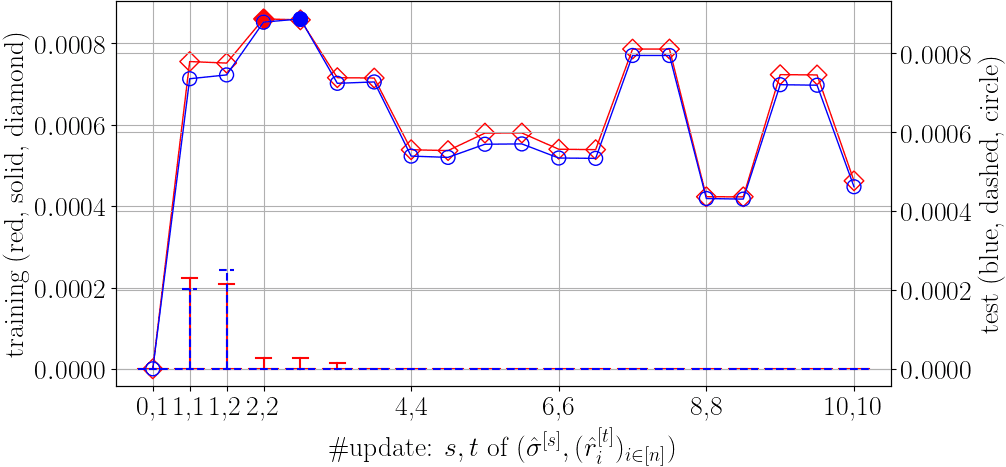}}
\end{tabular}
\caption{%
Part of results of synthetic data experiments, Procedure 1 in Section~\ref{sec:Synthetic} and Appendix~\ref{sec:NLL}:
For Cauchy-$N$ synthetic data with $N=1,5,25$ (left to right),
mean (marker) and 0.25, 0.5, and 0.75 quantiles (lower, middle, and upper bars) of 
1000 trial training (red, solid, diamond) and test (blue, dashed, circle) evaluation of 
WPP error \eqref{eq:WPE} with the NLL loss $\phi=\phi_\nll$, 
Kendall's Tau \eqref{eq:Kendall}, and tie rate \eqref{eq:TIE} (top to bottom)
for the isotonic Bradley-Terry model learned with the NLL loss $\phi=\phi_\nll$.
Smaller \eqref{eq:WPE}, or larger \eqref{eq:Kendall} indicates a better model.
The marker for the best model and model with the most ties was filled in, 
and the outer frame of the figure was highlighted in gray
if the best model was significantly better than the Bradley-Terry model
with respect to Mann-Whitney U test of the significance level 0.05.}
\label{fig:Cauchy-NLL}
\end{sidewaysfigure}
\begin{sidewaysfigure}
\centering%
\renewcommand{\arraystretch}{0.5}%
\renewcommand{\tabcolsep}{0.5pt}%
\begin{tabular}{cc|ccc|ccc|ccc}%
&&\multicolumn{3}{c|}{\tiny$N=1$, $|D_\tra|:|D_\tes|=$}&\multicolumn{3}{c|}{\tiny$N=5$, $|D_\tra|:|D_\tes|=$}&\multicolumn{3}{c}{\tiny$N=25$, $|D_\tra|:|D_\tes|=$}\\
&&{\tiny$1:9$}&{\tiny$5:5$}&{\tiny$9:1$}&{\tiny$1:9$}&{\tiny$5:5$}&{\tiny$9:1$}&{\tiny$1:9$}&{\tiny$5:5$}&{\tiny$9:1$}\\
\midrule
\multirow{3}{*}[-3.4mm]{\rotatebox{90}{\tiny $n=$}}
&\rotatebox{90}{\tiny\,~~~~\,$25$}&
{\includegraphics[width=2.0cm]{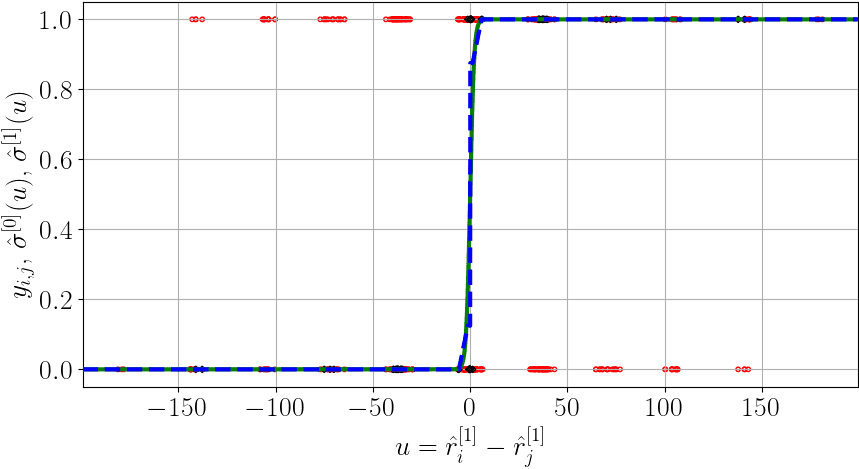}}&
{\includegraphics[width=2.0cm]{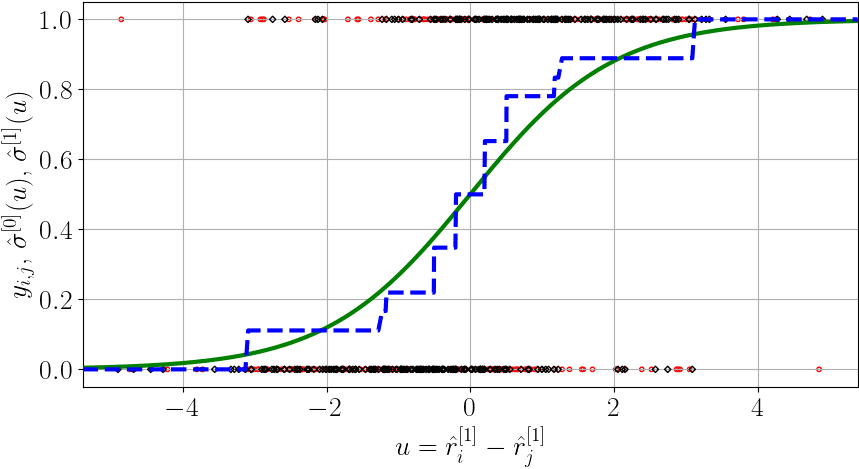}}&
{\includegraphics[width=2.0cm]{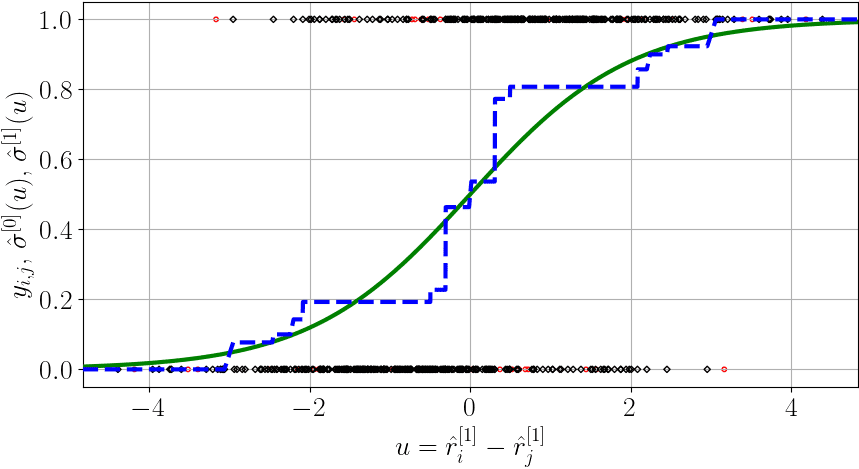}}&
{\includegraphics[width=2.0cm]{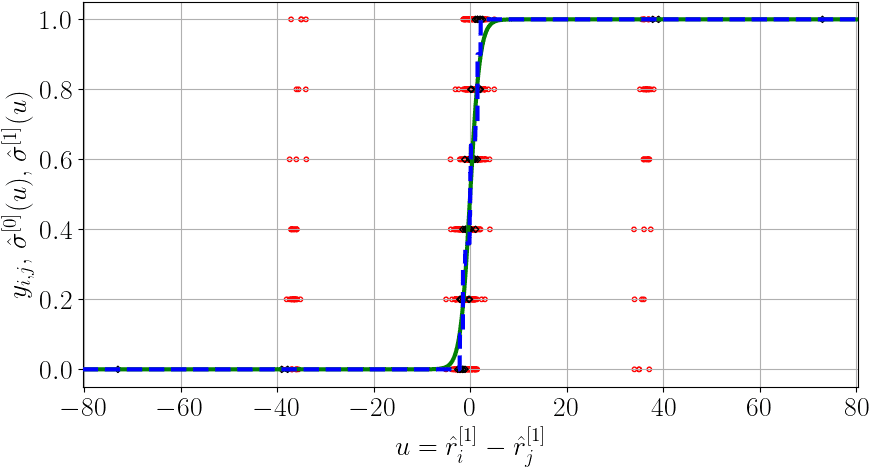}}&
{\includegraphics[width=2.0cm]{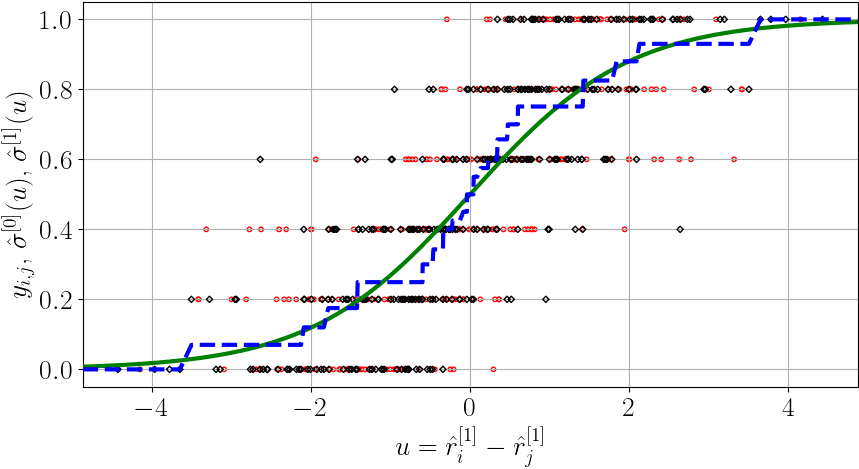}}&
{\includegraphics[width=2.0cm]{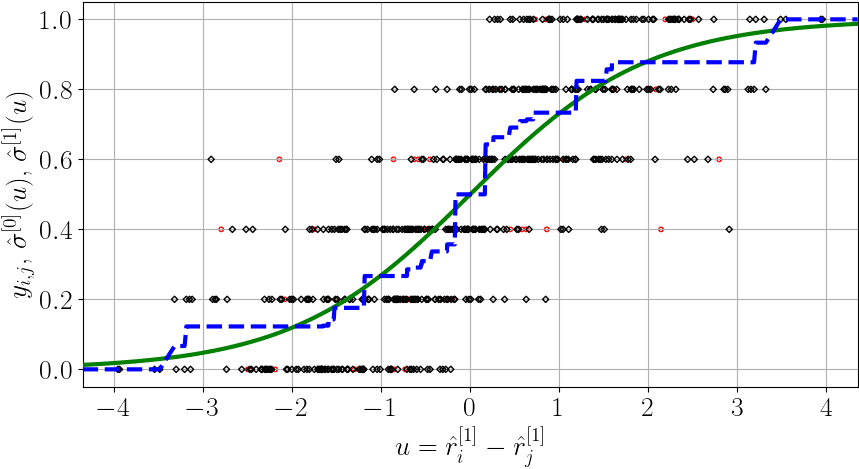}}&
{\includegraphics[width=2.0cm]{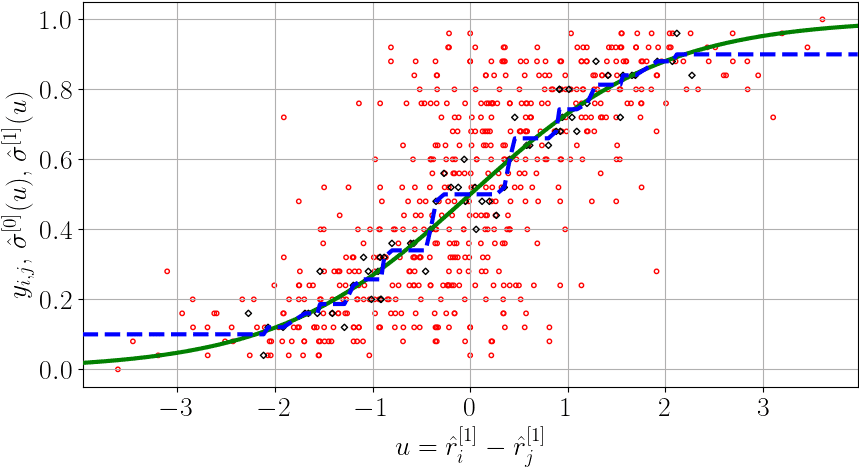}}&
{\includegraphics[width=2.0cm]{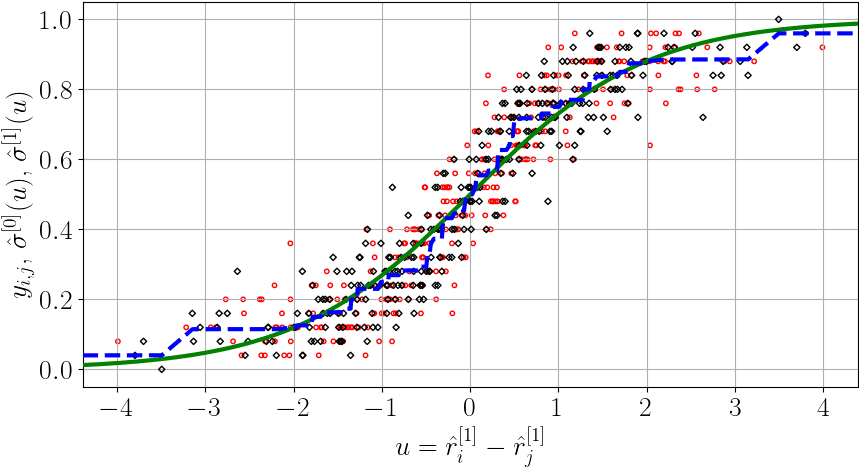}}&
{\includegraphics[width=2.0cm]{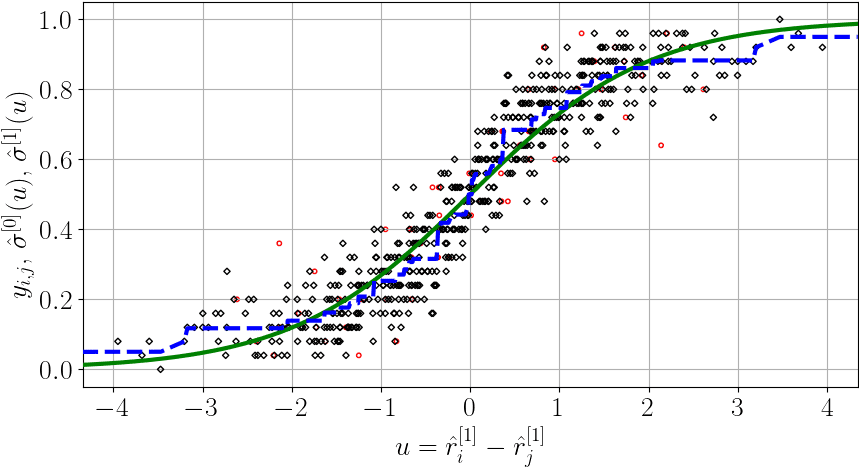}}\\
&\rotatebox{90}{\tiny\,~~~\,$100$}&
{\includegraphics[width=2.0cm]{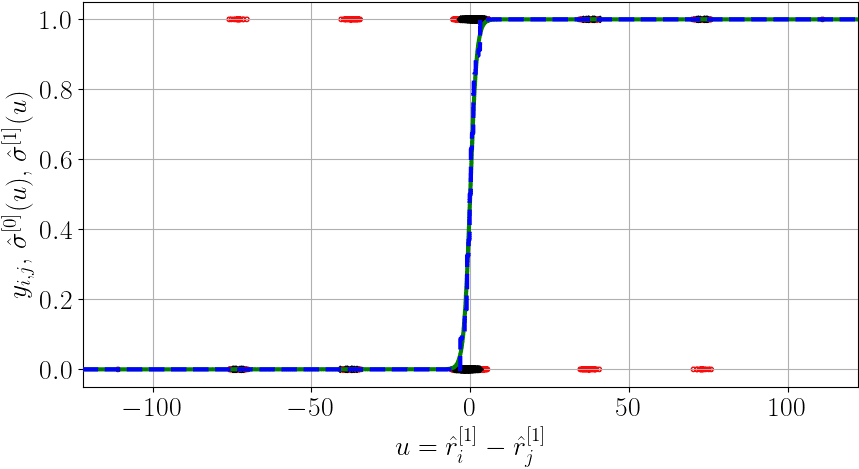}}&
{\includegraphics[width=2.0cm]{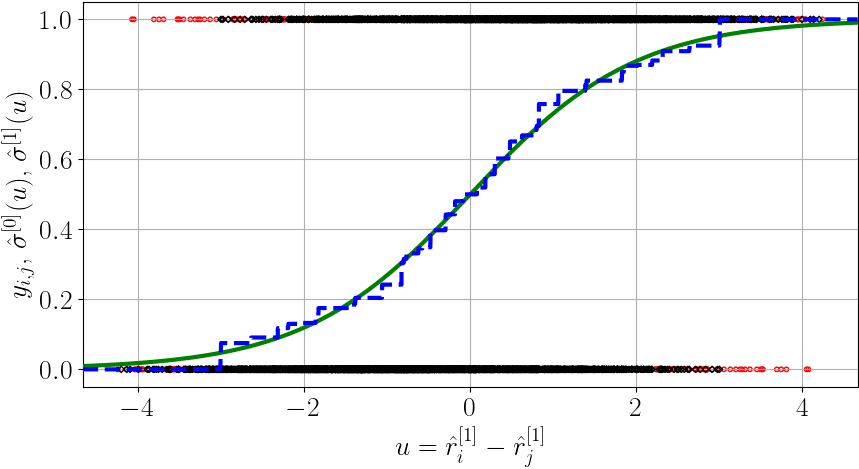}}&
{\includegraphics[width=2.0cm]{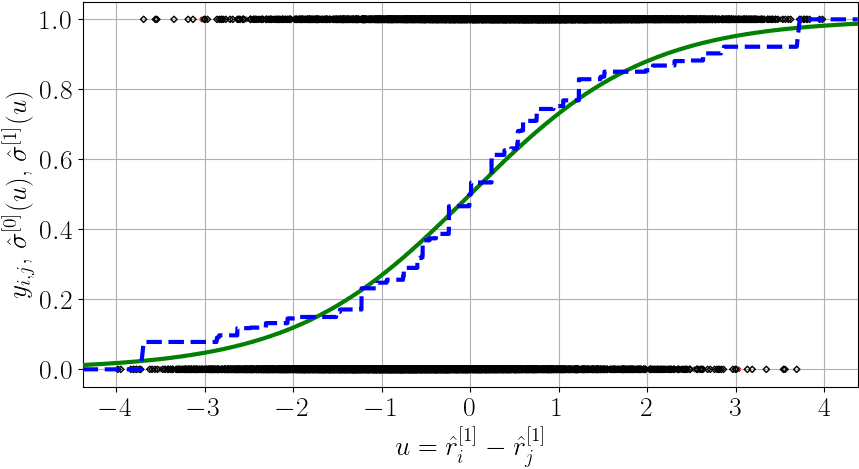}}&
{\includegraphics[width=2.0cm]{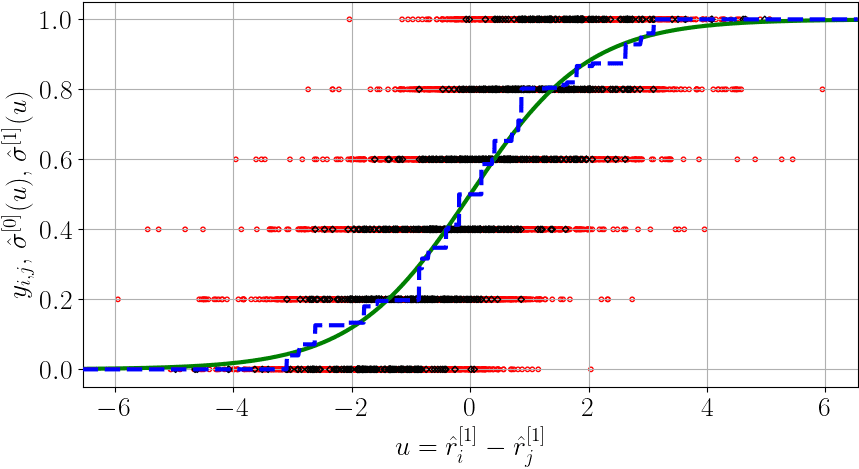}}&
{\includegraphics[width=2.0cm]{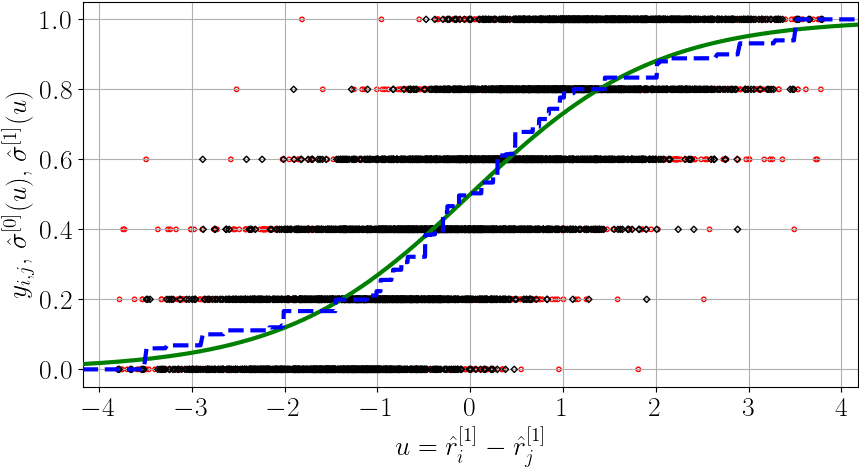}}&
{\includegraphics[width=2.0cm]{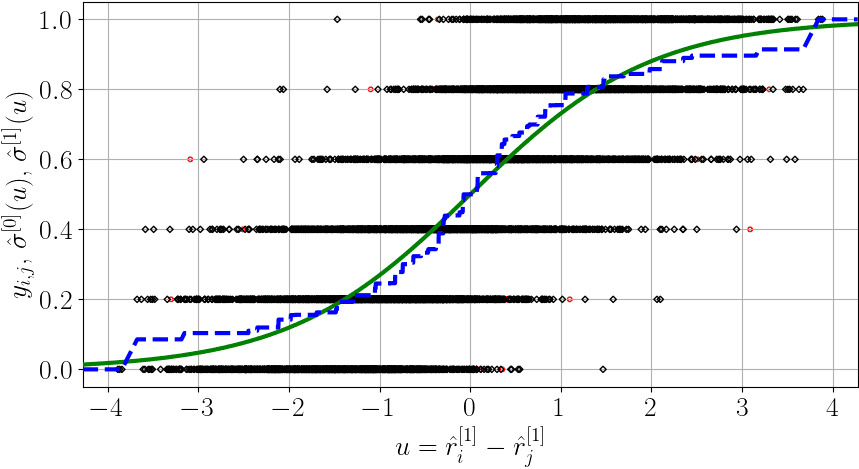}}&
{\includegraphics[width=2.0cm]{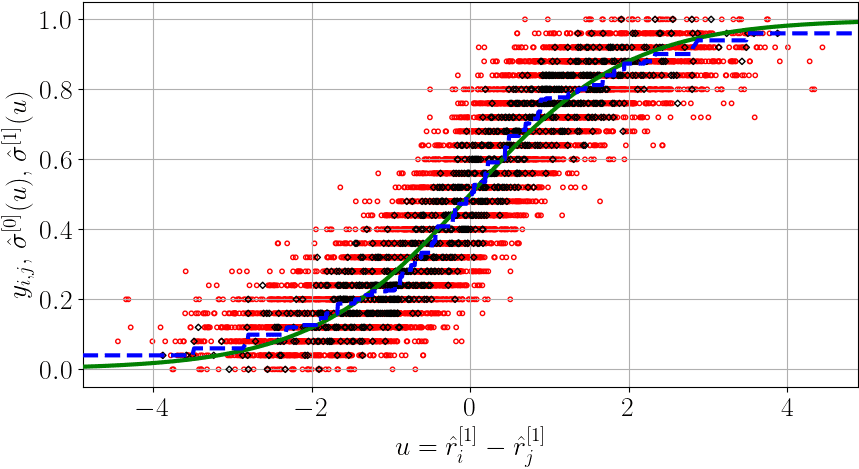}}&
{\includegraphics[width=2.0cm]{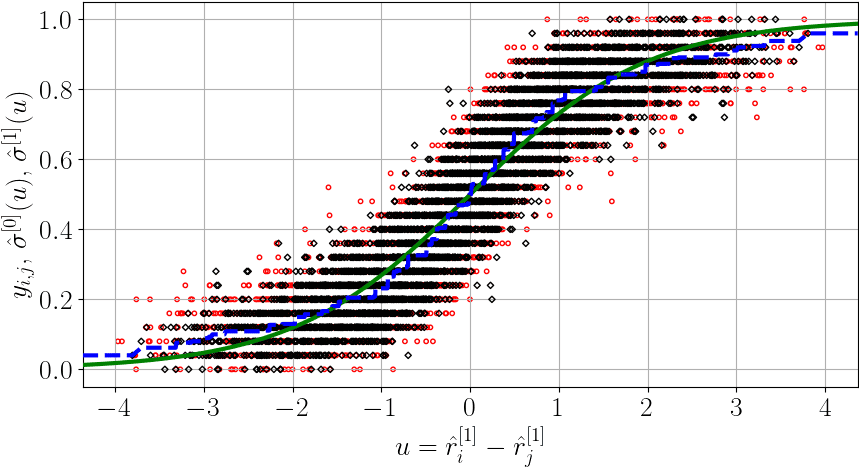}}&
{\includegraphics[width=2.0cm]{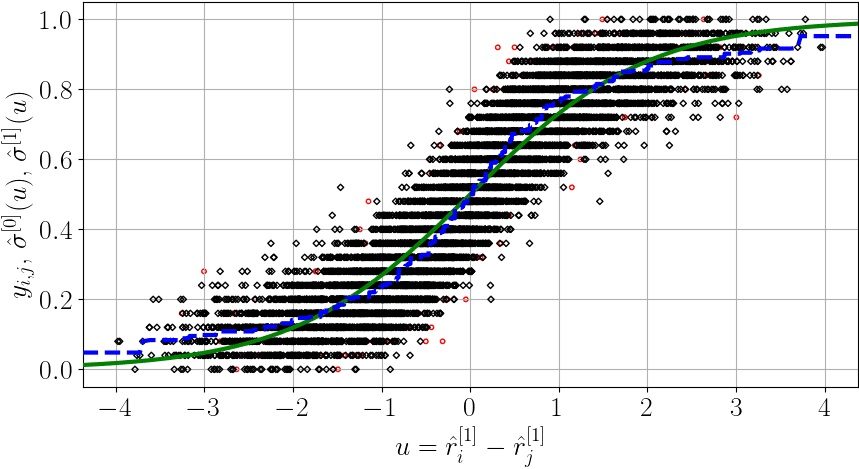}}\\
&\rotatebox{90}{\tiny\,~~~\,$400$}&
{\includegraphics[width=2.0cm]{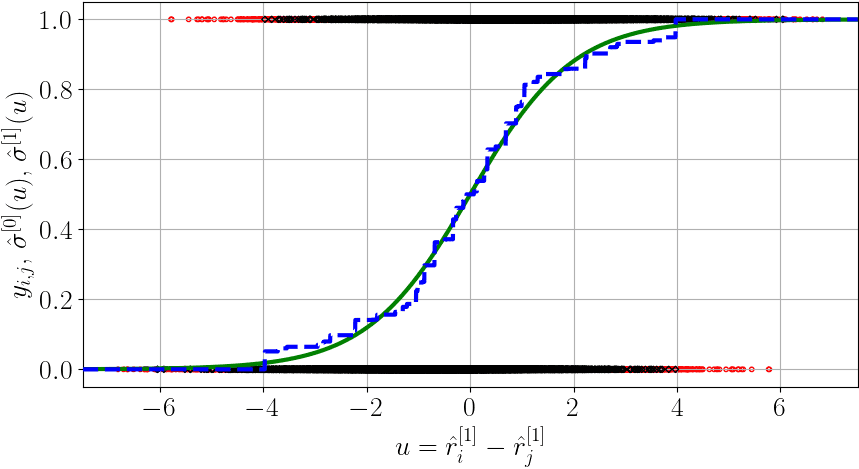}}&
{\includegraphics[width=2.0cm]{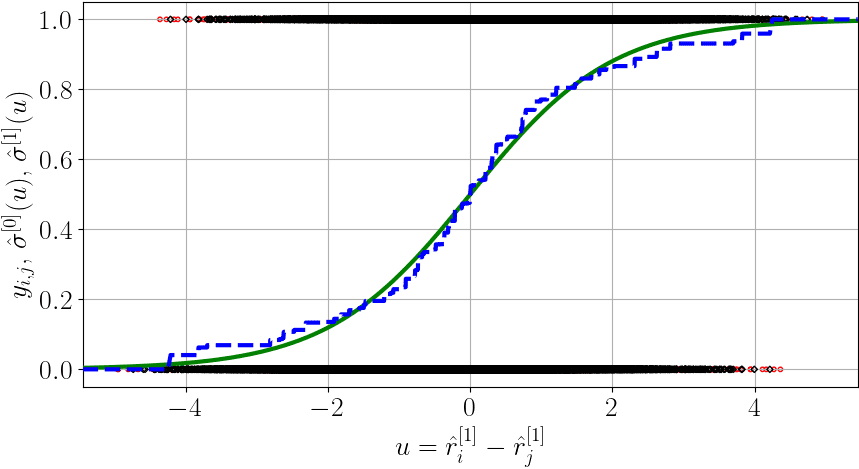}}&
{\includegraphics[width=2.0cm]{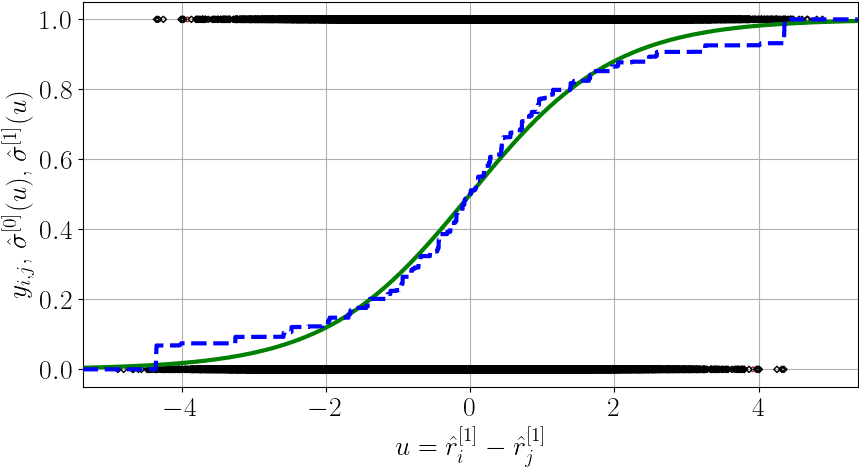}}&
{\includegraphics[width=2.0cm]{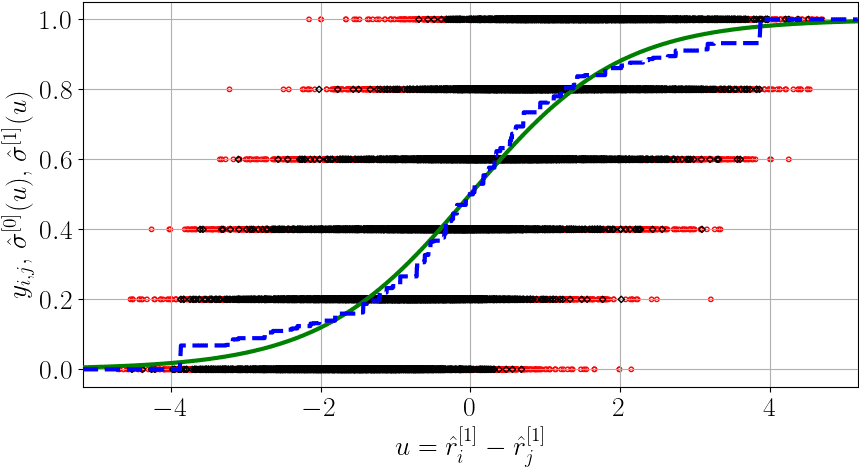}}&
{\includegraphics[width=2.0cm]{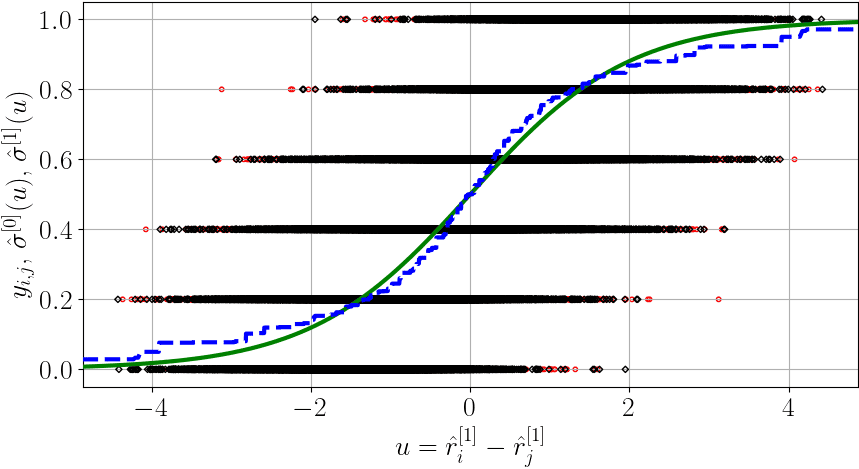}}&
{\includegraphics[width=2.0cm]{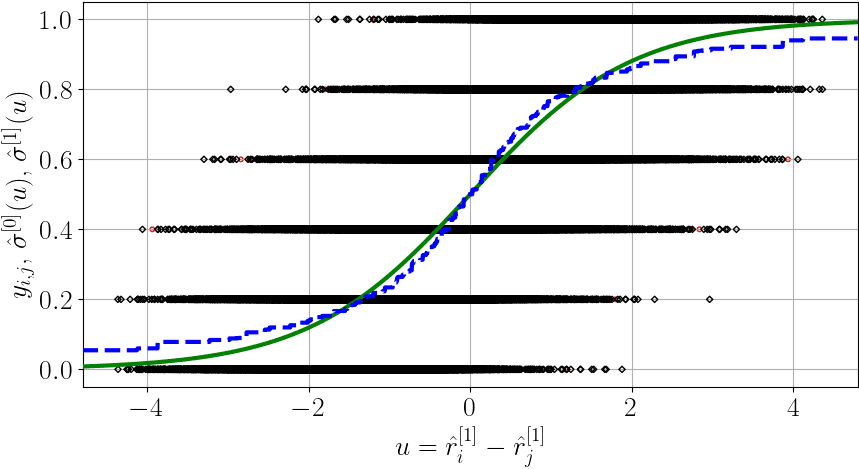}}&
{\includegraphics[width=2.0cm]{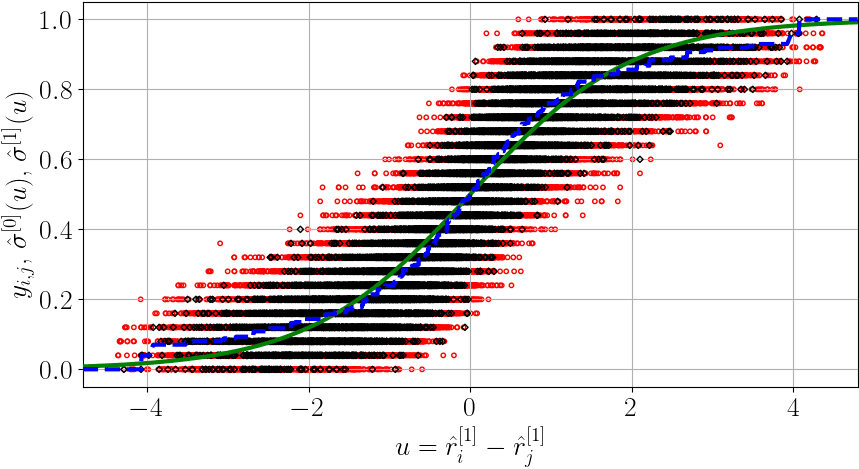}}&
{\includegraphics[width=2.0cm]{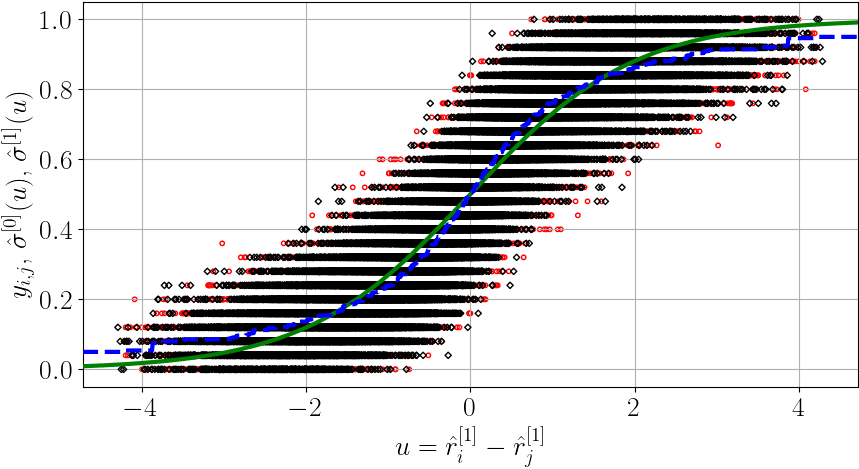}}&
{\includegraphics[width=2.0cm]{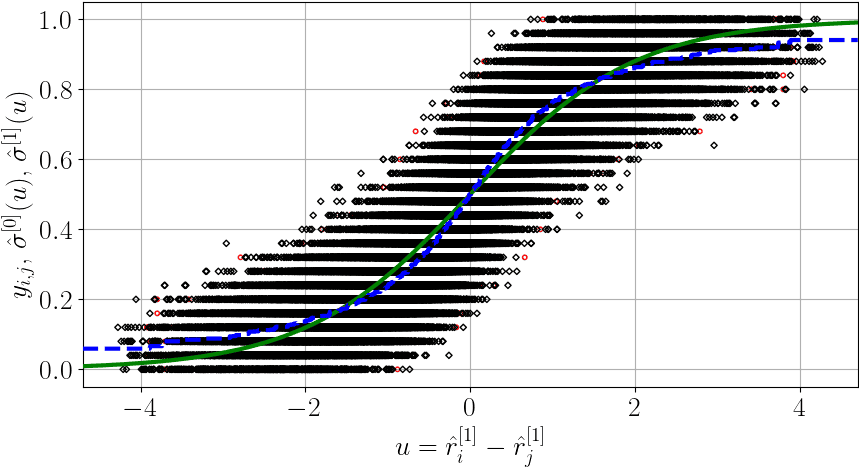}}
\end{tabular}
\begin{tabular}{ccccc}%
\multicolumn{5}{c}{\tiny Rescaled version: $n$, $N$, $|D_\tra|:|D_\tes|=$}\\
{\tiny25, 1, 1:9}&{\tiny25, 1, 5:5}&{\tiny25, 1, 9:1}&{\tiny25, 5, 1:9}&{\tiny100, 1, 1:9}\\
\midrule
{\includegraphics[width=2.0cm]{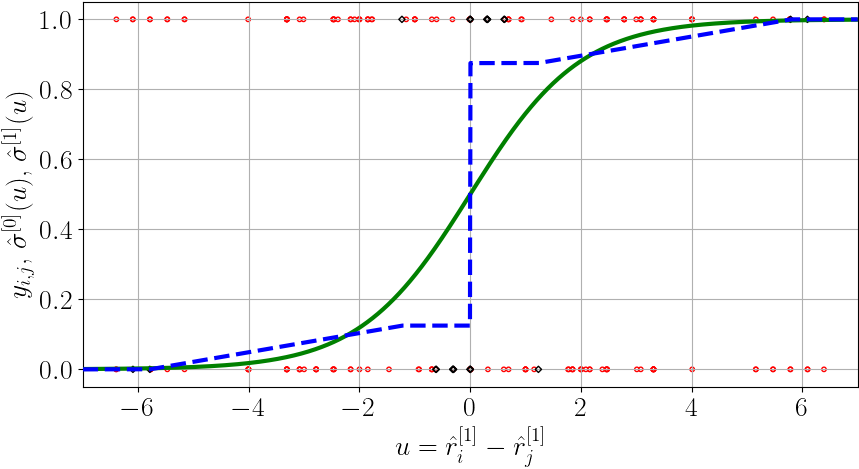}}&
{\includegraphics[width=2.0cm]{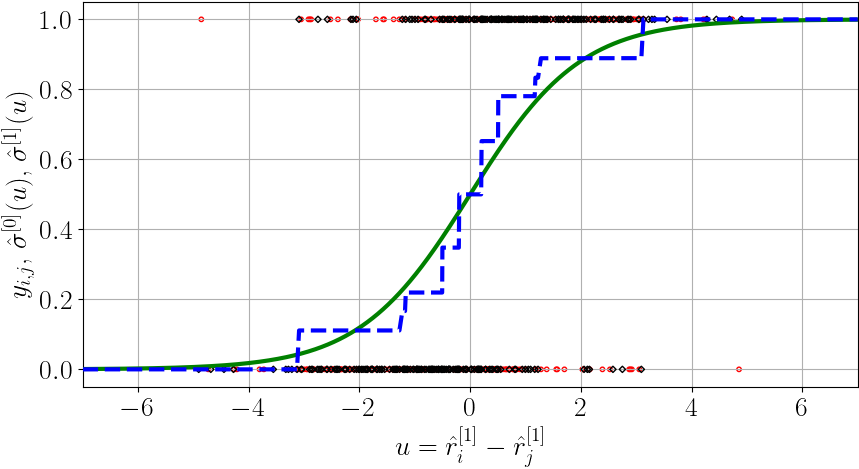}}&
{\includegraphics[width=2.0cm]{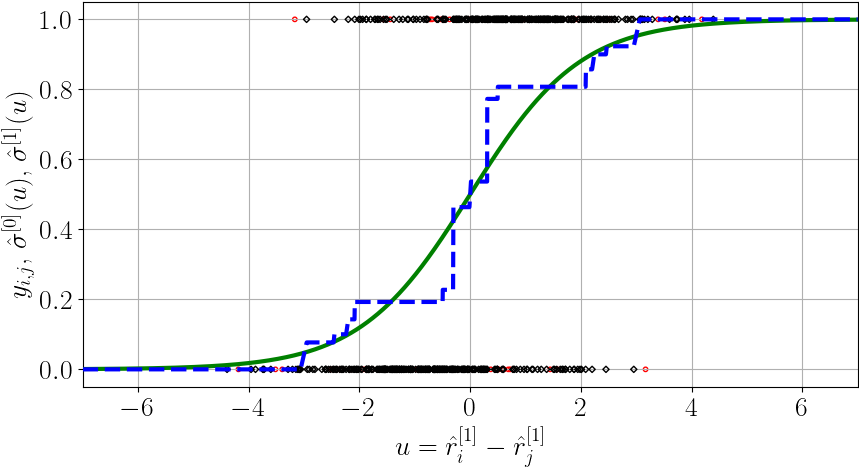}}&
{\includegraphics[width=2.0cm]{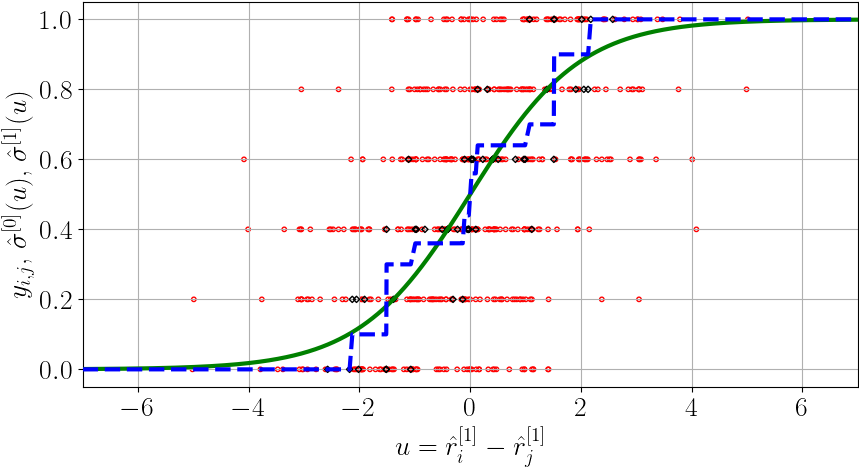}}&
{\includegraphics[width=2.0cm]{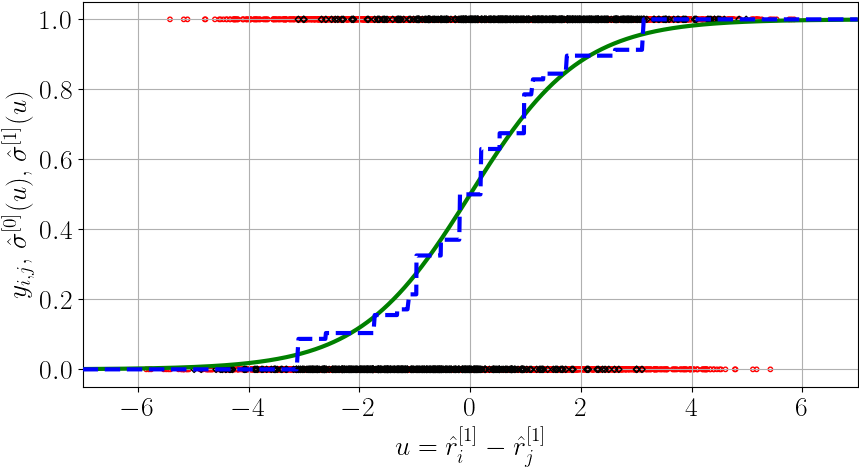}}
\end{tabular}
\caption{%
Part of results of synthetic data experiments, Procedure 1 in Section~\ref{sec:Synthetic} and Appendix~\ref{sec:NLL}:
For Cauchy-$N$ synthetic data with $N=1,5,25$ (left to right),
black diamonds and red circles are training and test data $y_{i,j}$, 
and a green curve and a blue polyline are
the Bradley-Terry model $\hat{\sigma}^{[0]}(\hat{r}_i^{[1]}-\hat{r}_j^{[1]})$ and
isotonic Bradley-Terry model $\hat{\sigma}^{[1]}(\hat{r}_i^{[1]}-\hat{r}_j^{[1]})$
learned with the NLL loss $\phi=\phi_\nll$ in a certain trial,
over the range $[-1.1\cdot\max_{i,j}|\hat{r}_i^{[1]}-\hat{r}_j^{[1]}|,1.1\cdot\max_{i,j}|\hat{r}_i^{[1]}-\hat{r}_j^{[1]}|]$
or $[-7,7]$ in the rescaled version.}
\label{fig:Res-Cauchy-NLL}
\end{sidewaysfigure}
\begin{sidewaystable}
\centering%
\renewcommand{\arraystretch}{0.5}%
\renewcommand{\tabcolsep}{0.5pt}%
\caption{%
Results of synthetic data experiments, Procedure 2 in Section~\ref{sec:Synthetic} and Appendix~\ref{sec:NLL}:
For Cauchy-$N$ synthetic data with $N=1,5,25$ (left to right),
mean and std ($\text{mean}_{\text{std}}$) of 1000 trial test evaluation of 
WPP error \eqref{eq:WPE} with the NLL loss $\phi=\phi_\nll$
and Kendall's Tau \eqref{eq:Kendall}
for the Bradley-Terry model (upper) and model-selected isotonic 
Bradley-Terry model learned with the NLL loss $\phi=\phi_\nll$ (lower).
Smaller \eqref{eq:WPE}, or larger \eqref{eq:Kendall} indicates a better model.
It also includes $p$-value for the Mann-Whitney U test in the parentheses:
A value below the significance level 0.05 implies 
that the isotonic Bradley-Terry model performs significantly better
and is highlighted in red bold.
-- implies that errors include NaN.}
\label{tab:Cauchy-NLL}
\scalebox{0.9}{\begin{minipage}{25cm}
\begin{tabular}{cc|ccccc|ccccc|ccccc}%
&&\multicolumn{5}{c|}{\tiny $N=1$, $|D_\tra|:|D_\tes|=$}&\multicolumn{5}{c|}{\tiny $N=5$, $|D_\tra|:|D_\tes|=$}&\multicolumn{5}{c}{\tiny $N=25$, $|D_\tra|:|D_\tes|=$}\\
&&{\tiny$1:9$}&{\tiny$3:7$}&{\tiny$5:5$}&{\tiny$7:3$}&{\tiny$9:1$}&{\tiny$1:9$}&{\tiny$3:7$}&{\tiny$5:5$}&{\tiny$7:3$}&{\tiny$9:1$}&{\tiny$1:9$}&{\tiny$3:7$}&{\tiny$5:5$}&{\tiny$7:3$}&{\tiny$9:1$}
\\\midrule
\multirow{15}{*}[-6.75mm]{\rotatebox{90}{\tiny\eqref{eq:WPE}, $n=$}}
&\multirow{3}{*}{\rotatebox{90}{\tiny\,$25$~~}}
&$11.8076_{4.3672}$&$2.1516_{1.6534}$&$.6032_{.4154}$&$.3801_{.2002}$&$.3341_{.1798}$&$1.3032_{.6484}$&$.3235_{.1244}$&$.2877_{.0172}$&$.2833_{.0169}$&$.2809_{.0250}$&$.3371_{.1100}$&$.2810_{.0106}$&$.2790_{.0112}$&$.2780_{.0123}$&$.2773_{.0166}$\\&
&--&$2.1516_{1.6534}$&--&--&--&--&--&--&--&--&--&--&--&--&--\\&
&($.7136$)&($.5000$)&($.5362$)&($.5438$)&($.5460$)&($.7645$)&($.8782$)&($.9476$)&($.8897$)&($.6829$)&($.8684$)&($.9655$)&($.9281$)&($.6875$)&($.4358$)\\
&\multirow{3}{*}{\rotatebox{90}{\tiny\,$50$~~}}
&$6.3959_{3.1135}$&$.4528_{.1934}$&$.3143_{.0543}$&$.2967_{.0256}$&$.2905_{.0254}$&$.5110_{.2260}$&$.2860_{.0097}$&$.2819_{.0085}$&$.2806_{.0091}$&$.2797_{.0125}$&$.2893_{.0247}$&$.2790_{.0073}$&$.2783_{.0074}$&$.2780_{.0077}$&$.2778_{.0093}$\\&
&$6.3959_{3.1135}$&--&--&--&--&--&--&--&--&--&--&--&--&--&--\\&
&($.5000$)&($.5119$)&($.5367$)&($.6271$)&($.5387$)&($.5739$)&($.9896$)&($.9988$)&($.9624$)&($.7921$)&($.9575$)&($.9768$)&($.8379$)&($.2425$)&($.1130$)\\
&\multirow{3}{*}{\rotatebox{90}{\tiny\,$100$~~}}
&$.9624_{.3679}$&$.3020_{.0186}$&$.2886_{.0083}$&$.2848_{.0085}$&$.2828_{.0124}$&$.2960_{.0235}$&$.2807_{.0052}$&$.2792_{.0053}$&$.2786_{.0057}$&$.2783_{.0072}$&$.2800_{.0049}$&$.2778_{.0049}$&$.2775_{.0050}$&$.2774_{.0051}$&\tcr{$\bf.2774_{.0059}$}\\&
&$.9624_{.3679}$&--&--&--&--&--&--&--&--&--&--&--&--&--&\tcr{--}\\&
&($.5000$)&($.6354$)&($.8181$)&($.8499$)&($.5953$)&($.8472$)&($1.0000$)&($1.0000$)&($.9995$)&($.6208$)&($.9986$)&($.9980$)&($.5629$)&($.0551$)&(\tcr{$\bf.0037$})\\
&\multirow{3}{*}{\rotatebox{90}{\tiny\,$200$~~}}
&$.3476_{.0412}$&$.2863_{.0044}$&$.2824_{.0044}$&$.2808_{.0048}$&$.2798_{.0066}$&$.2831_{.0036}$&$.2789_{.0036}$&$.2782_{.0037}$&$.2779_{.0038}$&$.2777_{.0045}$&$.2784_{.0035}$&$.2775_{.0035}$&$.2774_{.0036}$&\tcr{$\bf.2773_{.0036}$}&\tcr{$\bf.2772_{.0039}$}\\&
&$.3476_{.0412}$&--&--&--&--&--&--&--&--&--&--&--&--&\tcr{--}&\tcr{--}\\&
&($.5000$)&($.9847$)&($.9992$)&($.9750$)&($.8104$)&($1.0000$)&($1.0000$)&($1.0000$)&($.9838$)&($.3354$)&($.9859$)&($.8261$)&($.5511$)&(\tcr{$\bf.0136$})&(\tcr{$\bf.0000$})\\
&\multirow{3}{*}{\rotatebox{90}{\tiny\,$400$~~}}
&$.2927_{.0046}$&$.2816_{.0028}$&$.2798_{.0028}$&$.2790_{.0030}$&$.2786_{.0038}$&$.2800_{.0026}$&$.2781_{.0026}$&$.2778_{.0026}$&$.2776_{.0027}$&\tcr{$\bf.2775_{.0029}$}&$.2778_{.0025}$&$.2774_{.0026}$&\tcr{$\bf.2774_{.0026}$}&\tcr{$\bf.2773_{.0026}$}&\tcr{$\bf.2773_{.0027}$}\\&
&--&--&--&--&--&--&--&--&--&\tcr{--}&--&--&\tcr{--}&\tcr{--}&\tcr{--}\\&
&($.7571$)&($1.0000$)&($1.0000$)&($.9986$)&($.3936$)&($1.0000$)&($1.0000$)&($1.0000$)&($.9800$)&(\tcr{$\bf.0429$})&($.9437$)&($.3308$)&(\tcr{$\bf.0378$})&(\tcr{$\bf.0009$})&(\tcr{$\bf.0000$})\\
\midrule
\multirow{15}{*}[-6.75mm]{\rotatebox{90}{\tiny\eqref{eq:Kendall}, $n=$}}
&\multirow{3}{*}{\rotatebox{90}{\tiny\,$25$~~}}
&$.1681_{.0867}$&\tcr{$\bf.2679_{.0743}$}&\tcr{$\bf.3221_{.0720}$}&\tcr{$\bf.3479_{.0836}$}&\tcr{$\bf.3578_{.1319}$}&$.3226_{.0896}$&\tcr{$\bf.5402_{.0493}$}&\tcr{$\bf.5755_{.0456}$}&\tcr{$\bf.5903_{.0518}$}&\tcr{$\bf.5995_{.0780}$}&$.5724_{.0800}$&\tcr{$\bf.7514_{.0250}$}&\tcr{$\bf.7699_{.0220}$}&\tcr{$\bf.7774_{.0255}$}&\tcr{$\bf.7818_{.0416}$}\\&
&$.1705_{.0892}$&\tcr{$\bf.2799_{.0786}$}&\tcr{$\bf.3359_{.0776}$}&\tcr{$\bf.3619_{.0886}$}&\tcr{$\bf.3724_{.1421}$}&$.3289_{.0921}$&\tcr{$\bf.5542_{.0515}$}&\tcr{$\bf.5920_{.0471}$}&\tcr{$\bf.6071_{.0531}$}&\tcr{$\bf.6175_{.0812}$}&$.5770_{.0807}$&\tcr{$\bf.7634_{.0257}$}&\tcr{$\bf.7828_{.0224}$}&\tcr{$\bf.7905_{.0256}$}&\tcr{$\bf.7956_{.0414}$}\\&
&($.2606$)&(\tcr{$\bf.0004$})&(\tcr{$\bf.0000$})&(\tcr{$\bf.0001$})&(\tcr{$\bf.0091$})&($.0581$)&(\tcr{$\bf.0000$})&(\tcr{$\bf.0000$})&(\tcr{$\bf.0000$})&(\tcr{$\bf.0000$})&($.0879$)&(\tcr{$\bf.0000$})&(\tcr{$\bf.0000$})&(\tcr{$\bf.0000$})&(\tcr{$\bf.0000$})\\
&\multirow{3}{*}{\rotatebox{90}{\tiny\,$50$~~}}
&\tcr{$\bf.2113_{.0518}$}&\tcr{$\bf.3330_{.0364}$}&\tcr{$\bf.3631_{.0366}$}&\tcr{$\bf.3761_{.0404}$}&\tcr{$\bf.3843_{.0611}$}&\tcr{$\bf.4774_{.0416}$}&\tcr{$\bf.5831_{.0242}$}&\tcr{$\bf.5998_{.0234}$}&\tcr{$\bf.6065_{.0256}$}&\tcr{$\bf.6109_{.0365}$}&\tcr{$\bf.7146_{.0252}$}&\tcr{$\bf.7749_{.0114}$}&\tcr{$\bf.7833_{.0112}$}&\tcr{$\bf.7868_{.0124}$}&\tcr{$\bf.7889_{.0181}$}\\&
&\tcr{$\bf.2184_{.0543}$}&\tcr{$\bf.3461_{.0382}$}&\tcr{$\bf.3780_{.0393}$}&\tcr{$\bf.3918_{.0429}$}&\tcr{$\bf.4010_{.0664}$}&\tcr{$\bf.4898_{.0428}$}&\tcr{$\bf.6001_{.0253}$}&\tcr{$\bf.6174_{.0244}$}&\tcr{$\bf.6239_{.0263}$}&\tcr{$\bf.6278_{.0379}$}&\tcr{$\bf.7255_{.0260}$}&\tcr{$\bf.7881_{.0116}$}&\tcr{$\bf.7965_{.0114}$}&\tcr{$\bf.7995_{.0126}$}&\tcr{$\bf.8010_{.0181}$}\\&
&(\tcr{$\bf.0023$})&(\tcr{$\bf.0000$})&(\tcr{$\bf.0000$})&(\tcr{$\bf.0000$})&(\tcr{$\bf.0000$})&(\tcr{$\bf.0000$})&(\tcr{$\bf.0000$})&(\tcr{$\bf.0000$})&(\tcr{$\bf.0000$})&(\tcr{$\bf.0000$})&(\tcr{$\bf.0000$})&(\tcr{$\bf.0000$})&(\tcr{$\bf.0000$})&(\tcr{$\bf.0000$})&(\tcr{$\bf.0000$})\\
&\multirow{3}{*}{\rotatebox{90}{\tiny\,$100$~~}}
&\tcr{$\bf.2972_{.0254}$}&\tcr{$\bf.3727_{.0196}$}&\tcr{$\bf.3886_{.0198}$}&\tcr{$\bf.3958_{.0225}$}&\tcr{$\bf.4002_{.0322}$}&\tcr{$\bf.5618_{.0174}$}&\tcr{$\bf.6051_{.0138}$}&\tcr{$\bf.6129_{.0139}$}&\tcr{$\bf.6161_{.0151}$}&\tcr{$\bf.6180_{.0208}$}&\tcr{$\bf.7639_{.0084}$}&\tcr{$\bf.7859_{.0064}$}&\tcr{$\bf.7898_{.0065}$}&\tcr{$\bf.7913_{.0072}$}&\tcr{$\bf.7922_{.0099}$}\\&
&\tcr{$\bf.3089_{.0266}$}&\tcr{$\bf.3875_{.0211}$}&\tcr{$\bf.4044_{.0208}$}&\tcr{$\bf.4112_{.0237}$}&\tcr{$\bf.4158_{.0339}$}&\tcr{$\bf.5773_{.0182}$}&\tcr{$\bf.6223_{.0143}$}&\tcr{$\bf.6294_{.0143}$}&\tcr{$\bf.6317_{.0154}$}&\tcr{$\bf.6326_{.0211}$}&\tcr{$\bf.7771_{.0088}$}&\tcr{$\bf.7988_{.0064}$}&\tcr{$\bf.8017_{.0066}$}&\tcr{$\bf.8023_{.0072}$}&\tcr{$\bf.8026_{.0099}$}\\&
&(\tcr{$\bf.0000$})&(\tcr{$\bf.0000$})&(\tcr{$\bf.0000$})&(\tcr{$\bf.0000$})&(\tcr{$\bf.0000$})&(\tcr{$\bf.0000$})&(\tcr{$\bf.0000$})&(\tcr{$\bf.0000$})&(\tcr{$\bf.0000$})&(\tcr{$\bf.0000$})&(\tcr{$\bf.0000$})&(\tcr{$\bf.0000$})&(\tcr{$\bf.0000$})&(\tcr{$\bf.0000$})&(\tcr{$\bf.0000$})\\
&\multirow{3}{*}{\rotatebox{90}{\tiny\,$200$~~}}
&\tcr{$\bf.3525_{.0138}$}&\tcr{$\bf.3926_{.0120}$}&\tcr{$\bf.4010_{.0122}$}&\tcr{$\bf.4050_{.0132}$}&\tcr{$\bf.4072_{.0178}$}&\tcr{$\bf.5947_{.0100}$}&\tcr{$\bf.6149_{.0090}$}&\tcr{$\bf.6187_{.0092}$}&\tcr{$\bf.6205_{.0097}$}&\tcr{$\bf.6213_{.0119}$}&\tcr{$\bf.7810_{.0046}$}&\tcr{$\bf.7909_{.0042}$}&\tcr{$\bf.7927_{.0043}$}&\tcr{$\bf.7936_{.0046}$}&\tcr{$\bf.7940_{.0056}$}\\&
&\tcr{$\bf.3656_{.0146}$}&\tcr{$\bf.4070_{.0126}$}&\tcr{$\bf.4147_{.0127}$}&\tcr{$\bf.4181_{.0137}$}&\tcr{$\bf.4199_{.0184}$}&\tcr{$\bf.6108_{.0104}$}&\tcr{$\bf.6296_{.0092}$}&\tcr{$\bf.6319_{.0094}$}&\tcr{$\bf.6327_{.0098}$}&\tcr{$\bf.6327_{.0121}$}&\tcr{$\bf.7939_{.0048}$}&\tcr{$\bf.8019_{.0043}$}&\tcr{$\bf.8026_{.0044}$}&\tcr{$\bf.8025_{.0046}$}&\tcr{$\bf.8023_{.0056}$}\\&
&(\tcr{$\bf.0000$})&(\tcr{$\bf.0000$})&(\tcr{$\bf.0000$})&(\tcr{$\bf.0000$})&(\tcr{$\bf.0000$})&(\tcr{$\bf.0000$})&(\tcr{$\bf.0000$})&(\tcr{$\bf.0000$})&(\tcr{$\bf.0000$})&(\tcr{$\bf.0000$})&(\tcr{$\bf.0000$})&(\tcr{$\bf.0000$})&(\tcr{$\bf.0000$})&(\tcr{$\bf.0000$})&(\tcr{$\bf.0000$})\\
&\multirow{3}{*}{\rotatebox{90}{\tiny\,$400$~~}}
&\tcr{$\bf.3816_{.0083}$}&\tcr{$\bf.4031_{.0078}$}&\tcr{$\bf.4075_{.0079}$}&\tcr{$\bf.4095_{.0083}$}&\tcr{$\bf.4106_{.0104}$}&\tcr{$\bf.6099_{.0064}$}&\tcr{$\bf.6197_{.0062}$}&\tcr{$\bf.6215_{.0063}$}&\tcr{$\bf.6224_{.0064}$}&\tcr{$\bf.6229_{.0074}$}&\tcr{$\bf.7885_{.0029}$}&\tcr{$\bf.7932_{.0029}$}&\tcr{$\bf.7941_{.0029}$}&\tcr{$\bf.7945_{.0029}$}&\tcr{$\bf.7947_{.0035}$}\\&
&\tcr{$\bf.3944_{.0088}$}&\tcr{$\bf.4147_{.0081}$}&\tcr{$\bf.4179_{.0082}$}&\tcr{$\bf.4192_{.0084}$}&\tcr{$\bf.4197_{.0107}$}&\tcr{$\bf.6237_{.0066}$}&\tcr{$\bf.6308_{.0064}$}&\tcr{$\bf.6314_{.0064}$}&\tcr{$\bf.6313_{.0064}$}&\tcr{$\bf.6313_{.0074}$}&\tcr{$\bf.7994_{.0031}$}&\tcr{$\bf.8019_{.0029}$}&\tcr{$\bf.8018_{.0029}$}&\tcr{$\bf.8015_{.0029}$}&\tcr{$\bf.8012_{.0035}$}\\&
&(\tcr{$\bf.0000$})&(\tcr{$\bf.0000$})&(\tcr{$\bf.0000$})&(\tcr{$\bf.0000$})&(\tcr{$\bf.0000$})&(\tcr{$\bf.0000$})&(\tcr{$\bf.0000$})&(\tcr{$\bf.0000$})&(\tcr{$\bf.0000$})&(\tcr{$\bf.0000$})&(\tcr{$\bf.0000$})&(\tcr{$\bf.0000$})&(\tcr{$\bf.0000$})&(\tcr{$\bf.0000$})&(\tcr{$\bf.0000$})\\
\end{tabular}\end{minipage}}
\end{sidewaystable}

\begin{sidewaysfigure}
\centering%
\renewcommand{\arraystretch}{0.5}%
\renewcommand{\tabcolsep}{0.5pt}%
\begin{tabular}{c|ccc|ccc|ccc}%
&\multicolumn{3}{c|}{\tiny PL, $|D_\tra|:|D_\tes|=$}&\multicolumn{3}{c|}{\tiny MLB, $|D_\tra|:|D_\tes|=$}&\multicolumn{3}{c}{\tiny ATP, $|D_\tra|:|D_\tes|=$}\\
&{\tiny$1:9$}&{\tiny$5:5$}&{\tiny$9:1$}&{\tiny$1:9$}&{\tiny$5:5$}&{\tiny$9:1$}&{\tiny$1:9$}&{\tiny$5:5$}&{\tiny$9:1$}
\\\midrule
\raisebox{1.75ex}{\rotatebox{90}{\tiny\eqref{eq:WPE}}}&
{\includegraphics[width=2.0cm]{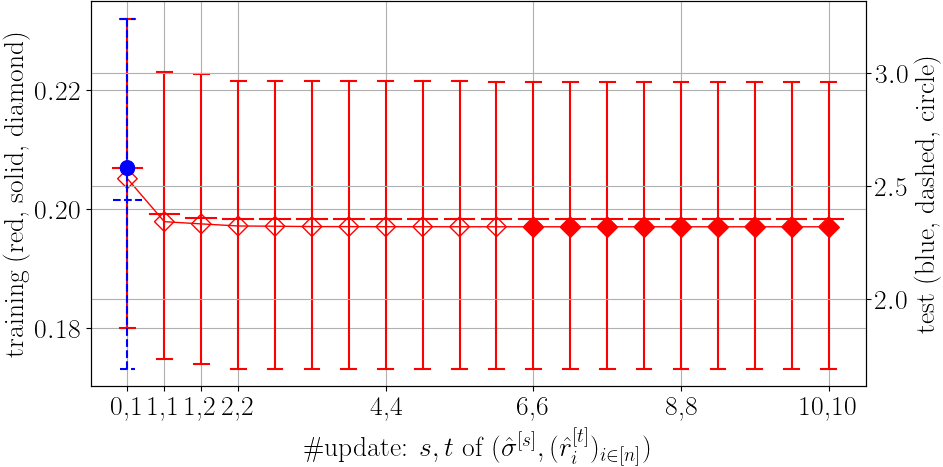}}&
{\includegraphics[width=2.0cm]{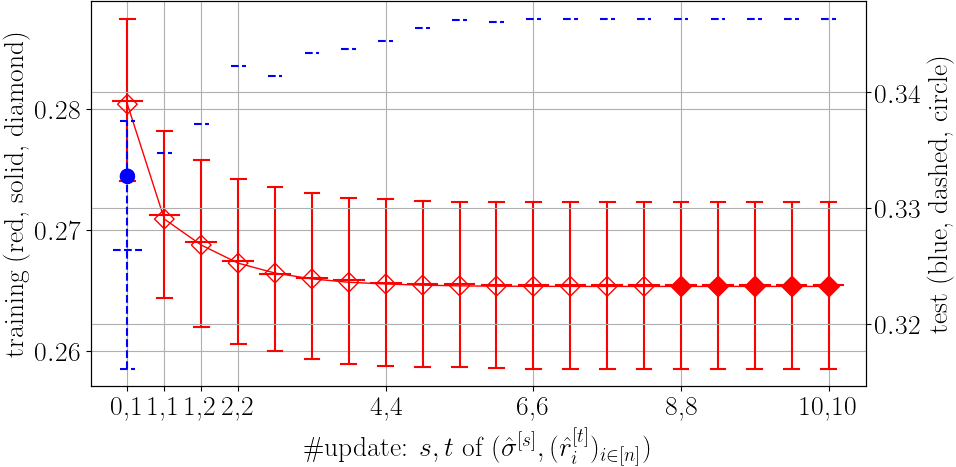}}&
{\includegraphics[width=2.0cm]{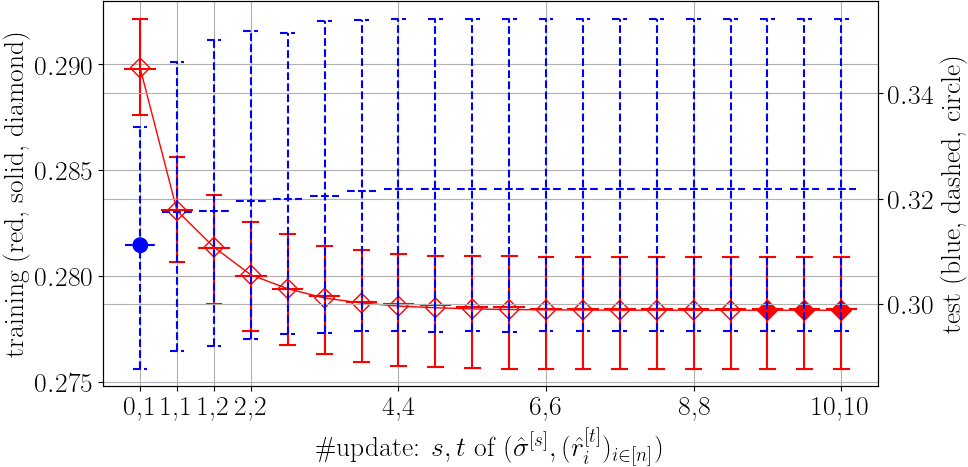}}&
{\includegraphics[width=2.0cm]{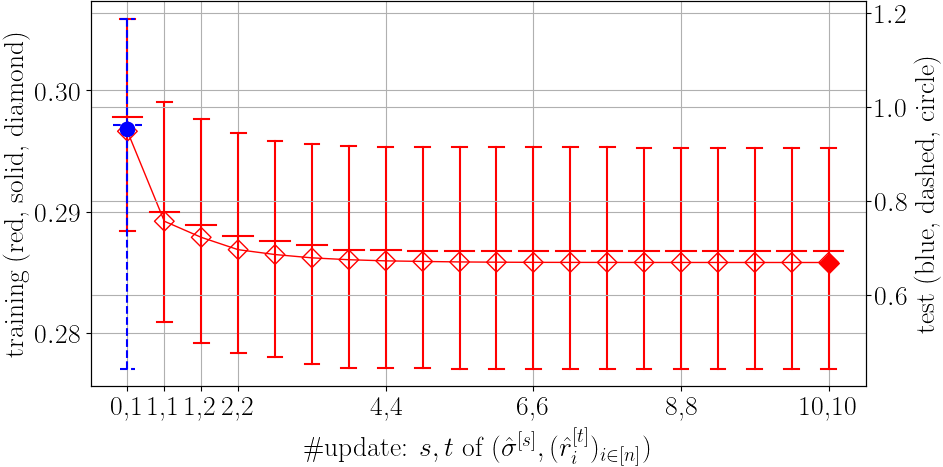}}&
{\includegraphics[width=2.0cm]{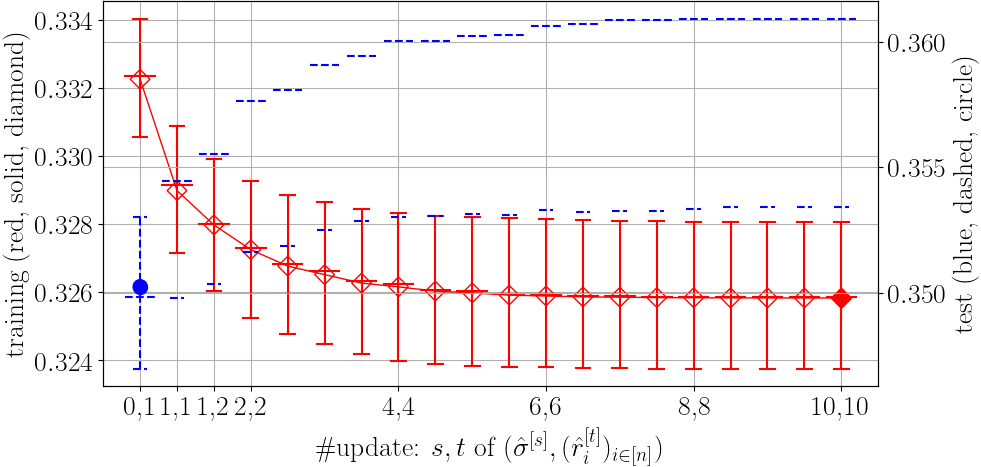}}&
{\includegraphics[width=2.0cm]{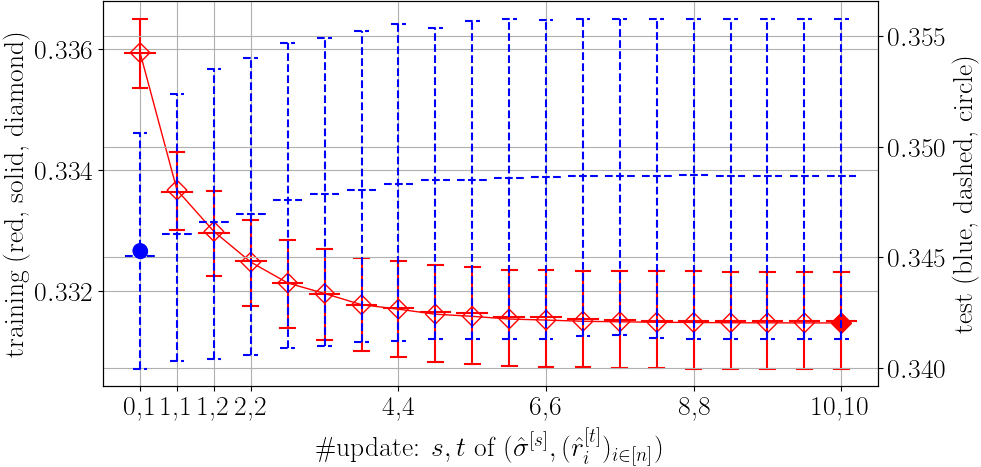}}&
{\includegraphics[width=2.0cm]{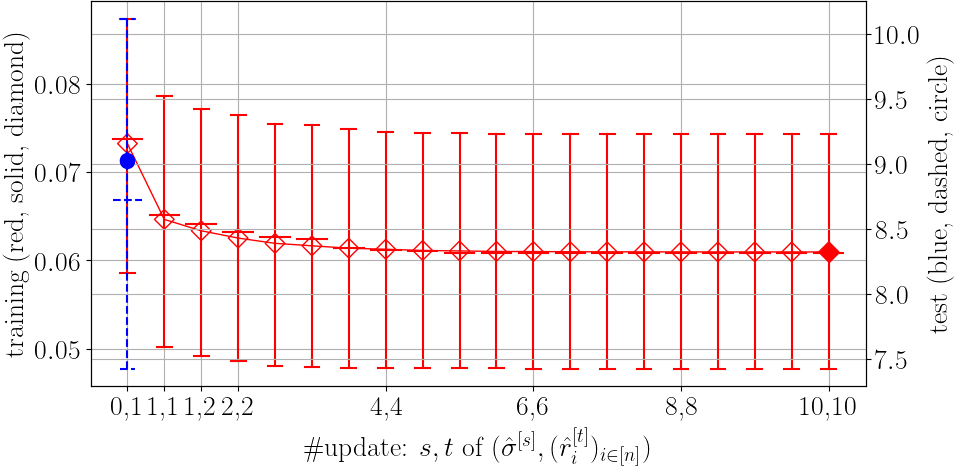}}&
{\includegraphics[width=2.0cm]{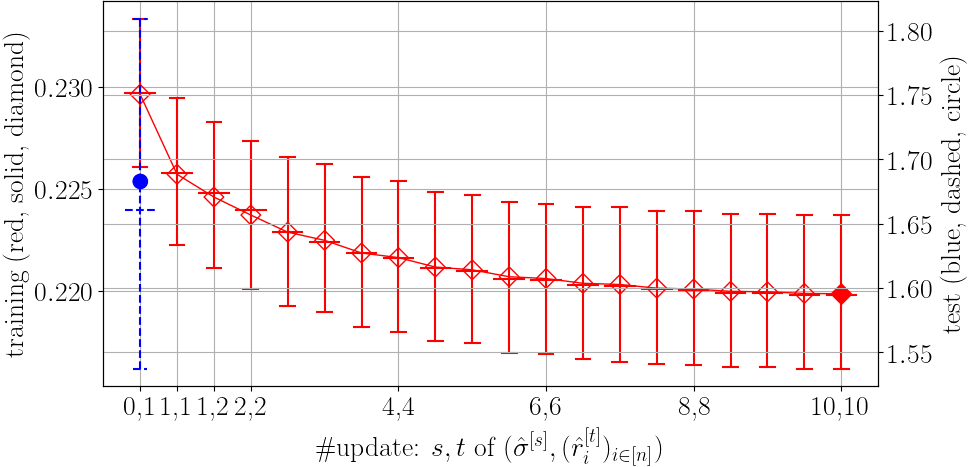}}&
{\includegraphics[width=2.0cm]{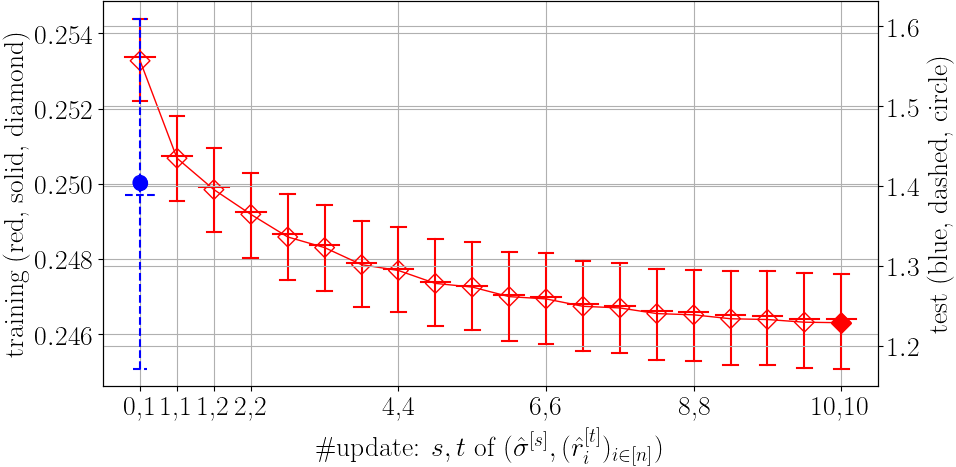}}
\\\midrule
\raisebox{1.75ex}{\rotatebox{90}{\tiny\eqref{eq:Kendall}}}&
{\includegraphics[width=2.0cm]{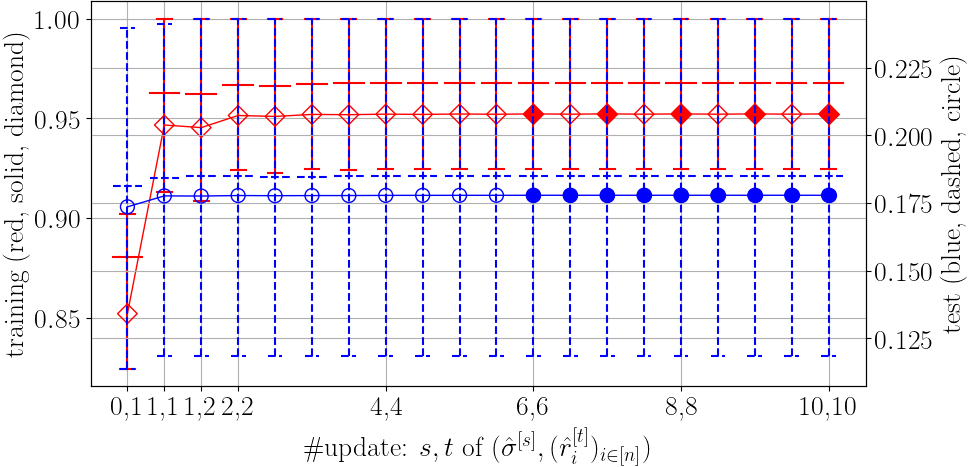}}&
\CF{\includegraphics[width=2.0cm]{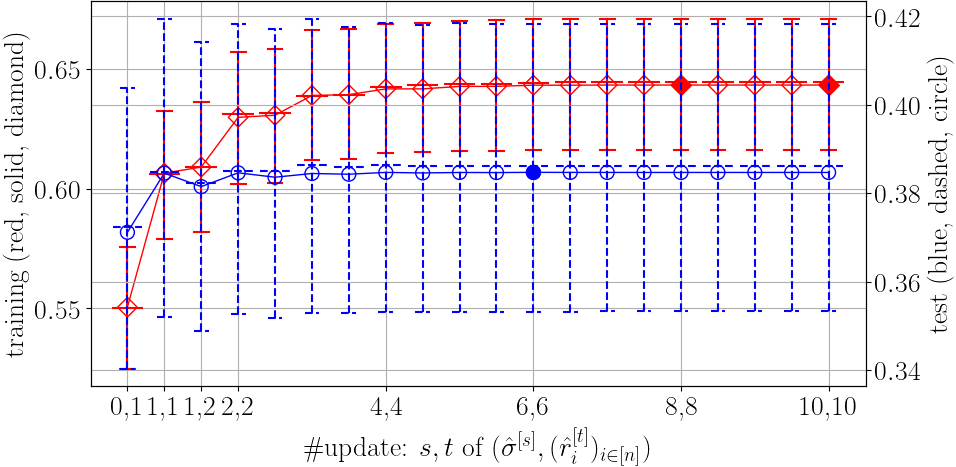}}&
\CF{\includegraphics[width=2.0cm]{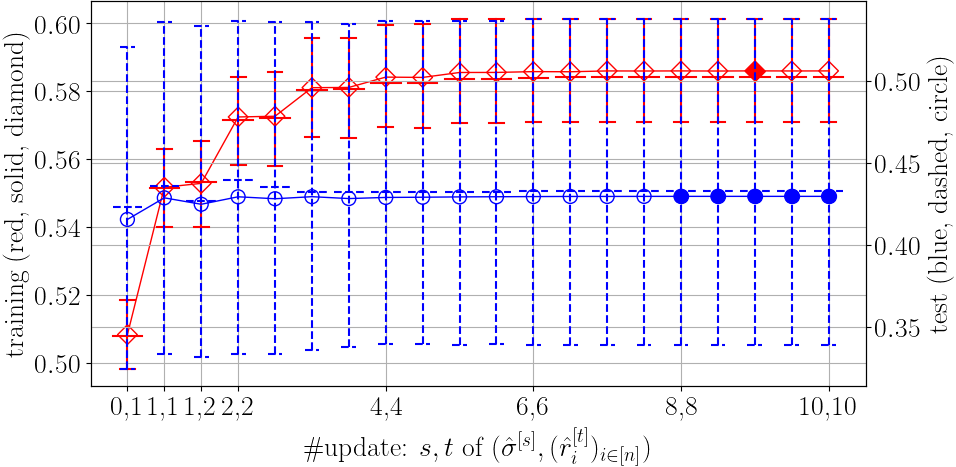}}&
{\includegraphics[width=2.0cm]{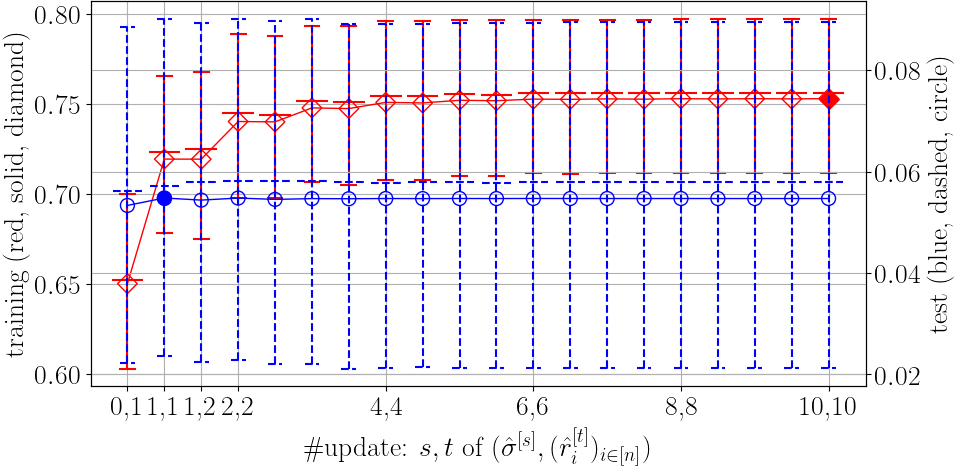}}&
\CF{\includegraphics[width=2.0cm]{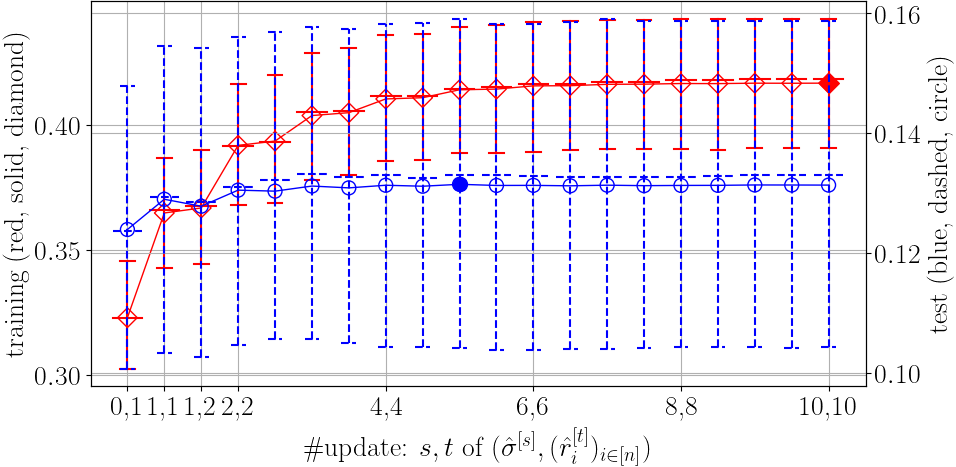}}&
\CF{\includegraphics[width=2.0cm]{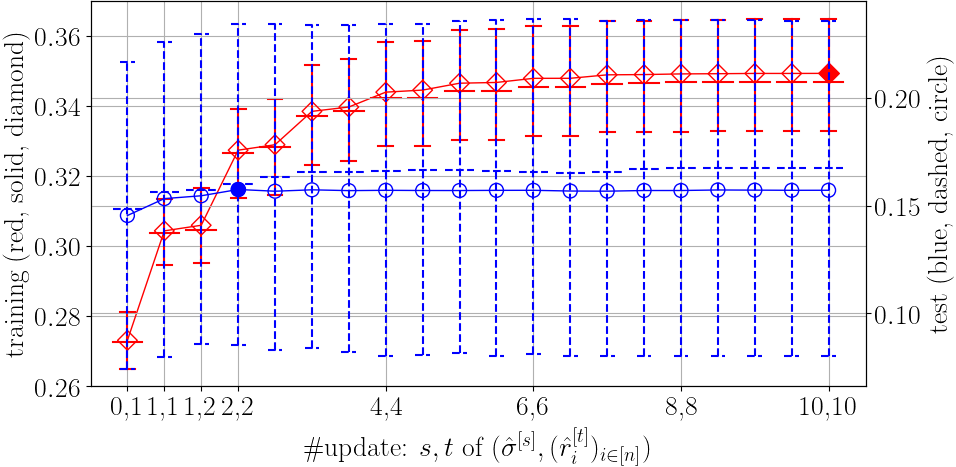}}&
\CF{\includegraphics[width=2.0cm]{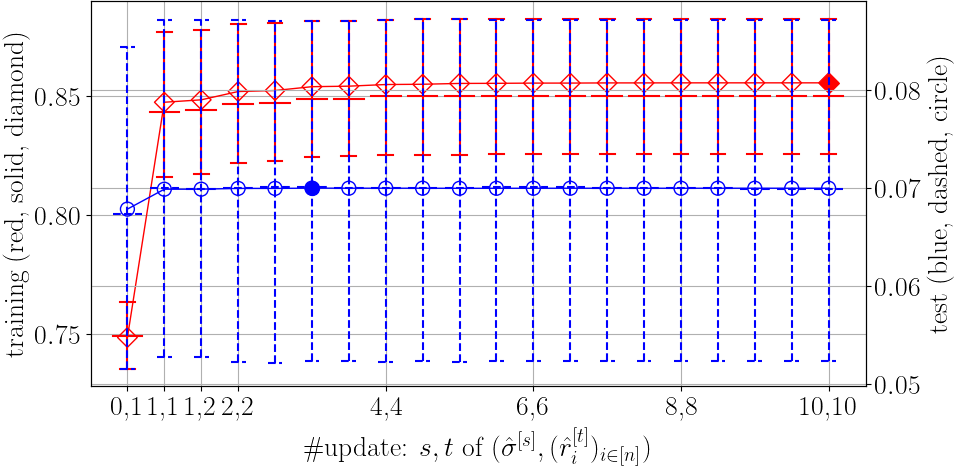}}&
\CF{\includegraphics[width=2.0cm]{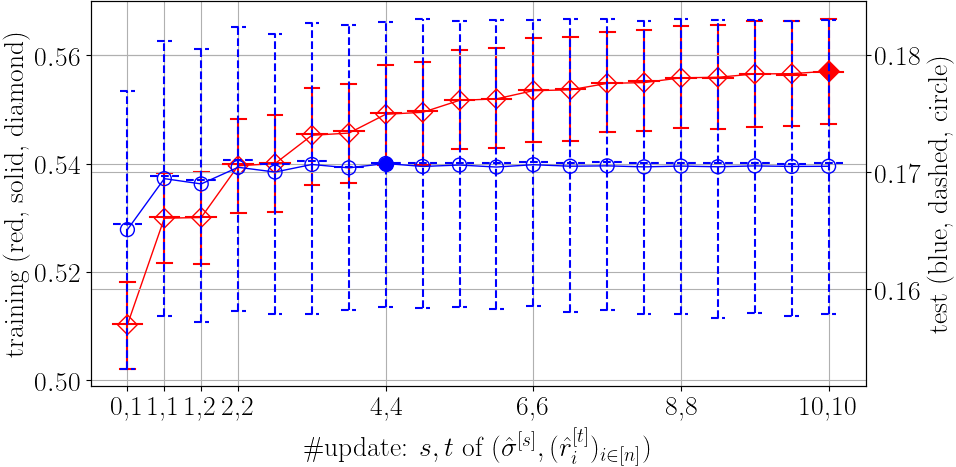}}&
\CF{\includegraphics[width=2.0cm]{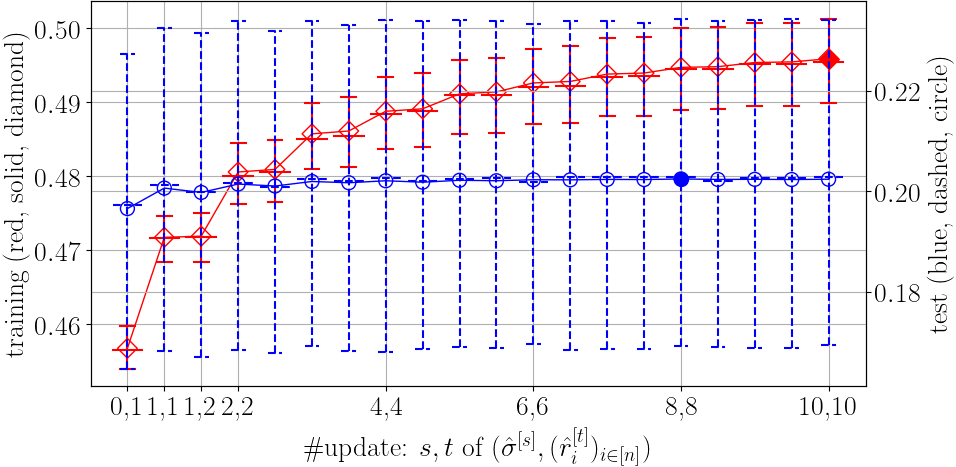}}
\\\midrule
\raisebox{1.75ex}{\rotatebox{90}{\tiny\eqref{eq:TIE}}}&
{\includegraphics[width=2.0cm]{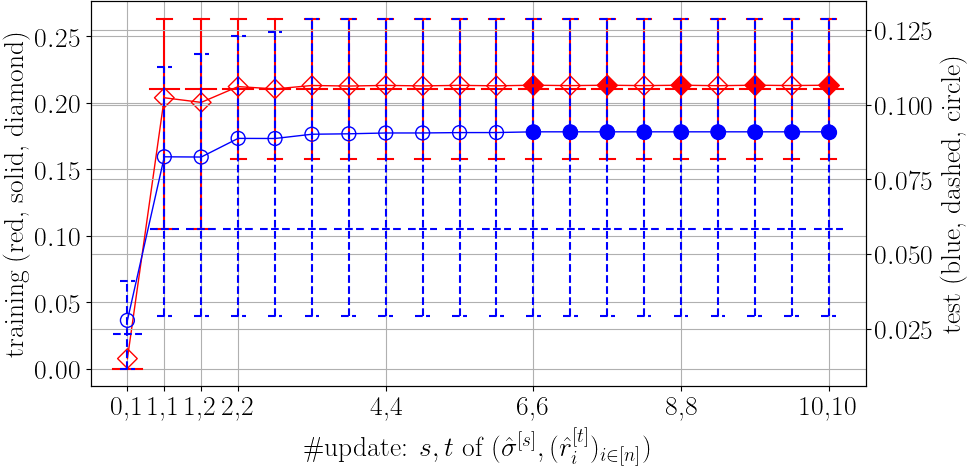}}&
{\includegraphics[width=2.0cm]{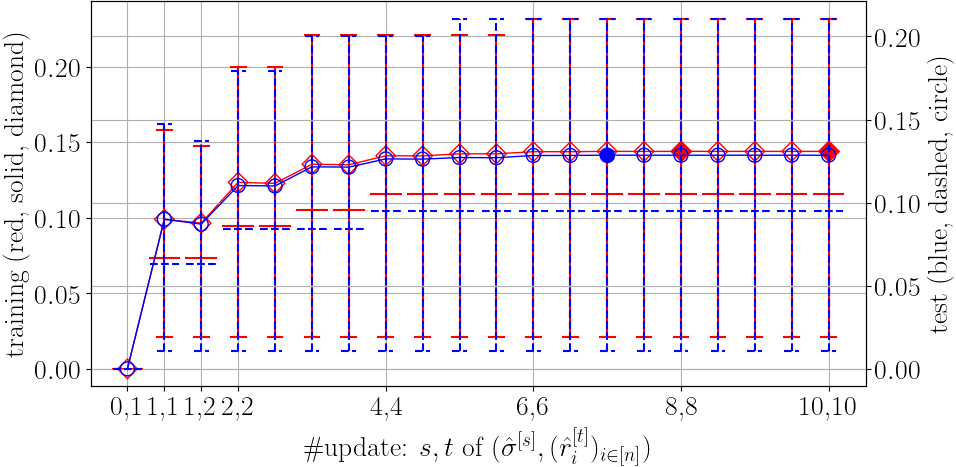}}&
{\includegraphics[width=2.0cm]{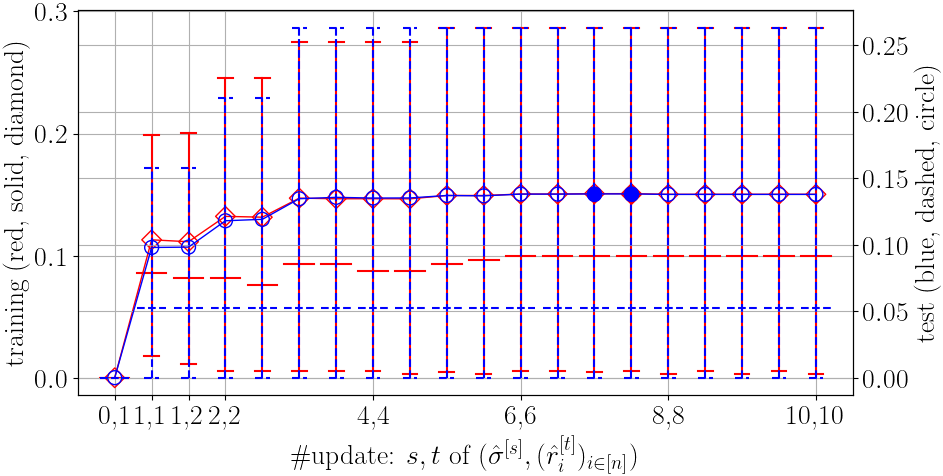}}&
{\includegraphics[width=2.0cm]{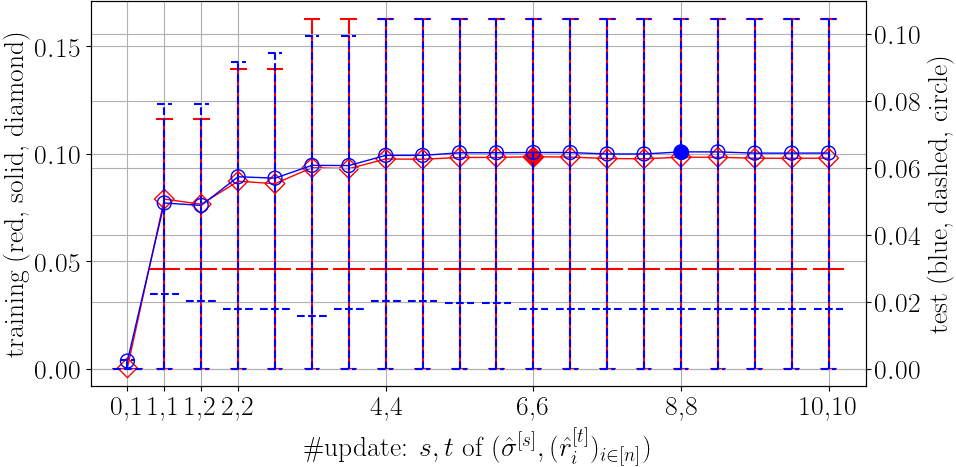}}&
{\includegraphics[width=2.0cm]{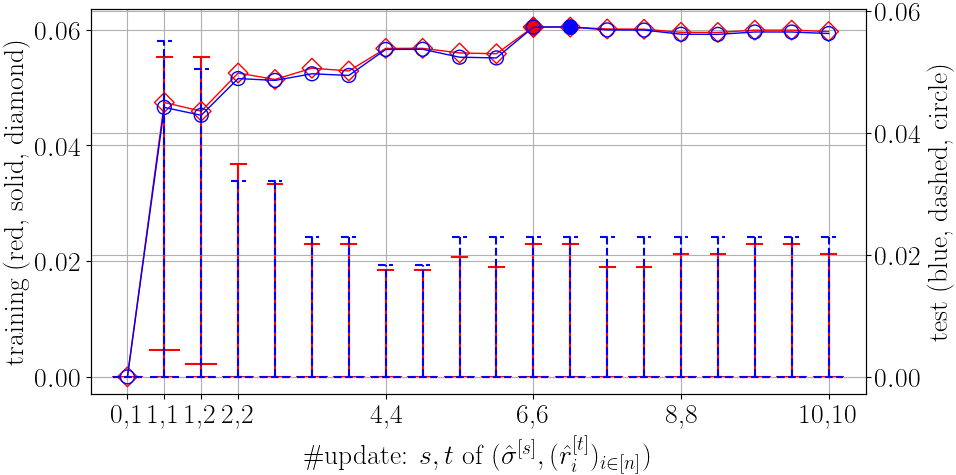}}&
{\includegraphics[width=2.0cm]{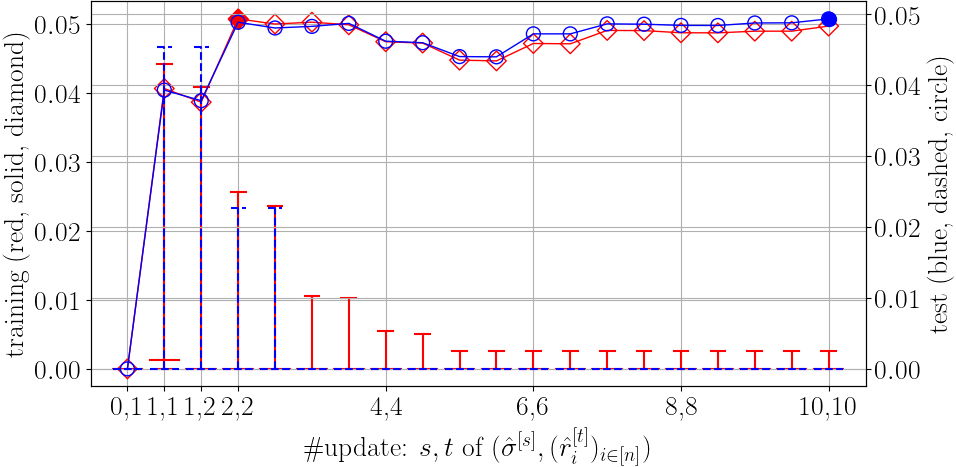}}&
{\includegraphics[width=2.0cm]{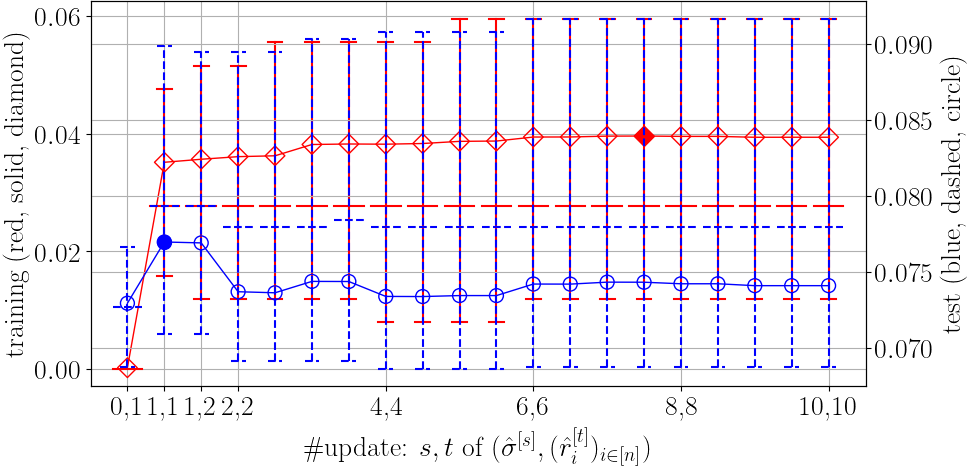}}&
{\includegraphics[width=2.0cm]{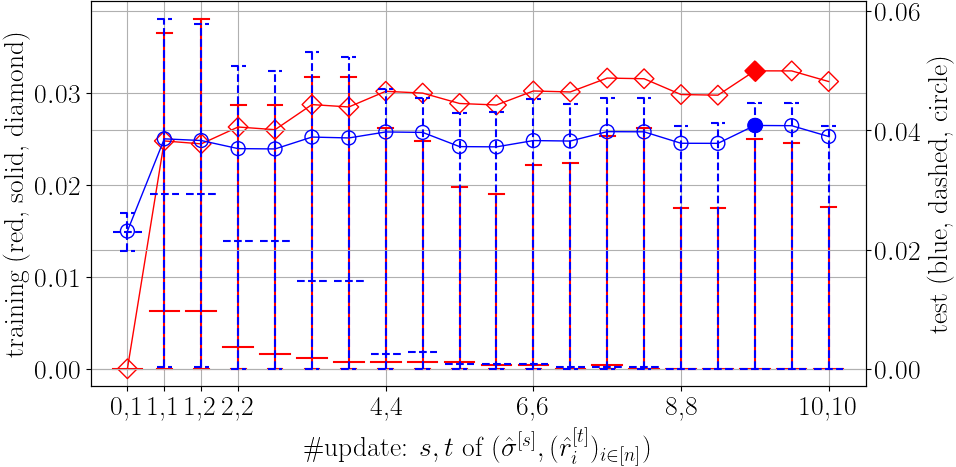}}&
{\includegraphics[width=2.0cm]{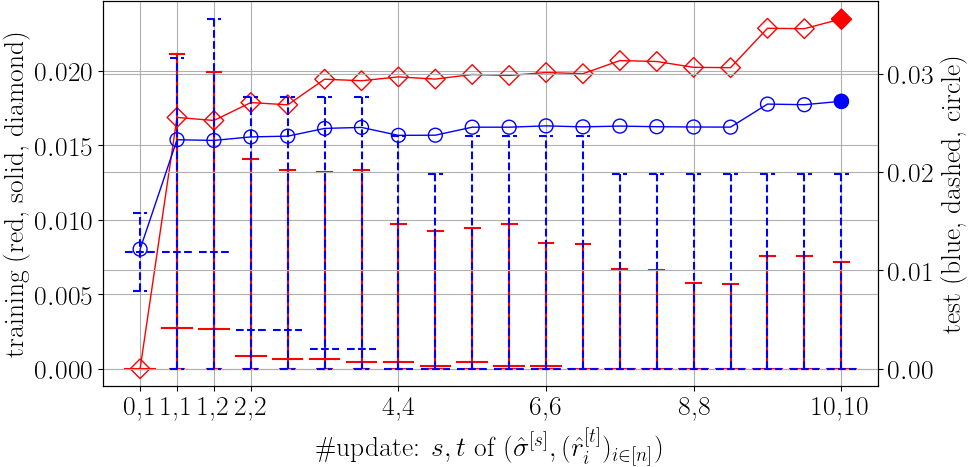}}
\end{tabular}
\caption{%
Part of results of real-world data experiments, Procedure 1 in Section~\ref{sec:RealWorld} and Appendix~\ref{sec:NLL}:
For PL, MLB, and ATP data (left to right),
mean (marker) and 0.25, 0.5, and 0.75 quantiles (lower, middle, and upper bars) of 
1000 trial training (red, solid, diamond) and test (blue, dashed, circle) evaluation of 
WPP error \eqref{eq:WPE} with the NLL loss $\phi=\phi_\nll$, 
Kendall's Tau \eqref{eq:Kendall}, and tie rate \eqref{eq:TIE} (top to bottom)
for the isotonic Bradley-Terry model learned with the NLL loss $\phi=\phi_\nll$.
Smaller \eqref{eq:WPE}, or larger \eqref{eq:Kendall} indicates a better model.
The marker for the best model and model with the most ties was filled in, 
and the outer frame of the figure was highlighted in gray
if the best model was significantly better than the Bradley-Terry model
with respect to Mann-Whitney U test of the significance level 0.05.}
\label{fig:Real-NLL}
\centering%
\renewcommand{\arraystretch}{0.5}%
\renewcommand{\tabcolsep}{0.5pt}%
\begin{tabular}{ccc|ccc|ccc}%
\multicolumn{9}{c}{~}\\
\multicolumn{9}{c}{~}\\
\multicolumn{3}{c|}{\tiny PL, $|D_\tra|:|D_\tes|=$}&\multicolumn{3}{c|}{\tiny MLB, $|D_\tra|:|D_\tes|=$}&\multicolumn{3}{c}{\tiny ATP, $|D_\tra|:|D_\tes|=$}\\
{\tiny$1:9$}&{\tiny$5:5$}&{\tiny$9:1$}&{\tiny$1:9$}&{\tiny$5:5$}&{\tiny$9:1$}&{\tiny$1:9$}&{\tiny$5:5$}&{\tiny$9:1$}\\
\midrule
{\includegraphics[width=2.0cm]{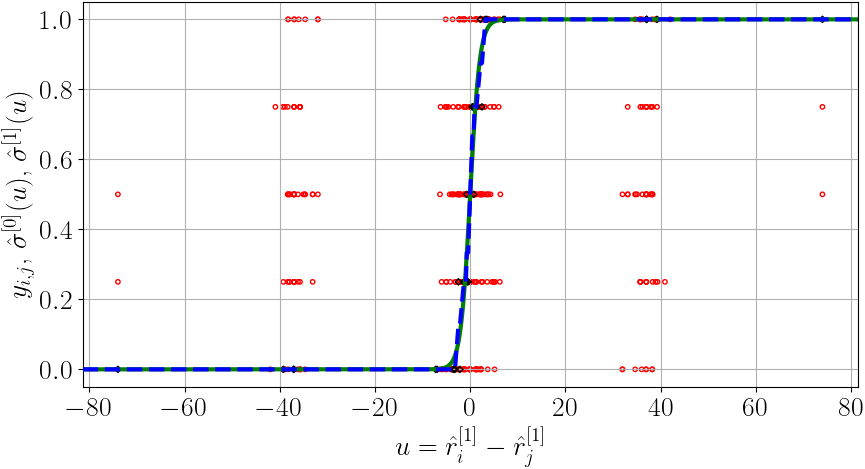}}&
{\includegraphics[width=2.0cm]{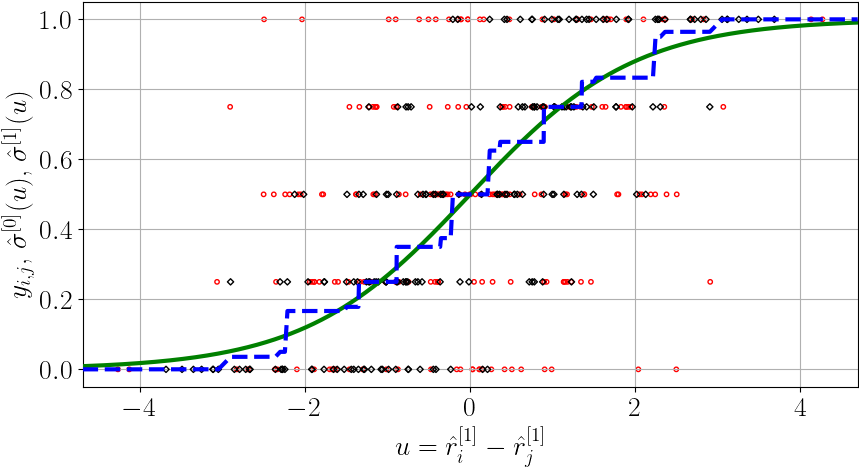}}&
{\includegraphics[width=2.0cm]{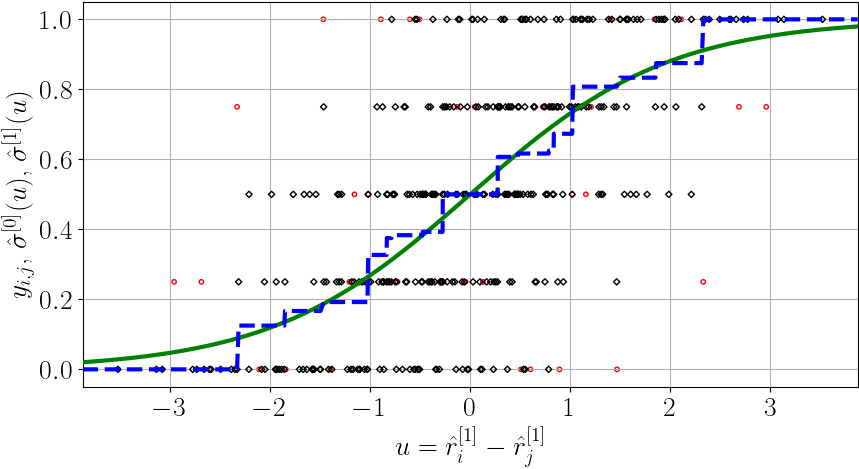}}&
{\includegraphics[width=2.0cm]{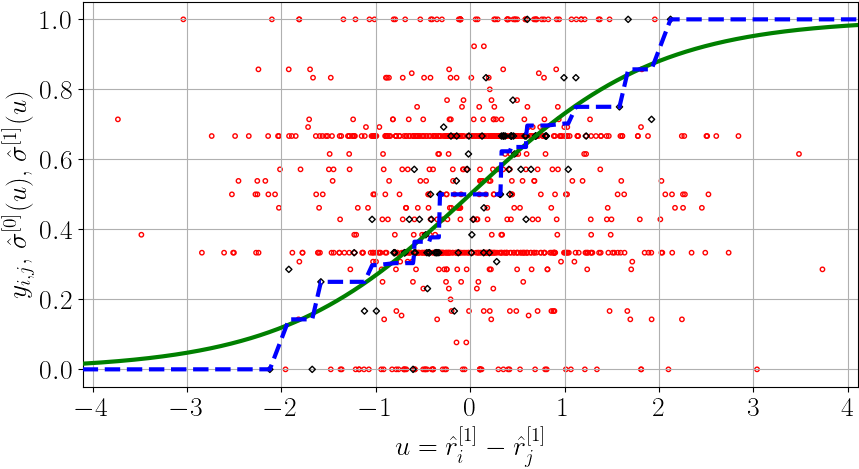}}&
{\includegraphics[width=2.0cm]{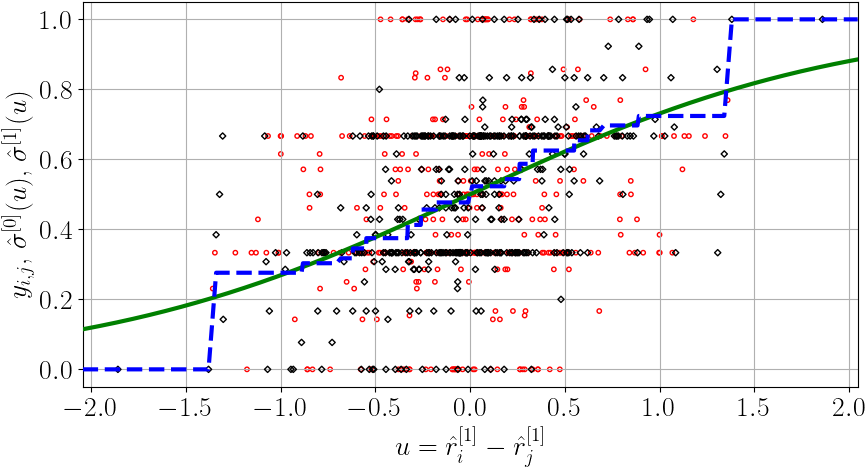}}&
{\includegraphics[width=2.0cm]{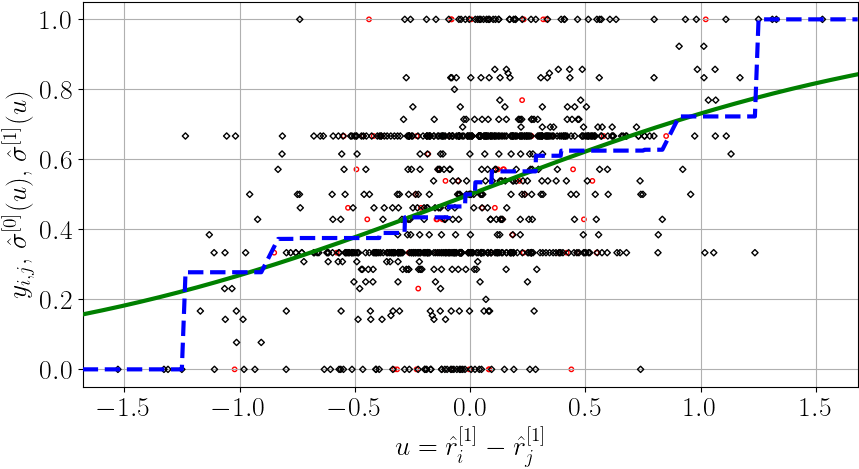}}&
{\includegraphics[width=2.0cm]{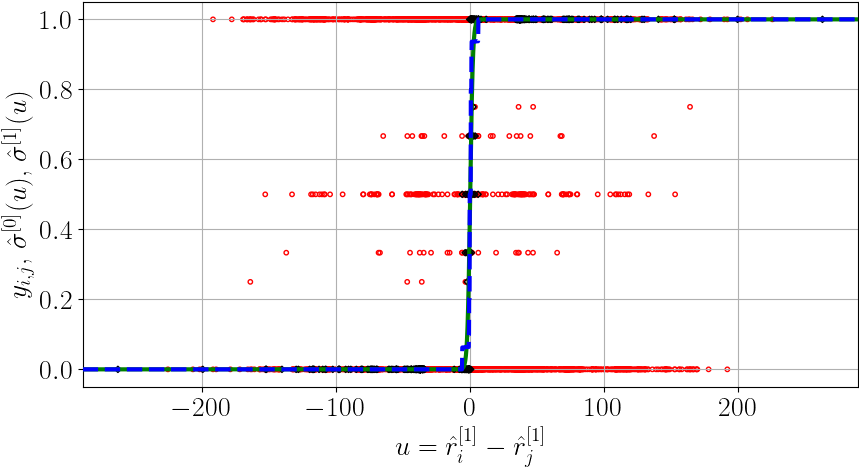}}&
{\includegraphics[width=2.0cm]{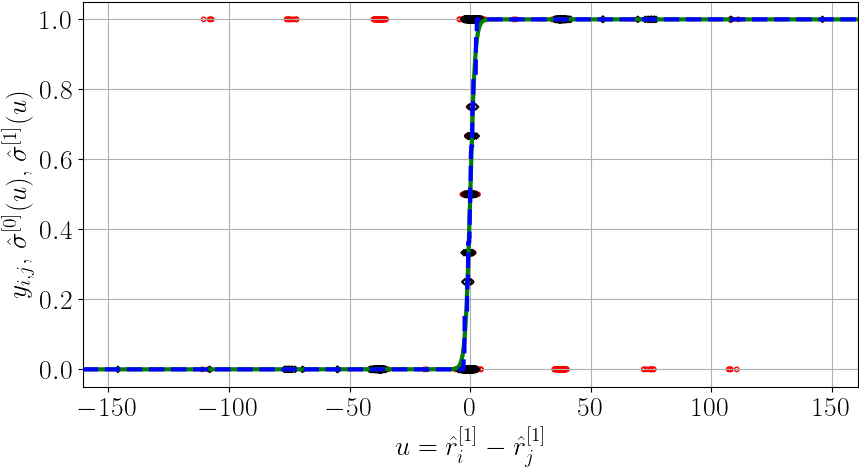}}&
{\includegraphics[width=2.0cm]{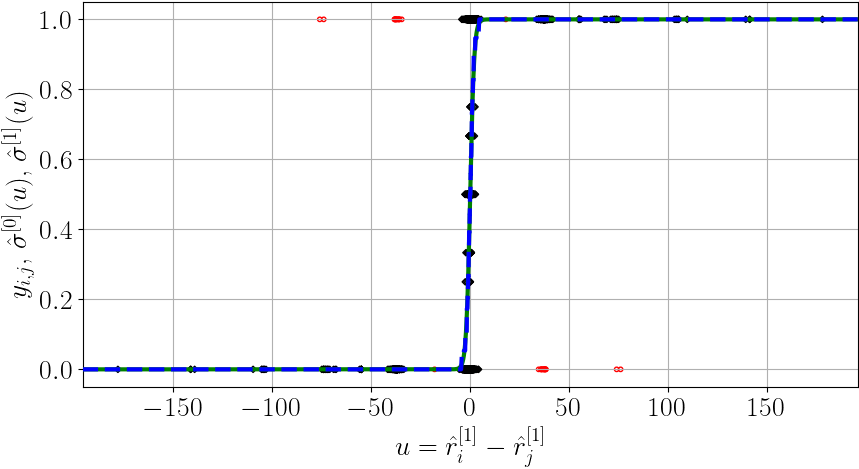}}
\end{tabular}
\begin{tabular}{cccc}%
\multicolumn{4}{c}{\tiny Rescaled version: Data, $|D_\tra|:|D_\tes|=$}\\
{\tiny PL, 1:9}&{\tiny ATP, 1:9}&{\tiny ATP, 5:5}&{\tiny ATP, 9:1}\\
\midrule
{\includegraphics[width=2.0cm]{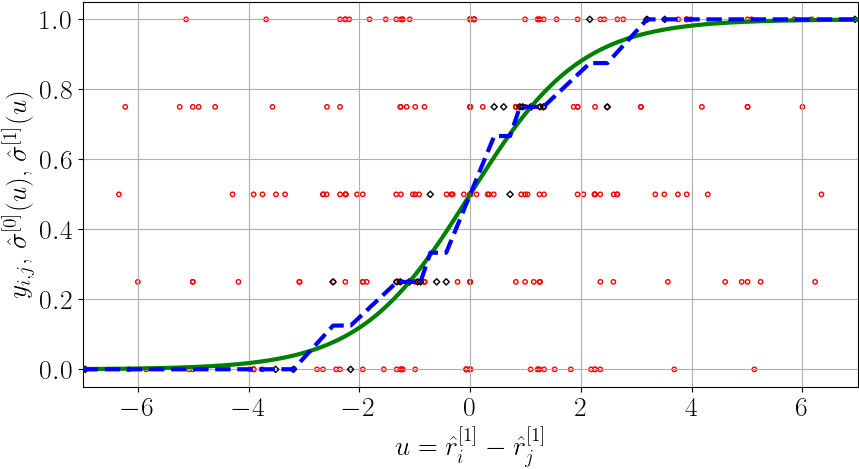}}&
{\includegraphics[width=2.0cm]{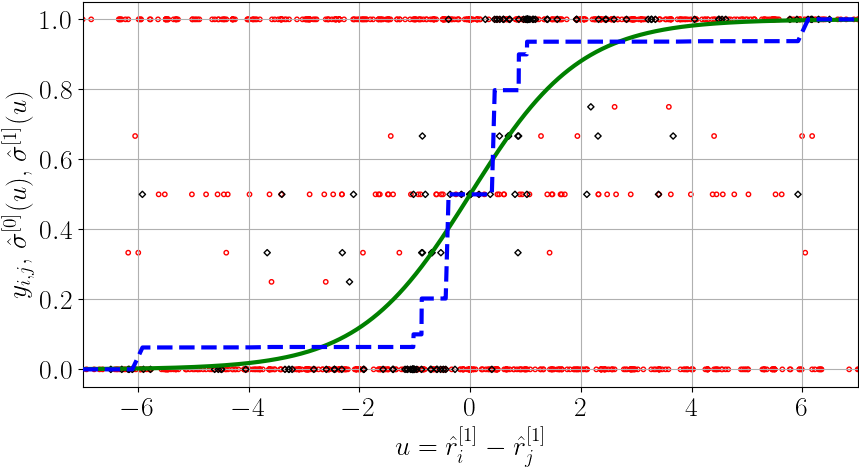}}&
{\includegraphics[width=2.0cm]{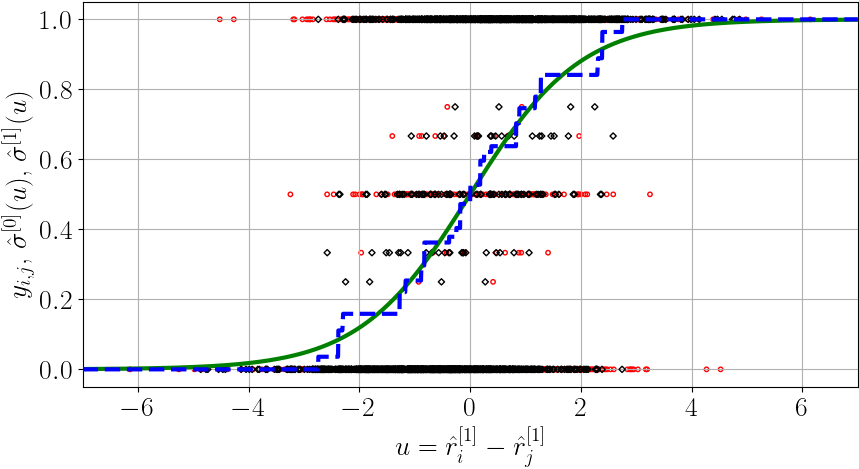}}&
{\includegraphics[width=2.0cm]{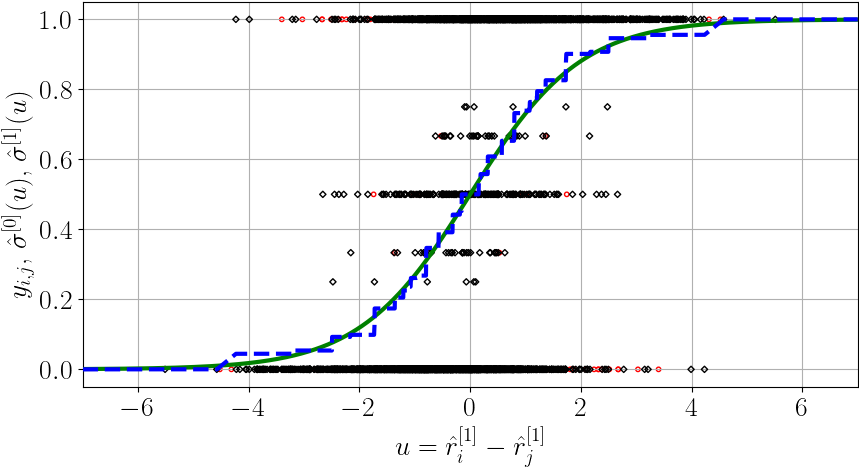}}
\end{tabular}
\caption{%
Part of results of real-world data experiments, Procedure 1 in Section~\ref{sec:RealWorld} and Appendix~\ref{sec:NLL}:
For PL, MLB, and ATP data (left to right),
black diamonds and red circles are training and test data $y_{i,j}$, 
and a green curve and a blue polyline are
the Bradley-Terry model $\hat{\sigma}^{[0]}(\hat{r}_i^{[1]}-\hat{r}_j^{[1]})$ and
isotonic Bradley-Terry model $\hat{\sigma}^{[1]}(\hat{r}_i^{[1]}-\hat{r}_j^{[1]})$
learned with the NLL loss $\phi=\phi_\nll$ in a certain trial,
over the range $[-1.1\cdot\max_{i,j}|\hat{r}_i^{[1]}-\hat{r}_j^{[1]}|,1.1\cdot\max_{i,j}|\hat{r}_i^{[1]}-\hat{r}_j^{[1]}|]$
or $[-7,7]$ in the rescaled version.}
\label{fig:Real-Res-Cauchy-NLL}
\end{sidewaysfigure}
\begin{sidewaystable}
\centering%
\renewcommand{\arraystretch}{0.5}%
\renewcommand{\tabcolsep}{0.5pt}%
\caption{%
Results of real-world data experiments, Procedure 2 in Section~\ref{sec:RealWorld} and Appendix~\ref{sec:NLL}:
For PL, MLB, and ATP data (left to right),
mean and std ($\text{mean}_{\text{std}}$) of 1000 trial test evaluation of 
WPP error \eqref{eq:WPE} with the NLL loss $\phi=\phi_\nll$
and Kendall's Tau \eqref{eq:Kendall}
for the Bradley-Terry model (upper) and model-selected isotonic 
Bradley-Terry model learned with the NLL loss $\phi=\phi_\nll$ (lower).
Smaller \eqref{eq:WPE}, or larger \eqref{eq:Kendall} indicates a better model.
It also includes $p$-value for the Mann-Whitney U test in the parentheses:
A value below the significance level 0.05 implies 
that the isotonic Bradley-Terry model performs significantly better
and is highlighted in red bold.
-- implies that errors include NaN.}
\label{tab:Real-NLL}
\scalebox{0.9}{\begin{minipage}{25cm}
\begin{tabular}{c|ccccc|ccccc|ccccc}%
&\multicolumn{5}{c|}{\tiny PL, $|D_\tra|:|D_\tes|=$}&\multicolumn{5}{c|}{\tiny MLB, $|D_\tra|:|D_\tes|=$}&\multicolumn{5}{c}{\tiny ATP, $|D_\tra|:|D_\tes|=$}\\
&{\tiny$1:9$}&{\tiny$3:7$}&{\tiny$5:5$}&{\tiny$7:3$}&{\tiny$9:1$}&{\tiny$1:9$}&{\tiny$3:7$}&{\tiny$5:5$}&{\tiny$7:3$}&{\tiny$9:1$}&{\tiny$1:9$}&{\tiny$3:7$}&{\tiny$5:5$}&{\tiny$7:3$}&{\tiny$9:1$}
\\\midrule
\multirow{3}{*}[-1.375mm]{\rotatebox{90}{\tiny\eqref{eq:WPE}}}
&$2.5804_{1.2334}$&$.5358_{.3756}$&$.3328_{.0542}$&$.3189_{.0318}$&$.3112_{.0334}$
&$.9516_{.5431}$&$.3638_{.0609}$&$.3502_{.0046}$&$.3469_{.0049}$&$.3453_{.0079}$
&$9.0269_{2.2872}$&$2.2744_{.3300}$&$1.6829_{.2006}$&$1.4985_{.1946}$&$1.4047_{.3220}$\\
&--&--&--&--&--
&--&--&--&--&--
&$9.0269_{2.2872}$&$2.2744_{.3300}$&$1.6829_{.2006}$&$1.4985_{.1946}$&$1.4047_{.3220}$\\
&($.9870$)&($.6934$)&($.7872$)&($.7159$)&($.7064$)
&($.7125$)&($.7757$)&($.9274$)&($.9519$)&($.9238$)
&($.5000$)&($.5000$)&($.5000$)&($.5000$)&($.5000$)\\
\midrule
\multirow{3}{*}[-1.375mm]{\rotatebox{90}{\tiny\eqref{eq:Kendall}}}
&$.1736_{.0932}$&\tcr{$\bf.3206_{.0575}$}&\tcr{$\bf.3711_{.0488}$}&\tcr{$\bf.3934_{.0693}$}&\tcr{$\bf.4154_{.1429}$}
&$.0533_{.0483}$&\tcr{$\bf.1031_{.0341}$}&\tcr{$\bf.1239_{.0354}$}&\tcr{$\bf.1375_{.0508}$}&\tcr{$\bf.1454_{.0989}$}
&$.0678_{.0239}$&\tcr{$\bf.1364_{.0173}$}&\tcr{$\bf.1651_{.0175}$}&\tcr{$\bf.1818_{.0234}$}&\tcr{$\bf.1966_{.0463}$}\\
&$.1779_{.0969}$&\tcr{$\bf.3299_{.0604}$}&\tcr{$\bf.3833_{.0513}$}&\tcr{$\bf.4051_{.0740}$}&\tcr{$\bf.4279_{.1460}$}
&$.0545_{.0497}$&\tcr{$\bf.1065_{.0363}$}&\tcr{$\bf.1297_{.0398}$}&\tcr{$\bf.1454_{.0560}$}&\tcr{$\bf.1551_{.1072}$}
&$.0693_{.0244}$&\tcr{$\bf.1414_{.0185}$}&\tcr{$\bf.1704_{.0186}$}&\tcr{$\bf.1870_{.0249}$}&\tcr{$\bf.2020_{.0485}$}\\
&($.1219$)&(\tcr{$\bf.0001$})&(\tcr{$\bf.0000$})&(\tcr{$\bf.0001$})&(\tcr{$\bf.0232$})
&($.2910$)&(\tcr{$\bf.0151$})&(\tcr{$\bf.0003$})&(\tcr{$\bf.0004$})&(\tcr{$\bf.0127$})
&($.0814$)&(\tcr{$\bf.0000$})&(\tcr{$\bf.0000$})&(\tcr{$\bf.0000$})&(\tcr{$\bf.0064$})\\
\end{tabular}\end{minipage}}
\end{sidewaystable}

\end{appendices}



\begin{thebibliography}{48}
\ifx \bisbn   \undefined \def \bisbn  #1{ISBN #1}\fi
\ifx \binits  \undefined \def \binits#1{#1}\fi
\ifx \bauthor  \undefined \def \bauthor#1{#1}\fi
\ifx \batitle  \undefined \def \batitle#1{#1}\fi
\ifx \bjtitle  \undefined \def \bjtitle#1{#1}\fi
\ifx \bvolume  \undefined \def \bvolume#1{\textbf{#1}}\fi
\ifx \byear  \undefined \def \byear#1{#1}\fi
\ifx \bissue  \undefined \def \bissue#1{#1}\fi
\ifx \bfpage  \undefined \def \bfpage#1{#1}\fi
\ifx \blpage  \undefined \def \blpage #1{#1}\fi
\ifx \burl  \undefined \def \burl#1{\textsf{#1}}\fi
\ifx \doiurl  \undefined \def \doiurl#1{\url{https://doi.org/#1}}\fi
\ifx \betal  \undefined \def \betal{\textit{et al.}}\fi
\ifx \binstitute  \undefined \def \binstitute#1{#1}\fi
\ifx \binstitutionaled  \undefined \def \binstitutionaled#1{#1}\fi
\ifx \bctitle  \undefined \def \bctitle#1{#1}\fi
\ifx \beditor  \undefined \def \beditor#1{#1}\fi
\ifx \bpublisher  \undefined \def \bpublisher#1{#1}\fi
\ifx \bbtitle  \undefined \def \bbtitle#1{#1}\fi
\ifx \bedition  \undefined \def \bedition#1{#1}\fi
\ifx \bseriesno  \undefined \def \bseriesno#1{#1}\fi
\ifx \blocation  \undefined \def \blocation#1{#1}\fi
\ifx \bsertitle  \undefined \def \bsertitle#1{#1}\fi
\ifx \bsnm \undefined \def \bsnm#1{#1}\fi
\ifx \bsuffix \undefined \def \bsuffix#1{#1}\fi
\ifx \bparticle \undefined \def \bparticle#1{#1}\fi
\ifx \barticle \undefined \def \barticle#1{#1}\fi
\bibcommenthead
\ifx \bconfdate \undefined \def \bconfdate #1{#1}\fi
\ifx \botherref \undefined \def \botherref #1{#1}\fi
\ifx \url \undefined \def \url#1{\textsf{#1}}\fi
\ifx \bchapter \undefined \def \bchapter#1{#1}\fi
\ifx \bbook \undefined \def \bbook#1{#1}\fi
\ifx \bcomment \undefined \def \bcomment#1{#1}\fi
\ifx \oauthor \undefined \def \oauthor#1{#1}\fi
\ifx \citeauthoryear \undefined \def \citeauthoryear#1{#1}\fi
\ifx \endbibitem  \undefined \def \endbibitem {}\fi
\ifx \bconflocation  \undefined \def \bconflocation#1{#1}\fi
\ifx \arxivurl  \undefined \def \arxivurl#1{\textsf{#1}}\fi
\csname PreBibitemsHook\endcsname

\bibitem[\protect\citeauthoryear{Bradley}{1976}]{bradley1976science}
\begin{barticle}
\bauthor{\bsnm{Bradley}, \binits{R.A.}}:
\batitle{{Science, statistics, and paired comparisons}}.
\bjtitle{Biometrics}
\bvolume{32}(\bissue{2}),
\bfpage{213}--\blpage{239}
(\byear{1976})
\doiurl{10.2307/2529494}
\end{barticle}
\endbibitem

\bibitem[\protect\citeauthoryear{Cattelan}{2012}]{cattelan2012models}
\begin{barticle}
\bauthor{\bsnm{Cattelan}, \binits{M.}}:
\batitle{{Models for Paired Comparison Data: A Review with Emphasis on
  Dependent Data}}.
\bjtitle{Statistical Science}
\bvolume{27}(\bissue{3}),
\bfpage{412}--\blpage{433}
(\byear{2012})
\doiurl{10.1214/12-sts396}
\end{barticle}
\endbibitem

\bibitem[\protect\citeauthoryear{Kissler and
  B{\"a}uml}{2000}]{kissler2000effects}
\begin{barticle}
\bauthor{\bsnm{Kissler}, \binits{J.}},
\bauthor{\bsnm{B{\"a}uml}, \binits{K.-H.}}:
\batitle{{Effects of the beholder's age on the perception of facial
  attractiveness}}.
\bjtitle{Acta Psychologica}
\bvolume{104}(\bissue{2}),
\bfpage{145}--\blpage{166}
(\byear{2000})
\doiurl{10.1016/s0001-6918(00)00018-4}
\end{barticle}
\endbibitem

\bibitem[\protect\citeauthoryear{Maydeu-Olivares and
  B{\"o}ckenholt}{2008}]{maydeu2008modeling}
\begin{barticle}
\bauthor{\bsnm{Maydeu-Olivares}, \binits{A.}},
\bauthor{\bsnm{B{\"o}ckenholt}, \binits{U.}}:
\batitle{{Modeling subjective health outcomes: Top 10 reasons to use
  Thurstone's method}}.
\bjtitle{Medical Care}
\bvolume{46}(\bissue{4}),
\bfpage{346}--\blpage{348}
(\byear{2008})
\doiurl{10.1097/mlr.0b013e31816dd8d9}
\end{barticle}
\endbibitem

\bibitem[\protect\citeauthoryear{Mazzuchi et~al.}{2008}]{mazzuchi2008paired}
\begin{barticle}
\bauthor{\bsnm{Mazzuchi}, \binits{T.A.}},
\bauthor{\bsnm{Linzey}, \binits{W.G.}},
\bauthor{\bsnm{Bruning}, \binits{A.}}:
\batitle{{A paired comparison experiment for gathering expert judgment for an
  aircraft wiring risk assessment}}.
\bjtitle{Reliability Engineering \& System Safety}
\bvolume{93}(\bissue{5}),
\bfpage{722}--\blpage{731}
(\byear{2008})
\doiurl{10.1016/j.ress.2007.03.011}
\end{barticle}
\endbibitem

\bibitem[\protect\citeauthoryear{Rafailov et~al.}{2023}]{rafailov2023direct}
\begin{bchapter}
\bauthor{\bsnm{Rafailov}, \binits{R.}},
\bauthor{\bsnm{Sharma}, \binits{A.}},
\bauthor{\bsnm{Mitchell}, \binits{E.}},
\bauthor{\bsnm{Manning}, \binits{C.D.}},
\bauthor{\bsnm{Ermon}, \binits{S.}},
\bauthor{\bsnm{Finn}, \binits{C.}}:
\bctitle{Direct preference optimization: Your language model is secretly a
  reward model}.
In: \bbtitle{Advances in Neural Information Processing Systems},
vol. \bseriesno{36},
pp. \bfpage{53728}--\blpage{53741}
(\byear{2023}).
\doiurl{10.52202/075280-2338}
\end{bchapter}
\endbibitem

\bibitem[\protect\citeauthoryear{Chiang et~al.}{2024}]{Chatbot}
\begin{bchapter}
\bauthor{\bsnm{Chiang}, \binits{W.-L.}},
\bauthor{\bsnm{Zheng}, \binits{L.}},
\bauthor{\bsnm{Sheng}, \binits{Y.}},
\bauthor{\bsnm{Angelopoulos}, \binits{A.N.}},
\bauthor{\bsnm{Li}, \binits{T.}},
\bauthor{\bsnm{Li}, \binits{D.}},
\bauthor{\bsnm{Zhu}, \binits{B.}},
\bauthor{\bsnm{Zhang}, \binits{H.}},
\bauthor{\bsnm{Jordan}, \binits{M.I.}},
\bauthor{\bsnm{Gonzalez}, \binits{J.E.}},
\bauthor{\bsnm{Stoica}, \binits{I.}}:
\bctitle{Chatbot arena: An open platform for evaluating llms by human
  preference}.
In: \bbtitle{Proceedings of the 41st International Conference on Machine
  Learning}
(\byear{2024})
\end{bchapter}
\endbibitem

\bibitem[\protect\citeauthoryear{Zermelo}{1929}]{Zermelo1929}
\begin{barticle}
\bauthor{\bsnm{Zermelo}, \binits{E.}}:
\batitle{{Die Berechnung der Turnier-Ergebnisse als ein Maximumproblem der
  Wahrscheinlichkeitsrechnung}}.
\bjtitle{Mathematische Zeitschrift}
\bvolume{29},
\bfpage{436}--\blpage{460}
(\byear{1929})
\doiurl{10.1007/bf01180541}
\end{barticle}
\endbibitem

\bibitem[\protect\citeauthoryear{Ford~Jr.}{1957}]{ford1957solution}
\begin{barticle}
\bauthor{\bsnm{Ford~Jr.}, \binits{L.R.}}:
\batitle{{Solution of a ranking problem from binary comparisons}}.
\bjtitle{The American Mathematical Monthly}
\bvolume{64}(\bissue{8, Part 2}),
\bfpage{28}--\blpage{33}
(\byear{1957})
\doiurl{10.2307/2308513}
\end{barticle}
\endbibitem

\bibitem[\protect\citeauthoryear{Bradley and Terry}{1952}]{bradley1952rank}
\begin{barticle}
\bauthor{\bsnm{Bradley}, \binits{R.A.}},
\bauthor{\bsnm{Terry}, \binits{M.E.}}:
\batitle{{Rank analysis of incomplete block designs: I. {The} method of paired
  comparisons}}.
\bjtitle{Biometrika}
\bvolume{39}(\bissue{3/4}),
\bfpage{324}--\blpage{345}
(\byear{1952})
\doiurl{10.2307/2527550}
\end{barticle}
\endbibitem

\bibitem[\protect\citeauthoryear{Luce}{1959}]{luce1959individual}
\begin{bbook}
\bauthor{\bsnm{Luce}, \binits{R.D.}}:
\bbtitle{{Individual Choice Behavior}}
vol. \bseriesno{4}.
\bpublisher{John Wiley \& Sons},
\blocation{New York}
(\byear{1959})
\end{bbook}
\endbibitem

\bibitem[\protect\citeauthoryear{Thurstone}{1927}]{thurstone1927law}
\begin{barticle}
\bauthor{\bsnm{Thurstone}, \binits{L.L.}}:
\batitle{{A law of comparative judgment}}.
\bjtitle{Psychological Review}
\bvolume{34}(\bissue{4}),
\bfpage{273}--\blpage{286}
(\byear{1927})
\doiurl{10.1037/h0070288}
\end{barticle}
\endbibitem

\bibitem[\protect\citeauthoryear{Mosteller}{1951}]{mosteller1951remarks1}
\begin{barticle}
\bauthor{\bsnm{Mosteller}, \binits{F.}}:
\batitle{{Remarks on the method of paired comparisons: I. {The} least squares
  solution assuming equal standard deviations and equal correlations}}.
\bjtitle{Psychometrika}
\bvolume{16}(\bissue{1}),
\bfpage{3}--\blpage{9}
(\byear{1951})
\doiurl{10.1007/978-0-387-44956-2_8}
\end{barticle}
\endbibitem

\bibitem[\protect\citeauthoryear{Stern}{1990}]{stern1990continuum}
\begin{barticle}
\bauthor{\bsnm{Stern}, \binits{H.}}:
\batitle{{A continuum of paired comparisons models}}.
\bjtitle{Biometrika}
\bvolume{77}(\bissue{2}),
\bfpage{265}--\blpage{273}
(\byear{1990})
\doiurl{10.2307/2336804}
\end{barticle}
\endbibitem

\bibitem[\protect\citeauthoryear{Ayer et~al.}{1955}]{ayer1955empirical}
\begin{barticle}
\bauthor{\bsnm{Ayer}, \binits{M.}},
\bauthor{\bsnm{Brunk}, \binits{H.D.}},
\bauthor{\bsnm{Ewing}, \binits{G.M.}},
\bauthor{\bsnm{Reid}, \binits{W.T.}},
\bauthor{\bsnm{Silverman}, \binits{E.}}:
\batitle{An empirical distribution function for sampling with incomplete
  information}.
\bjtitle{The Annals of Mathematical Statistics}
\bvolume{26}(\bissue{4}),
\bfpage{641}--\blpage{647}
(\byear{1955})
\doiurl{10.1214/aoms/1177728423}
\end{barticle}
\endbibitem

\bibitem[\protect\citeauthoryear{Brunk}{1955}]{brunk1955maximum}
\begin{barticle}
\bauthor{\bsnm{Brunk}, \binits{H.D.}}:
\batitle{Maximum likelihood estimates of monotone parameters}.
\bjtitle{The Annals of Mathematical Statistics}
\bvolume{26}(\bissue{4}),
\bfpage{607}--\blpage{616}
(\byear{1955})
\doiurl{10.1214/aoms/1177728420}
\end{barticle}
\endbibitem

\bibitem[\protect\citeauthoryear{Van~Eeden}{1956}]{van1956maximum}
\begin{barticle}
\bauthor{\bsnm{Van~Eeden}, \binits{C.}}:
\batitle{Maximum likelihood estimation of ordered probabilities}.
\bjtitle{Proceedings of the Koninklijke Nederlandse Akademie van Wetenschappen,
  Seriese A}
\bvolume{59}(\bissue{4}),
\bfpage{444}--\blpage{455}
(\byear{1956})
\doiurl{10.1016/s1385-7258(56)50060-1}
\end{barticle}
\endbibitem

\bibitem[\protect\citeauthoryear{Barlow et~al.}{1972}]{barlow1972statistical}
\begin{bbook}
\bauthor{\bsnm{Barlow}, \binits{R.E.}},
\bauthor{\bsnm{Bartholomew}, \binits{D.J.}},
\bauthor{\bsnm{Bremner}, \binits{J.M.}},
\bauthor{\bsnm{Brunk}, \binits{H.D.}}:
\bbtitle{Statistical Inference Under Order Restrictions: The Theory and
  Application of Isotonic Regression}.
\bpublisher{John Wiley \& Sons},
\blocation{London and New York}
(\byear{1972})
\end{bbook}
\endbibitem

\bibitem[\protect\citeauthoryear{Robertson et~al.}{1988}]{robertson1988order}
\begin{bbook}
\bauthor{\bsnm{Robertson}, \binits{T.}},
\bauthor{\bsnm{Wright}, \binits{F.T.}},
\bauthor{\bsnm{Dykstra}, \binits{R.L.}}:
\bbtitle{Order Restricted Statistical Inference}.
\bpublisher{John Wiley \& Sons},
\blocation{Chichester}
(\byear{1988})
\end{bbook}
\endbibitem

\bibitem[\protect\citeauthoryear{Chatterjee}{2015}]{chatterjee2015matrix}
\begin{barticle}
\bauthor{\bsnm{Chatterjee}, \binits{S.}}:
\batitle{{Matrix estimation by universal singular value thresholding}}.
\bjtitle{The Annals of Statistics}
\bvolume{43}(\bissue{1}),
\bfpage{177}--\blpage{214}
(\byear{2015})
\doiurl{10.1214/14-aos1272}
\end{barticle}
\endbibitem

\bibitem[\protect\citeauthoryear{Shah et~al.}{2016}]{shah2016stochastically}
\begin{bchapter}
\bauthor{\bsnm{Shah}, \binits{N.}},
\bauthor{\bsnm{Balakrishnan}, \binits{S.}},
\bauthor{\bsnm{Guntuboyina}, \binits{A.}},
\bauthor{\bsnm{Wainwright}, \binits{M.}}:
\bctitle{{Stochastically transitive models for pairwise comparisons:
  Statistical and computational issues}}.
In: \bbtitle{{Proceedings of the 33rd International Conference on Machine
  Learning}},
pp. \bfpage{11}--\blpage{20}
(\byear{2016}).
\doiurl{10.1109/tit.2016.2634418}
\end{bchapter}
\endbibitem

\bibitem[\protect\citeauthoryear{Chatterjee
  et~al.}{2018}]{chatterjee2018matrix}
\begin{barticle}
\bauthor{\bsnm{Chatterjee}, \binits{S.}},
\bauthor{\bsnm{Guntuboyina}, \binits{A.}},
\bauthor{\bsnm{Sen}, \binits{B.}}:
\batitle{{On matrix estimation under monotonicity constraints}}.
\bjtitle{Bernoulli}
\bvolume{24}(\bissue{2}),
\bfpage{1072}--\blpage{1100}
(\byear{2018})
\doiurl{10.3150/16-bej865}
\end{barticle}
\endbibitem

\bibitem[\protect\citeauthoryear{Elo and Sloan}{1978}]{elo1978rating}
\begin{bbook}
\bauthor{\bsnm{Elo}, \binits{A.E.}},
\bauthor{\bsnm{Sloan}, \binits{S.}}:
\bbtitle{The Rating of Chessplayers: Past and Present}.
\bpublisher{Arco Publishing, Inc.},
\blocation{New York}
(\byear{1978})
\end{bbook}
\endbibitem

\bibitem[\protect\citeauthoryear{Hunter}{2004}]{hunter2004mm}
\begin{barticle}
\bauthor{\bsnm{Hunter}, \binits{D.R.}}:
\batitle{{MM algorithms for generalized Bradley-Terry models}}.
\bjtitle{The Annals of Statistics}
\bvolume{32}(\bissue{1}),
\bfpage{384}--\blpage{406}
(\byear{2004})
\doiurl{10.1214/aos/1079120141}
\end{barticle}
\endbibitem

\bibitem[\protect\citeauthoryear{Nelder and
  Wedderburn}{1972}]{nelder1972generalized}
\begin{barticle}
\bauthor{\bsnm{Nelder}, \binits{J.A.}},
\bauthor{\bsnm{Wedderburn}, \binits{R.W.M.}}:
\batitle{{Generalized linear models}}.
\bjtitle{Journal of the Royal Statistical Society Series A: Statistics in
  Society}
\bvolume{135}(\bissue{3}),
\bfpage{370}--\blpage{384}
(\byear{1972})
\doiurl{10.2307/2344614}
\end{barticle}
\endbibitem

\bibitem[\protect\citeauthoryear{Bianco and Yohai}{1996}]{bianco1996robust}
\begin{bchapter}
\bauthor{\bsnm{Bianco}, \binits{A.M.}},
\bauthor{\bsnm{Yohai}, \binits{V.J.}}:
\bctitle{Robust estimation in the logistic regression model}.
In: \bbtitle{Robust Statistics, Data Analysis, and Computer Intensive Methods},
pp. \bfpage{17}--\blpage{34}
(\byear{1996}).
\doiurl{10.1007/978-1-4612-2380-1_2}
\end{bchapter}
\endbibitem

\bibitem[\protect\citeauthoryear{Yamasaki and Tanaka}{2023}]{yamasaki2023label}
\begin{botherref}
\oauthor{\bsnm{Yamasaki}, \binits{R.}},
\oauthor{\bsnm{Tanaka}, \binits{T.}}:
{Label Smoothing is Robustification against Model Misspecification}.
arXiv preprint arXiv:2305.08501
(2023)
\end{botherref}
\endbibitem

\bibitem[\protect\citeauthoryear{Kendall}{1938}]{kendall1938new}
\begin{barticle}
\bauthor{\bsnm{Kendall}, \binits{M.G.}}:
\batitle{{A new measure of rank correlation}}.
\bjtitle{Biometrika}
\bvolume{30}(\bissue{1-2}),
\bfpage{81}--\blpage{93}
(\byear{1938})
\doiurl{10.2307/2332226}
\end{barticle}
\endbibitem

\bibitem[\protect\citeauthoryear{Spearman}{1904}]{Spearman04}
\begin{barticle}
\bauthor{\bsnm{Spearman}, \binits{C.}}:
\batitle{{The Proof and Measurement of Association between Two Things}}.
\bjtitle{The American Journal of Psychology}
\bvolume{15}(\bissue{1}),
\bfpage{72}--\blpage{101}
(\byear{1904})
\doiurl{10.1037/11491-005}
\end{barticle}
\endbibitem

\bibitem[\protect\citeauthoryear{Noether}{1981}]{noether1981kendall}
\begin{barticle}
\bauthor{\bsnm{Noether}, \binits{G.E.}}:
\batitle{{Why kendall tau?}}
\bjtitle{Teaching Statistics}
\bvolume{3}(\bissue{2}),
\bfpage{41}--\blpage{43}
(\byear{1981})
\doiurl{10.1111/j.1467-9639.1981.tb00422.x}
\end{barticle}
\endbibitem

\bibitem[\protect\citeauthoryear{Brunk et~al.}{1957}]{brunk1957minimizing}
\begin{barticle}
\bauthor{\bsnm{Brunk}, \binits{H.D.}},
\bauthor{\bsnm{Ewing}, \binits{G.M.}},
\bauthor{\bsnm{Utz}, \binits{W.R.}}:
\batitle{Minimizing integrals in certain classes of monotone functions}.
\bjtitle{Pacific Journal of Mathematics}
\bvolume{7}(\bissue{1}),
\bfpage{833}--\blpage{847}
(\byear{1957})
\doiurl{10.2140/pjm.1957.7.833}
\end{barticle}
\endbibitem

\bibitem[\protect\citeauthoryear{Thompson~Jr.}{1962}]{thompson1962problem}
\begin{barticle}
\bauthor{\bsnm{Thompson~Jr.}, \binits{W.A.}}:
\batitle{The problem of negative estimates of variance components}.
\bjtitle{The Annals of Mathematical Statistics}
\bvolume{33}(\bissue{1}),
\bfpage{273}--\blpage{289}
(\byear{1962})
\doiurl{10.1214/aoms/1177704731}
\end{barticle}
\endbibitem

\bibitem[\protect\citeauthoryear{Lee}{1983}]{lee1983min}
\begin{barticle}
\bauthor{\bsnm{Lee}, \binits{C.-I.C.}}:
\batitle{The min-max algorithm and isotonic regression}.
\bjtitle{The Annals of Statistics}
\bvolume{11}(\bissue{2}),
\bfpage{467}--\blpage{477}
(\byear{1983})
\doiurl{10.1214/aos/1176346153}
\end{barticle}
\endbibitem

\bibitem[\protect\citeauthoryear{Best and Chakravarti}{1990}]{best1990active}
\begin{barticle}
\bauthor{\bsnm{Best}, \binits{M.J.}},
\bauthor{\bsnm{Chakravarti}, \binits{N.}}:
\batitle{Active set algorithms for isotonic regression; a unifying framework}.
\bjtitle{Mathematical Programming}
\bvolume{47}(\bissue{1}),
\bfpage{425}--\blpage{439}
(\byear{1990})
\doiurl{10.1007/bf01580873}
\end{barticle}
\endbibitem

\bibitem[\protect\citeauthoryear{de~Leeuw et~al.}{2010}]{de2010isotone}
\begin{barticle}
\bauthor{\bsnm{Leeuw}, \binits{J.}},
\bauthor{\bsnm{Hornik}, \binits{K.}},
\bauthor{\bsnm{Mair}, \binits{P.}}:
\batitle{Isotone optimization in {R}: Pool-adjacent-violators algorithm
  ({PAVA}) and active set methods}.
\bjtitle{Journal of Statistical Software}
\bvolume{32}(\bissue{5}),
\bfpage{1}--\blpage{24}
(\byear{2010})
\doiurl{10.18637/jss.v032.i05}
\end{barticle}
\endbibitem

\bibitem[\protect\citeauthoryear{Shah and Wainwright}{2018}]{Nihar18}
\begin{barticle}
\bauthor{\bsnm{Shah}, \binits{N.B.}},
\bauthor{\bsnm{Wainwright}, \binits{M.J.}}:
\batitle{Simple, robust and optimal ranking from pairwise comparisons}.
\bjtitle{Journal of Machine Learning Research}
\bvolume{18}(\bissue{199}),
\bfpage{1}--\blpage{38}
(\byear{2018})
\end{barticle}
\endbibitem

\bibitem[\protect\citeauthoryear{Han et~al.}{2020}]{han2020asymptotic}
\begin{barticle}
\bauthor{\bsnm{Han}, \binits{R.}},
\bauthor{\bsnm{Ye}, \binits{R.}},
\bauthor{\bsnm{Tan}, \binits{C.}},
\bauthor{\bsnm{Chen}, \binits{K.}}:
\batitle{{Asymptotic theory of sparse Bradley--Terry model}}.
\bjtitle{Annals of Applied Probability}
\bvolume{30}(\bissue{5}),
\bfpage{2491}--\blpage{2515}
(\byear{2020})
\doiurl{10.1214/20-aap1564}
\end{barticle}
\endbibitem

\bibitem[\protect\citeauthoryear{Simons and Yao}{1999}]{simons1999asymptotics}
\begin{barticle}
\bauthor{\bsnm{Simons}, \binits{G.}},
\bauthor{\bsnm{Yao}, \binits{Y.-C.}}:
\batitle{{Asymptotics when the number of parameters tends to infinity in the
  Bradley-Terry model for paired comparisons}}.
\bjtitle{The Annals of Statistics}
\bvolume{27}(\bissue{3}),
\bfpage{1041}--\blpage{1060}
(\byear{1999})
\doiurl{10.1214/aos/1018031267}
\end{barticle}
\endbibitem

\bibitem[\protect\citeauthoryear{Pedregosa et~al.}{2011}]{scikit-learn}
\begin{barticle}
\bauthor{\bsnm{Pedregosa}, \binits{F.}},
\bauthor{\bsnm{Varoquaux}, \binits{G.}},
\bauthor{\bsnm{Gramfort}, \binits{A.}},
\bauthor{\bsnm{Michel}, \binits{V.}},
\bauthor{\bsnm{Thirion}, \binits{B.}},
\bauthor{\bsnm{Grisel}, \binits{O.}},
\bauthor{\bsnm{Blondel}, \binits{M.}},
\bauthor{\bsnm{Prettenhofer}, \binits{P.}},
\bauthor{\bsnm{Weiss}, \binits{R.}},
\bauthor{\bsnm{Dubourg}, \binits{V.}},
\bauthor{\bsnm{Vanderplas}, \binits{J.}},
\bauthor{\bsnm{Passos}, \binits{A.}},
\bauthor{\bsnm{Cournapeau}, \binits{D.}},
\bauthor{\bsnm{Brucher}, \binits{M.}},
\bauthor{\bsnm{Perrot}, \binits{M.}},
\bauthor{\bsnm{Duchesnay}, \binits{E.}}:
\batitle{Scikit-learn: Machine learning in {P}ython}.
\bjtitle{Journal of Machine Learning Research}
\bvolume{12},
\bfpage{2825}--\blpage{2830}
(\byear{2011})
\end{barticle}
\endbibitem

\bibitem[\protect\citeauthoryear{Springall}{1973}]{springall1973response}
\begin{barticle}
\bauthor{\bsnm{Springall}, \binits{A.}}:
\batitle{{Response surface fitting using a generalization of the Bradley-Terry
  paired comparison model}}.
\bjtitle{Journal of the Royal Statistical Society Series C: Applied Statistics}
\bvolume{22}(\bissue{1}),
\bfpage{59}--\blpage{68}
(\byear{1973})
\doiurl{10.2307/2346303}
\end{barticle}
\endbibitem

\bibitem[\protect\citeauthoryear{De~Soete and
  Winsberg}{1993}]{de1993thurstonian}
\begin{barticle}
\bauthor{\bsnm{De~Soete}, \binits{G.}},
\bauthor{\bsnm{Winsberg}, \binits{S.}}:
\batitle{{A Thurstonian pairwise choice model with univariate and multivariate
  spline transformations}}.
\bjtitle{Psychometrika}
\bvolume{58}(\bissue{2}),
\bfpage{233}--\blpage{256}
(\byear{1993})
\doiurl{10.1007/bf02294575}
\end{barticle}
\endbibitem

\bibitem[\protect\citeauthoryear{David}{1963}]{david1963method}
\begin{bbook}
\bauthor{\bsnm{David}, \binits{H.A.}}:
\bbtitle{{The Method of Paired Comparisons}}
vol. \bseriesno{12}.
\bpublisher{Charles Griffin and Company Ltd.},
\blocation{London}
(\byear{1963})
\end{bbook}
\endbibitem

\bibitem[\protect\citeauthoryear{Glickman}{1995}]{glickman1995comprehensive}
\begin{barticle}
\bauthor{\bsnm{Glickman}, \binits{M.E.}}:
\batitle{A comprehensive guide to chess ratings}.
\bjtitle{American Chess Journal}
\bvolume{3}(\bissue{1}),
\bfpage{59}--\blpage{102}
(\byear{1995})
\end{barticle}
\endbibitem

\bibitem[\protect\citeauthoryear{Glenn and David}{1960}]{glenn1960ties}
\begin{barticle}
\bauthor{\bsnm{Glenn}, \binits{W.A.}},
\bauthor{\bsnm{David}, \binits{H.A.}}:
\batitle{{Ties in paired-comparison experiments using a modified
  Thurstone-Mosteller model}}.
\bjtitle{Biometrics}
\bvolume{16}(\bissue{1}),
\bfpage{86}--\blpage{109}
(\byear{1960})
\doiurl{10.2307/2527957}
\end{barticle}
\endbibitem

\bibitem[\protect\citeauthoryear{Rao and Kupper}{1967}]{rao1967ties}
\begin{barticle}
\bauthor{\bsnm{Rao}, \binits{P.V.}},
\bauthor{\bsnm{Kupper}, \binits{L.L.}}:
\batitle{{Ties in paired-comparison experiments: A generalization of the
  Bradley-Terry model}}.
\bjtitle{Journal of the American Statistical Association}
\bvolume{62}(\bissue{317}),
\bfpage{194}--\blpage{204}
(\byear{1967})
\doiurl{10.2307/2285921}
\end{barticle}
\endbibitem

\bibitem[\protect\citeauthoryear{Davidson}{1970}]{davidson1970extending}
\begin{barticle}
\bauthor{\bsnm{Davidson}, \binits{R.R.}}:
\batitle{{On extending the Bradley-Terry model to accommodate ties in paired
  comparison experiments}}.
\bjtitle{Journal of the American Statistical Association}
\bvolume{65}(\bissue{329}),
\bfpage{317}--\blpage{328}
(\byear{1970})
\doiurl{10.1080/01621459.1970.10481082}
\end{barticle}
\endbibitem

\bibitem[\protect\citeauthoryear{Plackett}{1975}]{plackett1975analysis}
\begin{barticle}
\bauthor{\bsnm{Plackett}, \binits{R.L.}}:
\batitle{{The analysis of permutations}}.
\bjtitle{Journal of the Royal Statistical Society Series C: Applied Statistics}
\bvolume{24}(\bissue{2}),
\bfpage{193}--\blpage{202}
(\byear{1975})
\doiurl{10.2307/2346567}
\end{barticle}
\endbibitem

\bibitem[\protect\citeauthoryear{Marden}{1996}]{marden1996analyzing}
\begin{bbook}
\bauthor{\bsnm{Marden}, \binits{J.I.}}:
\bbtitle{{Analyzing and Modeling Rank Data}}.
\bpublisher{CRC Press},
\blocation{Boca Raton}
(\byear{1996})
\end{bbook}
\endbibitem

\end{thebibliography}
\end{document}